\documentclass[journal]{IEEEtran}
\usepackage{cite}
\usepackage{tabularx}
\usepackage{multirow}
\usepackage{multicol}
\usepackage{color}
\usepackage{amssymb,amsfonts,amsthm}
\usepackage[linkcolor=blue,colorlinks=true, citecolor = blue]{hyperref}
\usepackage{amsmath}       
\usepackage{amssymb}
\usepackage{mathtools}
\usepackage{amsthm}
\usepackage{mathrsfs}
  {
      \theoremstyle{plain}
      \newtheorem{assumption}{Assumption}
  }
\usepackage{algorithm}
\usepackage{algpseudocode}  
\usepackage{graphicx}
\usepackage[misc]{ifsym} 
\usepackage{subfigure} 
\usepackage{textcomp}
\usepackage{xcolor}
\usepackage{dsfont}
\usepackage{mathrsfs}
\usepackage{bbm}
\usepackage{booktabs} 
\usepackage{diagbox}
\usepackage{verbatim}
\usepackage{setspace}
\newtheorem{lemma}{Lemma}
\newtheorem{theorem}{Theorem}
\newtheorem{definition}{Definition}

\allowdisplaybreaks[4]
\def\BibTeX{{\rm B\kern-.05em{\sc i\kern-.025em b}\kern-.08em
    T\kern-.1667em\lower.7ex\hbox{E}\kern-.125emX}}
\begin{document}
\title{On the Escaping Efficiency of Distributed Adversarial Training Algorithms}
\author{Ying~Cao*, Kun~Yuan, Ali H. Sayed
\thanks{*Corresponding author. Authors Ying Cao and Ali H. Sayed are with the Institute of Electrical and Micro Engineering in EPFL,
Lausanne. Kun Yuan is with the Center for Machine Learning Research (CMLR) in Peking University. Emails: \{ying.cao, ali.sayed\}@epfl.ch, kunyuan@pku.edu.cn.}
}
\maketitle
%


\begin{abstract}
Adversarial training has been widely studied in recent years due to its role in improving model robustness against adversarial attacks. This work compares the fundamental  performance of centralized and decentralized adversarial training algorithms within multi-agent learning environments. Previous studies have highlighted the importance of model flatness in determining robustness. To this end, we develop a general theoretical framework to study the escaping efficiency of these algorithms from local minima, which is closely related to the flatness of the resulting models. We show that when the perturbation bound is sufficiently small (i.e., when the attack strength is relatively mild) and a large batch size is used, decentralized adversarial training algorithms, including consensus and diffusion, are guaranteed to escape faster from local minima than the centralized strategy, thereby favoring flatter minima. However, as the perturbation bound increases, this trend may no longer hold. In the simulation results, we illustrate our theoretical findings and systematically compare the performance of models obtained through decentralized and centralized adversarial training algorithms. The results highlight the potential of decentralized strategies to enhance the robustness of models in distributed settings.
\end{abstract}


\section{Introduction}
Despite the success on a wide range of modern vision and language tasks, machine learning models have been shown to be surprisingly vulnerable to adversarial perturbations~\cite{szegedy2013intriguing, goodfellow2014explaining}. Basically, models that perform well under standard (clean) conditions can often be misled by imperceptible input modifications. One of the most widely used methodologies to mitigate this vulnerability is adversarial training~\cite{madry2017towards,wang23ad,bartoldson2024adversarial}, where adversarially perturbed samples are included during training to improve the robustness of models. Most earlier works studied adversarial robustness in the single-agent setting, where the entire system is composed of a single model. However, distributed learning systems, which involve multiple agents, are becoming increasingly relevant in a highly connected world and for large-scale applications \cite{vlaski2023networked,YuanWSYB24}. In the previous work \cite{cao2025decentralized}, we proposed a general distributed adversarial training framework to improve the robustness of networked models and analyzed the convergence behavior of the associated algorithms.

In distributed learning, a group of agents collaboratively optimize a global objective, where data and/or computation are distributed. To achieve this, two main algorithmic paradigms are commonly employed in algorithm design \cite{sayed2014adaptation,sayed_2023}. The first is known as the \emph{centralized} approach, in which all agents contribute to the update of a shared global model.  This approach has theoretical properties that are essentially consistent with traditional single-agent learning.  The second is the \emph{decentralized} strategy, which differs fundamentally from the centralized setting in how agents coordinate. The agents are connected by a predefined graph topology, and each agent keeps its local observations private and interacts only with its neighboring agents to update its model. Furthermore, within the decentralized setting, \emph{diffusion} and \emph{consensus} are two widely adopted strategies \cite{vlaski2023networked,sayed2014adaptation,sayed_2023,nedic2009distributed,TuS12} and are also the focus of this paper.

The decentralized framework is primarily motivated by its advantages in preserving data privacy, improving fault tolerance, and distributing the computational load across multiple agents \cite{abs-2111-04287,ZhuLWH25,RizkVS23,AlghunaimY22,YeLZZSJ24,SunLW23}. Regarding the fundamental performance analysis in the standard (clean) environment, early studies generally established that decentralized methods can, at best, match the optimization performance of the centralized approach~\cite{sayed2014adaptation,ChenS15b,YuanAYS20,vlaski2021second,VlaskiS21a,VogelsHJ23}. Some works such as \cite{DengSL023,ZhuHZNST22} suggested that the decentralized operation can even degrade the generalization performance of models relative to the centralized baseline. However, some recent studies have observed that decentralized training can be beneficial to generalization\cite{00010KJS21,BarsBTSN24,ZhuH0ST23,cao2024trade}. Specifically, reference \cite{00010KJS21} empirically found that the consensus distance, i.e., the difference among local models during intermediate stages of training with decentralized strategies, could improve generalization over the centralized training. Similarly, \cite{BarsBTSN24} established that poorly-connected graphs can be beneficial to generalization.  Similarly, reference \cite{ZhuH0ST23,cao2024trade} provably verified that decentralized training methods favor flatter local minima than the centralized approach in non-convex and large-batch settings, and such flatness has been linked to enhanced generalization ability~\cite{ForetKMN21}. In particular, \cite{cao2024trade} showed that decentralized strategies generally strike a more favorable balance between flatness and optimization compared to the centralized solution, thereby resulting in better classification accuracy. While these observations are limited to the clean setting without adversarial attacks, this work compares centralized and decentralized training in the adversarial setting, where adversarial perturbations are explicitly considered.

In the adversarial learning community, several prior works have highlighted the importance of the flatness of models in determining robustness \cite{ZhangHZCWW24, WuX020, Stutz0S21,abs-2010-03593}. Specifically, they found that flattening the loss landscape with respect to the model parameters tends to improve the robustness of models. Inspired by these findings, we investigate the implicit bias of decentralized and centralized adversarial training methods, particularly regarding their preference for flatter minima and the impact on model robustness. In particular, we focus on addressing the following two challenges:
\begin{itemize}
    \item \textit{In the adversarial training context, can decentralized methods favor flatter models than the centralized strategy, and under what conditions?} 
    \item \textit{Can decentralized adversarial training methods enhance robustness of models to adversarial attacks compared to their centralized counterpart?}
\end{itemize}
To answer these questions, we theoretically examine the escaping efficiency of distributed adversarial training algorithms, which measures how effectively they escape local minima of the non-convex adversarial risk. This behavior is closely related to their tendency to converge to flatter solutions. The key contributions of this work are summarized as follows:

(1) To analyze the escaping efficiency of distributed adversarial training algorithms, we extend the theoretical framework from the clean environment \cite{cao2024trade} to the adversarial learning context. Since the dependence of the Hessian matrix of the risk function on the iterates now renders the analysis rather intractable, we adopt a short-term model in which the original Hessian matrix is approximated by its value at a local minimum. To ensure that the short-term approximation reliably captures the dynamics of the true model in the vicinity of local minima, we rigorously verify that the approximation error between the true dynamics and the short-term model is negligible. It is important to note that this is not a straightforward extension, as the adversarial learning context requires appropriate modifications to the proof techniques due to changes in certain theoretical properties. For instance, the Lipschitz continuity of the risk function may no longer hold, and the presence of adversarial perturbations introduces additional complexities that must be carefully addressed.

(2) Using the short-term model, we derive closed-form expressions for the escaping efficiency of distributed adversarial training algorithms. This enables a direct  comparison of centralized, consensus, and diffusion strategies in the adversarial setting. Basically, our analysis confirms that the network heterogeneity and underlying graph structure contribute to faster escape from local minima for the family of decentralized adversarial training methods--namely, consensus and diffusion--compared to the centralized strategy. As a result, decentralized methods favor flatter minima than the centralized approach  in the large-batch setting. Furthermore, within the family of decentralized methods,  and due to its larger mean-square error performance, consensus exhibits higher escaping efficiency than diffusion, and thus favors flatter solutions. However, it is important to note that these results are valid only when the perturbation bound parameter, which intuitively quantifies the strength of adversarial attacks, remains sufficiently small. As the perturbation bound increases, the above trend may no longer hold.

(3) In the simulation section, we present experimental results comparing adversarial centralized, consensus, and diffusion training algorithms on CIFAR-10, CIFAR-100, and Tiny ImageNet-200 datasets. We vary the perturbation bounds to simulate different levels of attack strength and to demonstrate their impact on model performance. First, we visualize the flatness of models obtained by the three distributed adversarial training algorithms. As anticipated by our theoretical results, when adversarial attacks are relatively mild, i.e., when the perturbation bound is small, decentralized methods produce models with flatter landscapes with respect to the model parameters compared to the centralized solution. Moreover, consensus yields slightly flatter models than diffusion. However, as the attack strength increases, this trend does not necessarily hold; in some cases, decentralized methods even result in sharper models than the centralized strategy. Then, we evaluate the clean and robust accuracy of the trained models using AutoAttack \cite{Croce020a}. We observe that decentralized methods achieve higher clean and robust accuracy compared to the centralized approach in the large-batch setting, even in cases where the centralized models are flatter than those obtained by decentralized strategies under strong adversarial attacks. This suggests that while flatness remains an important factor in determining robustness, it is not the sole determinant. These findings further motivate us to examine the optimization performance of the three adversarial training methods. We find that, compared to decentralized methods, centralized adversarial training with vanilla SGD exhibits worse optimization performance in the large-batch setting under strong attacks. This helps explain why decentralized methods can achieve better robustness despite producing sharper models than the centralized method. Overall, these comparisons highlight the potential of decentralized training to improve robustness in the multi-agent adversarial learning framework.

\section{Preliminaries}
\subsection{Problem statement}
We consider a network of $K$ agents (or nodes) interconnected through a graph topology. Each agent observes data realizations drawn from independently distributed observations. The objective of multi-agent adversarial learning is to solve the following optimization problem in order to obtain a robust model which is capable of resisting adversarial perturbations in the data:
\begin{equation}\label{formulation_global}
\min\limits_{{w}\in \mathbbm{R}^M} \left\{J(w) \overset{\Delta}{=} \frac{1}{K}\sum\limits_{k=1}^K J_k( w)\right\}
\end{equation}
where $M$ is the dimension of the parameter $w$, and the empirical risk function of each local agent has the following formulation:
\begin{equation}\label{local_risk}
    J_k(w) \overset{\Delta}{=} \frac{1}{N_k}\sum_{i=1}^{N_k}\left\{ \max\limits_{\left\Vert\delta\right\Vert_p \le \epsilon}  Q_k(w;{x}_{k,i} +\delta, {y}_{k,i})\right\}
\end{equation}
where $Q_k(\cdot)$ is a nonconvex loss function, and $\delta$ represents the perturbation variable whose $\ell_p$ norm is constrained by a small value $\epsilon$. Regarding the choice of norm, the $\ell_2$ and $\ell_\infty$ norms are the most commonly used in the community of adversarial learning. The notation $(x_{k,i}, y_{k,i})$ refers to the $N_k$ training samples observed by agent $k$, where $x_{k,i}$ is the feature vector and $y_{k,i}$ is the corresponding label. These samples are drawn from a random source $(\boldsymbol{x}_k, \boldsymbol{y}_k)$ associated with agent $k$. In this paper, we examine the optimization behaviors of algorithms in the vicinity of local minima of the non-convex function $J(w)$. Specifically, assuming $w^{\star}$ is a local minimizer of $J(w)$, we focus on heterogeneous networks in which the local minimizers of $J_k(w)$ differ from $w^{\star}$. 

Consider the solution of the inner maximization problem in \eqref{local_risk} denoted by:
\begin{align}\label{tmaximizer_p}
    \boldsymbol{\delta}_k^{\star} (w) \overset{\Delta}{=} \mathop{\mathrm{argmax}}\limits_{\left\Vert\delta\right\Vert_{p}  \le \epsilon}   Q_k(w;\boldsymbol{x}_k +\delta,\boldsymbol{y}_k)
\end{align}
which is also referred to as the optimal perturbation corresponding to the model $w$ and the input sample $(\boldsymbol{x}_k, \boldsymbol{y}_k)$, we introduce the following assumption to simplify the analysis:
\begin{assumption}\label{assump_uniqueperturb}
(\textbf{Uniqueness of the optimal perturbation}) For any given model $w$ and sample $(\boldsymbol{x}_k, \boldsymbol{y}_k)$, the associated optimal perturbation $\boldsymbol{\delta}_k^{\star} (w)$, as defined in \eqref{tmaximizer_p}, is unique. 

 $\hfill\square$ 
\end{assumption}
By resorting to Danskin's theorem \cite{cao2025decentralized,bertsekas2009convex, Thekumparampil019}, Assumption \ref{assump_uniqueperturb} guarantees that $J_k(w)$ and $J(w)$ are differentiable, with
\begin{align}\label{def_f}
 \nabla J_k(w) &= \mathds{E} \nabla_w Q_k(w;\boldsymbol{x}_k+\boldsymbol{\delta}_k^{\star}(w),\boldsymbol{y}_k) \notag\\
 &= \frac{1}{N_k}\sum_{i=1}^{N_k} \nabla_w Q_k(w;x_{k,i}+ \delta^{\star}_{k,i}(w), y_{k,i})
\end{align}
where the dependence of $\boldsymbol{\delta}_k^{\star}(w)$ on $w$ can be ignored when computing $\nabla J_k(w)$. Notably, this assumption is weaker than the commonly adopted condition in the literature \cite{sinha2017certifying, WangM0YZG19}, which requires the loss function to be strongly-concave with respect to the perturbation variable. If the optimal perturbation is not unique, the robust risks $J_k(w)$ and $J(w)$ may be nonsmooth. In this case, continuing the analysis using the present gradient- and Hessian-based techniques would be analytically intractable.  Extending the analysis to nonunique optimal perturbations is therefore a useful direction for future work, in which establishing comparable results would require a different framework based on generalized gradients or subgradients.

\subsection{Distributed adversarial training algorithms}
One methodology to solve \eqref{formulation_global} is adversarial training, in which the inner maximization problem in \eqref{local_risk} is approximately solved. This reduces the minimax problem to a standard minimization formulation, enabling the application of traditional stochastic gradient methods to search for a robust minimizer. In this paper, we focus on three popular distributed stochastic gradient algorithms to seek a solution for \eqref{formulation_global}.
Basically, the mini-batch \textit{centralized} adversarial training algorithm follows the following recursive update:
\begin{subequations}
\begin{align}
\label{at_centra_1}
&\widehat{\boldsymbol{x}}_{k,n}^{b} = \boldsymbol{x}_{k,n}^{b} + \widehat{\boldsymbol{\delta}}^{b}_{k,n}\\
\label{at_centra_2}
&\boldsymbol{w}_n = \boldsymbol{w}_{n-1} - \mu\times\frac{1}{KB}\sum\limits_{k=1}^{K}\sum\limits_{b=1}^{B} \nabla_w Q(\boldsymbol{w}_{n-1};\widehat{\boldsymbol{x}}_{k,n}^{b},\boldsymbol{y}_{k,n}^b)
\end{align}
\end{subequations}
In \eqref{at_centra_1}--\eqref{at_centra_2}, $B$ denotes the local batch size of each agent, while $\mu$ is a small positive step-size parameter (also often referred to as learning rate). Additionally, $\widehat{\boldsymbol{\delta}}_{k,n}^b$ is an approximate worst-case perturbation associated with the model $\boldsymbol{w}_{n-1}$ and 
the sample $(\boldsymbol{x}_{k,n}^b, \boldsymbol{y}_{k,n}^b)$:
\begin{equation}\label{a_perturb_centra}
    \widehat{\boldsymbol{\delta}}^{b}_{k,n} \approx \mathop{\mathrm{argmax}}\limits_{\left\Vert\delta\right\Vert_{p}  \le \epsilon}   Q(\boldsymbol{w}_{n-1};\boldsymbol{x}_{k,n}^b +\delta,\boldsymbol{y}_{k,n}^b)
\end{equation}
Various attack methods have been proposed to compute the approximate perturbations in the literature, such as fast gradient methods \cite{goodfellow2014explaining} and projected gradient descent methods \cite{madry2017towards}. In this work, we do not impose any restriction on the attack methods used in the training process.

The update in Eqs.~\eqref{at_centra_1}--\eqref{at_centra_2} is mathematically equivalent to conventional single-agent training with a global batch size of $KB$, and it admits two mathematically equivalent implementations. In the central-server implementation \cite{sayed2014adaptation,sayed_2023}, all agents transmit a mini-batch of perturbed observations to a central server at every iteration. The central server then performs the gradient computation and model update. Alternatively, the same update can be implemented using synchronous data-parallel computation. In this implementation, each worker computes a local gradient using the same synchronized model, and the local gradients are then averaged to update the model. Thus, gradient computations can be distributed across workers, while all workers maintain one synchronized global model. Synchronous data-parallel frameworks, such as PyTorch DistributedDataParallel \cite{Li2020PyTorchDistributed}, fundamentally follow this structure by computing gradients locally and then average them across workers.

Another methodology for solving distributed learning problems is decentralized learning, where data remains local, and agents interact solely with their neighbors through a collaborative process. In this paper, we consider two widely recognized decentralized methods: the consensus and diffusion strategies. Before listing the algorithms, it is necessary to define the graph structure and the associated combination matrix $A$ that drives their cooperation. Basically, the agents are assumed to be linked by a weighted graph topology. The weight on the link from agent $\ell$ to agent $k$ is denoted by $a_{\ell k}$, which lies within $[0,1]$; this value is used to scale information sent from $\ell$ to $k$, and it will be strictly positive if there exists a link from $\ell$ to $k$ over which information can be shared. The weights $\{a_{\ell k}\}$ can be collected into a $K\times K$ combination matrix $A$.
\begin{assumption}\label{as_graph}
    (\textbf{Strongly-connected graph}) The graph is assumed to be strongly connected. This means that there exists a path with positive weights $\{a_{\ell k}\}$ linking any pair of agents and, in addition, at least one node $k$ in the network has a self-loop with $a_{kk}>0$. Moreover, for any pair of agents, it is assumed that $a_{\ell k} = a_{k\ell}$, and the entries on each column of $A$ are normalized to add up to $1$.
    
    $\hfill\square$ 
\end{assumption}
 Assumption \ref{as_graph} guarantees that $A$ is doubly-stochastic. Moreover, it follows from the Perron-Frobenius theorem \cite{sayed2014adaptation} that $A$ has a single eigenvalue at 1, and the corresponding right eigenvector is $ \frac{1}{K}\mathbbm{1}$, namely,
 \begin{align}
     A\mathbbm{1} = \mathbbm{1}
 \end{align}
Also, the matrix $A$ admits the following eigen-decomposition:
\begin{equation}\label{decomp_A}
    A = VPV^{\sf T}
\end{equation}
where 
\begin{equation}\label{vpv1}
     V = \left[\frac{1}{\sqrt{K}}\mathbbm{1} \quad V_\alpha \right],\quad
    P = \left[
        \begin{array}{cc}
        1&0\\
        0&P_\alpha
        \end{array}
    \right]
\end{equation}
Here, $P_{\alpha}\in \mathbbm{R}^{(K-1)\times(K-1)}$ is a diagonal matrix containing the second largest-magnitude eigenvalue $\lambda_2$ to the smallest-magnitude eigenvalue $\lambda_K$ of $A$. $V_\alpha$ is composed of the corresponding eigenvectors associated with these eigenvalues. For later use, we extend these matrices into block forms through the Kronecker product operation:
\begin{align}\label{vpv222}
    &\mathcal{P}_{\alpha} = P_\alpha\otimes I_M, \quad
\mathcal{V}_\alpha = V_\alpha\otimes I_M 
\end{align}

 Using the combination matrix $A$, we now present the adversarial versions of the diffusion and consensus strategies \cite{cao2025decentralized}. Basically, the adversarial diffusion training algorithm consists of the following three steps:
\begin{subequations}
\begin{align}
\label{diff_a_1}
&\widehat{\boldsymbol{x}}_{k,n}^{b} = \boldsymbol{x}_{k,n}^{b} + \widehat{\boldsymbol{\delta}}^{b}_{k,n}\\
\label{diff_a_2}
    &\boldsymbol{\phi}_{k,n} = \boldsymbol{w}_{k,n-1} - \frac{\mu}{B}\sum\limits_{b=1}^{B}\nabla_w Q_k(\boldsymbol{w}_{k,n-1};\widehat{\boldsymbol{x}}_{k,n}^{b}, \boldsymbol{y}_{k,n}^b)\\
\label{diff_a_3}
&\boldsymbol{w}_{k,n}  = \sum\limits_{\ell \mathcal{\in N}_k} a_{\ell k} \boldsymbol{\phi}_{\ell,n}  
\end{align}
\end{subequations}
At every iteration $n$, every agent $k$ samples $B$ data points and compute their adversarial perturbations and uses \eqref{diff_a_2} to locally update its iterate $\boldsymbol{w}_{k,n-1}$ to the intermediate value $\boldsymbol{\phi}_{k,n}$. Subsequently, the same agent combines the intermediate iterates from across its neighbors using \eqref{diff_a_3}. Note that in decentralized algorithms, the approximate worst-case perturbation $\widehat{\boldsymbol{\delta}}^{b}_{k,n}$ is computed with the local model $\boldsymbol{w}_{k,n-1}$:
\begin{equation}\label{a_perturb_decentra}
\widehat{\boldsymbol{\delta}}^{b}_{k,n} \approx \mathop{\mathrm{argmax}}\limits_{\left\Vert\delta\right\Vert_{p}  \le \epsilon}   Q_k(\boldsymbol{w}_{k,n-1};\boldsymbol{x}_{k,n}^b +\delta,\boldsymbol{y}_{k,n}^b)
\end{equation}
It is useful to remark that the centralized implementation \eqref{at_centra_1}--\eqref{at_centra_2} can be viewed as a special case of \eqref{diff_a_1}--\eqref{diff_a_3} if we select the combination matrix as $A=\frac{1}{K}\mathbbm{1}\mathbbm{1}^{\sf T}$ and if loss functions of all agents are identical. 

In comparison, the consensus strategy follows a different collaborative process:
\begin{subequations}
\begin{align}
\label{consen_a_1}
&\boldsymbol{\phi}_{k,n} = \sum\limits_{\ell \mathcal{\in N}_k} a_{\ell k} \boldsymbol{w}_{\ell,n-1}\\
\label{consen_a_2}
&\widehat{\boldsymbol{x}}_{k,n}^{b} = \boldsymbol{x}_{k,n}^{b} + \widehat{\boldsymbol{\delta}}^{b}_{k,n}\\
\label{consen_a_3}
&\boldsymbol{w}_{k,n}  = \boldsymbol{\phi}_{k,n}  - \frac{\mu}{B}\sum\limits_{b=1}^{B}\nabla_w Q_k(\boldsymbol{w}_{k,n-1};\widehat{\boldsymbol{x}}_{k,n}^{b}, \boldsymbol{y}_{k,n}^b)
\end{align}
\end{subequations}
In this case, the existing iterates $\boldsymbol{w}_{\ell,n-1}$ are first combined to generate the intermediate value $\boldsymbol{\phi}_{k,n}$, after which (\ref{consen_a_2})--(\ref{consen_a_3}) are applied. Observe the asymmetry on the right-hand side in (\ref{consen_a_2}). The starting iterate is $\boldsymbol{\phi}_{k,n}$, while the gradient of the loss function is evaluated at $\boldsymbol{w}_{k,n-1}$. In contrast, the same iterate $\boldsymbol{w}_{k,n-1}$ appears in both terms on the RHS of (\ref{diff_a_2}).

\section{Escaping efficiency of multi-agent adversarial training algorithms}
In this section, we examine how the three distributed adversarial training algorithms escape from local minima of the adversarial risk function $J(w)$. To do so, we first characterize the basin (or valley) of a local minimum.
\begin{definition}
   (\textbf{Basin of attraction} \cite{cao2024trade,MoriLLU22}) For a given local minimizer $w^{\star}$, its basin of attraction $\Omega(w^\star)$ is defined as the set of all points starting from which the iterates $\boldsymbol{w}_{k,n}$ or $\boldsymbol{w}_n$  converge to $w^{\star}$ as $n\to \infty$ under sufficiently small $\mu$. This definition assumes that there is no noise in the gradient-descent algorithms.
   
$\hfill\square$
\end{definition}
 
 If $\boldsymbol{w}_{k,n}$ or $\boldsymbol{w}_n \notin  \Omega(w^{\star})$, then we say the algorithms escape from the basin of $w^{\star}$ exactly. We then quantify the escape ability of algorithms by introducing the concept of \textit{escaping efficiency}:
\begin{definition}\label{def_ee}
   (\textbf{Escaping efficiency} \cite{cao2024trade}) Assume all agents start from points close to a local minimizer of $J(w)$, which is denoted by $w^{\star}$, the escaping efficiency over the network at iteration $n$ is defined by:
    \begin{align}\label{def_ee_eq1}
    \mathrm{ER}_n \overset{\Delta}{=}\frac{1}{K}\sum\limits_{k=1}^{K} \mathds{E} J(\boldsymbol{w}_{k,n}) - J(w^\star)
\end{align}
The larger the value of $\mathrm{ER}_n$ is, the farther the network model--comprising $K$ local models--will be from the local minimum $J(w^{\star})$ on average, indicating higher escaping efficiency.

$\hfill\square$ 
\end{definition}

The concepts of escaping efficiency and basin of attraction are extended from the clean environment \cite{ZhuWYWM19, cao2024trade,MoriLLU22 }, where they are used to analyze the escape behavior of traditional stochastic gradient algorithms from local minima without accounting for adversarial perturbations. In this paper, we adapt the two concepts
to the adversarial context to analyze the escape behavior of the stochastic gradient-based adversarial training algorithms. In general, for a fixed $n$, a larger value of $\mathrm{ER}_n$ implies an expected faster escape from a local minimum. 

\subsection{Adversarial training dynamics}
To enable the analysis, we first introduce the stochastic gradient noise process associated with each agent in the general mini-batch setting. For agent $k$ at iteration $n$, the gradient noise under a mini-batch of size $B$ is defined as:
\begin{align}\label{s_gn_k_B}
\boldsymbol{s}_{k,n}^{B}(w) \overset{\Delta}{=} \frac{1}{B}\sum\limits_{b=1}^{B}\nabla_w Q_k(w;\widehat{\boldsymbol{x}}_{k,n}^{b},\boldsymbol{y}_{k,n}^b) - \mathds{E}\nabla_w Q_k(w;\widehat{\boldsymbol{x}}_{k,n}^{b} ,\boldsymbol{y}_{k,n}^b)
\end{align}
In particular, for the single-sample case where $B = 1$, the gradient noise simplifies to
\begin{align}\label{s_gn_k}
\boldsymbol{s}_{k,n}(w) \overset{\Delta}{=} \nabla_w Q_k(w;\widehat{\boldsymbol{x}}_{k,n},\boldsymbol{y}_{k,n}) - \mathds{E}\nabla_w Q_k(w;\widehat{\boldsymbol{x}}_{k,n} ,\boldsymbol{y}_{k,n})
\end{align}
where in the empirical setting, the expectation is taken over all training samples, namely, 
\begin{align}
 \mathds{E}\nabla_w Q_k(w;\widehat{\boldsymbol{x}}_{k,n} ,\boldsymbol{y}_{k,n}) {=} \frac{1}{N_k}  \sum_{i=1}^{N_k}\nabla_w Q_k(w;\widehat{{x}}_{k,i},{y}_{k,i})
\end{align}
Moreover, we define the covariance matrix of the gradient noise as follows:
\begin{align}\label{com-gd}
     R_{s,k,n}^B(w) \overset{\Delta}{=} \mathds{E}\left\{ \boldsymbol{s}^{B}_{k,n}(w)\boldsymbol{s}^{B}_{k,n}(w)^{\sf T}\right\}
\end{align}
Specifically, when $B=1$, we have
\begin{align}\label{com-gd-1}
     R_{s,k,n}(w) \overset{\Delta}{=} \mathds{E}\left\{ \boldsymbol{s}_{k,n}(w)\boldsymbol{s}_{k,n}(w)^{\sf T}\right\}
\end{align}
Accordingly, in the case of $B=1$, we define the covariance matrix of the gradient noise associated with the global risk function at $w^{\star}$, which will be used later:
\begin{align}\label{d_gcmstar}
     \bar{R}_s &\overset{\Delta}{=} \mathds{E}\left\{\left(\frac{1}{K}\sum\limits_{k=1}^{K}\boldsymbol{s}_{k,n}(w^\star)\right)\left(\frac{1}{K}\sum\limits_{\ell=1}^{K} \boldsymbol{s}_{\ell,n}(w^\star)\right)^{\sf T}\right\}
 \end{align}

Then, we introduce the following block error vector which captures the deviation of agents from the local minimizer $w^{\star}$:
\begin{align}\label{d_cwt}{\widetilde{\boldsymbol{\scriptstyle\mathcal{W}}}}_{n} \overset{\Delta}{=} {\mathop{\rm col}}\{\widetilde{\boldsymbol{w}}_{k,n} \overset{\Delta}{=}\boldsymbol{w}_{k,n} - w^{\star}\}_{k=1}^{K}
\end{align}
Additionally, we consider the extended combination matrix in block form:
\begin{align}\label{ex_ca}
\mathcal{A} \overset{\Delta}{=} A \otimes I_{M}   
\end{align}
where $I_{M}$ is the identity matrix of size $M$, and $\otimes$ denotes the Kronecker product operation. To facilitate a unified representation of the updated rules in  \eqref{at_centra_1}--\eqref{at_centra_2}, \eqref{consen_a_1}--\eqref{consen_a_3}, and \eqref{diff_a_1}--\eqref{diff_a_3}, we further introduce two matrices, $A_1$ and $A_2$, whose values change depending on the specific algorithm. Specifically, for the consensus algorithm in \eqref{consen_a_1}--\eqref{consen_a_3}, we set
\begin{align}
    A_1 {=} A, \quad\;A_2 {=} I_{K}
\end{align}
while for the diffusion algorithm in \eqref{diff_a_1}--\eqref{diff_a_3}, we use
\begin{align}
   A_1 {=} I_K, \quad\;A_2 {=} A
\end{align}
and, for the centralized method with \eqref{at_centra_1}--\eqref{at_centra_2}, we define
\begin{align}
A_1  {=}  I_K, \quad A_2  {=}  \pi\mathbbm{1}^{\sf T} = \frac{1}{K}\mathbbm{1}\mathbbm{1}^{\sf T}
\end{align}
Similar to \eqref{ex_ca}, we also consider the extended block form of $A_1$ and ${A_2}$:
\begin{align}\label{ex_ca12}
\mathcal{A}_1 \overset{\Delta}{=} A_1 \otimes I_{M}  \quad\;  \mathcal{A}_2 \overset{\Delta}{=} A_2 \otimes I_{M}
\end{align}
with which we obtain the following unified recursion that covers the dynamics of the centralized, consensus, and diffusion adversarial training algorithms:
\begin{align}
\label{unified_re_1}
\widetilde{\boldsymbol{\scriptstyle\mathcal{W}}}_{n} =& \mathcal{A}_2\Bigg(\mathcal{A}_1\widetilde{\boldsymbol{\scriptstyle\mathcal{W}}}_{n-1} \notag\\
&- \mu\mathop{\rm col}\limits_k\left\{\frac{1}{B}\sum_b\nabla_w Q_k(\boldsymbol{w}_{k,n-1};\widehat{\boldsymbol{x}}_{k,n}^{b}, \boldsymbol{y}_{k,n}^b)\right\}\Bigg)\notag\\
=& \mathcal{A}\widetilde{\boldsymbol{\scriptstyle\mathcal{W}}}_{n-1} - \mu\mathcal{A}_2\mathop{\rm col}\limits_k\bigg\{\frac{1}{B}\sum_b\nabla_w Q_k(\boldsymbol{w}_{k,n-1};\widehat{\boldsymbol{x}}_{k,n}^{b},\boldsymbol{y}_{k,n}^b) \notag\\
&- \mathds{E}\nabla_w Q_k(\boldsymbol{w}_{k,n-1};\widehat{\boldsymbol{x}}_{k,n}, \boldsymbol{y}_{k,n}) \notag\\
&+ \mathds{E}\nabla_w Q_k(\boldsymbol{w}_{k,n-1};\widehat{\boldsymbol{x}}_{k,n},\boldsymbol{y}_{k,n})- \nabla J_k(\boldsymbol{w}_{k,n-1}) \notag\\
&+ \nabla J_k(\boldsymbol{w}_{k,n-1}) - \nabla J_k(w^{\star}) +  \nabla J_k(w^{\star})\bigg\} \notag\\
{=}& \mathcal{A}\widetilde{\boldsymbol{\scriptstyle\mathcal{W}}}_{n-1} - \mu\mathcal{A}_2\mathop{\rm col}\limits_k\Big\{\boldsymbol{s}_{k,n}^B(\boldsymbol{w}_{k,n-1}) \notag\\
&+ \mathds{E}\big\{\nabla_w Q_k(\boldsymbol{w}_{k,n-1};\widehat{\boldsymbol{x}}_{k,n},\boldsymbol{y}_{k,n}) \notag\\
&- \nabla_w Q_k(\boldsymbol{w}_{k,n-1};\boldsymbol{x}_{k,n}^{\star},\boldsymbol{y}_{k,n})\big\} \notag\\
& + \nabla J_k(\boldsymbol{w}_{k,n-1}) - \nabla J_k(w^{\star}) + \nabla J_k(w^{\star})\Big\}\notag\\
\overset{(a)}{=}& \mathcal{A}\widetilde{\boldsymbol{\scriptstyle\mathcal{W}}}_{n-1} - \mu\mathcal{A}_2\mathop{\rm col}\limits_k\left\{\nabla J_k(\boldsymbol{w}_{k,n-1}) - \nabla J_k(w^{\star})\right\}\notag\\
&- \mu \mathcal{A}_2d - \mu\mathcal{A}_2 \boldsymbol{e}_{n-1} - \mu\mathcal{A}_2\boldsymbol{s}_{n}^{B} \\
\overset{(b)}{=} & \mathcal{A}\widetilde{\boldsymbol{\scriptstyle\mathcal{W}}}_{n-1} - \mu\mathcal{A}_2\mathop{\rm col}\limits_k\left\{H_{k,n-1}\widetilde{\boldsymbol{w}}_{k,n-1}\right\} - \mu \mathcal{A}_2d \notag\\
&- \mu\mathcal{A}_2 \boldsymbol{e}_{n-1} - \mu\mathcal{A}_2\boldsymbol{s}_{n}^{B}\notag \\
\label{unified_re_2}
\overset{(c)}{=}& \mathcal{A}_2\left(\mathcal{A}_1 - \mu\mathcal{H}_{n-1}\right)\widetilde{\boldsymbol{\scriptstyle\mathcal{W}}}_{n-1} -  \mu\mathcal{A}_2d - \mu\mathcal{A}_2 \boldsymbol{e}_{n-1} \notag\\
&- \mu\mathcal{A}_2\boldsymbol{s}_{n}^{B}
\end{align}
where
\begin{align}
 &\boldsymbol{x}_{k,n}^{\star} = \boldsymbol{x}_{k,n} + \boldsymbol{\delta}_{k,n}^{\star}\\
 &\boldsymbol{\delta}_{k,n}^{\star} = \mathop{\mathrm{argmax}}\limits_{\left\Vert\delta\right\Vert_{p}  \le \epsilon}   Q_k(\boldsymbol{w}_{k,n-1};\boldsymbol{x}_{k,n} + \delta, \boldsymbol{y}_{k,n})
\end{align}
and, in $(a)$, we introduce the following block vectors:
\begin{align}
& d \overset{\Delta}{=}\mathop{\rm col}\limits_k \left\{\nabla J_k(w^{\star})\right\}\notag\\
   & \boldsymbol{s}_{n}^{B} \overset{\Delta}{=} \mathop{\rm col}\limits_k \left\{\boldsymbol{s}_{k,n}^B(\boldsymbol{w}_{k,n-1})\right\}\notag\\
    &\boldsymbol{e}_{n-1} \overset{\Delta}{=} \mathop{\rm col}\limits_k\Big\{\boldsymbol{e}_{k,n-1} \overset{\Delta}{=}\mathds{E}\big\{\nabla_w Q_k(\boldsymbol{w}_{k,n-1};\widehat{\boldsymbol{x}}_{k,n},\boldsymbol{y}_{k,n}) \notag\\
    &\quad\quad\quad\ - \nabla_w Q_k(\boldsymbol{w}_{k,n-1};\boldsymbol{x}_{k,n}^{\star},\boldsymbol{y}_{k,n})\big\}\Big\}
\end{align}
where $d$ denotes the stacked gradients of the local risk functions evaluated at $w^{\star}$, $\boldsymbol{s}_{n}^B$ collects the gradient noise across the network, while $\boldsymbol{e}_{n-1}$ arises from the approximation error by using $\widehat{\boldsymbol{\delta}}^{b}_{k,n}$ as a surrogate for the optimal perturbation $\boldsymbol{\delta}_{k,n}^{b,\star}$. Furthermore, \eqref{unified_re_1} can be further decomposed into \eqref{unified_re_2}, which takes the form of the recursion involving the Hessian matrix. Specifically, $(b)$ follows from the mean-value theorem:
\begin{align}
    \nabla J_k(\boldsymbol{w}_{k,n-1}) - \nabla J_k(w^{\star}) = H_{k,n-1}\widetilde{\boldsymbol{w}}_{k,n-1}
\end{align}
with
\begin{align}
   H_{k,n-1}\overset{\Delta}{=}\int_0^1 \nabla^2 J_k(w^\star - t\widetilde{\boldsymbol{w}}_{k,n-1})dt 
\end{align}
and, in $(c)$, we collect Hessian matrices from all agents into a block one:
\begin{align}
    \mathcal{H}_{n-1}\overset{\Delta}{=}\mathrm{diag}\left\{H_{1,n-1},H_{2,n-1},\ldots H_{K,n-1}\right\}
\end{align}

In \eqref{unified_re_1} and \eqref{unified_re_2}, we present two equivalent forms of the recursion for ${\widetilde{\boldsymbol{\scriptstyle\mathcal{W}}}}_{n}$, where \eqref{unified_re_1} is expressed in terms of the gradient difference $\nabla J_k(\boldsymbol{w}_{k,n-1}) - \nabla J_k(w^{\star})$, while \eqref{unified_re_2} is based on the Hessian matrix $\mathcal{H}_{n-1}$. In what follows, we use the recursion in \eqref{unified_re_1} to establish bounds for $\mathds{E}\Vert\widetilde{\boldsymbol{\scriptstyle\mathcal{W}}}_{n}\Vert^2$, and use \eqref{unified_re_2} to analyze the excess risk $\mathrm{ER}_n$.

\subsection{Assumptions and Lemmas}
Prior to the analysis, we impose several mild assumptions and establish the associated lemmas for later use.

First, we introduce the following smoothness conditions on the loss function $Q_k$, which are widely adopted in the optimization literature \cite{sinha2017certifying, WangM0YZG19,Thekumparampil019,LiuSLTS20}:
\begin{assumption}\label{assump_sc}
(\textbf{Smoothness conditions}) For each agent $k$, the gradients of the loss function relative to $w$ and $x$ are Lipschitz. Specifically, it holds that
\begin{subequations}
\begin{align}
\label{assump_sc_eq1}
&\left\Vert\nabla_w  Q_k(w_2;x,y) -  \nabla_w  Q_k(w_1;x,y) \right\Vert\le   L\left\Vert w_2-w_1\right\Vert\\
\label{assump_sc_eq2}
&\left\Vert\nabla_w  Q_k(w;x_2,y) - \nabla_w Q_k(w;x_1,y) \right\Vert\le L\left\Vert x_2-x_1\right\Vert
\end{align}
\end{subequations}
and
\begin{subequations}
\begin{align}
\label{assump_sc_eq3}
&\left\Vert\nabla_x  Q_k(w_2;x,y) - \nabla_x  Q_k(w_1;x,y) \right\Vert\le L\left\Vert w_2-w_1\right\Vert\\
\label{assump_sc_eq4}
&\left\Vert\nabla_x  Q_k(w;x_2,y) - \nabla_x  Q_k(w;x_1,y) \right\Vert\le L\left\Vert x_2-x_1\right\Vert
\end{align}
\end{subequations}
$\hfill\square$
\end{assumption}


Under Assumptions \ref{assump_uniqueperturb} and \ref{assump_sc},  reference \cite{cao2025decentralized} established the following affine Lipschitz property:
\begin{lemma}\label{affine_l}
(\textbf{Affine Lipschitz}) For each agent $k$, and any $w_1$, $w_2$, $\delta_1$, $\delta_2$, it holds that
\begin{align}\label{affine_l_eq_1}
 &\Vert \nabla_w  Q_k(w_2; \boldsymbol{x}_k+\delta_2,\boldsymbol{y}_k) - \nabla_w Q_k(w_1; \boldsymbol{x}_k+\delta_1,\boldsymbol{y}_k)\Vert  \notag\\
 &\le L\Vert w_2 - w_1 \Vert + O(\epsilon)   
\end{align}
where the $\ell_p$ norms of $\delta_1$ and $\delta_2$ are bounded by $\epsilon$. Then, the gradient of $J_k$ is affine Lipschitz, namely,
\begin{align}\label{affine_l_eq_2}
    \Vert\nabla J_k(w_2) - \nabla J_k (w_1) \Vert \le L\Vert w_2 - w_1\Vert + O(\epsilon)
\end{align}
$\hfill\square$
\end{lemma}

Contrary to the traditional clean setting \cite{cao2024trade}, where the risk functions are typically assumed to satisfy Lipschitz conditions--implying that the associated Hessian matrices are uniformly bounded--it can only be guaranteed in the adversarial context that $J_k(w)$ is affine-Lipschitz, as shown in \eqref{affine_l_eq_2}. In particular, an extra constant term of order $O(\epsilon)$, which is introduced by the perturbation, appears on the RHS of \eqref{affine_l_eq_2}. As a result, the Hessian matrices evaluated at arbitrary model parameters in the adversarial setting are no longer uniformly bounded. Therefore, the proof techniques used in \cite{cao2024trade}
cannot be directly applied and require appropriate modifications.

In addition, with Assumptions \ref{assump_uniqueperturb} and \ref{assump_sc}, it can be verified that the gradient noise is zero-mean, and that its second- and fourth-order moments are bounded:
\begin{lemma}\label{lm_zbgn}
(\textbf{Zero-mean and bounded gradient noise terms}) Let $\mathcal{F}_{n-1}$ denote the filtration generated by the collection $\{\boldsymbol{w}_{k,j}\}_{j= 1\le n-1, k = 1,\ldots, K}$, representing the history of the random process up to iteration $n-1$. For any $\boldsymbol{w} \in\mathcal{F}_{n-1}$, we define the error vector
\begin{align}
\widetilde{\boldsymbol{w}} \overset{\Delta}{ = } \boldsymbol{w} - w^{\star}
\end{align}
Then, under Assumption \ref{assump_sc}, the gradient noise defined in \eqref{s_gn_k_B} is zero-mean, namely,
\begin{align}\label{sgn_0} 
\mathds{E}\left\{\boldsymbol{s}_{k,n}^{B }(\boldsymbol{w})\vert{\mathcal{F}}_{n-1}\right\}=0 
\end{align}
and, its second- and fourth-order moments are upper bounded by terms that scales inversely with the batch size:
\begin{align}
\label{sgn_2}
\mathds{E}\{\Vert \boldsymbol{s}_{k,n}^{B}(\boldsymbol{w})\Vert^2\vert\mathcal{F}_{n-1} \}&\le O\left(\frac{1}{B}\right)\Vert \widetilde{\boldsymbol{w}}\Vert^2 +O \left(\frac{1}{B}\right)\\
\label{sgn_4}
\mathds{E}\{\Vert \boldsymbol{s}_{k,n}^{B}(\boldsymbol{w})\Vert^4\vert\mathcal{F}_{n-1} \} & \le O \left(\frac{1}{B^2}\right)\Vert \widetilde{\boldsymbol{w}}\Vert^4  + O \left(\frac{1}{B^2}\right)
\end{align}
Moreover, the covariance matrix of the gradient noise, as defined in \eqref{com-gd} and \ref{com-gd-1}, also scales inversely with the batch size $B$:
\begin{align}\label{gcm_scale}
    R_{s,k,n}^B(w) = \frac{1}{B}R_{s,k,n}(w)
\end{align}
\end{lemma}
\begin{proof}
    See Appendix \ref{p_lm_zbgn}.
\end{proof}

Next, since our objective is to examine how algorithms escape from local minima of the adversarial risk $J(w)$, it is natural to begin the analysis under the assumption that algorithms are already operating near such local minima. This point of view is widely adopted in the literature when examining the escaping behaviors of algorithms from either saddle points \cite{VlaskiS21, VlaskiS21a} or local minima \cite{ZhuWYWM19, XieSS21, NguyenSGR19} in the clean environment. A natural question that follows is: how close can models get to $w^{\star}$? Motivated by the convergence results established in the literature \cite{cao2025decentralized}, we use the following assumption:
\begin{assumption}\label{assump_ip}
(\textbf{Initial points}) All models begin their updates from points sufficiently close to $w^\star$, that is,
\begin{align}
\label{assump_ip_eq_1}
   &\mathds{E}\Vert \boldsymbol{w}_{k,-1} - w^\star\Vert^2 \le o\left(\frac{\mu}{B}\right) + O(\epsilon^{\frac{7}{4}})\\
\label{assump_ip_eq_2}
        &\mathds{E}\Vert \boldsymbol{w}_{k,-1} - w^\star\Vert^4 \le o\left(\frac{\mu^2}{B^2}\right) + O(\epsilon^{\frac{7}{2}})
\end{align}
 $\hfill\square$ 
\end{assumption}
It is important to note that the ``starting points" referenced here are not the initializations in the traditional optimization sense, but rather the points from which our analysis begins. Thus, the full training process need not start near a local minimizer.  Basically, Eq. \eqref{assump_ip_eq_1} can be justified based on the convergence results established in \cite{cao2025decentralized}. Specifically, by selecting a small step size $\mu' \le O(\mu^2)$, we know from \cite{cao2025decentralized} that the centralized adversarial training algorithm, which aligns with the single-agent case, approaches an $O(\frac{\mu'}{B}) + O(\epsilon^2)$-neighborhood of a local minimum of $J(w)$ after sufficient iterations in nonconvex environments, where the $O(\frac{\mu'}{B})$ term arises from the gradient noise, while the $O(\epsilon^2)$ term originates from perturbations. By further utilizing the local convexity of $J(w)$ near a local minimum, it follows from \cite{cao2025decentralized} that after sufficient iterations:
\begin{align}\label{assump_ip_eq_1_e}
\mathds{E}\Vert \boldsymbol{w}_{k,n} - w^{\star}\Vert^2 \le  O\left(\frac{\mu'}{B}\right) + O(\epsilon^{\frac{7}{4}}) = o(\frac{\mu}{B}) + O(\epsilon^{\frac{7}{4}})
\end{align}
where, again, the $O\left(\frac{\mu'}{B}\right)$ term stems from the gradient noise, and the $O(\epsilon^{\frac{7}{4}})$ term is attributed to the approximation error when solving the maximization problem over the perturbation variable.  Furthermore, it can be shown that the centralized algorithm attains an $O\left(\frac{\mu'^2}{B^2}\right) + O(\epsilon^{\frac{7}{2}})$-neighborhood of $w^{\star}$ in the fourth-order moment sense after enough iterations:
\begin{align}\label{assump_ip_eq_2_e}
     \mathds{E}\Vert \boldsymbol{w}_{k,n} - w^{\star}\Vert^4 \le O\left(\frac{\mu'^2}{B^2}\right) + O(\epsilon^{\frac{7}{2}}) = o(\frac{\mu^2}{B^2}) + O(\epsilon^{\frac{7}{2}})
\end{align}
The detailed proof for these results can be found in Appendix \ref{p_assumpip}. Accordingly, Assumption \ref{assump_ip} is justified.

Next, considering the Hessian matrix of $J_k(w)$ at $w^{\star}$ denoted by
\begin{align}\label{localhsws}
  H_k^\star \overset{\Delta}{=} \nabla^2 J_k(w^\star)
\end{align}
and the Hessian matrix of $J(w)$ at $w^{\star}$:
\begin{align}\label{glbhsws}
        \bar{H} \overset{\Delta}{=} \nabla^2 J(w^\star) = \frac{1}{K}\sum_{k=1}^K \nabla^2 J_k(w^\star) = \frac{1}{K}\sum_{k=1}^K H_k^\star
\end{align}
we introduce the following assumption:
\begin{assumption}\label{assump_sh}
(\textbf{Small Hessian disagreement})  The Hessian matrix of $J_k(w)$ evaluated at $w^\star$ is close to the  Hessian matrix of $J(w)$ at $w^\star$, namely,
   \begin{align}
       \Vert H_k^{\star} - \bar{H}\Vert \le \rho
   \end{align}
   with a small constant $\rho$. $\hfill\square$
\end{assumption}
The motivation behind this assumption is identical to that in the multi-agent clean case. A detailed explanation can be found in \cite{cao2024trade}.

\subsection{Escaping efficiency of algorithms}
Recalling Definition \ref{def_ee}, the escaping efficiency is quantified by the excess risk of $J(w)$, which we denote by $\mathrm{ER}_n$. In general, the excess risk value near a local minimum satisfies the following equality (see Eq.(205) in \cite{cao2024trade}):
\begin{align}\label{eq_ern}
    \mathrm{ER}_n = \frac{1}{2K}\mathds{E}\Vert{\widetilde{\boldsymbol{\scriptstyle\mathcal{W}}}}_{n}\Vert^2_{I\otimes\bar{H}}\pm O((\mathds{E}\Vert{\widetilde{\boldsymbol{\scriptstyle\mathcal{W}}}}_{n}\Vert^4)^{\frac{3}{4}}) 
\end{align}
If $(\mathds{E}\Vert{\widetilde{\boldsymbol{\scriptstyle\mathcal{W}}}}_{n}\Vert^4)^{\frac{3}{4}}$ is negligible compared to $\mathds{E}\Vert{\widetilde{\boldsymbol{\scriptstyle\mathcal{W}}}}_{n}\Vert^2$, then the excess risk $\mathrm{ER}_n$ can be well represented by the leading quadratic term involving $\mathds{E}\Vert{\widetilde{\boldsymbol{\scriptstyle\mathcal{W}}}}_{n}\Vert^2$.  This approximation has been verified in the clean setting in \cite{cao2024trade}. In this work, we aim to establish a similar result in the adversarial setting.

Moreover, recalling \eqref{unified_re_2}, since the block matrix $\mathcal{H}_{n-1}$ depends on $\widetilde{\boldsymbol{\scriptstyle\mathcal{W}}}$, it is intractable to express $\mathrm{ER}_n$ directly with $\mathds{E}\Vert{\widetilde{\boldsymbol{\scriptstyle\mathcal{W}}}}_{n}\Vert^2$, even if the dominance of the leading term in \eqref{eq_ern} is established. To address the issue, and following the approach in \cite{cao2024trade, VlaskiS21a, sayed2014adaptation}, we use the Hessian matrix at $w^{\star}$ to replace $\mathcal{H}_{n-1}$. We define the following block diagonal matrix:
\begin{align}\label{collec_hk}
\mathcal{H} \overset{\Delta}{= } \mathrm{diag}\left\{H_{1}^\star,H_{2}^\star,\ldots H_{K}^\star\right\}
\end{align}
then, we introduce the following recursion, which is also called a short-term model, to replace the recursion in \eqref{unified_re_2}:
\begin{align}\label{approx_unified_re_2}
{\widetilde{\boldsymbol{\scriptstyle\mathcal{W}}}}_{n}' = \mathcal{A}_2\left(\mathcal{A}_1 - \mu\mathcal{H}\right)\widetilde{\boldsymbol{\scriptstyle\mathcal{W}}}_{n-1}' -  \mu\mathcal{A}_2d - \mu\mathcal{A}_2\boldsymbol{s}_{n}^{B}
\end{align}
Naturally, it is important to analyze the approximation error introduced by replacing the original dynamics with the short-term model ${\widetilde{\boldsymbol{\scriptstyle\mathcal{W}}}}_{n}'$. To do so, we separately establish bounds for $\mathds{E}\Vert{\widetilde{\boldsymbol{\scriptstyle\mathcal{W}}}}_{n}\Vert^2$, $\mathds{E}\Vert{\widetilde{\boldsymbol{\scriptstyle\mathcal{W}}}}_{n}\Vert^4$, and the approximation error $\vert\mathds{E}\Vert{\widetilde{\boldsymbol{\scriptstyle\mathcal{W}}}}_{n}\Vert^2 - \mathds{E}\Vert{\widetilde{\boldsymbol{\scriptstyle\mathcal{W}}}}_{n}'\Vert^2\vert$ for decentralized and centralized methods in following lemmas. 

\begin{lemma}\label{mse_2}
(\textbf{Bounds for second-order moments}) For a fixed small step size $\mu$ and local batch size $B$, and under assumptions \ref{assump_uniqueperturb}--\ref{assump_ip}, it can be verified that for both \textbf{consensus} and \textbf{diffusion} adversarial training methods, the second-order moments of the error vector  
${\widetilde{\boldsymbol{\scriptstyle\mathcal{W}}}}_{n}$ are upper bounded over a finite time horizon $n \le O(\frac{1}{\mu})$. Specifically, it holds that
 \begin{align}
  \label{mse_2_decentra}
\mathds{E}\Vert{\widetilde{\boldsymbol{\scriptstyle\mathcal{W}}}}_{n}\Vert^2 &\le O\left(\frac{\mu}{B}\right) + O(\mu^2) + O(\epsilon^{\frac{7}{4}})
\end{align}
where the $O\left(\frac{\mu}{B}\right)$ term arises from the gradient noise, the $O(\mu^2)$ term reflects the influence of the network heterogeneity and the graph structure, and the $O(\epsilon^{\frac{7}{4}})$ is due to the presence of adversarial perturbations. Similarly, for the \textbf{centralized} strategy, the following bound holds:
 \begin{align}
  \label{mse_2_centra}
\mathds{E}\Vert{\widetilde{\boldsymbol{\scriptstyle\mathcal{W}}}}_{n}\Vert^2 &\le O\left(\frac{\mu}{B}\right) + O(\epsilon^{\frac{7}{4}})
\end{align}
\end{lemma}
\begin{proof}
See Appendix \ref{p_mse_2}.
\end{proof}
Comparing the bounds in \eqref{mse_2_decentra} and \eqref{mse_2_centra}, we observe that decentralized methods incorporate an extra $O\left(\mu^2\right)$ term compared to the centralized strategy. However, the influence of the additional term become significant only in the large-batch context when the following inequality is satisfied, 
 \begin{align}\label{rmuB}
 \frac{1}{B} \le O(\mu)  
 \end{align}
 otherwise, it will be dominated by $O(\frac{\mu}{B})$, and the difference between decentralized and centralized methods becomes negligible.  Since our objective is to understand and highlight the distinction between the centralized and decentralized training strategies in the adversarial environment, we focus on the large-batch setting throughout this paper. As a result, for the three methods, it holds that
  \begin{align}
  \label{mse_2_unified}
\mathds{E}\Vert{\widetilde{\boldsymbol{\scriptstyle\mathcal{W}}}}_{n}\Vert^2 &\le O(\mu^2) + O(\epsilon^{\frac{7}{4}})
\end{align}

We then establish bounds for $\mathds{E}\Vert{\widetilde{\boldsymbol{\scriptstyle\mathcal{W}}}}_{n}\Vert^4$. The primary motivation for establishing this result is its subsequent use in analyzing the approximation error induced by the short-term model.
\begin{lemma}\label{mse_4}
(\textbf{Bounds for fourth-order moments}) For a fixed small step size $\mu$, under assumptions \ref{assump_uniqueperturb}--\ref{assump_ip} and \eqref{rmuB}, and in a finite number of iterations $n\le O(1/\mu)$, the fourth-order moments of the error vector ${\widetilde{\boldsymbol{\scriptstyle\mathcal{W}}}}_{n}$ for the three distributed adversarial training methods are upper bounded by:
 \begin{align}
 \label{mse_decentra_w4}
\mathds{E}\Vert{\widetilde{\boldsymbol{\scriptstyle\mathcal{W}}}}_{n}\Vert^4 &\le O(\mu^4) + O(\epsilon^{\frac{7}{2}})
\end{align}
\end{lemma}
\begin{proof}
See Appendix \ref{p_mse_4}.
\end{proof}

We next examine the approximation error introduced by the short-term model. 
\begin{lemma}\label{mse_ae}
(\textbf{Approximation error}) Under the same conditions as in Lemma \ref{mse_4}, and assuming in addition that the following condition holds,
\begin{align}\label{l_ae_1}
    \epsilon \le O(\mu^2)
\end{align}
the second-order moments of the short-term model associated with the three distributed adversarial training methods are upper bounded, namely,
 \begin{align}\label{l_ae_2}
\mathds{E}\Vert\widetilde{\boldsymbol{\scriptstyle\mathcal{W}}}_n' \Vert^2 {\le}& O(\mu^2)
\end{align}
Moreover, the approximation error of the short term model is also upper bounded as follows:
\begin{align}\label{l_ae_3}
\left\vert \mathds{E}\Vert\widetilde{\boldsymbol{\scriptstyle\mathcal{W}}}_n' \Vert^2 - \mathds{E}\Vert\widetilde{\boldsymbol{\scriptstyle\mathcal{W}}}_n\Vert^2 \right\vert {\le}& O(\mu^3)
\end{align}
\end{lemma}
\begin{proof}
    See appendix \ref{p_mse_ae}.
\end{proof}

From Lemmas \ref{mse_2}, \ref{mse_4} and \ref{mse_ae}, it follows that  $\mathds{E}\Vert{\widetilde{\boldsymbol{\scriptstyle\mathcal{W}}}}_{n}\Vert^2$ and  $\mathds{E}\Vert{\widetilde{\boldsymbol{\scriptstyle\mathcal{W}}}}_{n}'\Vert^2$ dominate $\vert \mathds{E}\Vert\widetilde{\boldsymbol{\scriptstyle\mathcal{W}}}_n'\Vert^2 - \mathds{E}\Vert\widetilde{\boldsymbol{\scriptstyle\mathcal{W}}}_n \Vert^2\vert $ in magnitude across the three distributed adversarial training methods when $\mu$ and $\epsilon$ are sufficiently small in the large-batch setting. Therefore, we can evaluate the $\mathrm{ER}_n$ performance of the three methods by using the short-term model, namely,
\begin{align}\label{eq_ern_2}
    \mathrm{ER}_n = \frac{1}{2K}\mathds{E}\Vert{\widetilde{\boldsymbol{\scriptstyle\mathcal{W}}}}_{n}'\Vert^2_{I\otimes\bar{H}}\pm o(\mathds{E}\Vert{\widetilde{\boldsymbol{\scriptstyle\mathcal{W}}}}_{n}'\Vert^2) 
\end{align}
and the second term on the RHS of \eqref{eq_ern_2} can be omitted. Note that \eqref{l_ae_1} is a sufficient condition to guarantee \eqref{l_ae_3}. Without this condition, the approximation error $\left|\mathds{E}\Vert\widetilde{\boldsymbol{\scriptstyle\mathcal{W}}}_n'\Vert^2-\mathds{E}\Vert\widetilde{\boldsymbol{\scriptstyle\mathcal{W}}}_n\Vert^2\right|$ is not guaranteed to be negligible compared with $\mathds{E}\Vert\widetilde{\boldsymbol{\scriptstyle\mathcal{W}}}_n\Vert^2$. Consequently, the short-term model would not approximate  the true model well, and the analysis cannot continue.

Building upon Lemmas \ref{mse_2}, \ref{mse_4} and \ref{mse_ae}, and using \eqref{eq_ern_2}, we establish the following theorem to demonstrate the escaping efficiency of the three distributed adversarial training algorithms from local minima.

\begin{theorem}\label{th_ee}
{(\textbf{Escaping efficiency of distributed algorithms})}.
Consider a network of agents running distributed adversarial training algorithms governed by the recursion (\ref{unified_re_1}). Under assumptions \ref{assump_uniqueperturb}--\ref{assump_sh} and the conditions specified in \eqref{rmuB} and \eqref{l_ae_1}, after $n$ iterations with 
\begin{align}
    n \le O(1/\mu)
\end{align}
the following results holds:  
\begin{align}
\label{th_ee_1}
   \mathrm{ER}_{n,cen} = & {\frac{\mu}{B} e(n)} \\
\label{th_ee_2} 
   \mathrm{ER}_{n,con} = &{\frac{\mu}{B} e(n)+ \mu^2 f_{con}(n)} \\
\label{th_ee_3}
   \mathrm{ER}_{n,dif} = &  {\frac{\mu}{B}e(n)+ \mu^2 f_{dif}(n)} 
\end{align}
where
\begin{align}
\label{th_ee_4}
&e(n) = \frac{1}{4} \mathrm{Tr}\left(\left(I - (I - \mu \bar{H})^{2(n+1)}\right)\bar{R}_s \right)\\
\label{th_ee_5}
&f_{con}(n) = \frac{1}{2K}\Vert d^{\sf T}\mathcal{V}_\alpha( I - \mathcal{P}_{\alpha})^{-1}(I - \mathcal{P}_{\alpha}^{n+1})\Vert^2_{I\otimes\bar{H}}\\
\label{th_ee_6}
&f_{dif}(n) = \frac{1}{2K}\Vert d^{\sf T}\mathcal{V}_\alpha \mathcal{P}_\alpha( I - \mathcal{P}_{\alpha})^{-1}(I - \mathcal{P}_{\alpha}^{n+1})\Vert^2_{I\otimes\bar{H}}
\end{align}
and $\mathrm{{ER}}_{n,cen}$, $\mathrm{{ER}}_{n,con}$, and $ \mathrm{{ER}}_{n,dif}$ represent the excess risk of the centralized, consensus and diffusion methods at iteration $n$, respectively. Moreover, it holds that
\begin{align}\label{th_ee_7}
\mathrm{ER}_{n,cen} \le \mathrm{ER}_{n,dif} \le \mathrm{ER}_{n,con}
\end{align}    
\end{theorem}
\begin{proof}
The proof of Theorem \ref{th_ee} follows a similar line of reasoning as in Appendices E and F of \cite{cao2024trade}, with the key distinction that adversarial perturbations are explicitly accounted for in the definitions of the Hessian matrix $\bar{H}$ and the covariance matrix $\bar{R}_s$. In contrast, \cite{cao2024trade} focuses exclusively on the clean case.
\end{proof}

Theorem \ref{th_ee} shows the escaping efficiency of the three distributed adversarial training algorithms from local minima. Similar to Appendix E of \cite{cao2024trade}, in theory, the algorithms require $O(1/\mu)$ number of iterations to exit the local basin. Theorem \ref{th_ee} is therefore a local and finite-horizon (i.e., non-asymptotic) result derived under a fixed step size $\mu$ over finite number of iterations. We observe from \eqref{th_ee_1}--\eqref{th_ee_3} that, similar to findings in the clean case \cite{cao2024trade}, decentralized adversarial training strategies--namely, diffusion and consensus--implicitly incorporate extra $O(\mu^2)$ terms due to the network heterogeneity and the graph structure compared to the centralized method. Specifically, in the homogeneous networks where $w^{\star}_k = w^{\star}$ for all agents, $d$ vanishes. In addition, as mentioned earlier, the centralized update in \eqref{at_centra_1}--\eqref{at_centra_2} is fundamentally equivalent to the diffusion update in \eqref{diff_a_1}--\eqref{diff_a_3} with the combination matrix set to $A=\frac{1}{K}\mathbbm{1}\mathbbm{1}^{\sf T}$. In such case, the matrix $\mathcal{P}_{\alpha}$ is $0$ and so that $f_{dif}(n)$ becomes 0 in the centralized method. In these cases, the extra $O(\mu^2)$ become $0$.  On the contrary, in heterogeneous networks and decentralized settings, the presence of nonzero $O(\mu^2)$
terms leads to larger $\mathrm{ER}_n$ values for decentralized adversarial training methods compared to their centralized counterpart. As a result, decentralized adversarial training methods exhibit faster escape behavior from local minima. Furthermore, as indicated by \eqref{th_ee_7}, consensus escapes faster from local minima than diffusion.

Notably, the above results rely on the conditions specified in \eqref{rmuB} and \eqref{l_ae_1}. First, similar to the clean case in \cite{cao2024trade}, \eqref{rmuB} implies that the batch size should be sufficiently large; otherwise, the $O(\mu^2)$ terms \eqref{th_ee_5}--\eqref{th_ee_6} may become negligible in comparison to the $O(\mu/B)$ term. In such cases, the performance differences among decentralized and centralized methods become less pronounced. Second, \eqref{l_ae_1} requires $\epsilon$ to be sufficiently small, which intuitively corresponds to the case where the attack strength of adversarial perturbations is relatively mild. Otherwise, the approximation error of the short-term model may no longer be negligible compared to the true model. In such cases, the conclusions drawn from Theorem~\ref{th_ee} may no longer be valid.

Similar to \cite{cao2024trade}, we can further relate the escaping efficiency of algorithms to the flatness of local minima. In general, the \emph{flatness} of a local minimum characterizes how the risk value changes with perturbations of the model parameter \cite{cao2024trade,Stutz0S21}. Locally, this variation is closely related to the Hessian matrix \(\bar{H}\), as shown in Eq. (9) of \cite{cao2024trade}. Following the reasoning in Eqs.~(67)--(73) of \cite{cao2024trade}, we similarly conclude that the three algorithms tend to stay around flat local minima, where $\mathrm{Tr}(\bar{H})$ is small. As discussed before, decentralized adversarial training methods escape more quickly from local minima than the centralized strategy. As a result, they are more likely to leave sharp basins and converge to a flatter ones. This demonstrates the tendency of decentralized methods to favor flatter minima over the centralized method in the adversarial environment. Furthermore, within the decentralized methods, the consensus strategy demonstrates a stronger preference for flatter minima compared to the diffusion one.


\subsection{Discussion}

Reference~\cite{cao2024trade} studied the escaping efficiency framework for standard distributed learning. The definitions of basin of attraction and escaping efficiency, the graph eigendecomposition, and the general idea of the short-term model used in this paper are inherited from that work. However, extending the framework to adversarial training is not achieved by simply replacing clean samples with adversarial ones. First, as shown in Lemma~\ref{affine_l}, the gradient of the adversarial risk satisfies only the affine-Lipschitz bound where an additional $O(\epsilon)$ term is involved. Therefore, unlike in the clean analysis, a uniform bound on the Hessian at arbitrary iterates is no longer guaranteed. Second, as \eqref{local_risk} shows, the adversarial risk contains an inner maximization problem. As the exact inner maximizer is generally intractable, adversarial training uses an approximate maximizer in practice. This introduces the additional error term $\boldsymbol{e}_{n-1}$ into the true recursion of $\widetilde{\boldsymbol{\scriptstyle\mathcal{W}}}_{n}$, as shown in \eqref{unified_re_1}. Third, due to adversarial perturbations, the initial-point bounds in Assumption~\ref{assump_ip} contain extra $O(\epsilon^{\frac{7}{4}})$ and $O(\epsilon^{\frac{7}{2}})$ terms, as shown in \eqref{assump_ip_eq_1} and \eqref{assump_ip_eq_2}, while the clean case does not have them. Note that they are not imposed arbitrarily, but are motivated by the convergence results, as discussed following Assumption \ref{assump_ip}. Together, These differences lead to extra technical difficulties in the proofs compared to~\cite{cao2024trade}.

To address these difficulties, we use the gradient-form recursion in \eqref{unified_re_1} to derive bounds, instead of directly applying the Hessian-based arguments from~\cite{cao2024trade}. Specifically, we apply the affine-Lipschitz gradient bound and the initial-point assumption, and control the error introduced by the approximate inner maximization. As a result, we derive new second- and fourth-order moment bounds in Lemma~\ref{mse_2} and \ref{mse_4}. Compared to \cite{cao2024trade}, these bounds contain the additional  perturbation-dependent terms $O(\epsilon^{\frac{7}{4}})$ and $O(\epsilon^{\frac{7}{2}})$. We further bound the error between the true dynamics and the short-term model, which gives the additional sufficient condition $\epsilon\le O(\mu^2)$ and establishes an $O(\mu^3)$ approximation error, as shown in Lemma \ref{mse_ae}. 

Once the validity of the short-term model is guaranteed, the final network-algebra calculation used to obtain the closed-form solution expressions in Theorem \ref{th_ee} follows a similar line of reasoning to \cite{cao2024trade}. However, the underlying quantities have different meanings. In  \cite{cao2024trade},  $\bar{H}$ refers to the Hessian of the clean risk at a local minimum of the clean risk, while $\bar{H}$ in this paper is the Hessian of the adversarial risk at minima of the adversarial risk. The gradient covariance matrix $\bar{R}_S$ follows a similar case. 

\section{Simulation results}
In this section, we compare the performance of the three distributed adversarial training algorithms on the CIFAR-10~\cite{krizhevsky2009learning}, CIFAR-100~\cite{krizhevsky2009learning} and Tiny ImageNet-200~\cite{5206848} datasets. A brief summary of the experimental setup is provided below, while additional implementation details can be found in Appendix \ref{appen_as}. We use WideResNet-28-10  for CIFAR-10 and Tiny ImageNet-200, and WideResNet-34-10 for CIFAR-100. We simulate three graph topologies, including a randomly generated doubly-stochastic and strongly-connected graph using the Metropolis rule as described in \cite{sayed2014adaptation}, the ring, and the grid structure. 
We also consider different network sizes, with $K\in\{8,16,32\}$. We evaluate each algorithm using local batch sizes with $B\in\{128, 256\}$. In the decentralized experiments, we randomly split the training dataset into $K$ disjoint subsets, with each agent observing data from a single subset only. For the centralized baseline, training follows the traditional single-agent training pipeline with the effective batch size $K \times B$.
We consider adversarial perturbations constrained by $\ell_2$ and $\ell_\infty$ norms. Specifically, for $\ell_\infty$ bounded attacks, we use  $\epsilon \in \{8/255$, $3/255$\} on CIFAR-10 and CIFAR-100 and $\epsilon=3/255$ on Tiny ImageNet-200. For $\ell_2$-bounded perturbations, we simulate $\epsilon=128/255$ for all three datasets. Following prior works~\cite{abs-2010-03593,RiceWK20}, we adopt the piece-wise learning rate schedule throughout all experiments. This schedule can also be motivated by combining Theorem~\ref{th_ee} with earlier convergence results [8], from which the step size $\mu$ plays two different roles. A relatively large $\mu$ is used in the early stage to help the algorithms explore the risk landscape and locate a flatter region, while a smaller step size is used during the later stage to guarantee stability.

Since \emph{overfitting} is a widely observed phenomenon in adversarial training, we evaluate the performance of algorithms using both the \emph{best} and \emph{final} models, where the \emph{best} model refers to the one that achieves highest validation performance during training, while the \emph{final} model denotes the one obtained at convergence. In the main text, to maintain consistency with the theoretical analysis, we present the simulation results for the classical PGD training \cite{madry2017towards} using the SGD optimizer with the momentum parameter set to $0$. To enable a more comprehensive comparison of the robustness of models trained by the three algorithms, we also report results using the SGD optimizer with momentum and the TRADES method \cite{ZhangYJXGJ19}. Due to space limitations, the main text reports only results for the random graph structure with $K=16$. Results for other graph topologies and network sizes, together with those for SGD with momentum and TRADES, are provided in the Appendix \ref{appen_as} of the supplementary material.

\subsection{Flatness visualization}\label{flat_visual_main}
We begin by visualizing the flatness of models obtained by the three algorithms. Specifically, we employ the risk landscape visualization method proposed in \cite{Li0TSG18} to examine the geometry around the learned models. Basically, for a given model parameterized by $w$, we compute the risk value $J(w+\alpha\boldsymbol{v})$ at the perturbed parameter $w+\alpha\boldsymbol{v}$, where  $\boldsymbol{v}$ is a random direction used to perturb the model parameter, and $\alpha$ controls the scale of the perturbation. For example, in a neural network, for the $i_{th}$ filter in layer $\ell$,  the corresponding random direction is rescaled as
\begin{align}
\boldsymbol{v}_{\ell,i}
\leftarrow
\frac{\boldsymbol{v}_{\ell,i}}{\Vert \boldsymbol{v}_{\ell,i}\Vert_\mathrm{F}}
\Vert w_{\ell,i}\Vert_\mathrm{F},
\end{align}
so that
\begin{align}
\Vert \boldsymbol{v}_{\ell,i}\Vert_\mathrm{F}= 
\Vert w_{\ell,i}\Vert_\mathrm{F}
\end{align}
where $\boldsymbol{v}_{\ell,i}$ and $w_{\ell,i}$ denote the $i_{th}$ filters in layer $\ell$ of the random direction and the model parameter, respectively. Thus, the perturbation $\alpha\boldsymbol{v}_{\ell,i}$ is controlled relative to the scale of each filter. This filter-wise normalization method helps mitigate not only scaling effects across layers and filters, but also visualization artifacts caused by function-preserving parameter rescaling \cite{Li0TSG18}. We also average the risk curves over multiple independently sampled random directions to guarantee that results are not occasional. Note that Hessian-based measures are not adopted here because the exact Hessian of the robust risk in \eqref{local_risk} needs to account for the dependence of the optimal perturbation $\delta^\star$ on $w$, which is difficult to characterize, particularly for the norm-constrained inner maximization problem. Existing Hessian-based diagnostics generally use surrogate Hessians evaluated at fixed adversarial examples \cite{LiuSLTS20,sitawarin2021sat} and can also be sensitive to parameter rescaling \cite{Stutz0S21}. We therefore use the filter-wise normalized robust-risk landscape as a simple and scale-aware measure of local flatness.

Figure \ref{loss_visual_cifar10_sgd} illustrates the results obtained on the CIFAR-10 dataset, whereas those on CIFAR-100 and Tiny ImageNet-200 are presented in Figure \ref{loss_visual_cifar100_sgd} and \ref{loss_visual_imagenet_sgd}, respectively. We focus on the \emph{final} models in the main text because they reflect the solutions converged at the end of training. This aligns with the broader objective of our theoretical analysis: understanding the types of minima to which the algorithms tend to converge through their local escape behavior. The \emph{best} models, selected according to validation performance, may capture temporary visits to sharp or flat regions during training and therefore do not necessarily represent the solutions ultimately reached. We report their results in Appendix~\ref{appendix_gb} as a complementary evaluation of practical performance. Figures \ref{loss_visual_cifar10_sgd}--\ref{loss_visual_imagenet_sgd} indicate that when adversarial perturbations are constrained by $\epsilon = 128/255$ using the $\ell_2$ norm or $\epsilon = 3/255$ using the $\ell_\infty$ norm, which reflect relatively mild attack scenarios, decentralized adversarial training methods tend to produce flatter models compared to their centralized counterpart. However, when the attack strength increases to $\epsilon = 8/255$ with the $\ell_\infty$ norm, this trend may no longer hold. Furthermore, within the decentralized methods, consensus generally yields slightly flatter models than diffusion. Given prior observations in the literature that flatter models often exhibit improved robustness \cite{WuX020, Stutz0S21,abs-2010-03593}, our simulation results naturally raise the following question: Do flatter models always guarantee better robustness against adversarial attacks in distributed adversarial training?

\begin{figure*}[htbp]
\begin{center}
\subfigure[CIFAR-10: B=128]{\includegraphics[width=.3\linewidth]
{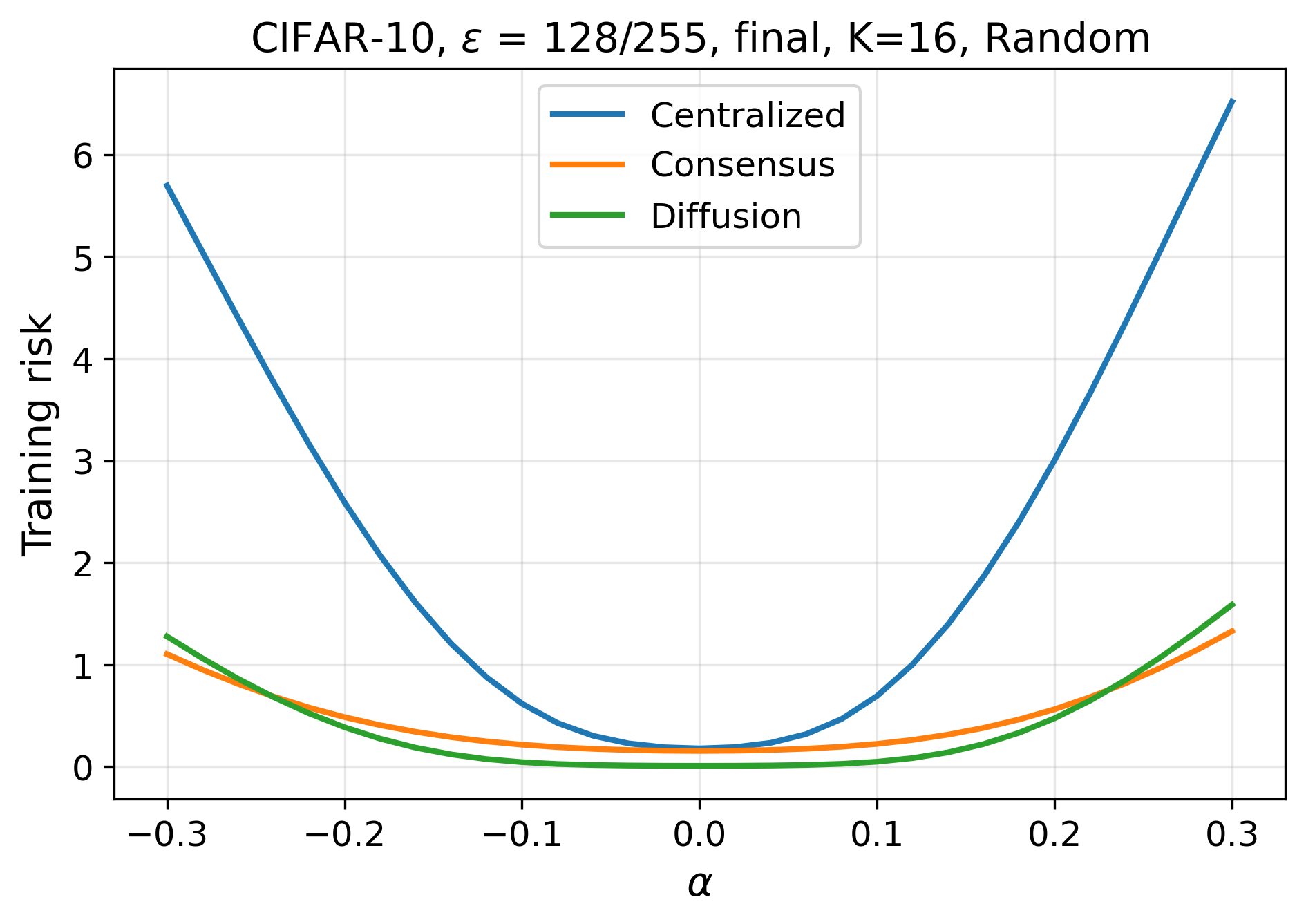}}
\subfigure[CIFAR-10: B=128]
{\includegraphics[width=.3\linewidth]{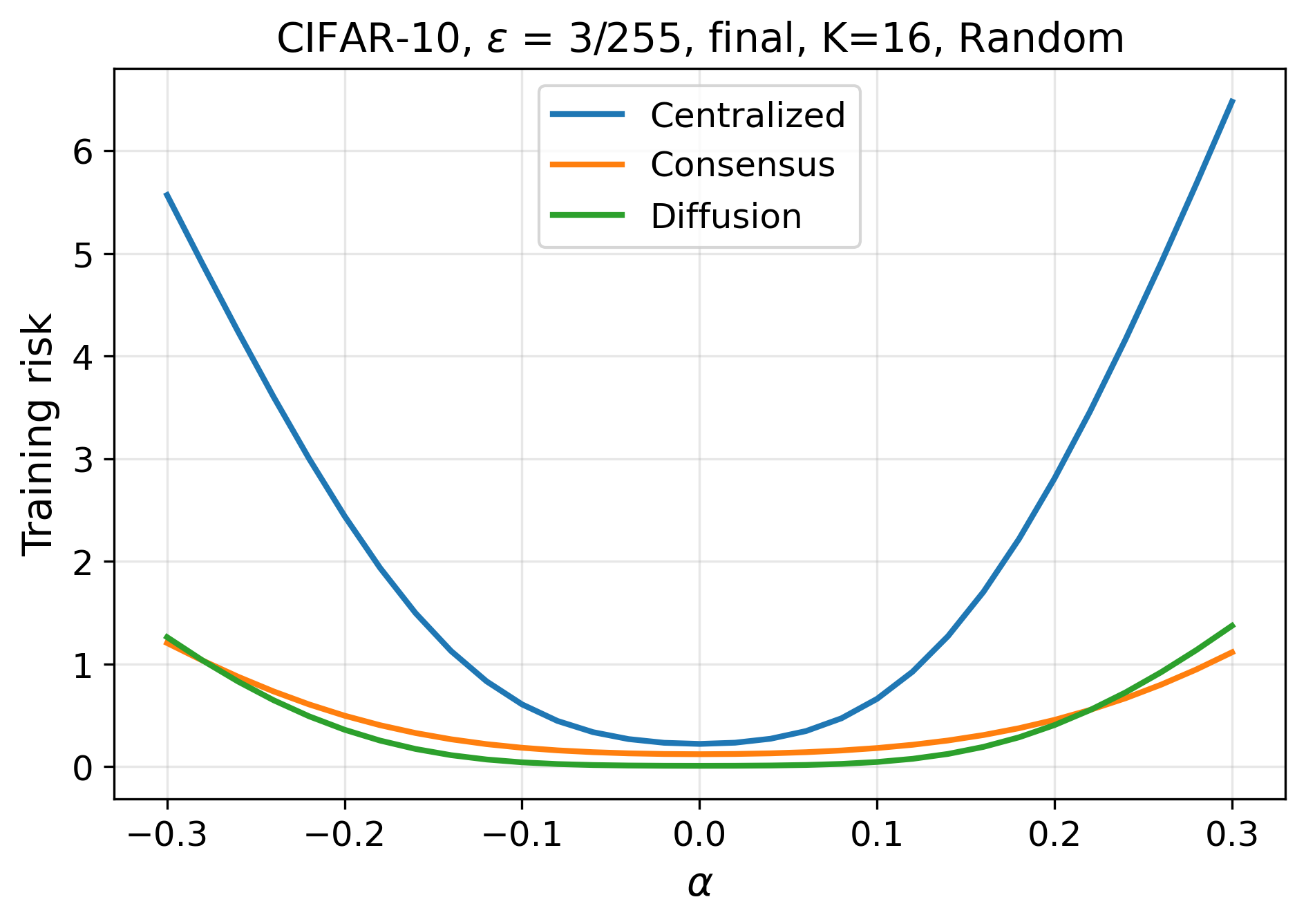}}
\subfigure[CIFAR-10: B=128]
{\includegraphics[width=.3\linewidth]
{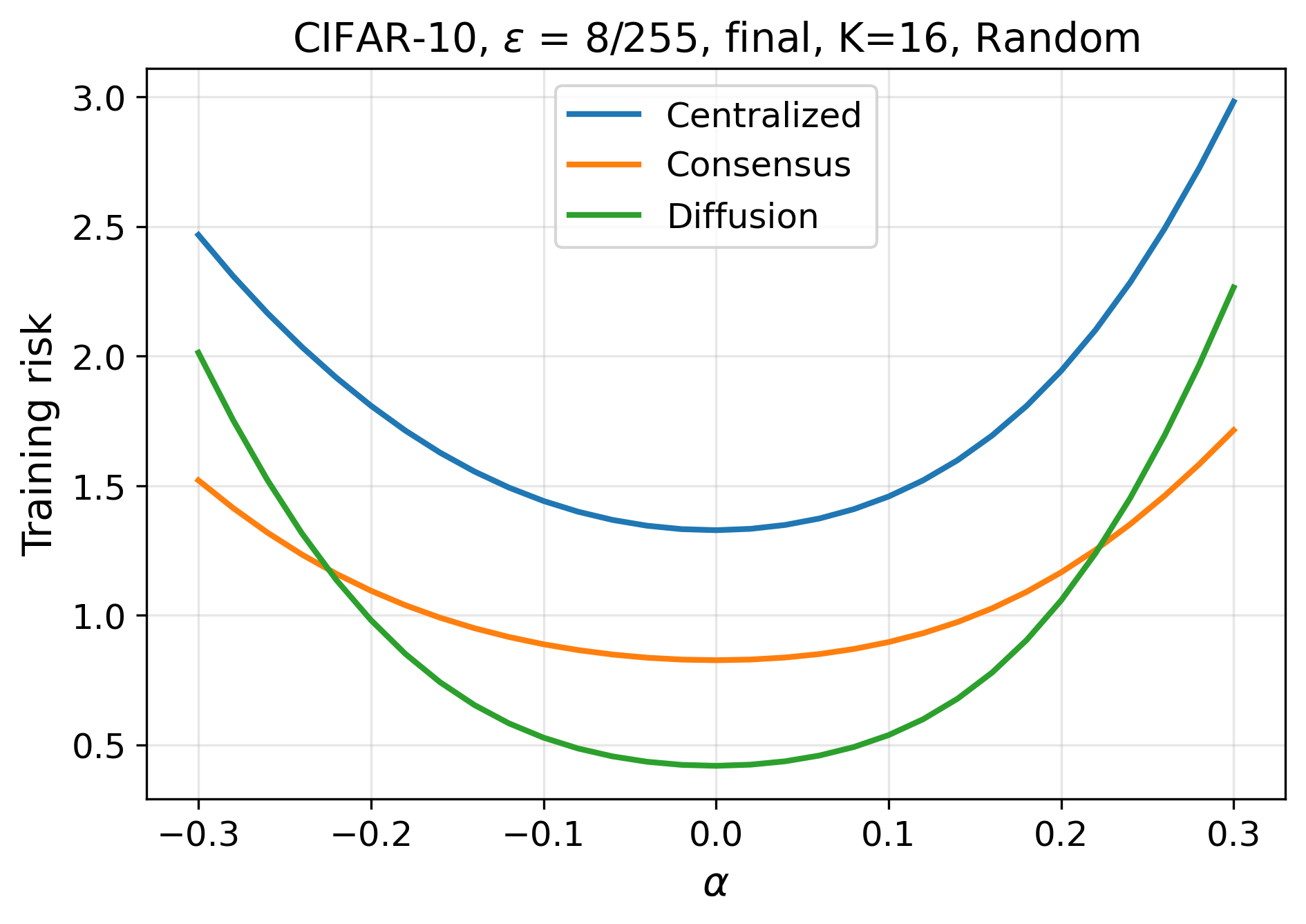}}

\subfigure[CIFAR-10: B=256]{\includegraphics[width=.3\linewidth]
{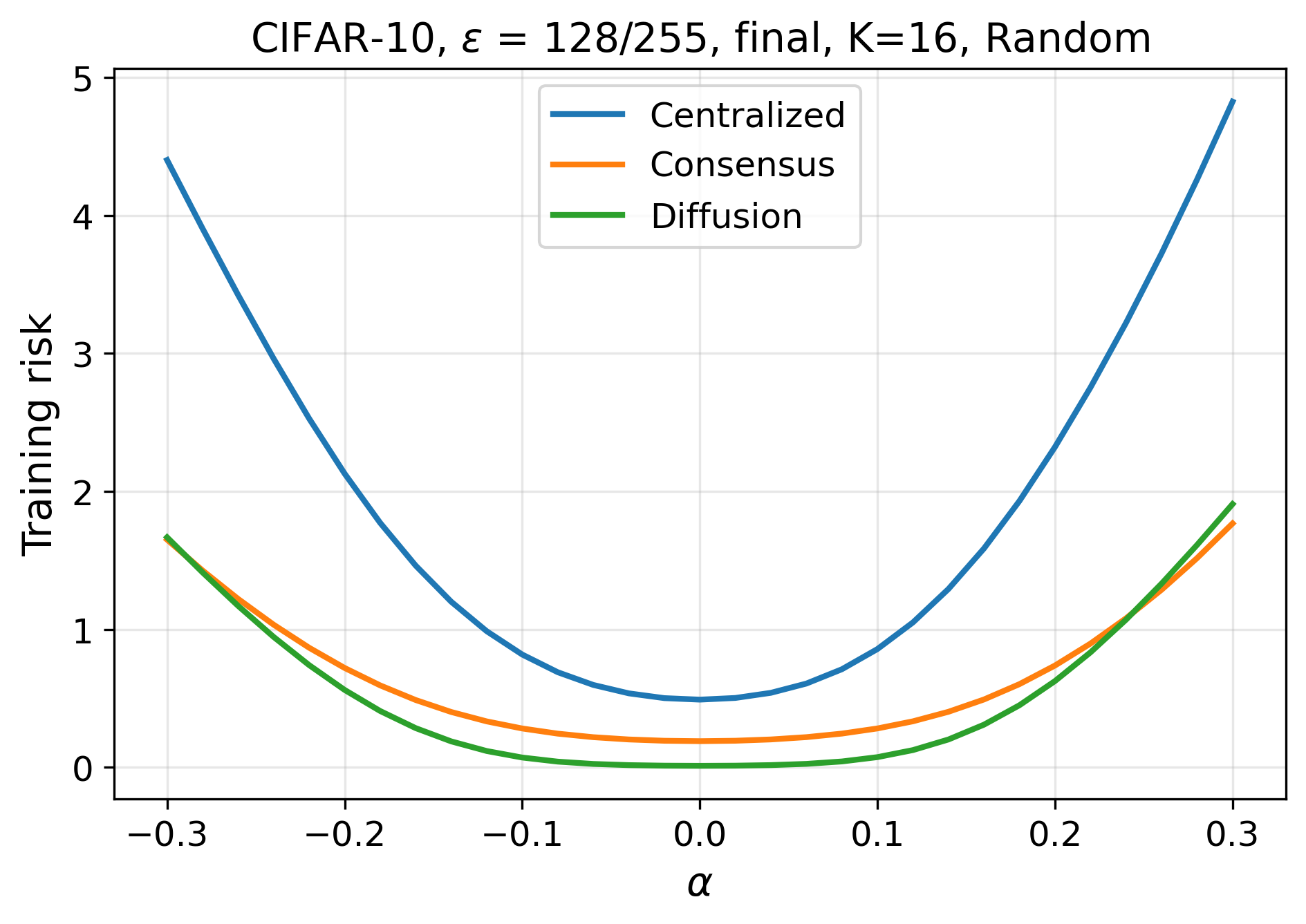}}
\subfigure[CIFAR-10: B=256 ]
{\includegraphics[width=.3\linewidth]{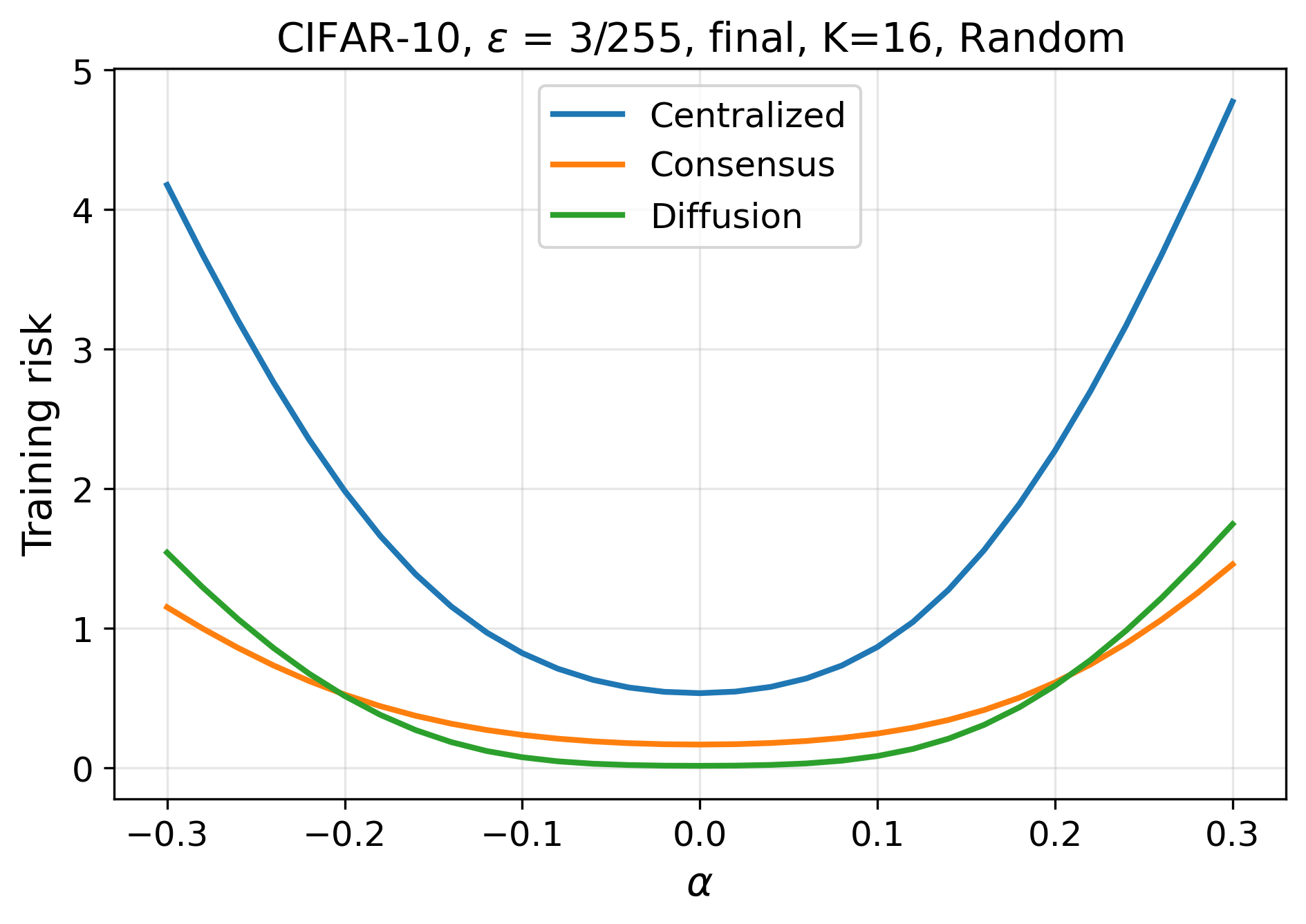}}
\subfigure[CIFAR-10: B=256 ]
{\includegraphics[width=.3\linewidth]
{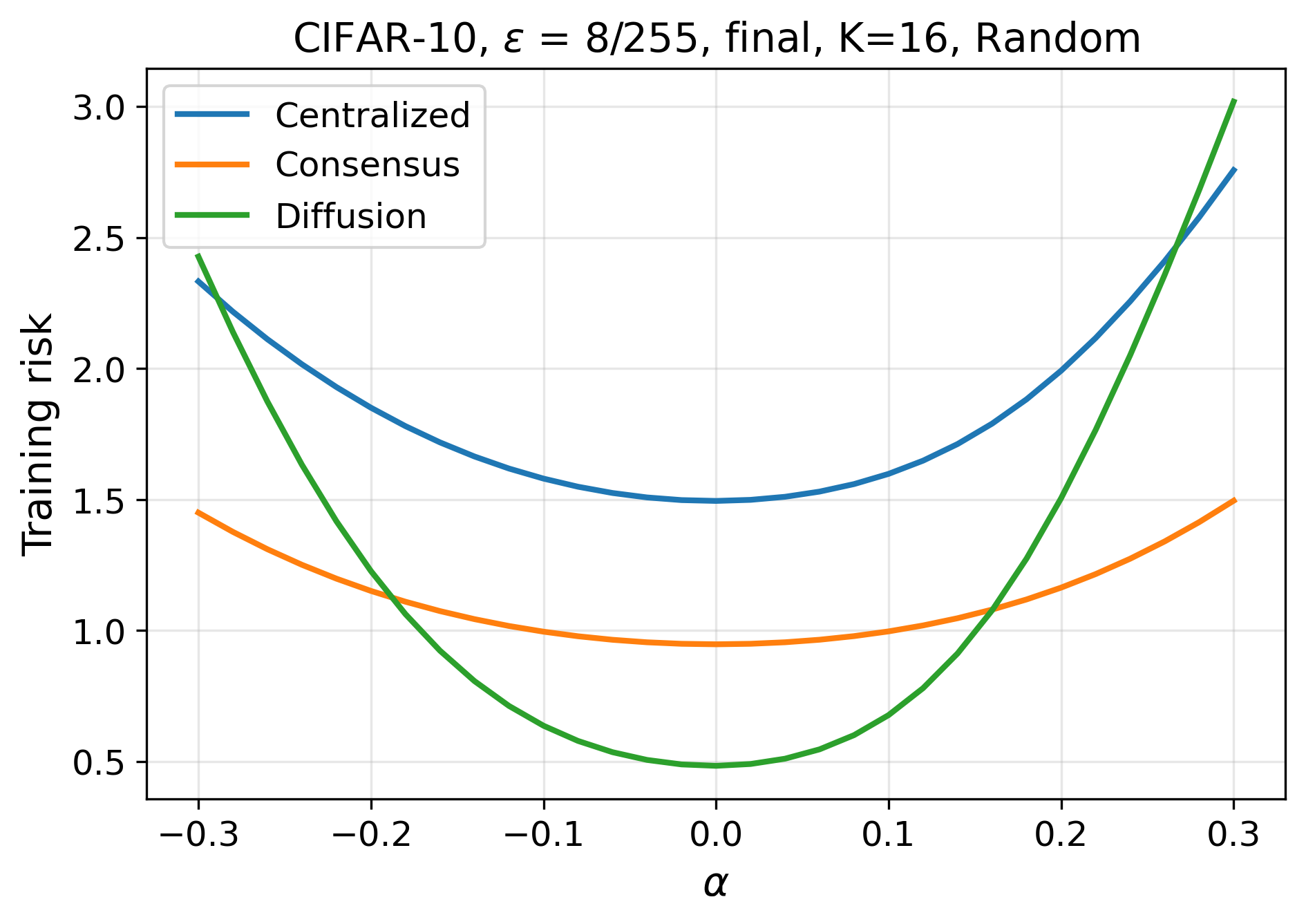}}

\caption{Flatness visualization of the \emph{final} models trained on the CIFAR-10 dataset.  Wider valleys correspond to flatter minima. The terminology \emph{random} in the title of each plots mean the randomly generated graph. From these figures, we observe that when adversarial perturbations are bounded by $128/255$ under the $\ell_2$ norm or by $3/255$ under $\ell_\infty$ norm--i.e., when the attack strength is relatively mild--decentralized adversarial training methods consistently yield flatter solutions compared to the centralized counterpart. However, when the attack strength increases to $\epsilon = 8/255$ under the $\ell_\infty$ norm, this principle could be broken. For example, in panels (c) and (f), the models obtained via diffusion are noticeably sharper than those obtained via centralized training.}
\label{loss_visual_cifar10_sgd}
\end{center}
\end{figure*}

\begin{figure*}[htbp]
\begin{center}
\subfigure[CIFAR-100: B=128]
{\includegraphics[width=.3\linewidth]
{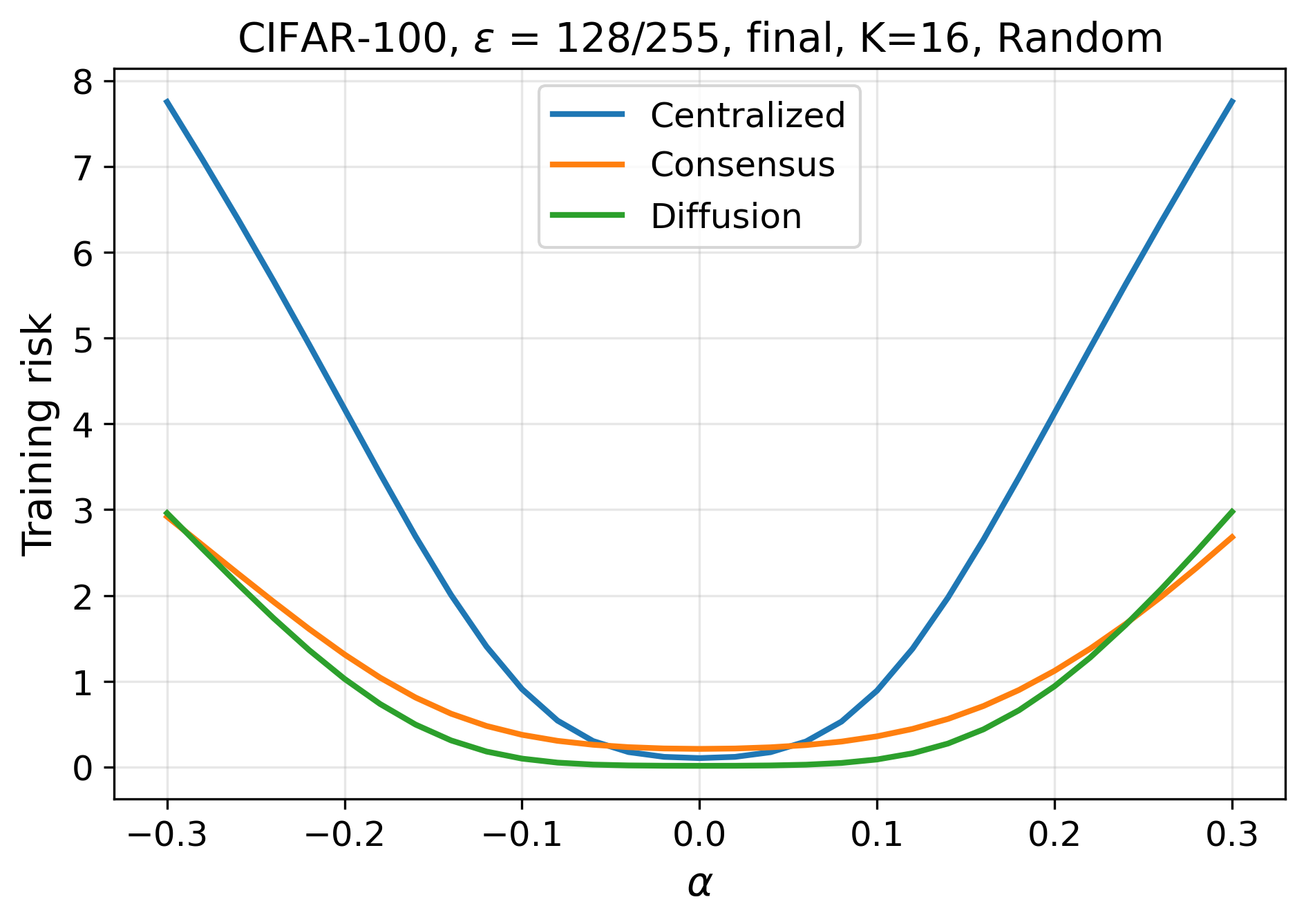}}
\subfigure[CIFAR-100: B=128]
{\includegraphics[width=.3\linewidth]
{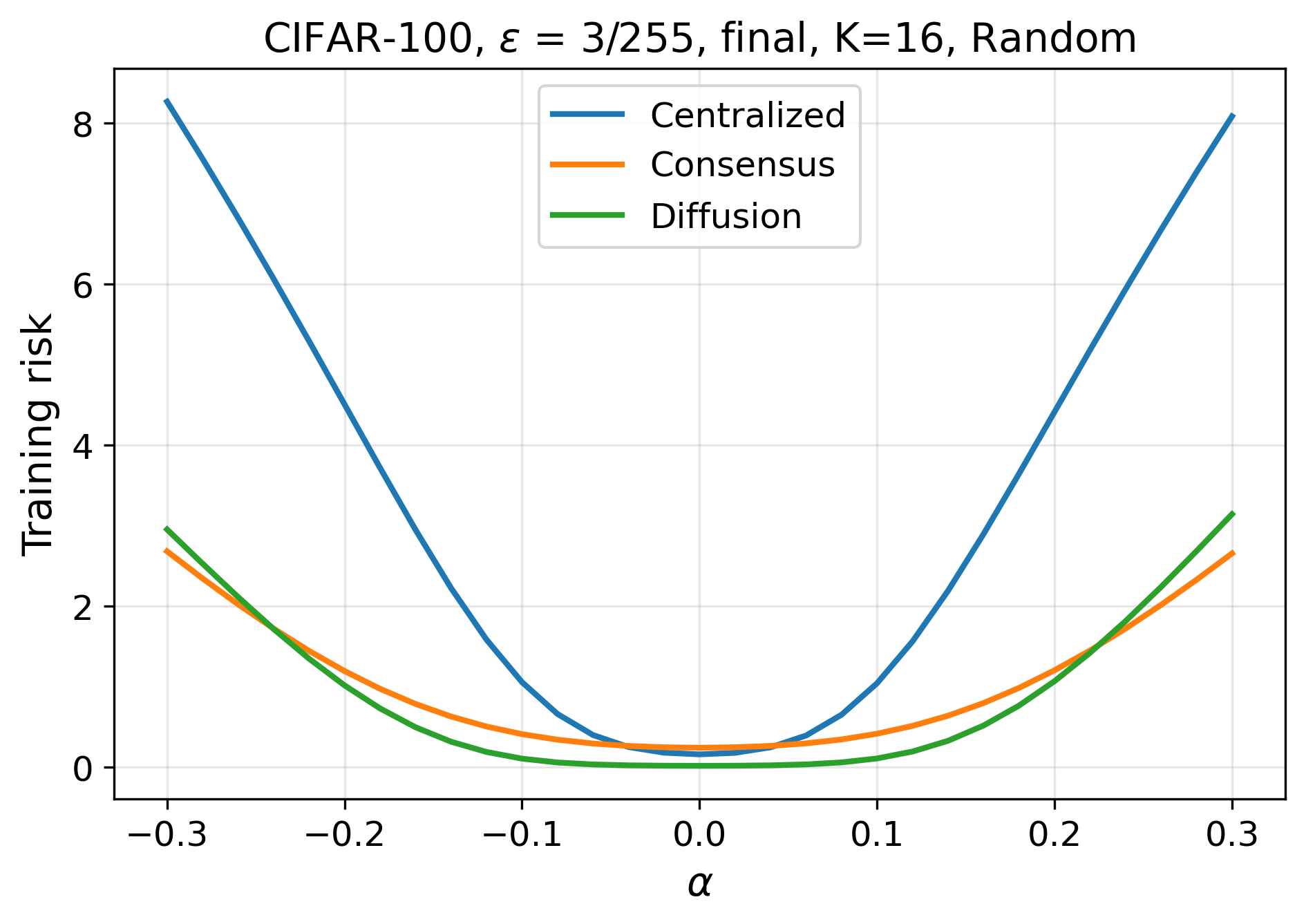}}
\subfigure[CIFAR-100: B=128]
{\includegraphics[width=.3\linewidth]
{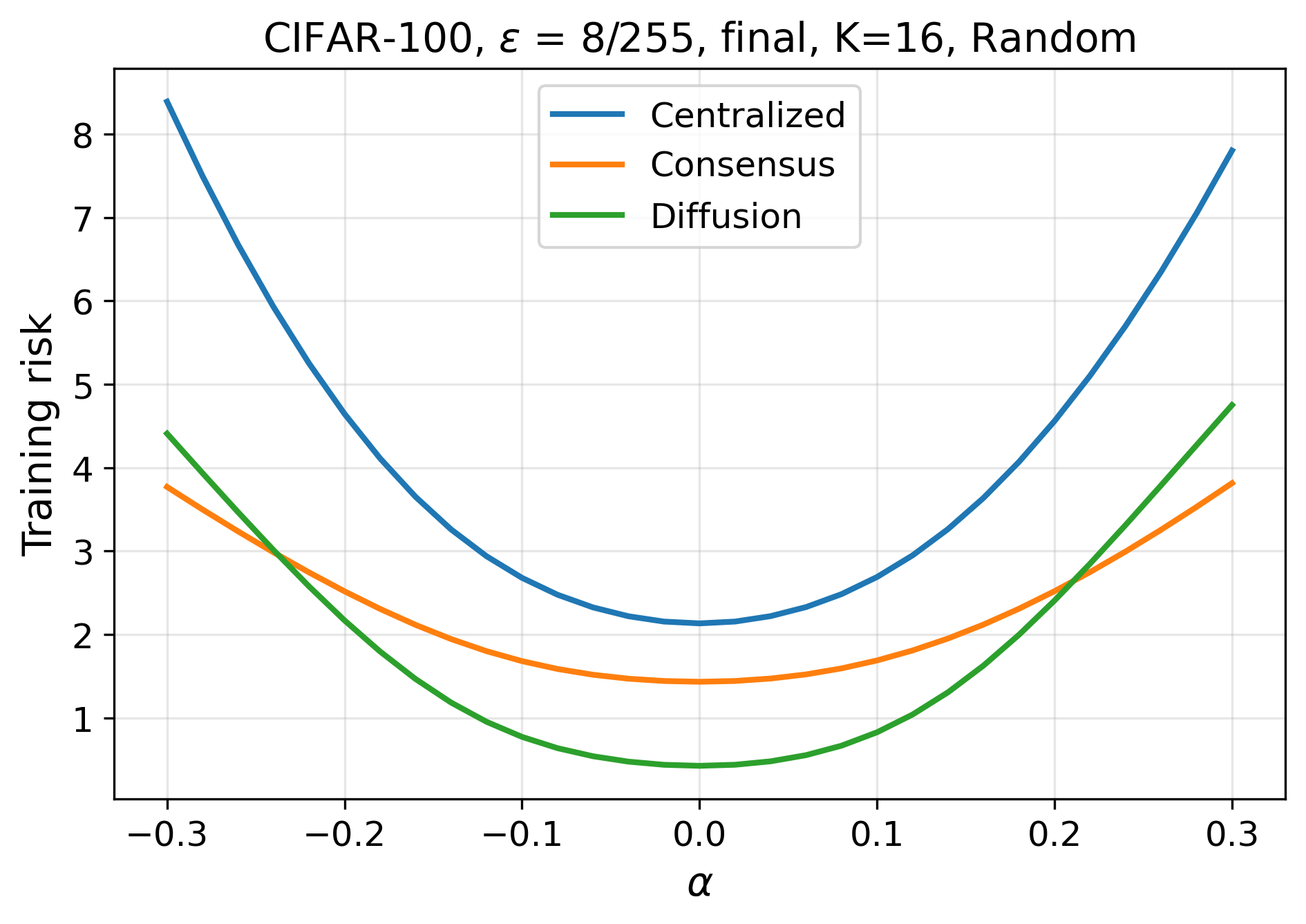}}
\subfigure[CIFAR-100: B=256 ]
{\includegraphics[width=.3\linewidth]
{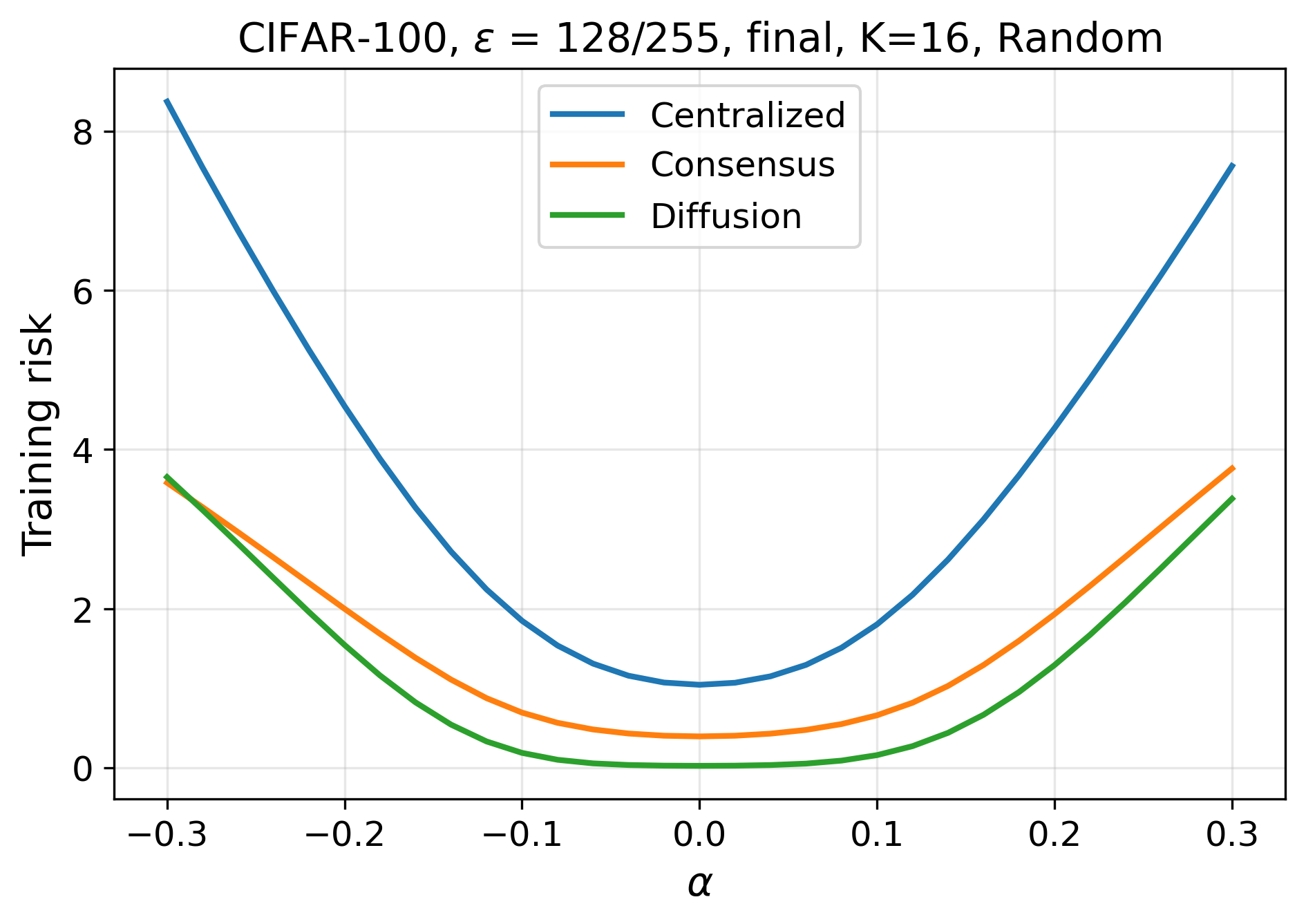}}
\subfigure[CIFAR-100: B=256]
{\includegraphics[width=.3\linewidth]
{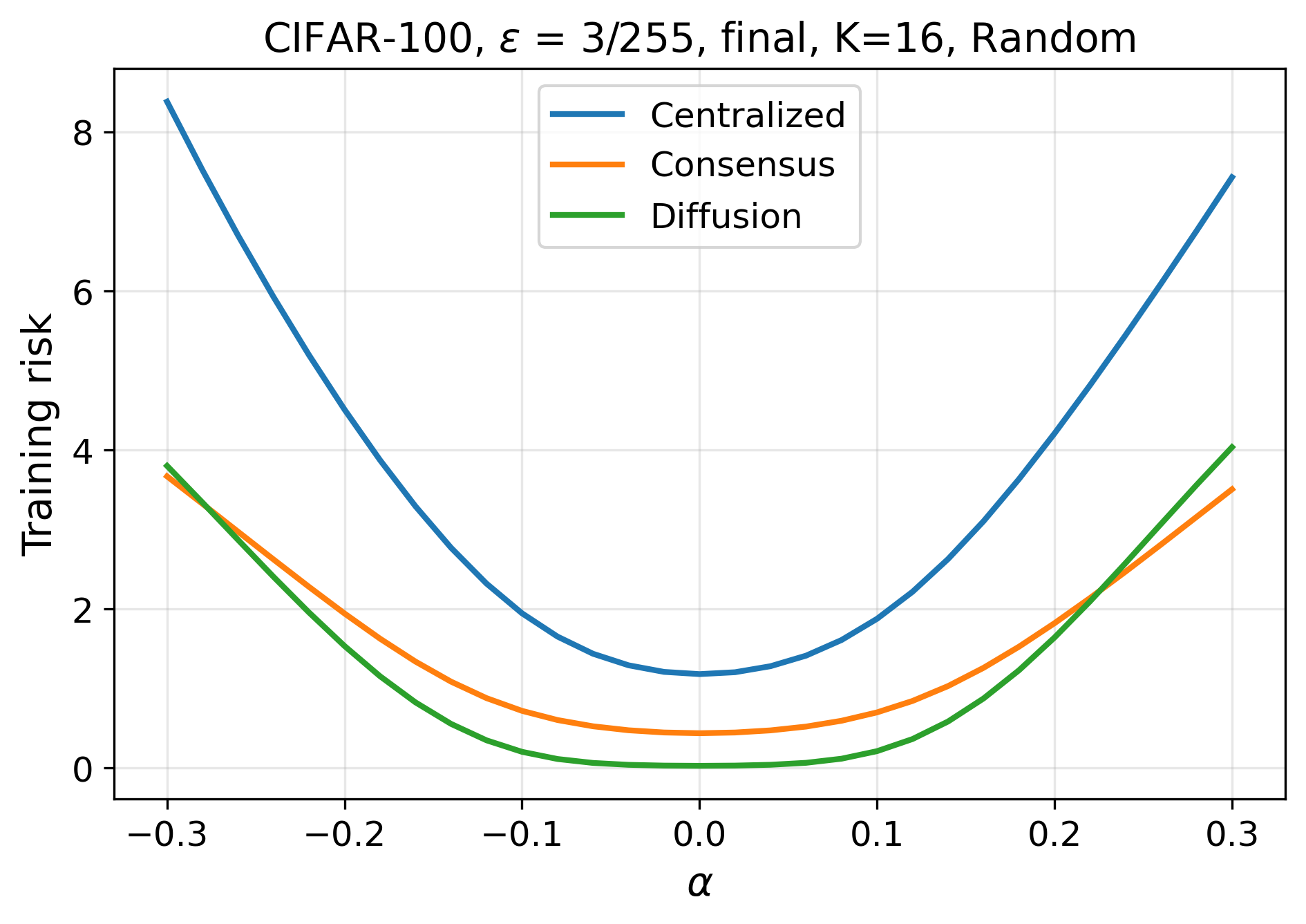}}
\subfigure[CIFAR-100: B=256]
{\includegraphics[width=.3\linewidth]
{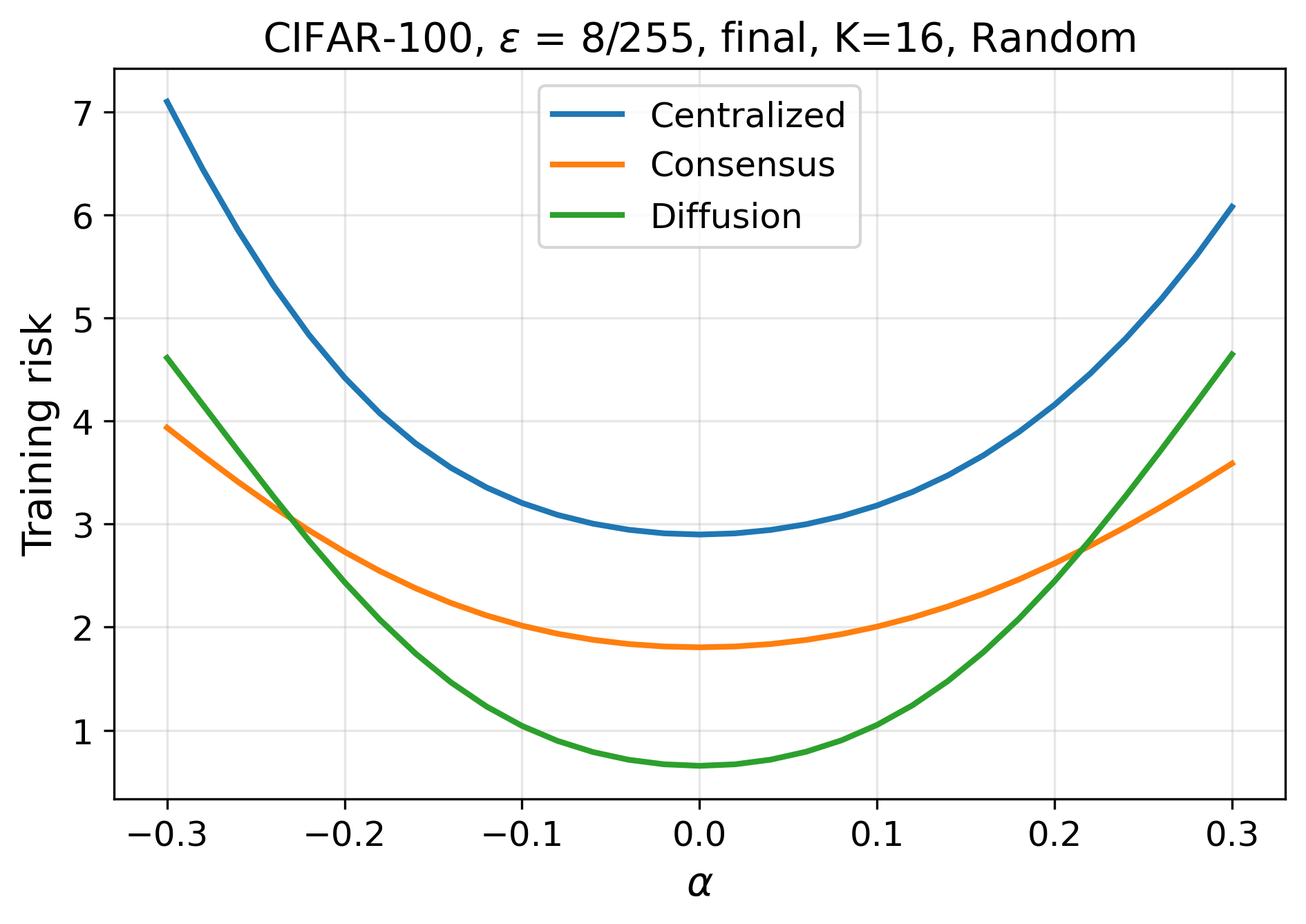}}
\caption{Flatness visualization of the \emph{final} models trained on the CIFAR-100 dataset. The observed phenomenon is consistent with Fig.~\ref{loss_visual_cifar10_sgd}.}
\label{loss_visual_cifar100_sgd}
\end{center}
\end{figure*}

\begin{figure*}[htbp]
\begin{center}
\subfigure[Tiny ImageNet-200: B=128]{\includegraphics[width=.3\linewidth]
{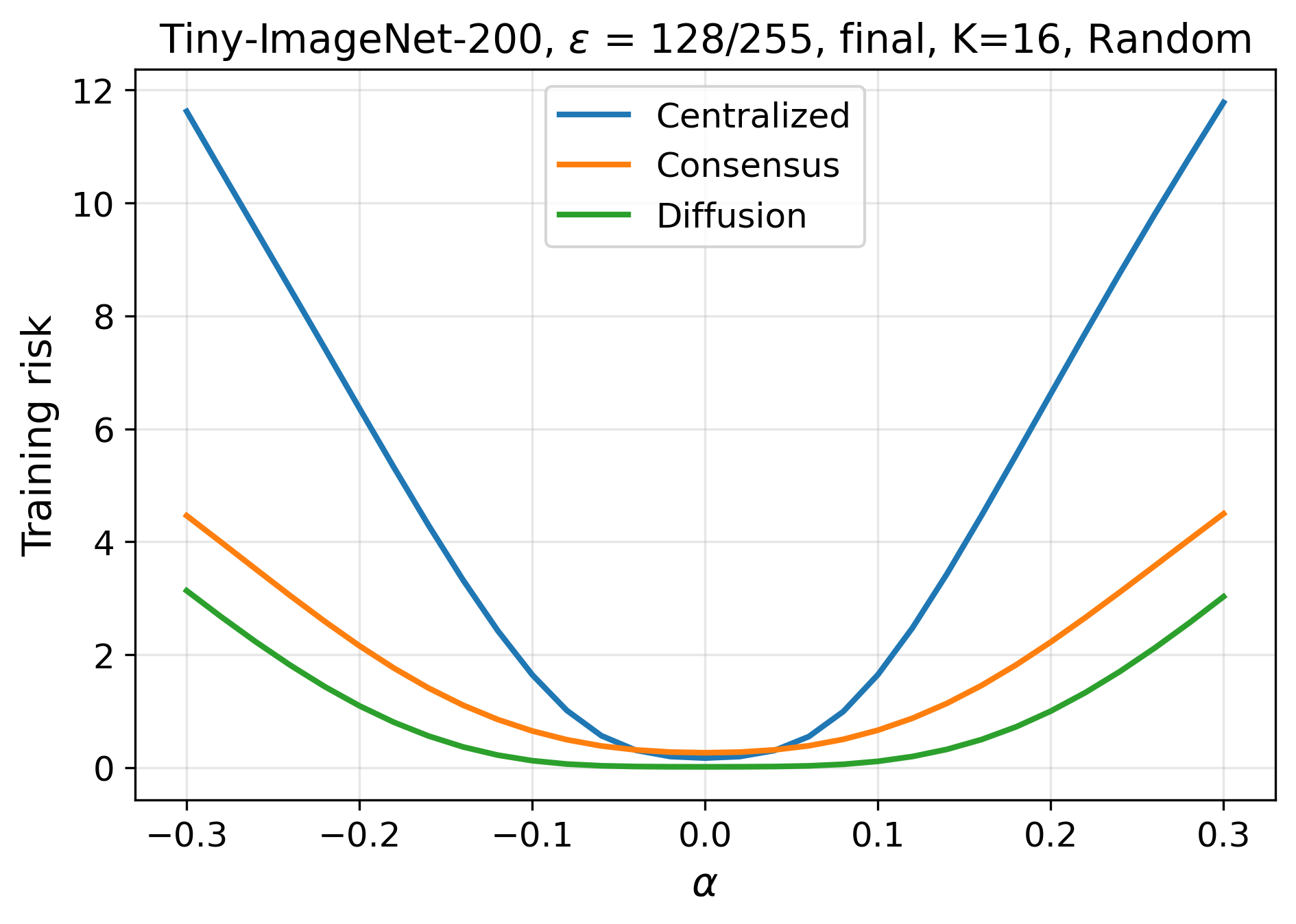}}
\subfigure[Tiny ImageNet-200: B=128]
{\includegraphics[width=.3\linewidth]
{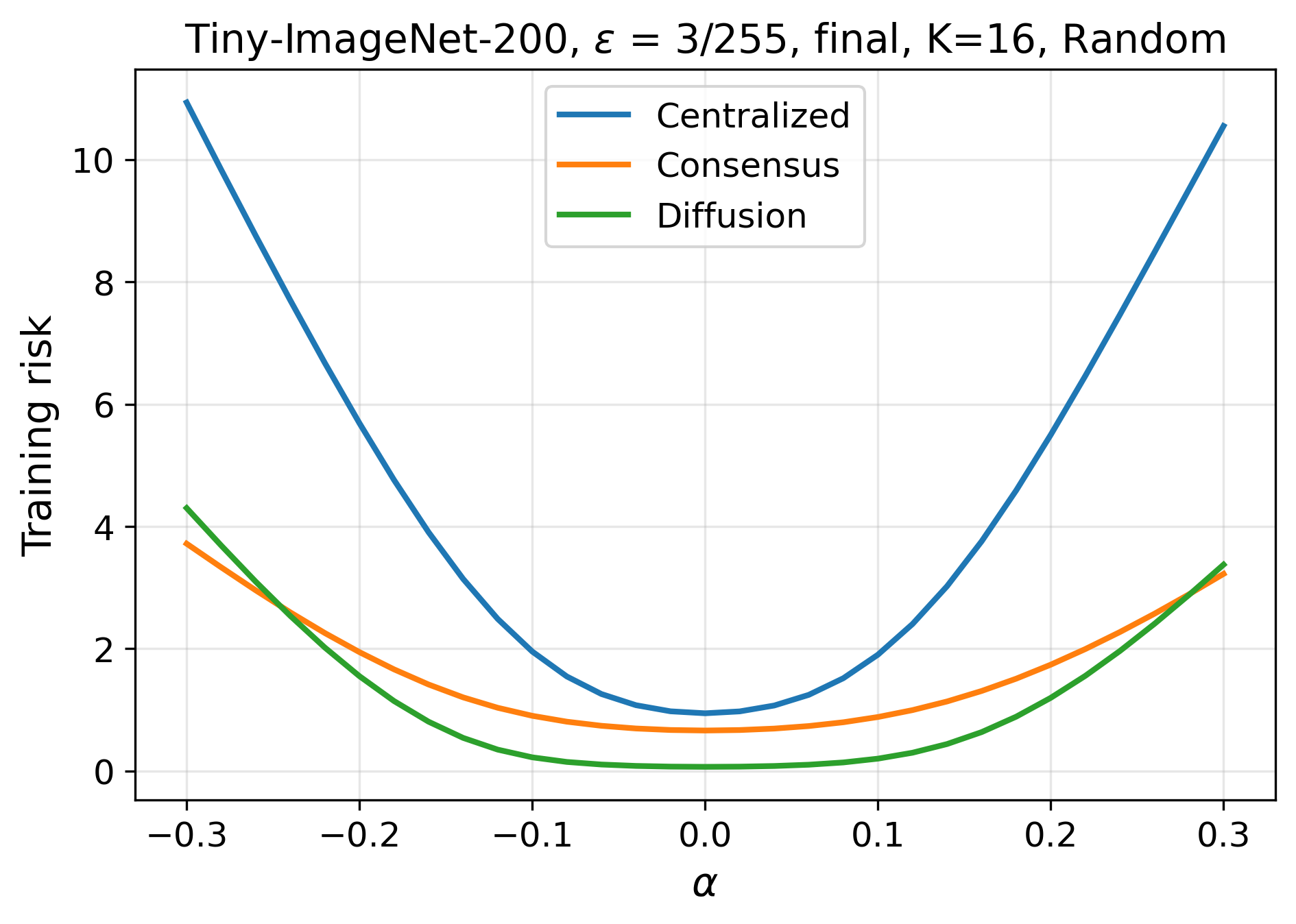}}
\caption{Flatness visualization of the models trained on the Tiny ImageNet-200 dataset. The observed phenomenon is consistent with Fig.~\ref{loss_visual_cifar10_sgd}.}
\label{loss_visual_imagenet_sgd}
\end{center}
\end{figure*}

\subsection{Clean and robust performance comparison of centralized and decentralized methods}
We then evaluate the performance of the trained models using AutoAttack \cite{Croce020a}, a standardized and widely adopted benchmark that combines several strong, parameter-free adversarial attacks. The results are reported in Tables~\ref{robustacc_sgd_cifar10}, \ref{robustacc_sgd_cifar100}, and \ref{robustacc_sgd_imagenet}, which present both the clean accuracy tested on unperturbed data and the robust accuracy measured on adversarially perturbed data. As shown in Tables~\ref{robustacc_sgd_cifar10}, \ref{robustacc_sgd_cifar100}, and \ref{robustacc_sgd_imagenet}, models obtained using decentralized training consistently outperform those produced by the centralized method on both clean and adversarial examples. Moreover, consensus generally achieves slightly higher robust accuracy than diffusion. It is worth noting that one may question the relatively lower robust accuracy reported in Table~\ref{robustacc_sgd_cifar10} and \ref{robustacc_sgd_cifar100}, as it appears suboptimal compared to values commonly reported in the literature~\cite{abs-2010-03593}. Specifically, without using additional data and under the $\ell_\infty$-bounded perturbations with $\epsilon = 8/255$, the robust accuracy of the \emph{best} on CIFAR-10  typically reaches around $51\%$, and on CIFAR-100 around $20\%$. This discrepancy primarily stems from the fact that our experimental setting differs from that we focus on \textit{large-batch} settings, whereas the literature primarily investigates small batch size. To be consistent with the literature, we additionally train neural networks using TRADES and the SGD momentum optimizer with large batch size. The corresponding results, presented in Appendix~\ref{appen_as}, show that the robust accuracy in large-batch settings indeed drops compared to the small-batch setting used in the literature~\cite{abs-2010-03593}. We focus on large batch sizes both to align with our theoretical findings and because they empirically meet the practical requirements of modern large-scale, multi-agent systems \cite{ZhangLZCCFMHH022,KeskarMNST17}.

By combining the results in Tables~\ref{robustacc_sgd_cifar10}--\ref{robustacc_sgd_imagenet} with the flatness measurements illustrated in Figures \ref{loss_visual_cifar10_sgd}--\ref{loss_visual_imagenet_sgd}, we observe the following: when the attack strength is relatively mild--specifically, with $\epsilon$ set to $128/255$ under the $\ell_2$ norm or $3/255$ under the $\ell_\infty$ norm, decentralized methods produce both flatter models and superior performance on clean and adversarial data compared to centralized training. However, when the attack strength increases to $8/255$ with the $\ell_\infty$ norm, decentralized training do not necessarily produce flatter models. Despite this, they still achieve better performance than the centralized approach. These observations suggest that, although flatness plays a significant role in enhancing adversarial robustness, it may not be the sole determining factor, particularly in the presence of strong adversarial attacks.

To gain further insights into the performance of the distributed adversarial training algorithms, and motivated by \cite{cao2024trade}, we examine their optimization behavior, which is another important factor influencing the performance of machine learning models. The corresponding results for the three datasets are shown in Figures \ref{optim_random_16_cifar10}--\ref{optim_random_16_imagenet}. In the large-batch setting with $\ell_\infty$ bounded by $\epsilon=8/255$, centralized adversarial training exhibits noticeably higher training error than the two decentralized algorithms. This observation helps explain why models trained with decentralized methods are more robust than those with the centralized approach even though decentralized models exhibiting sharper loss landscapes. Moreover, diffusion shows improved optimization performance compared to consensus; however, since consensus generally produces flatter models than diffusion, it can lead to higher robustness. Overall, these simulation results demonstrate the potential of decentralized adversarial training in improving model robustness within distributed learning frameworks.
 
\begin{table*}[!t]
\footnotesize
\centering
\begin{spacing}{1.6}
\caption{Clean and robust accuracy of the obtained models evaluated using AutoAttack on CIFAR-10}
\label{robustacc_sgd_cifar10}
\begin{tabular}{cccccccccccccccccc}
\toprule
\multirow{2}{*}{dataset} & \multirow{2}{*}{graph}& \multirow{2}{*}{K}&\multirow{2}{*}{norm} & \multirow{2}{*}{$\epsilon$} & \multirow{2}{*}{local batch}  & \multirow{2}{*}{method} &  \multicolumn{2}{c}{best} & \multicolumn{2}{c}{final} \\
\cline{8-9}\cline{10-11}
&&&&&&&clean($\%$) & AA($\%$) & clean($\%$) & AA($\%$) \\
\hline
\multirow{18}{*}{CIFAR-10}&\multirow{18}{*}{random} &\multirow{18}{*}{16}& \multirow{12}{*}{$\ell_\infty$} & \multirow{6}{*}{$8/255$} & \multirow{3}{*}{128} &  centralized &71.11&35.89&71.44&35.86  \\
    &&&&&&consensus & 80.95&\textbf{45.16}&82.74&\textbf{43.05}\\
    &&&&&&diffusion & \textbf{84.02}& 44.46& \textbf{84.38}& 38.71\\
\cline{6-11}
&&&&&\multirow{3}{*}{256} & centralized&65.95&33.93&66.35&34.06  \\
    &&&&&&consensus &81.18&\textbf{43.15}&82.05&\textbf{42.75}\\
    &&&&&&diffusion &\textbf{83.89}&40.68&\textbf{83.83}&37.37\\
\cline{5-11}
&&&& \multirow{6}{*}{$3/255$} &\multirow{3}{*}{128} & centralized&84.91&60.97&84.60&59.55  \\
    &&&&&&consensus&90.66&\textbf{71.90}&90.73&\textbf{69.18} \\
    &&&&&&diffusion&\textbf{90.95}&70.70&\textbf{90.89}&68.52 \\
\cline{6-11}
&&&&&\multirow{3}{*}{256} & centralized  &83.20&59.70&83.29&59.27\\
    &&&&&&consensus &89.88&\textbf{70.55}&90.12&\textbf{67.95}\\
    &&&&&&diffusion &\textbf{90.61}&67.00&\textbf{90.99}&66.15\\
\cline{4-11}
&&& \multirow{6}{*}{$\ell_2$} & \multirow{6}{*}{$128/255$} & \multirow{3}{*}{128} & centralized &83.96&54.03&83.60&52.26\\
    &&&&&&consensus &89.55&\textbf{64.27}&89.79&\textbf{60.76}\\
    &&&&&&diffusion &\textbf{90.07}&63.27&\textbf{90.03}&60.02\\
\cline{6-11}
&&&&&\multirow{3}{*}{256} & centralized &82.22&54.19&82.51&53.44 \\
    &&&&&&consensus & 88.71&\textbf{63.03}&88.81&\textbf{59.50} \\
    &&&&&&diffusion &\textbf{89.32}&60.99&\textbf{89.32}&58.32 \\
\bottomrule      
\end{tabular}
\end{spacing}
\end{table*}

\begin{table*}[!t]
\footnotesize
\centering
\begin{spacing}{1.6}
\caption{Clean and robust accuracy of the obtained models evaluated using AutoAttack on CIFAR-100}
\label{robustacc_sgd_cifar100}
\begin{tabular}{cccccccccccccccccc}
\toprule
\multirow{2}{*}{dataset} & \multirow{2}{*}{graph}& \multirow{2}{*}{K}&\multirow{2}{*}{norm} & \multirow{2}{*}{$\epsilon$} & \multirow{2}{*}{local batch}  & \multirow{2}{*}{method} &  \multicolumn{2}{c}{best} & \multicolumn{2}{c}{final} \\
\cline{8-9}\cline{10-11}
&&&&&&&clean($\%$) & AA($\%$) & clean($\%$) & AA($\%$) \\
\hline
\multirow{18}{*}{CIFAR-100}&\multirow{18}{*}{random} &\multirow{18}{*}{16}& \multirow{12}{*}{$\ell_\infty$} & \multirow{6}{*}{$8/255$} & \multirow{3}{*}{128} &  centralized & 46.95&16.28&47.09&15.47 \\
    &&&&&&consensus & 58.24&\textbf{23.70}&\textbf{58.85}&\textbf{21.52} \\
    &&&&&&diffusion &\textbf{58.92}&20.92&57.67&18.86\\
\cline{6-11}
&&&&&\multirow{3}{*}{256} & centralized &39.96&15.32&40.06&15.40 \\
    &&&&&&consensus & 55.33&\textbf{22.49}&\textbf{57.09}&\textbf{21.07}\\
    &&&&&&diffusion & \textbf{57.40} & 20.13& 55.90& 17.60\\
\cline{5-11}
&&&& \multirow{6}{*}{$3/255$} &\multirow{3}{*}{128} & centralized & 55.13&29.07&54.57&28.27  \\
    &&&&&&consensus& \textbf{68.90}&\textbf{42.72}&\textbf{68.08}&39.01 \\
    &&&&&&diffusion & 67.15&40.33&66.95&\textbf{39.39}\\
\cline{6-11}
&&&&&\multirow{3}{*}{256} & centralized&53.43&30.47&53.11&28.72  \\
    &&&&&&consensus &\textbf{67.47}&\textbf{41.10}&\textbf{66.53}&\textbf{38.07}\\
    &&&&&&diffusion &67.17&38.71&66.33&37.55\\
\cline{4-11}
&&& \multirow{6}{*}{$\ell_2$} & \multirow{6}{*}{$128/255$} & \multirow{3}{*}{128} & centralized & 54.44&25.49&54.07&25.20\\
    &&&&&&consensus & \textbf{67.08}&\textbf{38.09}&\textbf{66.30}&\textbf{34.33}\\
    &&&&&&diffusion & 66.32&34.60&65.77&33.50\\
\cline{6-11}
&&&&&\multirow{3}{*}{256} & centralized  & 53.07&26.76&53.03&25.76\\
    &&&&&&consensus &\textbf{66.15}&\textbf{36.34}&\textbf{65.70}&\textbf{32.93}\\
    &&&&&&diffusion &65.11&33.24&65.08&32.13\\
\bottomrule      
\end{tabular}
\end{spacing}
\end{table*}

\begin{table*}[!t]
\footnotesize
\centering
\begin{spacing}{1.6}
\caption{Clean and robust accuracy of the obtained models evaluated using AutoAttack on Tiny ImageNet-200}
\label{robustacc_sgd_imagenet}
\begin{tabular}{cccccccccccccccccc}
\toprule
\multirow{2}{*}{dataset} & \multirow{2}{*}{graph}& \multirow{2}{*}{K}&\multirow{2}{*}{norm} & \multirow{2}{*}{$\epsilon$} & \multirow{2}{*}{local batch}  & \multirow{2}{*}{method} &  \multicolumn{2}{c}{best} & \multicolumn{2}{c}{final} \\
\cline{8-9}\cline{10-11}
&&&&&&&clean($\%$) & AA($\%$) & clean($\%$) & AA($\%$) \\
\hline
\multirow{6}{*}{Tiny ImageNet-200}&\multirow{6}{*}{random} &\multirow{6}{*}{16}& \multirow{3}{*}{$\ell_\infty$} & \multirow{3}{*}{$3/255$} &\multirow{3}{*}{128} & centralized & 47.11 & 22.22& 46.11& 19.67\\
    &&&&&&consensus & \textbf{57.34} & \textbf{32.54}& \textbf{56.19}& \textbf{27.75}\\
    &&&&&&diffusion &54.92& 27.11& 54.03& 25.50\\
\cline{4-11}
&&& \multirow{3}{*}{$\ell_2$} & \multirow{3}{*}{$128/255$} & \multirow{3}{*}{128} & centralized & 47.13& 24.41&46.32& 22.44\\
    &&&&&&consensus & \textbf{60.01}& \textbf{37.60}&\textbf{ 58.27}&\textbf{31.67}\\
    &&&&&&diffusion &57.65&31.73&57.15&30.87\\
\bottomrule      
\end{tabular}
\end{spacing}
\end{table*}

\begin{figure*}[htbp]
\begin{center}
\subfigure[CIFA-10: B=128]{\includegraphics[width=.3\linewidth]{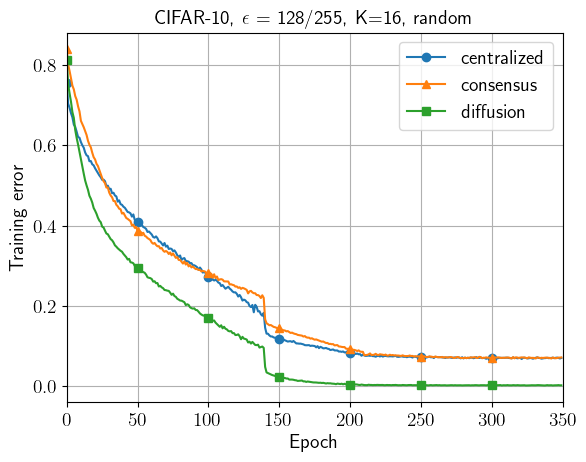}}
\subfigure[CIFAR-10: B=128]{\includegraphics[width=.3\linewidth]{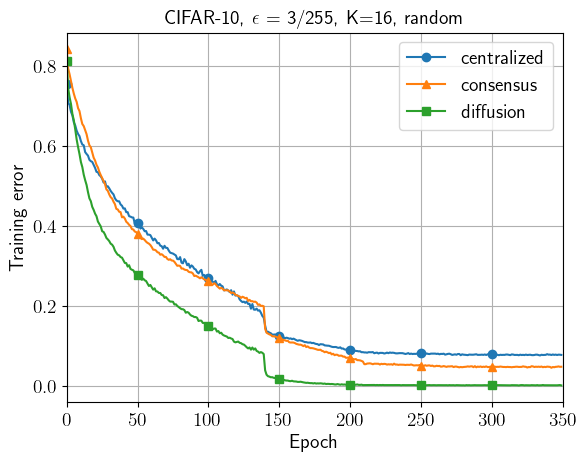}}
\subfigure[CIFAR-10: B=128]{\includegraphics[width=.3\linewidth]{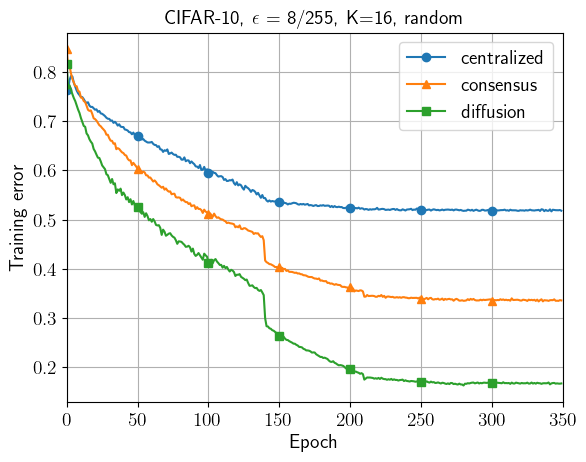}}
\subfigure[CIFAR-10: B=256]{\includegraphics[width=.3\linewidth]{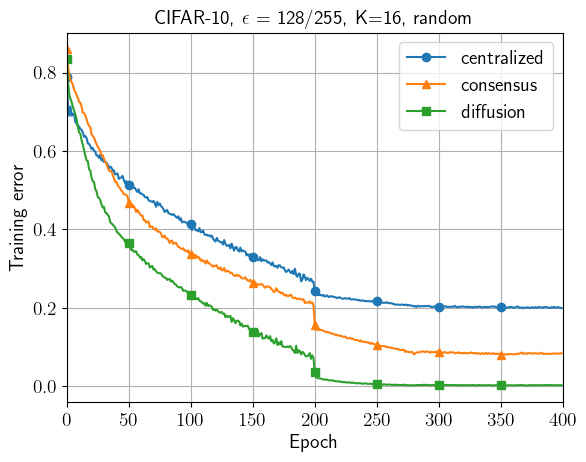}}
\subfigure[CIFAR-10: B=256]{\includegraphics[width=.3\linewidth]{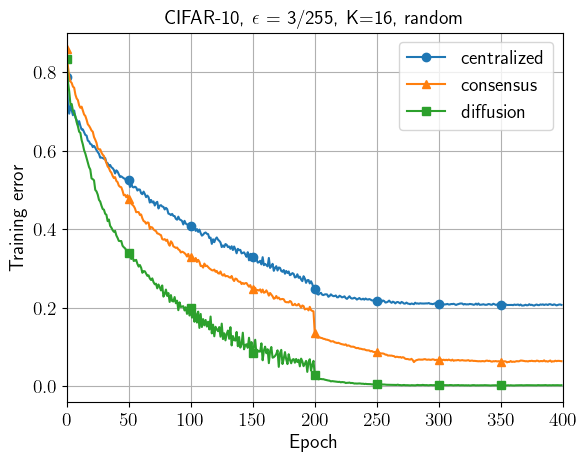}}
\subfigure[CIFAR-10: B=256]{\includegraphics[width=.3\linewidth]{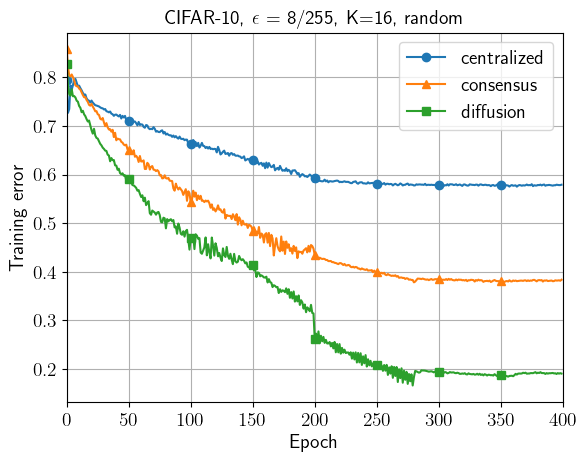}}
\caption{The evolution of the training error on CIFAR-10 using a \emph{random} graph topology with $K=16$ nodes. We observe that under strong adversarial attacks, specifically when the perturbation bound is $\epsilon = 8/255$, the optimization performance of centralized models in the large-batch setting can be bad.}
\label{optim_random_16_cifar10}
\end{center}
\end{figure*}

\begin{figure*}[htbp]
\begin{center}
\subfigure[CIFAR-100: B=128]{\includegraphics[width=.3\linewidth]{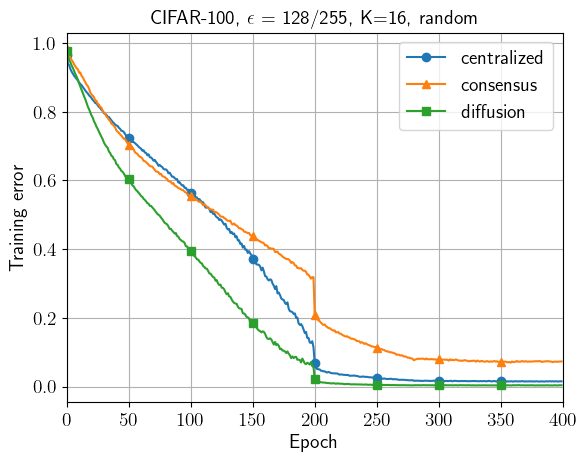}}
\subfigure[CIFAR-100: B=128]{\includegraphics[width=.3\linewidth]{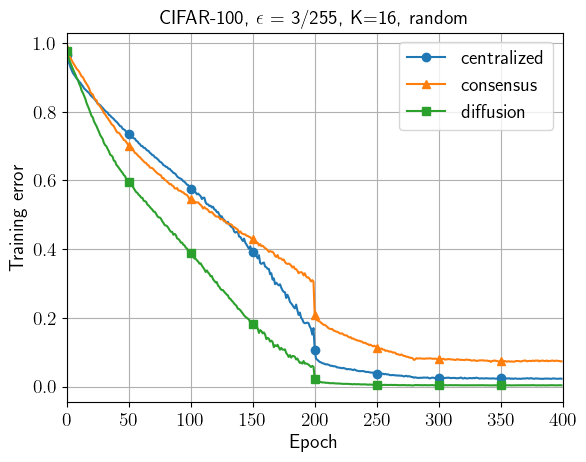}}
\subfigure[CIFAR-100: B=128]{\includegraphics[width=.3\linewidth]{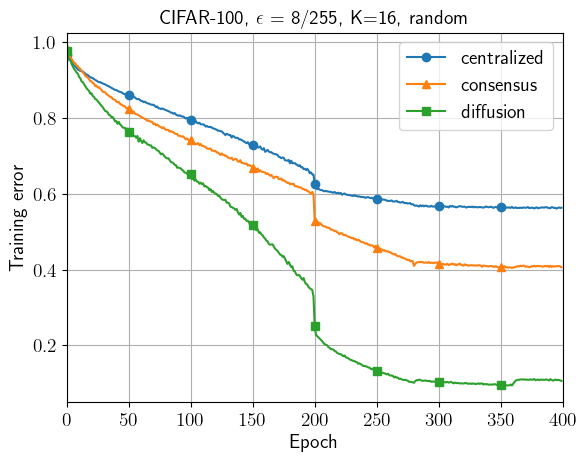}}
\subfigure[CIFAR-100: B=256]{\includegraphics[width=.3\linewidth]{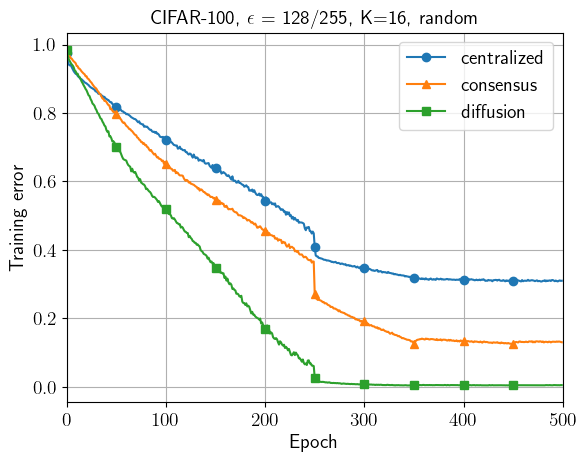}}
\subfigure[CIFAR-100: B=256]{\includegraphics[width=.3\linewidth]{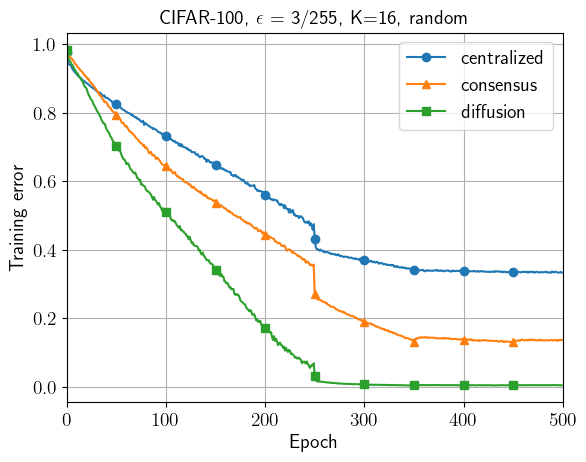}}
\subfigure[CIFAR-100: B=256]{\includegraphics[width=.3\linewidth]{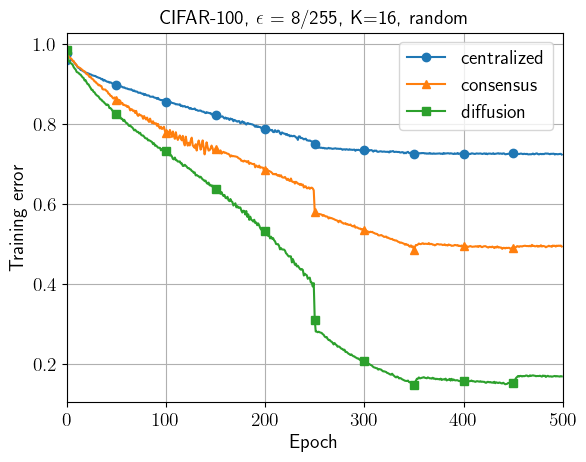}}
\caption{The evolution of the training error on CIFAR-100 using the \emph{random} graph structure with $K=16$ nodes. Similar to Figure \ref{optim_random_16_cifar10}, when $\epsilon = 8/255$, the optimization performance of centralized models are bad.}
\label{optim_random_16_cifar100}
\end{center}
\end{figure*}

\begin{figure*}[htbp]
\begin{center}
\subfigure[Tiny ImageNet-200: B=128]{\includegraphics[width=.3\linewidth]{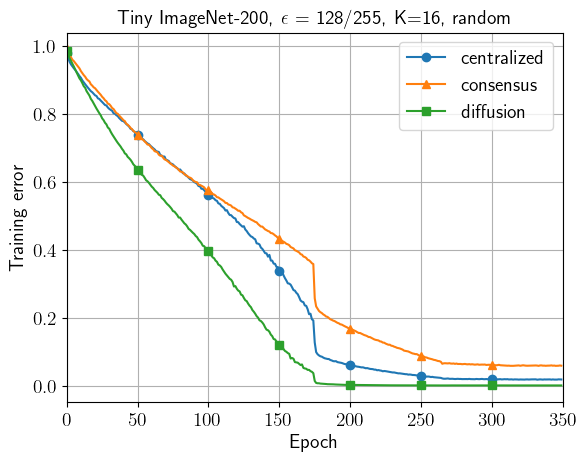}}
\subfigure[Tiny ImageNet-200: B=128]
{\includegraphics[width=.3\linewidth]{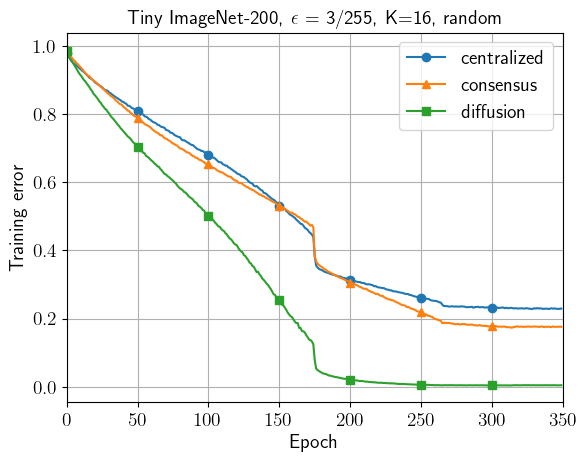}}
\caption{The evolution of the training error on Tiny ImageNet-200 using the \emph{random} graph structure with $K=16$ nodes.}
\label{optim_random_16_imagenet}
\end{center}
\end{figure*}

\section{Conclusion}
We investigated the learning behavior of stochastic gradient descent-based distributed adversarial training algorithms near local minima of non-convex robust risk functions. Theoretical results indicate that, in the large-batch setting, when the perturbation bound is sufficiently small, i.e., the attack strength is relatively mild, decentralized adversarial training algorithms escape local minima more efficiently and tend to favor flatter solutions compared to the centralized baseline. However, as the attack strength increases, this behavior may no longer hold. Our simulation results further demonstrate that decentralized adversarial training algorithms can achieve higher accuracy than the centralized baseline.

In this paper, our objective is to understand the implicit bias of distributed adversarial training algorithms. Several directions can be considered in future work. First, due to limited computational resources, our simulations are conducted on relatively small-scale datasets without the use of additional training data. While state-of-the-art studies in the single-agent setting often use extra data \cite{wang23ad,bartoldson2024adversarial},  extending such experiments to the multi-agent setting would require substantially greater computing resources and is beyond the scope of this work. We encourage future studies, particularly those with access to large-scale computing resources, to evaluate decentralized adversarial training algorithms in such large-scale scenarios. Second, our experiments are implemented on a single GPU to focus on the fundamental learning behavior of centralized and decentralized adversarial training algorithms. Future work can deploy these methods on real distributed platforms and compare their system-level performance, such as communication cost, computational efficiency, and runtime. Finally, this work provides a starting point for comparing centralized and decentralized adversarial training algorithms. We mainly focus on synchronous SGD. Future work can extend the analysis to other optimizers and asynchronous communication settings.

\section*{Acknowledgments}
We acknowledge the assistance of ChatGPT in improving the English presentation in this paper.

\bibliography{references}

\begin{thebibliography}{10}
\providecommand{\url}[1]{#1}
\csname url@samestyle\endcsname
\providecommand{\newblock}{\relax}
\providecommand{\bibinfo}[2]{#2}
\providecommand{\BIBentrySTDinterwordspacing}{\spaceskip=0pt\relax}
\providecommand{\BIBentryALTinterwordstretchfactor}{4}
\providecommand{\BIBentryALTinterwordspacing}{\spaceskip=\fontdimen2\font plus
\BIBentryALTinterwordstretchfactor\fontdimen3\font minus \fontdimen4\font\relax}
\providecommand{\BIBforeignlanguage}[2]{{%
\expandafter\ifx\csname l@#1\endcsname\relax
\typeout{** WARNING: IEEEtran.bst: No hyphenation pattern has been}%
\typeout{** loaded for the language `#1'. Using the pattern for}%
\typeout{** the default language instead.}%
\else
\language=\csname l@#1\endcsname
\fi
#2}}
\providecommand{\BIBdecl}{\relax}
\BIBdecl

\bibitem{szegedy2013intriguing}
C.~Szegedy, W.~Zaremba, I.~Sutskever, J.~Bruna, D.~Erhan, I.~Goodfellow, and R.~Fergus, ``Intriguing properties of neural networks,'' in \emph{Proc. ICLR}, Banff, 2014, pp. 1--10.

\bibitem{goodfellow2014explaining}
I.~Goodfellow, J.~Shlens, and C.~Szegedy, ``Explaining and harnessing adversarial examples,'' in \emph{Proc. ICLR}, San Diego, 2015, pp. 1--11.

\bibitem{madry2017towards}
A.~Madry, A.~Makelov, L.~Schmidt, D.~Tsipras, and A.~Vladu, ``Towards deep learning models resistant to adversarial attacks,'' in \emph{Proc. ICLR}, Vancouver, 2018, pp. 1--23.

\bibitem{wang23ad}
Z.~Wang, T.~Pang, C.~Du, M.~Lin, W.~Liu, and S.~Yan, ``Better diffusion models further improve adversarial training,'' in \emph{Proc. {ICML}}, Hawaii, 2023, pp. 36\,246--36\,263.

\bibitem{bartoldson2024adversarial}
B.~R. Bartoldson, J.~Diffenderfer, K.~Parasyris, and B.~Kailkhura, ``Adversarial robustness limits via scaling-law and human-alignment studies,'' in \emph{Proc. {ICML}}, Vienna, 2024, pp. 3046--3072.

\bibitem{vlaski2023networked}
S.~Vlaski, S.~Kar, A.~H. Sayed, and J.~M. Moura, ``Networked signal and information processing: Learning by multiagent systems,'' \emph{IEEE Signal Processing Magazine}, vol.~40, no.~5, pp. 92--105, 2023.

\bibitem{YuanWSYB24}
L.~Yuan, Z.~Wang, L.~Sun, P.~S. Yu, and C.~G. Brinton, ``Decentralized federated learning: {A} survey and perspective,'' \emph{{IEEE} Internet Things J.}, vol.~11, no.~21, pp. 34\,617--34\,638, 2024.

\bibitem{cao2025decentralized}
Y.~Cao, E.~Rizk, S.~Vlaski, and A.~H. Sayed, ``Decentralized adversarial training over graphs,'' \emph{{IEEE} Trans. Inf. Theory}, vol.~71, no.~7, pp. 5570--5600, 2025.

\bibitem{sayed2014adaptation}
A.~H. Sayed, ``Adaptation, learning, and optimization over networks,'' \emph{Foundations and Trends in Machine Learning}, vol.~7, no. 4-5, pp. 311--801, 2014.

\bibitem{sayed_2023}
------, \emph{Inference and Learning from Data}.\hskip 1em plus 0.5em minus 0.4em\relax Cambridge University Press, 2022.

\bibitem{nedic2009distributed}
A.~Nedic and A.~E. Ozdaglar, ``Distributed subgradient methods for multi-agent optimization,'' \emph{IEEE Transactions on Automatic Control}, vol.~54, no.~1, pp. 48--61, 2009.

\bibitem{TuS12}
S.~Tu and A.~H. Sayed, ``Diffusion strategies outperform consensus strategies for distributed estimation over adaptive networks,'' \emph{{IEEE} Trans. Signal Process.}, vol.~60, no.~12, pp. 6217--6234, 2012.

\bibitem{abs-2111-04287}
B.~Ying, K.~Yuan, H.~Hu, Y.~Chen, and W.~Yin, ``Bluefog: Make decentralized algorithms practical for optimization and deep learning,'' \emph{arXiv preprint arXiv:2111.04287}, 2021.

\bibitem{ZhuLWH25}
T.~Zhu, W.~Li, C.~Wang, and F.~He, ``{DICE:} data influence cascade in decentralized learning,'' in \emph{Proc. {ICLR}}, Singapore, 2025, pp. 1--36.

\bibitem{RizkVS23}
E.~Rizk, S.~Vlaski, and A.~H. Sayed, ``Enforcing privacy in distributed learning with performance guarantees,'' \emph{{IEEE} Trans. Signal Process.}, vol.~71, pp. 3385--3398, 2023.

\bibitem{AlghunaimY22}
S.~A. Alghunaim and K.~Yuan, ``A unified and refined convergence analysis for non-convex decentralized learning,'' \emph{{IEEE} Trans. Signal Process.}, vol.~70, pp. 3264--3279, 2022.

\bibitem{YeLZZSJ24}
Z.~Ye, W.~Luo, Q.~Zhou, Z.~Zhu, Y.~Shi, and Y.~Jia, ``Gradient inversion attacks: Impact factors analyses and privacy enhancement,'' \emph{{IEEE} Trans. Pattern Anal. Mach. Intell.}, vol.~46, no.~12, pp. 9834--9850, 2024.

\bibitem{SunLW23}
T.~Sun, D.~Li, and B.~Wang, ``Decentralized federated averaging,'' \emph{{IEEE} Trans. Pattern Anal. Mach. Intell.}, vol.~45, no.~4, pp. 4289--4301, 2023.

\bibitem{ChenS15b}
J.~Chen and A.~H. Sayed, ``On the learning behavior of adaptive networks - part {II:} performance analysis,'' \emph{{IEEE} Trans. Inf. Theory}, vol.~61, no.~6, pp. 3518--3548, 2015.

\bibitem{YuanAYS20}
K.~Yuan, S.~A. Alghunaim, B.~Ying, and A.~H. Sayed, ``On the influence of bias-correction on distributed stochastic optimization,'' \emph{{IEEE} Trans. Signal Process.}, vol.~68, pp. 4352--4367, 2020.

\bibitem{vlaski2021second}
S.~Vlaski and A.~H. Sayed, ``Second-order guarantees of stochastic gradient descent in nonconvex optimization,'' \emph{{IEEE} Transactions on Automatic Control}, vol.~67, no.~12, pp. 6489--6504, 2021.

\bibitem{VlaskiS21a}
------, ``Distributed learning in non-convex environments - part {II:} polynomial escape from saddle-points,'' \emph{{IEEE} Trans. Signal Process.}, vol.~69, pp. 1257--1270, 2021.

\bibitem{VogelsHJ23}
T.~Vogels, H.~Hendrikx, and M.~Jaggi, ``Beyond spectral gap: The role of the topology in decentralized learning,'' \emph{J. Mach. Learn. Res.}, vol.~24, pp. 355:1--355:31, 2023.

\bibitem{DengSL023}
X.~Deng, T.~Sun, S.~Li, and D.~Li, ``Stability-based generalization analysis of the asynchronous decentralized {SGD},'' in \emph{Proc. AAAI}, {Washington}, 2023, pp. 7340--7348.

\bibitem{ZhuHZNST22}
T.~Zhu, F.~He, L.~Zhang, Z.~Niu, M.~Song, and D.~Tao, ``Topology-aware generalization of decentralized {SGD},'' in \emph{Proc. ICML}, {Baltimore}, 2022, pp. 27\,479--27\,503.

\bibitem{00010KJS21}
L.~Kong, T.~Lin, A.~Koloskova, M.~Jaggi, and S.~U. Stich, ``Consensus control for decentralized deep learning,'' in \emph{Proc. ICML}, 2021, pp. 5686--5696.

\bibitem{BarsBTSN24}
B.~L. Bars, A.~Bellet, M.~Tommasi, K.~Scaman, and G.~Neglia, ``Improved stability and generalization guarantees of the decentralized {SGD} algorithm,'' in \emph{Proc. {ICML}}, Vienna, 2024.

\bibitem{ZhuH0ST23}
T.~Zhu, F.~He, K.~Chen, M.~Song, and D.~Tao, ``Decentralized {SGD} and average-direction {SAM} are asymptotically equivalent,'' in \emph{Proc. ICML}, {Honolulu}, 2023, pp. 43\,005--43\,036.

\bibitem{cao2024trade}
Y.~Cao, Z.~Wu, K.~Yuan, and A.~H. Sayed, ``On the trade-off between flatness and optimization in distributed learning,'' \emph{{IEEE} Trans. Pattern Anal. Mach. Intell.}, vol.~47, no.~10, pp. 8873--8888, 2025.

\bibitem{ForetKMN21}
P.~Foret, A.~Kleiner, H.~Mobahi, and B.~Neyshabur, ``Sharpness-aware minimization for efficiently improving generalization,'' in \emph{Proc. ICLR}, 2021.

\bibitem{ZhangHZCWW24}
Y.~Zhang, H.~He, J.~Zhu, H.~Chen, Y.~Wang, and Z.~Wei, ``On the duality between sharpness-aware minimization and adversarial training,'' in \emph{Proc. {ICML}}, Vienna, 2024.

\bibitem{WuX020}
D.~Wu, S.~Xia, and Y.~Wang, ``Adversarial weight perturbation helps robust generalization,'' in \emph{Proc. NeurIPS}, 2020, pp. 1--12.

\bibitem{Stutz0S21}
D.~Stutz, M.~Hein, and B.~Schiele, ``Relating adversarially robust generalization to flat minima,'' in \emph{Proc. {ICCV}}, Montreal,2021, pp. 7787--7797.

\bibitem{abs-2010-03593}
\BIBentryALTinterwordspacing
S.~Gowal, C.~Qin, J.~Uesato, T.~Mann, and P.~Kohli, ``Uncovering the limits of adversarial training against norm-bounded adversarial examples,'' 2021. [Online]. Available: \url{https://arxiv.org/abs/2010.03593}
\BIBentrySTDinterwordspacing

\bibitem{Croce020a}
F.~Croce and M.~Hein, ``Reliable evaluation of adversarial robustness with an ensemble of diverse parameter-free attacks,'' in \emph{Proc. ICML}, 2020, pp. 2206--2216.

\bibitem{bertsekas2009convex}
D.~Bertsekas, \emph{Convex {O}ptimization {T}heory}.\hskip 1em plus 0.5em minus 0.4em\relax Athena Scientific, 2009.

\bibitem{Thekumparampil019}
K.~K. Thekumparampil, P.~Jain, P.~Netrapalli, and S.~Oh, ``Efficient algorithms for smooth minimax optimization,'' in \emph{Proc. NeurIPS}, Vancouver, 2019, pp. 12\,659--12\,670.

\bibitem{sinha2017certifying}
A.~Sinha, H.~Namkoong, R.~Volpi, and J.~Duchi, ``Certifying some distributional robustness with principled adversarial training,'' in \emph{Proc. ICLR}, Vancouver, 2018, pp. 1--34.

\bibitem{WangM0YZG19}
Y.~Wang, X.~Ma, J.~Bailey, J.~Yi, B.~Zhou, and Q.~Gu, ``On the convergence and robustness of adversarial training,'' in \emph{Proc. {ICML}}, Long Beach, 2019, pp. 6586--6595.

\bibitem{Li2020PyTorchDistributed}
S.~Li, Y.~Zhao, R.~Varma, O.~Salpekar, P.~Noordhuis, T.~Li, A.~Paszke, J.~Smith, B.~Vaughan, P.~Damania, and S.~Chintala, ``Pytorch distributed: Experiences on accelerating data parallel training,'' \emph{Proc. {VLDB} Endow.}, vol.~13, no.~12, pp. 3005--3018, 2020.

\bibitem{MoriLLU22}
T.~Mori, Z.~Liu, K.~Liu, and M.~Ueda, ``Power-law escape rate of {SGD},'' in \emph{Proc. ICML}, {Baltimore}, 2022, pp. 15\,959--15\,975.

\bibitem{ZhuWYWM19}
Z.~Zhu, J.~Wu, B.~Yu, L.~Wu, and J.~Ma, ``The anisotropic noise in stochastic gradient descent: Its behavior of escaping from sharp minima and regularization effects,'' in \emph{Proc. ICML}, Long Beach, 2019, pp. 7654--7663.

\bibitem{LiuSLTS20}
C.~Liu, M.~Salzmann, T.~Lin, R.~Tomioka, and S.~S{\"{u}}sstrunk, ``On the loss landscape of adversarial training: Identifying challenges and how to overcome them,'' in \emph{Proc. NeurIPS}, 2020.

\bibitem{VlaskiS21}
S.~Vlaski and A.~H. Sayed, ``Distributed learning in non-convex environments - part {I:} agreement at a linear rate,'' \emph{{IEEE} Trans. Signal Process.}, vol.~69, pp. 1242--1256, 2021.

\bibitem{XieSS21}
Z.~Xie, I.~Sato, and M.~Sugiyama, ``A diffusion theory for deep learning dynamics: Stochastic gradient descent exponentially favors flat minima,'' in \emph{Proc. ICLR}, 2021.

\bibitem{NguyenSGR19}
T.~H. Nguyen, U.~Simsekli, M.~G{\"{u}}rb{\"{u}}zbalaban, and G.~Richard, ``First exit time analysis of stochastic gradient descent under heavy-tailed gradient noise,'' in \emph{Proc. NeurIPS}, {Vancouver}, 2019, pp. 273--283.

\bibitem{krizhevsky2009learning}
A.~Krizhevsky, ``Learning multiple layers of features from tiny images,'' University of Toronto, Tech. Rep., 2009.

\bibitem{5206848}
J.~Deng, W.~Dong, R.~Socher, L.-J. Li, K.~Li, and L.~Fei-Fei, ``Imagenet: A large-scale hierarchical image database,'' in \emph{Proc. IEEE CVPR}, 2009, pp. 248--255.

\bibitem{RiceWK20}
L.~Rice, E.~Wong, and J.~Z. Kolter, ``Overfitting in adversarially robust deep learning,'' in \emph{Proc. {ICML}}, 2020, pp. 8093--8104.

\bibitem{ZhangYJXGJ19}
H.~Zhang, Y.~Yu, J.~Jiao, E.~P. Xing, L.~E. Ghaoui, and M.~I. Jordan, ``Theoretically principled trade-off between robustness and accuracy,'' in \emph{Proc. ICML,}, Long Beach, 2019, pp. 7472--7482.

\bibitem{Li0TSG18}
H.~Li, Z.~Xu, G.~Taylor, C.~Studer, and T.~Goldstein, ``Visualizing the loss landscape of neural nets,'' in \emph{Proc. NeurIPS}, {Montr{\'{e}}al}, 2018, pp. 6391--6401.

\bibitem{sitawarin2021sat}
C.~Sitawarin, S.~Chakraborty, and D.~Wagner, ``Sat: Improving adversarial training via curriculum-based loss smoothing,'' in \emph{Proc. the 14th ACM Workshop on Artificial Intelligence and Security}, 2021, pp. 25--36.

\bibitem{ZhangLZCCFMHH022}
G.~Zhang, S.~Lu, Y.~Zhang, X.~Chen, P.~Chen, Q.~Fan, L.~Martie, L.~Horesh, M.~Hong, and S.~Liu, ``Distributed adversarial training to robustify deep neural networks at scale,'' in \emph{Proc. {UAI}}, vol. 180, Eindhoven, 2022, pp. 2353--2363.

\bibitem{KeskarMNST17}
N.~S. Keskar, D.~Mudigere, J.~Nocedal, M.~Smelyanskiy, and P.~T.~P. Tang, ``On large-batch training for deep learning: Generalization gap and sharp minima,'' in \emph{Proc. ICLR}, {Toulon}, 2017.

\end{thebibliography}
\bibliographystyle{IEEEtran}
\begin{IEEEbiography}[{\includegraphics[width=1in,height=1.25in,clip,keepaspectratio]{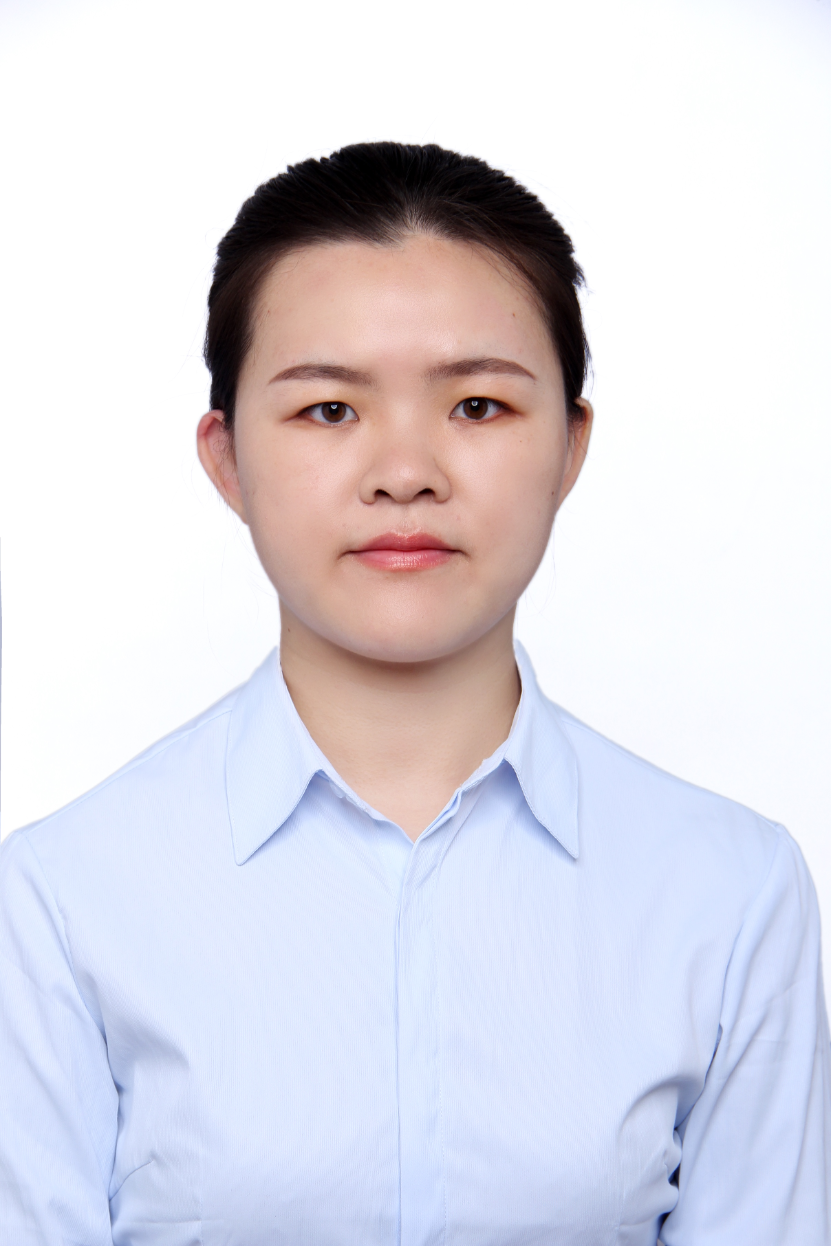}}]{Ying Cao}
is a postdoctoral researcher at EPFL. She received her Ph.D. degree from EPFL in 2026 and her bachelor's and master’s degree in Electronic and Information Engineering from Northwestern Polytechnical University, Xi’an, China, in 2017 and 2020, respectively. She received the China National Scholarship in 2018 and 2019. Her research interests include decentralized learning and adversarial robustness. 
\end{IEEEbiography}
\vskip -2\baselineskip plus -1fil 
\begin{IEEEbiography}[{\includegraphics[width=1in,height=1.25in,clip,keepaspectratio]{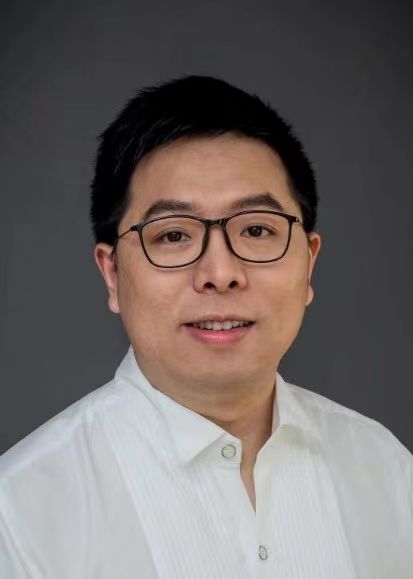}}]{Kun Yuan}
is an Assistant Professor at the Center for Machine Learning Research (CMLR) at Peking University. He completed his Ph.D. in Electrical and Computer Engineering at UCLA in 2019 and subsequently worked at Alibaba (US) Group from 2019 to 2022. His research concentrates on the theoretical and algorithmic foundations of optimization, signal processing, machine learning, and data science. Dr. Yuan is a recipient of the 2017 IEEE Signal Processing Society Young Author Best Paper Award and the 2017 ICCM Distinguished Paper Award.
\end{IEEEbiography}
\vskip -2\baselineskip plus -1fil 
\begin{IEEEbiography}[{\includegraphics[width=1in,height=1.25in,clip,keepaspectratio]{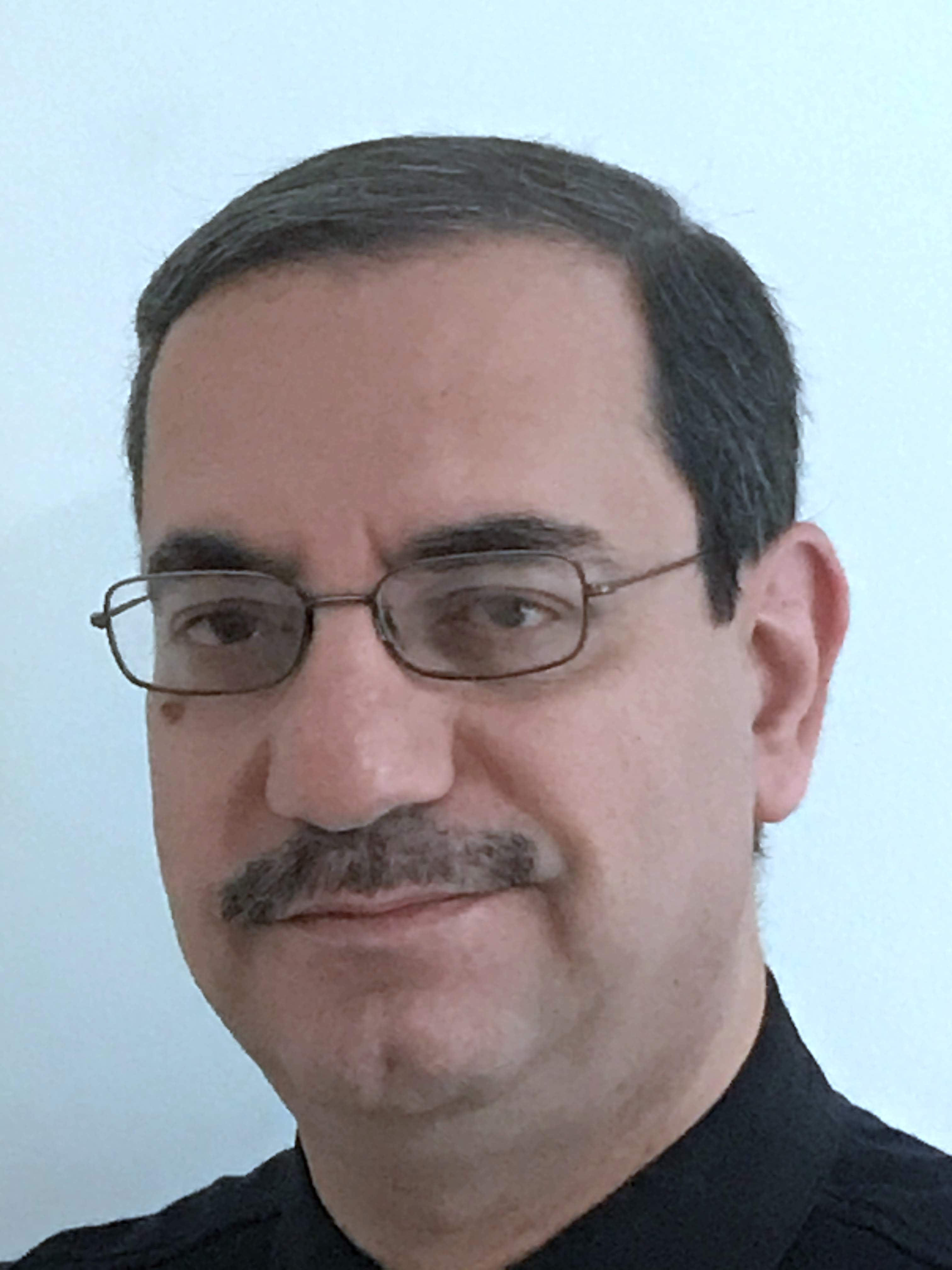}}]{Ali H. Sayed}
is Professor of Electrical and Micro Engineering at EPFL, where he also served as Dean of Engineering.  He served before as distinguished professor and former chairman of electrical engineering at UCLA. He also served as President of the IEEE Signal Processing Society and as a member of the IEEE Board of Directors. He is an author of over 650 scholarly publications and 10 books. He is a member of the US National Academy of Engineering and the World Academy of Sciences. His work has received several awards, including the IEEE Fourier Technical Field Award and the IEEE  Wiener Society Award.
\end{IEEEbiography}
\newpage
\onecolumn
\appendices
\section{Proof for Lemma \ref{lm_zbgn}}\label{p_lm_zbgn}
\begin{proof}
We begin by proving that the gradient noise is zero-mean. Basically, \eqref{s_gn_k_B} gives
\begin{align}\label{p_s0}
\mathds{E}\left\{s_{k,n}^B(\boldsymbol{w})\vert\mathcal{F}_{n-1}\right\} &= \mathds{E} \left\{\frac{1}{B}\sum\limits_b \nabla_w Q_k(\boldsymbol{w};\widehat{\boldsymbol{x}}_{k,n}^b,\boldsymbol{y}_{k,n}^b)  - \mathds{E}\nabla_w Q_k(\boldsymbol{w};\widehat{\boldsymbol{x}}_{k,n},\boldsymbol{y}_{k,n})\right\} \notag\\
&= \frac{1}{B}\sum\limits_b \mathds{E}\nabla_w Q_k(\boldsymbol{w};\widehat{\boldsymbol{x}}_{k,n}^b,\boldsymbol{y}_{k,n}^b) - \mathds{E}\nabla_w Q_k(\boldsymbol{w};\widehat{\boldsymbol{x}}_{k,n},\boldsymbol{y}_{k,n})\notag\\
&= 0
\end{align}

We now show that the second- and fourth-order moments of the gradient noise are bounded. To do so, we introduce the optimal perturbation corresponding to the parameter $w^{\star}$ and the sample $\boldsymbol{x}_{k,n}$:
\begin{align}\label{dstar_wstar}
    \boldsymbol{\delta}^{\star}(w^{\star}) \overset{\Delta}{=}\mathop{\mathrm{argmax}}\limits_{\left\Vert\delta\right\Vert_{p}  \le \epsilon}   Q_k(w^{\star};\boldsymbol{x}_{k,n} + \delta,\boldsymbol{y}_{k,n})
\end{align}
we also define the corresponding adversarially perturbed sample as:
\begin{align}\label{xstar_wstar}
    \boldsymbol{x}_{k,n}^{\star}(w^{\star}) \overset{\Delta}{=} \boldsymbol{x}_{k,n} + \boldsymbol{\delta}^{\star}(w^{\star})
\end{align}
Then, in the case where the batch size $B = 1$, we have
\begin{align}\label{p_s2}
&\mathds{E}\{\Vert \boldsymbol{s}_{k,n}(\boldsymbol{w})\Vert^2\vert\mathcal{F}_{n-1} \}\notag\\
&= \mathds{E}\Vert\nabla_w Q_k(\boldsymbol{w};\widehat{\boldsymbol{x}}_{k,n}, \boldsymbol{y}_{k,n}) - \mathds{E}\nabla_w Q_k(\boldsymbol{w};\widehat{\boldsymbol{x}}_{k,n},\boldsymbol{y}_{k,n})\Vert^2 \notag\\
&\overset{(a)}{\le} 4\mathds{E}\Vert\nabla_w Q_k(\boldsymbol{w};\widehat{\boldsymbol{x}}_{k,n},\boldsymbol{y}_{k,n})\Vert^2\notag\\
&=4\mathds{E}\Vert\nabla_w Q_k(\boldsymbol{w};\widehat{\boldsymbol{x}}_{k,n},\boldsymbol{y}_{k,n}) - \nabla_w Q_k(w^{\star};\boldsymbol{x}_{k,n}^{\star}(w^{\star}),\boldsymbol{y}_{k,n}) + \nabla_w Q_k(w^{\star};\boldsymbol{x}_{k,n}^{\star}(w^{\star}),\boldsymbol{y}_{k,n})\Vert^2\notag\\
&\overset{(b)}{\le}8\mathds{E}\Vert\nabla_w Q_k(\boldsymbol{w};\widehat{\boldsymbol{x}}_{k,n},\boldsymbol{y}_{k,n}) - \nabla_w Q_k(w^{\star};\widehat{\boldsymbol{x}}_{k,n},\boldsymbol{y}_{k,n}) + \nabla_w Q_k(w^{\star};\widehat{\boldsymbol{x}}_{k,n},\boldsymbol{y}_{k,n}) \notag\\
&\quad\;- \nabla_w Q_k(w^{\star};\boldsymbol{x}_{k,n}^{\star}(w^{\star}),\boldsymbol{y}_{k,n})\Vert^2 + 8\mathds{E}\Vert\nabla_w Q_k(w^{\star};\boldsymbol{x}_{k,n}^{\star}(w^{\star}),\boldsymbol{y}_{k,n})\Vert^2\notag\\
&\overset{(c)}{\le}16\mathds{E}\Vert\nabla_w Q_k(\boldsymbol{w};\widehat{\boldsymbol{x}}_{k,n},\boldsymbol{y}_{k,n}) - \nabla_w Q_k(w^{\star};\widehat{\boldsymbol{x}}_{k,n},\boldsymbol{y}_{k,n})\Vert^2 + 16\mathds{E}\Vert\nabla_w Q_k(w^{\star};\widehat{\boldsymbol{x}}_{k,n},\boldsymbol{y}_{k,n}) \notag\\
&\quad\;- \nabla_w Q_k(w^{\star};\boldsymbol{x}_{k,n}^{\star}(w^{\star}),\boldsymbol{y}_{k,n})\Vert + 8\mathds{E}\Vert\nabla_w Q_k(w^{\star};\boldsymbol{x}_{k,n}^{\star}(w^{\star}),\boldsymbol{y}_{k,n})\Vert^2\notag\\
&\overset{(d)}{\le}16L^2\Vert\widetilde{\boldsymbol{w}}\Vert^2+ O(\epsilon^2) + 8\mathds{E}\Vert\nabla_w Q_k(w^{\star};\boldsymbol{x}_{k,n}^{\star}(w^{\star}),\boldsymbol{y}_{k,n})\Vert^2
\end{align}
where $(a)$, $(b)$, and $(c)$ follow from Jensen's inequality, and $(d)$ follows from the Lipschitz conditions in Assumption \ref{assump_sc} and the constraint of the perturbation variable. That is,
\begin{align}\label{size_e}
    \Vert\nabla_w Q_k(w^{\star};\widehat{\boldsymbol{x}}_{k,n},\boldsymbol{y}_{k,n}) - \nabla_w Q_k(w^{\star};\boldsymbol{x}_{k,n}^{\star}(w^{\star}),\boldsymbol{y}_{k,n})\Vert^2 \overset{(a)}{\le} L^2\Vert\widehat{\boldsymbol{\delta}}_{k,n} - \boldsymbol{\delta}^{\star}(w^\star)\Vert^2 \overset{(b)}{\le}  O(\epsilon^2) 
\end{align}
where $(a)$ follows from \eqref{assump_sc_eq2}, and $(b)$ follows from Eq.(144) of \cite{cao2025decentralized}, which makes the use of the equivalence of vector norms, namely,
\begin{align}\label{delta_d}
  \Vert\widehat{\boldsymbol{\delta}}_{k,n} - \boldsymbol{\delta}^{\star}(w^\star)\Vert \le O(\epsilon)  
\end{align}

Similarly, for the fourth-order moments of the gradient noise, we have
\begin{align}\label{p_s4}
 &\mathds{E}\{\Vert \boldsymbol{s}_{k,n}(\boldsymbol{w})\Vert^4 \vert\mathcal{F}_{n-1} \} \notag\\
 & = \mathds{E}\Vert\nabla_w Q_k(\boldsymbol{w};\widehat{\boldsymbol{x}}_{k,n},\boldsymbol{y}_{k,n}) - \mathds{E}\nabla_w Q_k(\boldsymbol{w};\widehat{\boldsymbol{x}}_{k,n},\boldsymbol{y}_{k,n})\Vert^4\notag\\
 &\overset{(a)}{\le}16\mathds{E}\Vert\nabla_w Q_k(\boldsymbol{w};\widehat{\boldsymbol{x}}_{k,n},\boldsymbol{y}_{k,n})\Vert^4\notag\\
 &=16\mathds{E}\Vert\nabla_w Q_k(\boldsymbol{w};\widehat{\boldsymbol{x}}_{k,n},\boldsymbol{y}_{k,n}) - \nabla_w Q_k(w^{\star};\boldsymbol{x}_{k,n}^{\star}(w^{\star}),\boldsymbol{y}_{k,n}) + \nabla_w Q_k(w^{\star};\boldsymbol{x}_{k,n}^{\star}(w^{\star}),\boldsymbol{y}_{k,n})\Vert^4\notag\\
 &\overset{(b)}{\le}128\mathds{E}\Vert\nabla_w Q_k(\boldsymbol{w};\widehat{\boldsymbol{x}}_{k,n},\boldsymbol{y}_{k,n}) - \nabla_w Q_k(w^{\star};\widehat{\boldsymbol{x}}_{k,n},\boldsymbol{y}_{k,n}) + \nabla_w Q_k(w^{\star};\widehat{\boldsymbol{x}}_{k,n},\boldsymbol{y}_{k,n}) \notag\\
&\quad\;- \nabla_w Q_k(w^{\star};\boldsymbol{x}_{k,n}^{\star}(w^{\star}),\boldsymbol{y}_{k,n})\Vert^4 + 128\mathds{E}\Vert\nabla_w Q_k(w^{\star};\boldsymbol{x}_{k,n}^{\star}(w^{\star}),\boldsymbol{y}_{k,n})\Vert^4\notag\\
&\overset{(c)}{\le}1024\mathds{E}\Vert\nabla_w Q_k(\boldsymbol{w};\widehat{\boldsymbol{x}}_{k,n},\boldsymbol{y}_{k,n}) - \nabla_w Q_k(w^{\star};\widehat{\boldsymbol{x}}_{k,n},\boldsymbol{y}_{k,n})\Vert^4 + 1024\mathds{E}\Vert\nabla_w Q_k(w^{\star};\widehat{\boldsymbol{x}}_{k,n},\boldsymbol{y}_{k,n}) \notag\\
&\quad\; - \nabla_w Q_k(w^{\star};\boldsymbol{x}_{k,n}^{\star}(w^{\star}),\boldsymbol{y}_{k,n})\Vert^4 + 128\mathds{E}\Vert\nabla_w Q_k(w^{\star};\boldsymbol{x}_{k,n}^{\star}(w^{\star}),\boldsymbol{y}_{k,n})\Vert^4\notag\\
&\overset{(d)}{\le}1024L^4\Vert\widetilde{\boldsymbol{w}}\Vert^4+ O(\epsilon^4) + 128\mathds{E}\Vert\nabla_w Q_k(w^{\star};\boldsymbol{x}_{k,n}^{\star}(w^{\star}),\boldsymbol{y}_{k,n})\Vert^4
\end{align}
again, $(a)$, $(b)$, and $(c)$ follow from Jensen's inequality, and $(d)$ follows from Assumption \ref{assump_sc} and bounded perturbations. 

Then, similar to Appendix A of \cite{cao2024trade}, the size of the gradient noise scales inversely with the batch size $B$, namely, 
\begin{align}\label{p_s2_B}
 &\mathds{E}\{\Vert \boldsymbol{s}_{k,n}^B(\boldsymbol{w})\Vert^2\vert\mathcal{F}_{n-1} \} = \frac{1}{B} \mathds{E}\{\Vert \boldsymbol{s}_{k,n}(\boldsymbol{w})\Vert^2\vert\mathcal{F}_{n-1} \} \overset{(a)}{\le} O\left(\frac{1}{B}\right)\Vert\widetilde{\boldsymbol{w}}\Vert^2 + O\left(\frac{1}{B}\right)\\
 \label{p_s4_B}
 &\mathds{E}\{\Vert \boldsymbol{s}_{k,n}^B(\boldsymbol{w})\Vert^4\vert\mathcal{F}_{n-1} \} \le O\left(\frac{1}{B^2}\right) \mathds{E}\{\Vert \boldsymbol{s}_{k,n}(\boldsymbol{w})\Vert^4\vert\mathcal{F}_{n-1} \} \overset{(b)}{\le} O\left(\frac{1}{B^2}\right)\Vert\widetilde{\boldsymbol{w}}\Vert^2 + O\left(\frac{1}{B^2}\right)
\end{align}
where $(a)$ follows from \eqref{p_s2}, and $(b)$ follows from \eqref{p_s4}.

Finally, the proof of \eqref{gcm_scale} is the same with Eq. (92) of \cite{cao2024trade}.
\end{proof}

\section{Decomposition of \texorpdfstring{\eqref{unified_re_1}}{(ref)}}
Recalling the eigen-structure of $A$ in \eqref{decomp_A}--\eqref{vpv222}, we further introduce the following block matrices extended by the Kronecker product operation:
\begin{equation}
    \mathcal{V} = V\otimes I_M,\quad\; \mathcal{P} = P\otimes I_M,\quad\;  \mathds{1} = \mathbbm{1}\otimes I_M
\end{equation}
These extended matrices allow us to analyze the coupled dynamics of the entire network in a compact block form. Using the eigen-structure of $A$, the unified recursion in \eqref{unified_re_1} can be further decomposed. Basically, by multiplying the left of \eqref{unified_re_1}, we obtain
\begin{align}\label{unified_decomp}
\mathcal{V}^{\sf T}{\widetilde{\boldsymbol{\scriptstyle\mathcal{W}}}}_{n} =& \left[\begin{array}{c}
         (\frac{1}{\sqrt{K}}\mathbbm{1}^{\sf T} \otimes I_M){\widetilde{\boldsymbol{\scriptstyle\mathcal{W}}}}_{n} \\
         (V_\alpha^{\sf T}\otimes I_M){\widetilde{\boldsymbol{\scriptstyle\mathcal{W}}}}_{n}
    \end{array}\right] \overset{\Delta}{=} \left[\begin{array}{c}
         \bar{\boldsymbol{w}}_n  \\
          \check{\boldsymbol{w}}_n
    \end{array}\right]\notag\\
    =& \mathcal{P}\mathcal{V}^{\sf T}{\widetilde{\boldsymbol{\scriptstyle\mathcal{W}}}}_{n-1} - \mu\mathcal{V}^{\sf T}\mathcal{A}_2\left(\mathop{\rm col}\limits_k\left\{\nabla J_k(\boldsymbol{w}_{k,n-1}) - \nabla J_k(w^{\star})\right\} + d +  \boldsymbol{e}_{n-1} + \boldsymbol{s}_{n}^{B}\right)\notag\\
    =& \left[\begin{array}{cc}
        I_M & \boldsymbol{0} \\
        \boldsymbol{0} & P_{\alpha}\otimes I_M
    \end{array}\right]\left[\begin{array}{c}
         \bar{\boldsymbol{w}}_{n-1}  \\
          \check{\boldsymbol{w}}_{n-1}
    \end{array}\right] \notag\\
    &- \mu\left[\begin{array}{c}\frac{1}{\sqrt{K}}\mathds{1}^{\sf T}\\
   \mathcal{V}_{\alpha}^{\sf T}
\end{array}\right]\mathcal{A}_2\left(\mathop{\rm col}\limits_k\left\{\nabla J_k(\boldsymbol{w}_{k,n-1}) - \nabla J_k(w^{\star})\right\} + d +  \boldsymbol{e}_{n-1} + \boldsymbol{s}_{n}^{B}\right)
\end{align}
Note that
\begin{align}\label{tw-bw-cw}
\Vert{\widetilde{\boldsymbol{\scriptstyle\mathcal{W}}}}_{n}\Vert^2 = \Vert\mathcal{V}^{\sf T}{\widetilde{\boldsymbol{\scriptstyle\mathcal{W}}}}_{n}\Vert^2 = \Vert\bar{\boldsymbol{w}}_n\Vert^2 + \Vert\check{\boldsymbol{w}}_n\Vert^2
\end{align}

Applying the transformation in \eqref{unified_decomp} yields a decomposition of \eqref{unified_re_1} into two parts, characterized by the quantities $\{\bar{\boldsymbol{w}}_n, 
\check{\boldsymbol{w}}_n\}$:
\begin{align}\label{barw_re1}
\bar{\boldsymbol{w}}_n =& \bar{\boldsymbol{w}}_{n-1} - \frac{\mu}{\sqrt{K}}\sum_k\left(\nabla  J_k(\boldsymbol{w}_{k,n-1}) - \nabla J_k(w^{\star})\right) - \underbrace{\frac{\mu}{\sqrt{K}}\sum_k\nabla J_k(w^{\star})}_{0} - \frac{\mu}{\sqrt{K}}\sum_k\boldsymbol{e}_{k,n-1} \notag\\
&- \frac{\mu}{\sqrt{K}}\sum_k\boldsymbol{s}_{k,n}^B(\boldsymbol{w}_{k,n-1})\\
\label{checkw_re1}
\check{\boldsymbol{w}}_n =& \mathcal{P}_\alpha\check{\boldsymbol{w}}_{n-1} - \mu\mathcal{V}_{\alpha}^{\sf T}\mathcal{A}_2\left(\mathop{\rm col}\limits_k\left\{\nabla J_k(\boldsymbol{w}_{k,n-1}) - \nabla J_k(w^{\star})\right\} + d +  \boldsymbol{e}_{n-1} + \boldsymbol{s}_{n}^{B}\right)
\end{align}
Note that in the centralized method where both $\mathcal{P}_\alpha$ and $\mathcal{V}_{\alpha}$ are zero matrices, the component $\check{\boldsymbol{w}}_n$ is eliminated, and only the term $\bar{\boldsymbol{w}}_n$ remains in the recursion.

\section{Proof for Assumption \ref{assump_ip}}\label{p_assumpip}
\begin{proof}
    We first justify \eqref{assump_ip_eq_1}, which follows from \cite{cao2025decentralized}. Although \cite{cao2025decentralized} originally indicated that the perturbation-related term is of order $O(\epsilon)$, it can 
actually be verified that the order can be made arbitrarily close to $2$. Specifically, referring to Eq.(186) of \cite{cao2025decentralized}, where the approximation error term originates, it can be alternatively split as:
\begin{align}\label{p_assumip_wc2}
O(\epsilon)\Vert\bar{\boldsymbol{w}}_{n-1}\Vert = O(\epsilon^{\frac{7}{8}})\times O(\epsilon^{\frac{1}{8}})\Vert\bar{\boldsymbol{w}}_{n-1}\Vert {\le} O(\epsilon^{\frac{7}{4}}) + O(\epsilon^{\frac{1}{4}})\Vert\bar{\boldsymbol{w}}_{n-1}\Vert^2
\end{align}
Using this decomposition and the proof techniques provided in Appendix E of \cite{cao2025decentralized}, the bound in \eqref{assump_ip_eq_1} can be verified.

We now justify the bound in \eqref{assump_ip_eq_2}. Note that in the centralized setting, all agents share a single model denoted by $\boldsymbol{w}_{c,n}$, meaning their parameters are identical at every iteration. For the centralized setting, multiplying both sides of \eqref{barw_re1} by $\frac{1}{\sqrt{K}}$ and using the step size $\mu'$ gives:
\begin{align}\label{re_wc}
    \widetilde{\boldsymbol{w}}_{c,n} = \widetilde{\boldsymbol{w}}_{c,n-1} - \mu'\left(\nabla J(\boldsymbol{w}_{c,n-1}) - \nabla J(w^{\star})\right) - \mu'\bar{\boldsymbol{e}}_{n-1} - \mu'\bar{\boldsymbol{s}}_{n}^{B}
\end{align}
where
\begin{align}\label{p_assumip_des}
    &\bar{\boldsymbol{e}}_{n-1} \overset{\Delta}{=}\frac{1}{K}\sum_k\boldsymbol{e}_{k,n-1}(\boldsymbol{w}_{c,n-1})\\
    &\bar{\boldsymbol{s}}_{n}^{B} \overset{\Delta}{=}\frac{1}{K}\sum_k\boldsymbol{s}_{k,n}^B(\boldsymbol{w}_{c,n-1})
\end{align}

To proceed, we prove useful lemmas for $\bar{\boldsymbol{e}}_{n-1}$ and $\bar{\boldsymbol{s}}_{n}^{B}$. Basically,
\begin{align}\label{p_assumip_pe}
   \Vert\bar{\boldsymbol{e}}_{n-1}\Vert &\le \frac{1}{K}\sum_k\Vert\boldsymbol{e}_{k,n-1}(\boldsymbol{w}_{c,n-1})\Vert\notag\\
   &\le \frac{1}{K}\sum_k\mathds{E}\Vert\nabla_w Q_k(\boldsymbol{w}_{c,n-1};\widehat{\boldsymbol{x}}_{k,n},\boldsymbol{y}_{k,n}) - \nabla_w Q_k(\boldsymbol{w}_{c,n-1};\boldsymbol{x}_{k,n}^{\star},\boldsymbol{y}_{k,n})\Vert \notag\\
   &\overset{(a)}{\le} O(\epsilon)
\end{align}
where $(a)$ follows from \eqref{assump_sc_eq2} and \eqref{delta_d}. For the gradient noise, following a similar process to Lemma \ref{lm_zbgn}, we establish the following bounds:
\begin{align}
\label{p_assumip_ps0}
    \mathds{E}\{\bar{\boldsymbol{s}}_n^{B}\vert\mathcal{F}_{n-1}\} \overset{(a)}{=}& 0\\
    \label{p_assumip_ps2}
     \mathds{E}\left\{\Vert\bar{\boldsymbol{s}}_n^{B}\Vert^2\vert\mathcal{F}_{n-1}\right\} = &\mathds{E}\left\{\left\Vert\frac{1}{K}\sum_k{\boldsymbol{s}}_{k,n}^{B}(\boldsymbol{w}_{c,n-1})\right\Vert^2\bigg\vert\mathcal{F}_{n-1}\right\} \notag\\
     \overset{(b)}{=}& \frac{1}{K^2}\sum_k \mathds{E}\left\{\Vert{\boldsymbol{s}}_{k,n}^{B}(\boldsymbol{w}_{c,n-1})\Vert^2\vert\mathcal{F}_{n-1}\right\} \notag\\
     \overset{(c)}{\le}& O\left(\frac{1}{B}\right)\Vert\widetilde{\boldsymbol{w}}_{c,n-1}\Vert^2 +  O\left(\frac{1}{B}\right)\\
     \label{p_assumip_ps4}
     \mathds{E}\left\{\Vert\bar{\boldsymbol{s}}_n^{B}\Vert^4\vert\mathcal{F}_{n-1}\right\} \le& \frac{1}{K}\sum_k\mathds{E}\left\{\Vert{\boldsymbol{s}}_{k,n}^{B}(\boldsymbol{w}_{c,n-1})\Vert^4\vert\mathcal{F}_{n-1}\right\}\notag\\
     \overset{(d)}{\le}&O\left(\frac{1}{B^2}\right)\Vert\widetilde{\boldsymbol{w}}_{c,n-1}\Vert^4 +  O\left(\frac{1}{B^2}\right)
\end{align}
where $(a)$, $(c)$, and $(d)$ follow from \eqref{sgn_0}, \eqref{sgn_2}, \eqref{sgn_4}, respectively. The step $(b)$ follows from the independence among agents.

In addition, for any two random vectors $\boldsymbol{a}$ and $\boldsymbol{b}$, we have
\begin{align}\label{sum4_1}
    \Vert \boldsymbol{a} + \boldsymbol{b}\Vert^4 =& \Vert \boldsymbol{a}\Vert^4 + \Vert \boldsymbol{b}\Vert^4 + 2\Vert \boldsymbol{a}\Vert^2 \Vert \boldsymbol{b}\Vert^2 + 4\boldsymbol{a}^{\sf T}\boldsymbol{b}\Vert \boldsymbol{a}\Vert^2+ 4\boldsymbol{a}^{\sf T}\boldsymbol{b}\Vert \boldsymbol{b}\Vert^2 + 4(\boldsymbol{a}^{\sf T}\boldsymbol{b})^2 \notag\\
    \overset{(a)}{\le} &\Vert\boldsymbol{a}\Vert^4 + 3\Vert\boldsymbol{b}\Vert^4+ 8\Vert\boldsymbol{a}\Vert^2\Vert\boldsymbol{b}\Vert^2 + 4\Vert \boldsymbol{a}\Vert^2(\boldsymbol{a}^{\sf T}\boldsymbol{b})
\end{align}
where $(a)$ follows from  Cauchy–Schwarz inequality and Jensen's inequality. Moreover, if $\mathds{E}\boldsymbol{b}$ = 0, we have
\begin{align}\label{sum4_2}
    \mathds{E}\Vert\boldsymbol{a}+\boldsymbol{b}\Vert^4 \le \mathds{E}\Vert\boldsymbol{a}\Vert^4 + 3\mathds{E}\Vert\boldsymbol{b}\Vert^4+ 8\mathds{E}\Vert\boldsymbol{a}\Vert^2\Vert\boldsymbol{b}\Vert^2
\end{align}

We now prove the upper bound for $\mathds{E}\Vert\widetilde{\boldsymbol{w}}_{c,n-1}\Vert^4$. Assume $J(w)$ is $v$-strongly concave near $w^{\star}$, recalling \eqref{re_wc} gives:
\begin{align}\label{p_wc4}
    \mathds{E}\Vert\widetilde{\boldsymbol{w}}_{c,n}\Vert^4 =& \mathds{E}\Vert\widetilde{\boldsymbol{w}}_{c,n-1} - \mu'\left(\nabla J(\boldsymbol{w}_{c,n-1}) - \nabla J(w^{\star})\right) - \mu'\bar{\boldsymbol{e}}_{n-1} - \mu'\bar{\boldsymbol{s}}_{n}^{B}\Vert^4\notag\\
    \overset{(a)}{\le}& \mathds{E}\Vert\widetilde{\boldsymbol{w}}_{c,n-1} - \mu'\left(\nabla J(\boldsymbol{w}_{c,n-1}) - \nabla J(w^{\star})\right) - \mu'\bar{\boldsymbol{e}}_{n-1}\Vert^4 + 3\mu'^4\mathds{E}\Vert\bar{\boldsymbol{s}}_{n}^{B}\Vert^4 \notag\\
    & +8\mu'^2\mathds{E}\Vert\widetilde{\boldsymbol{w}}_{c,n-1} - \mu'\left(\nabla J(\boldsymbol{w}_{c,n-1}) - \nabla J(w^{\star})\right) - \mu'\bar{\boldsymbol{e}}_{n-1}\Vert^2\Vert\bar{\boldsymbol{s}}_{n}^{B}\Vert^2
\end{align}
where $(a)$ follows from \eqref{p_assumip_ps0}, \eqref{sum4_2}, and the Law of total expectation. To proceed, we bound each term on the right-hand side of \eqref{p_wc4} separately. Basically, for the first term, we have
\begin{align}\label{p_wc4_1}
&\mathds{E}\Vert\widetilde{\boldsymbol{w}}_{c,n-1} - \mu'\left(\nabla J(\boldsymbol{w}_{c,n-1}) - \nabla J(w^{\star})\right) - \mu'\bar{\boldsymbol{e}}_{n-1}\Vert^4 \notag\\
\overset{(a)}{\le}&\mathds{E}\Vert\widetilde{\boldsymbol{w}}_{c,n-1}\Vert^4 + 3\mu'^4\mathds{E}\Vert\nabla J(\boldsymbol{w}_{c,n-1}) - \nabla J(w^{\star}) + \bar{\boldsymbol{e}}_{n-1}\Vert^4 \notag\\&+
8\mu'^2\mathds{E}\Vert\widetilde{\boldsymbol{w}}_{c,n-1}\Vert^2\Vert\nabla J(\boldsymbol{w}_{c,n-1}) - \nabla J(w^{\star}) + {\bar{\boldsymbol{e}}}_{n-1}\Vert^2 \notag\\
&- 4\mu'\mathds{E}\Vert\widetilde{\boldsymbol{w}}_{c,n-1}\Vert^2\widetilde{\boldsymbol{w}}_{c,n-1}^{\sf T}(\nabla J(\boldsymbol{w}_{c,n-1}) - \nabla J(w^{\star}))- 4\mu'\mathds{E}\Vert\widetilde{\boldsymbol{w}}_{c,n-1}\Vert^2\widetilde{\boldsymbol{w}}_{c,n-1}^{\sf T}\bar{\boldsymbol{e}}_{n-1}
\end{align}
where $(a)$ follows from \eqref{sum4_1}. For the terms on the right-hand side of \eqref{p_wc4_1}, we further have
\begin{align}\label{p_wc4_11}
    3\mu'^4\mathds{E}\Vert\nabla J(\boldsymbol{w}_{c,n-1}) - \nabla J(w^{\star}) + \bar{\boldsymbol{e}}_{n-1}\Vert^4 \le& O(\mu'^4)\mathds{E}\Vert\nabla J(\boldsymbol{w}_{c,n-1}) - \nabla J(w^{\star}) \Vert^4 + O(\mu'^4)\mathds{E}\Vert\bar{\boldsymbol{e}}_{n-1}\Vert^4\notag\\
    \overset{(a)}{\le} &O(\mu'^4)\mathds{E}\Vert\widetilde{\boldsymbol{w}}_{c,n-1}\Vert^4 + O(\mu'^4\epsilon^4)
\end{align}
\begin{align}\label{p_wc4_12}
    8\mu'^2\mathds{E}\Vert\widetilde{\boldsymbol{w}}_{c,n-1}\Vert^2\Vert\nabla J(\boldsymbol{w}_{c,n-1}) - \nabla J(w^{\star}) + {\bar{\boldsymbol{e}}}_{n-1}\Vert^2 \le & O(\mu'^2)\mathds{E}\Vert\widetilde{\boldsymbol{w}}_{c,n-1}\Vert^2\left(2\Vert\nabla J(\boldsymbol{w}_{c,n-1}) - \nabla J(w^{\star})\Vert^2 + 2\Vert{\bar{\boldsymbol{e}}}_{n-1}\Vert^2\right)\notag\\
    \overset{(b)}{\le}& O(\mu'^2)\mathds{E}\Vert\widetilde{\boldsymbol{w}}_{c,n-1}\Vert^2 (O(1)\Vert\widetilde{\boldsymbol{w}}_{c,n-1}\Vert^2 + O(\epsilon^2)) \notag\\
    =& O(\mu'^2)\mathds{E}\Vert\widetilde{\boldsymbol{w}}_{c,n-1}\Vert^4 + O(\mu'^2\epsilon^2)\mathds{E}\Vert\widetilde{\boldsymbol{w}}_{c,n-1}\Vert^2
\end{align}
\begin{align}\label{p_wc4_13}
   - \widetilde{\boldsymbol{w}}_{c,n-1}^{\sf T}(\nabla J(\boldsymbol{w}_{c,n-1}) - \nabla J(w^{\star})) \overset{(c)}{\le} -v\Vert\widetilde{\boldsymbol{w}}_{c,n-1}\Vert^2
\end{align}
\begin{align}\label{p_wc4_14}
-\Vert\widetilde{\boldsymbol{w}}_{c,n-1}\Vert^2\widetilde{\boldsymbol{w}}_{c,n-1}^{\sf T}\bar{\boldsymbol{e}}_{n-1} &\overset{(d)}{\le} \Vert\widetilde{\boldsymbol{w}}_{c,n-1}\Vert^2 (O(\epsilon)\Vert\widetilde{\boldsymbol{w}}_{c,n-1}\Vert)\notag\\
&= \Vert\widetilde{\boldsymbol{w}}_{c,n-1}\Vert^2 (O(\epsilon^{\frac{7}{8}}) \times O(\epsilon^{\frac{1}{8}})\Vert\widetilde{\boldsymbol{w}}_{c,n-1}\Vert) \notag\\
&\le O(\epsilon^{\frac{1}{4}})\Vert\widetilde{\boldsymbol{w}}_{c,n-1}\Vert^4 + O(\epsilon^{\frac{7}{4}})\Vert\widetilde{\boldsymbol{w}}_{c,n-1}\Vert^2
\end{align}
where $(a)$ and $(b)$ follow from \eqref{affine_l_eq_2} and \eqref{p_assumip_pe}, $(c)$ follows from the strong convexity of $J(w)$, and $(d)$ follows from \eqref{p_assumip_pe}.

Substituting \eqref{p_wc4_11}--\eqref{p_wc4_14} into \eqref{p_wc4_1} gives
\begin{align}\label{p_wc4_1_f}
  &\mathds{E}\Vert\widetilde{\boldsymbol{w}}_{c,n-1} - \mu'\left(\nabla J(\boldsymbol{w}_{c,n-1}) - \nabla J(w^{\star})\right) - \mu'\bar{\boldsymbol{e}}_{n-1}\Vert^4\notag\\
  {\le}& (1 - O(\mu') + o(\mu'))\mathds{E}\Vert\widetilde{\boldsymbol{w}}_{c,n-1}\Vert^4 + O(\mu'^4\epsilon^4) + O(\mu'^2\epsilon^2)\mathds{E}\Vert\widetilde{\boldsymbol{w}}_{c,n-1}\Vert^2 +  O(\mu'\epsilon^\frac{7}{4})\mathds{E}\Vert\widetilde{\boldsymbol{w}}_{c,n-1}\Vert^2\notag\\
\overset{(a)}{\le}& (1 - O(\mu') + o(\mu'))\mathds{E}\Vert\widetilde{\boldsymbol{w}}_{c,n-1}\Vert^4 + O(\frac{\mu'^2}{B}\epsilon^{\frac{7}{4}}) + O(\mu'\epsilon^{\frac{7}{2}})
\end{align}
for small $\mu'$ and $\epsilon$, where $(a)$ follows from \eqref{assump_ip_eq_1_e}.

Then, for the second term in \eqref{p_wc4}, it follows from \eqref{p_assumip_ps4} that:
\begin{align}\label{p_wc4_2}
3\mu'^4\mathds{E}\Vert\bar{\boldsymbol{s}}_{n}^{B}\Vert^4 \le  O\left(\frac{\mu'^4}{B^2}\right)\mathds{E}\Vert\widetilde{\boldsymbol{w}}_{c,n-1}\Vert^4 + O\left(\frac{\mu'^4}{B^2}\right) 
\end{align}

As for the third term of \eqref{p_wc4}, we have
\begin{align}\label{p_wc4_3}
&8\mu'^2\mathds{E}\Vert\widetilde{\boldsymbol{w}}_{c,n-1} - \mu'\left(\nabla J(\boldsymbol{w}_{c,n-1}) - \nabla J(w^{\star})\right) - \mu'\bar{\boldsymbol{e}}_{n-1}\Vert^2\Vert\bar{\boldsymbol{s}}_{n}^{B}\Vert^2\notag\\
&\overset{(a)}{\le}O(\mu'^2)\mathds{E}\bigg\{\left(\Vert\widetilde{\boldsymbol{w}}_{c,n-1}\Vert^2 - 2\mu'\widetilde{\boldsymbol{w}}_{c,n-1}^{\sf T}\left(\nabla J(\boldsymbol{w}_{c,n-1}) - \nabla J(w^{\star})\right) - 2\mu'\widetilde{\boldsymbol{w}}_{c,n-1}^{\sf T}\bar{\boldsymbol{e}}_{n-1} + \mu'^2\Vert\nabla J(\boldsymbol{w}_{c,n-1}) - \nabla J(w^{\star}) + \bar{\boldsymbol{e}}_{n-1}\Vert^2 \right)\notag\\
&\quad\ \times\left(O(\frac{1}{B})\Vert\widetilde{\boldsymbol{w}}_{c,n-1}\Vert^2 + O(\frac{1}{B})\right) \bigg\}\notag\\
&\overset{(b)}{\le}O\left(\frac{\mu'^2}{B}\right)\mathds{E}\Big\{\big(\Vert\widetilde{\boldsymbol{w}}_{c,n-1}\Vert^2 - 2\mu'v\Vert\widetilde{\boldsymbol{w}}_{c,n-1}\Vert^2 + \mu'\Vert\widetilde{\boldsymbol{w}}_{c,n-1}\Vert^2 + O(\mu'\epsilon^2) + O(\mu'^2)\Vert\widetilde{\boldsymbol{w}}_{c,n-1}\Vert^2\big)\big(\Vert\widetilde{\boldsymbol{w}}_{c,n-1}\Vert^2+ O(1)\big)\Big\}\notag\\
 &=O\left(\frac{\mu'^2}{B}\right)\mathds{E}\Vert\widetilde{\boldsymbol{w}}_{c,n-1}\Vert^4 + O\left(\frac{\mu'^2}{B}\right)\mathds{E}\Vert\widetilde{\boldsymbol{w}}_{c,n-1}\Vert^2 + O\left(\frac{\mu'^3}{B}\epsilon^2\right)\notag\\
&\overset{(c)}{\le} O\left(\frac{\mu'^2}{B}\right)\mathds{E}\Vert\widetilde{\boldsymbol{w}}_{c,n-1}\Vert^4 + O\left(\frac{\mu'^3}{B^2}\right)+O\left(\frac{\mu'^2}{B}\epsilon^{\frac{7}{4}}\right)
\end{align}
where $(a)$ follows from \eqref{p_assumip_ps2} and the Law of total expectation, $(b)$ follows from the strong convexity of $J(w)$, \eqref{affine_l_eq_2}, and \eqref{p_assumip_pe}, and $(c)$ follows from \eqref{assump_ip_eq_1_e}.

Substituting \eqref{p_wc4_1_f}--\eqref{p_wc4_3} into \eqref{p_wc4}, we obtain
\begin{align}\label{p_wc4_f}
 \mathds{E}\Vert\widetilde{\boldsymbol{w}}_{c,n}\Vert^4 =& (1 - O(\mu') + o(\mu'))\mathds{E}\Vert\widetilde{\boldsymbol{w}}_{c,n-1}\Vert^4 + O\left(\frac{\mu'^3}{B^2}\right) + O(\mu'\epsilon^{\frac{7}{2}}) + O\left(\frac{\mu'^2}{B}\epsilon^{\frac{7}{4}}\right) 
\end{align}
where 
\begin{align}
    1 - O(\mu') + o(\mu') \le 1 
\end{align}
for sufficiently small $\mu'$, and
\begin{align}
  O\left(\frac{\mu'^2}{B}\epsilon^{\frac{7}{4}}\right) = O(\mu')(\frac{\mu'}{B}\times \epsilon^{\frac{7}{4}}) \le  O\left(\frac{\mu'^3}{B^2}\right) + O(\mu'\epsilon^{\frac{7}{2}})
\end{align}

Finally, iterating \eqref{p_wc4_f}, after sufficient iterations, we obtain
\begin{align}
    \mathop{\lim}\limits_{n\to\infty} \mathds{E}\left\Vert \widetilde{\boldsymbol{w}}_{c,n}\right\Vert^2 \le O(\frac{\mu'^2}{B^2}) + O(\epsilon^{\frac{7}{2}})  
\end{align}
which means that the centralized algorithm can attain an $O(\frac{\mu'^2}{B^2}) + O(\epsilon^{\frac{7}{2}})$-neighborhood of $w^{\star}$ in the fourth-order moment sense after enough iterations.
\end{proof}

\section{Proof for Lemma \ref{mse_2}}\label{p_mse_2}
\begin{proof}\label{p_mse2}
In this section, we establish bounds for $\mathds{E}\Vert{\widetilde{\boldsymbol{\scriptstyle\mathcal{W}}}}_{n}\Vert^2$ for the three distributed adversarial training algorithms. As discussed following lemma \ref{affine_l}, unlike the clean setting \cite{cao2024trade}, the affine Lipschitz property implies that the Hessian matrices in the adversarial context are not uniformly bounded. As a result, the proof techniques used in \cite{cao2024trade} for the clean case cannot be applied directly to this paper. To solve the problem, we use the recursion in \eqref{unified_re_1}, which is formulated in terms of gradients, to derive bounds for $\mathds{E}\Vert{\widetilde{\boldsymbol{\scriptstyle\mathcal{W}}}}_{n}\Vert^2$ and $\mathds{E}\Vert{\widetilde{\boldsymbol{\scriptstyle\mathcal{W}}}}_{n}\Vert^4$.

With \eqref{barw_re1} and \eqref{checkw_re1}, we derive bounds for $\mathds{E}\Vert\bar{\boldsymbol{w}}_n\Vert^2$ and $\mathds{E}\Vert\check{\boldsymbol{w}}_n\Vert^2$. Basically, we consider $0< t = 1 - O(\mu)< 1$, then, for $\mathds{E}\Vert\bar{\boldsymbol{w}}_n\Vert^2$,  \eqref{barw_re1} gives:
\begin{align}\label{bar_w_21}
    \mathds{E}\Vert\bar{\boldsymbol{w}}_n\Vert^2 &\overset{(a)}{=} \mathds{E}\Vert\bar{\boldsymbol{w}}_{n-1} - \frac{\mu}{\sqrt{K}}\sum_k(\nabla J_k(\boldsymbol{w}_{k,n-1}) - \nabla J_k(w^{\star})) - \frac{\mu}{\sqrt{K}}\sum_k \boldsymbol{e}_{k,n-1}\Vert^2 \notag\\
    &\quad+ \frac{\mu^2}{K}\sum_k \mathds{E}\Vert\boldsymbol{s}_{k,n}^B(\boldsymbol{w}_{k,n-1})\Vert^2\notag\\
    &= \mathds{E}\Bigg\Vert t\frac{1}{t}\bar{\boldsymbol{w}}_{n-1} + (1-t)\frac{1}{1-t}\frac{\mu}{\sqrt{K}}\left(\sum_k (\nabla J_k(\boldsymbol{w}_{k,n-1}) - \nabla J_k(w^{\star})) + \sum_k \boldsymbol{e}_{k,n-1}\right)\Bigg\Vert^2\notag\\
    &\quad\; + \frac{\mu^2}{K}\sum_k \mathds{E}\Vert\boldsymbol{s}_{k,n}^B\Vert^2\notag\\
    &\overset{(b)}{\le} \frac{1}{1-O(\mu)}\mathds{E}\Vert\bar{\boldsymbol{w}}_{n-1}\Vert^2 + \frac{2\mu^2}{O(\mu)}\sum_k\mathds{E}\Vert\nabla J_k(\boldsymbol{w}_{k,n-1}) - \nabla J_k(w^{\star})\Vert^2 \notag\\
    &\quad\; + \frac{2\mu^2}{O(\mu)}\sum_k \mathds{E}\Vert \boldsymbol{e}_{k,n-1}\Vert^2 + \frac{\mu^2}{K}\sum_k \mathds{E}\Vert\boldsymbol{s}_{k,n}^B\Vert^2
\end{align}
where $(a)$ follows from \eqref{sgn_0} and the independence of data among agents, namely, for any two agents $k$ and $\ell$, it holds that:
\begin{equation}\label{sk_sl_0a}
\mathds{E}\left[\boldsymbol{s}_{k,n}^B(\boldsymbol{s}_{\ell,n}^{B})^{\sf  T}\vert\mathcal{F}_{n-1}\right] = \mathds{E}\left[\boldsymbol{s}_{k,n}^B\vert\mathcal{F}_{n-1}\right]\times\mathds{E}\left[(\boldsymbol{s}_{\ell,n}^B)^{\sf T}\vert\mathcal{F}_{n-1}\right] = 0
\end{equation}
and $(b)$ follows from Jensen's inequality. Moreover, for the terms on the RHS of \eqref{bar_w_21}, we have
\begin{align}
\label{bar_w_211}
    &\Vert\nabla J_k (\boldsymbol{w}_{k,n-1}) - \nabla J_k(w^{\star})\Vert^2 \overset{(a)}{\le} 2L^2\Vert\widetilde{\boldsymbol{w}}_{k,n-1}\Vert^2 + O(\epsilon^2)\\
    \label{bar_w_212}
    &\Vert\boldsymbol{e}_{k,n-1}\Vert^2 \le \mathds{E}\Vert\nabla_w Q_k(\boldsymbol{w}_{k,n-1};\widehat{\boldsymbol{x}}_{k,n},\boldsymbol{y}_{k,n}) - \nabla_w Q_k(\boldsymbol{w}_{k,n-1};\boldsymbol{x}_{k,n}^{\star},\boldsymbol{y}_{k,n})\Vert^2 \overset{(b)}{\le} O(\epsilon^2)\\
 \label{bar_w_213}   
&\mathds{E}\{\Vert\boldsymbol{s}_{k,n}^B\Vert^2\vert\mathcal{F}_{n-1}\}\overset{(c)}{\le} O\left(\frac{1}{B}\right)\mathds{E}\Vert\widetilde{\boldsymbol{w}}_{k,n-1}\Vert^2 + O(\frac{1}{B})
    \end{align}
where $(a)$ follows from Lemma \ref{affine_l}, $(b)$ follows a similar reasoning to \eqref{size_e}, and $(c)$ follows from \eqref{sgn_2}. By substituting \eqref{bar_w_211}--\eqref{bar_w_213} into \eqref{bar_w_21}, we obtain
\begin{align}\label{bar_w_22}
\mathds{E}\Vert\bar{\boldsymbol{w}}_{n}\Vert^2 \le& \frac{1}{1-O(\mu)}\mathds{E}\Vert\bar{\boldsymbol{w}}_{n-1}\Vert^2 + O(\mu)\mathds{E}\Vert{\widetilde{\boldsymbol{\scriptstyle\mathcal{W}}}}_{n-1}\Vert^2 + O(\mu\epsilon^2)+ O\left(\frac{\mu^2}{B}\right)\mathds{E}\Vert{\widetilde{\boldsymbol{\scriptstyle\mathcal{W}}}}_{n-1}\Vert^2 + O\left(\frac{\mu^2}{B}\right)\notag\\
\overset{(a)}{=} &\left(\frac{1}{1-O(\mu)} + O(\mu)\right)\mathds{E}\Vert\bar{\boldsymbol{w}}_{n-1}\Vert^2 + O(\mu)\mathds{E}\Vert\check{\boldsymbol{w}}_{n-1}\Vert^2 + O(\mu\epsilon^2) + O\left(\frac{\mu^2}{B}\right)
\end{align}
where $(a)$ follows from \eqref{tw-bw-cw}.

We next establish bounds for $\mathds{E}\Vert\check{\boldsymbol{w}}_n\Vert^2$. Basically, recalling \eqref{checkw_re1} gives:
\begin{align}\label{check_w_21}
\mathds{E}\Vert\check{\boldsymbol{w}}_n\Vert^2 =& \mathds{E}\Vert\mathcal{P}_\alpha\check{\boldsymbol{w}}_{n-1} -  \mu\mathcal{V}_{\alpha}^{\sf T}\mathcal{A}_2\mathop{\rm col}\limits_k\left\{\nabla J_k(\boldsymbol{w}_{k,n-1}) - \nabla J_k(w^{\star})\right\} -  \mu\mathcal{V}_{\alpha}^{\sf T}\mathcal{A}_2d -\mu\mathcal{V}_{\alpha}^{\sf T}\mathcal{A}_2\boldsymbol{e}_{n-1}\Vert^2 \notag\\
&+ \mu^2\mathds{E}\Vert\mathcal{V}_{\alpha}^{\sf T}\mathcal{A}_2\boldsymbol{s}_{n}^B\Vert^2
\end{align}
for which we have
\begin{align}
\label{check_w_211}
&\mathds{E}\{\Vert\boldsymbol{s}_{n}^B\Vert^2 \vert\mathcal{F}_{n-1}\} = \sum\limits_k \mathds{E}\{\Vert\boldsymbol{s}_{k,n}^B(\boldsymbol{w}_{k,n-1})\Vert^2\vert\mathcal{F}_{n-1} \}\le O\left(\frac{1}{B}\right)\Vert {\widetilde{\boldsymbol{\scriptstyle\mathcal{W}}}}_{n} \Vert^2 + O\left(\frac{1}{B}\right)\\
\label{check_w_212}
&\Vert\boldsymbol{e}_{n-1}\Vert^2 = \sum\limits_k \Vert\boldsymbol{e}_{k,n-1}\Vert^2 \le O(\epsilon^2)
\end{align}

We substitute \eqref{check_w_211} and \eqref{check_w_212} into \eqref{check_w_21}, and apply the Jensen's inequality on it, then we have
\begin{align}\label{check_w_22}
\mathds{E}\Vert\check{\boldsymbol{w}}_n\Vert^2 {\le}& \rho(\mathcal{P}_\alpha)\mathds{E}\Vert\check{\boldsymbol{w}}_{n-1}\Vert^2 + O(\mu^2)\sum_k\mathds{E}\Vert\nabla J_k(\boldsymbol{w}_{k,n-1}) - \nabla J_k(w^{\star})\Vert^2 + O(\mu^2) \notag\\
& + O(\mu^2\epsilon^2) + O(\frac{\mu^2}{B})\mathds{E}\Vert{\widetilde{\boldsymbol{\scriptstyle\mathcal{W}}}}_{n-1}\Vert^2 + O(\frac{\mu^2}{B})\notag\\
\overset{(a)}{=}& (\rho(\mathcal{P}_\alpha) + O(\mu^2))\mathds{E}\Vert\check{\boldsymbol{w}}_{n-1}\Vert^2 + O(\mu^2)\mathds{E}\Vert\bar{\boldsymbol{w}}_{n-1}\Vert^2 + O(\mu^2)
\end{align}
where $\rho(\mathcal{P}_\alpha)$ is the spectral radius of $\mathcal{P}_\alpha$ which is smaller than 1 \cite{sayed2014adaptation}, $(a)$ follows from \eqref{bar_w_211} and \eqref{tw-bw-cw}, and the last constant term $O(\mu^2)$ arises from the network heterogeneity
which guarantees $d\neq 0$, otherwise this term reduces to $O(\mu^2\epsilon^2) + O(\frac{\mu^2}{B})$.

Combining \eqref{bar_w_22} and \eqref{check_w_22}, we have
\begin{align}\label{sum_tildew2}
    \mathds{E}\left[\begin{array}{c}
         \Vert\bar{\boldsymbol{w}}_{n}\Vert^2\\
          \Vert\check{\boldsymbol{w}}_{n}\Vert^2
    \end{array}\right]\le \left[\begin{array}{cc}
       \frac{1 + O(\mu)}{1-O(\mu)} & O(\mu) \\
         O(\mu^2)&  \rho(P_\alpha)+O(\mu^2)
    \end{array}\right]\left[\begin{array}{c}
         \mathds{E}\Vert\bar{\boldsymbol{w}}_{n-1}\Vert^2\\
         \mathds{E} \Vert\check{\boldsymbol{w}}_{n-1}\Vert^2
    \end{array}\right] + \left[\begin{array}{c}
         O(\frac{\mu^2}{B}) + O(\mu\epsilon^2)\\
         O(\mu^2)
    \end{array}\right]
\end{align}

Let  
\begin{equation}\label{d_gamma_1}
    \Gamma_1 = \left[\begin{array}{cc}
       \frac{1+O(\mu)}{1-O(\mu)} & O(\mu) \\
         O(\mu^2)&  \rho(P_\alpha)+O(\mu^2)
    \end{array}\right]
\end{equation}
then, iterating \eqref{sum_tildew2} gives
\begin{align}\label{sum_tildew2_2}
    \mathds{E}\left[\begin{array}{c}
         \Vert\bar{\boldsymbol{w}}_{n}\Vert^2\\
          \Vert\check{\boldsymbol{w}}_{n}\Vert^2
    \end{array}\right]&\le \Gamma_1^{n+1}\mathds{E}\left[\begin{array}{c}
         \Vert\bar{\boldsymbol{w}}_{-1}\Vert^2\\
          \Vert\check{\boldsymbol{w}}_{-1}\Vert^2
    \end{array}\right] + \sum\limits_{i = 0}^{n} \Gamma_1^{i}\left[\begin{array}{c}
         O(\frac{\mu^2}{B})\\
         O(\mu^2)
    \end{array}\right] \notag\\
    &= \Gamma_1^{n+1}\mathds{E}\left[\begin{array}{c}
         \Vert\bar{\boldsymbol{w}}_{-1}\Vert^2\\
          \Vert\check{\boldsymbol{w}}_{-1}\Vert^2
    \end{array}\right]  + (I - \Gamma_1)^{-1}(I-\Gamma_1^{n+1})\left[\begin{array}{c}
         O(\frac{\mu^2}{B}) + O(\mu\epsilon^2)\\
         O(\mu^2)
    \end{array}\right]
\end{align}
To proceed, we compute $(I - \Gamma_1)^{-1}(I-\Gamma_1^{n+1})$. Basically, as
\begin{align}
    1 - \frac{1 + O(\mu)}{1 - O(\mu)} = - O(\mu)
\end{align}
we get
\begin{align}\label{gamma_1_in}
    (I - \Gamma_1)^{-1} =& \left[\begin{array}{cc}
        1 - \frac{1+O(\mu)}{1-O(\mu)} & -O(\mu) \\
        -O(\mu^2) &  1 - \rho(P_\alpha)+O(\mu^2) + O(\mu^2)
    \end{array}\right]^{-1} \notag\\
    =& \left[\begin{array}{cc}
        -O(\mu) & -O(\mu) \\
        -O(\mu^2) &  O(1)
    \end{array}\right]^{-1} \notag\\
    =& \left[\begin{array}{cc}
        -O(\frac{1}{\mu}) & -O(1) \\
        -O(\mu) &  O(1)
    \end{array}\right]
\end{align}
Then, for the matrix power $\Gamma_1^{n+1}$, we have
\begin{align}\label{gamma1_p}
    \Gamma_1^{n+1} = \left[\begin{array}{cc}
      (\frac{1+O(\mu)}{1-O(\mu)})^{n+1}   & O(\mu) \\
        O(\mu^2) & \rho^{n+1}(P_\alpha)
    \end{array}\right] 
\end{align} 
Moreover, by resorting to Lemma 2 in \cite{VlaskiS21}, with $n\le O(\frac{1}{\mu})$, it holds that
\begin{align}\label{gamma1_p1}
   \left(\frac{1+O(\mu)}{1-O(\mu)}\right)^{n+1}  = O(1), \quad\quad 1 -  \left(\frac{1+O(\mu)}{1-O(\mu)}\right)^{n+1}  = -O(1)
\end{align}
In addition, according to assumption \ref{assump_ip}, we have
\begin{align}
\label{ip_bar}
  &\mathds{E}\Vert\bar{\boldsymbol{w}}_{-1}\Vert^2 = \mathds{E}\Vert \frac{1}{\sqrt{K}}\sum\limits_{k=1}^{K}\widetilde{\boldsymbol{w}}_{k,-1}\Vert^2  \le \sum\limits_{k=1}^{K}\mathds{E}\Vert\widetilde{\boldsymbol{w}}_{k,-1}\Vert^2 \le o(\frac{\mu}{B}) + O(\epsilon^{\frac{7}{4}}) \\
  \label{ip_check}
  &\mathds{E}\Vert\check{\boldsymbol{w}}_{-1}\Vert^2 =\mathds{E}\Vert\mathcal{V}^{\sf T}_{\alpha}{\widetilde{\boldsymbol{\scriptstyle\mathcal{W}}}}_{-1}\Vert^2 \le \Vert\mathcal{V}^{\sf T}_{\alpha}\Vert^2 \mathds{E}\Vert{\widetilde{\boldsymbol{\scriptstyle\mathcal{W}}}}_{-1}\Vert^2 \le o(\frac{\mu}{B}) + O(\epsilon^{\frac{7}{4}})
\end{align}
Substituting  \eqref{gamma_1_in}--\eqref{ip_check} into \eqref{sum_tildew2_2}, we obtain
\begin{align}\label{sum_tildew3}
     \mathds{E}\left[\begin{array}{c}
         \Vert\bar{\boldsymbol{w}}_{n}\Vert^2\\
          \Vert\check{\boldsymbol{w}}_{n}\Vert^2
    \end{array}\right] &\le \left[\begin{array}{c}
          o(\frac{\mu}{B}) + O(\epsilon^{\frac{7}{4}})\\
           o(\frac{\mu}{B}) + O(\epsilon^{\frac{7}{4}})
    \end{array}\right]  + \left[\begin{array}{cc}
        -O(\frac{1}{\mu}) & -O(1) \\
        -O(\mu) &  O(1)
    \end{array}\right]\notag\\
    &\quad\;\times\left[\begin{array}{cc}
     -O(1) & -O(\mu) \\
        -O(\mu^2) & O(1) 
    \end{array}\right]\left[\begin{array}{c}
         O(\frac{\mu^2}{B}) + O(\epsilon^2)\\
         O(\mu^2)
    \end{array}\right] \notag\\
    &\le \left[\begin{array}{c}
         O(\frac{\mu}{B})+O(\mu^2) + O(\epsilon^{\frac{7}{4}})\\
         O(\mu^2) + O(\epsilon^{\frac{7}{4}}) + o(\frac{\mu}{B})
    \end{array}\right]
\end{align}
according to which we have
\begin{align}\label{sum_tildew4}
\mathds{E}\Vert{\widetilde{\boldsymbol{\scriptstyle\mathcal{W}}}}_{n}\Vert^2 = \mathds{E}(\Vert\bar{\boldsymbol{w}}_{n}\Vert^2+\Vert\check{\boldsymbol{w}}_{n}\Vert^2)\le O(\frac{\mu}{B}) + O( \mu^2) + O(\epsilon^{\frac{7}{4}})
\end{align}
where the term $O\left(\frac{\mu}{B}\right)$ originates from the gradient noise, the term $O(\mu^2)$ is attributed to the network heterogeneity, and $O(\epsilon^{\frac{7}{4}})$ appears due to the adversarial perturbations.

Following the explanation after \eqref{check_w_22}, in the totally \textbf{homogeneous networks} where $d=0$, the bound of $\mathds{E}\Vert\check{\boldsymbol{w}}_{n}\Vert^2$ in \eqref{sum_tildew3} becomes:
\begin{align}
 \mathds{E}\Vert\check{\boldsymbol{w}}_{n}\Vert^2 \le O(\mu^2\epsilon^2) +  O\left(\frac{\mu^2}{B}\right) + O(\epsilon^{\frac{7}{4}}) + o\left(\frac{\mu}{B}\right) \le O(\epsilon^{\frac{7}{4}}) + o\left(\frac{\mu}{B}\right)
\end{align}
As a result, we have
\begin{align}
\mathds{E}\Vert{\widetilde{\boldsymbol{\scriptstyle\mathcal{W}}}}_{n}\Vert^2 \le O\left(\frac{\mu}{B}\right) + O\left(\epsilon^{\frac{7}{4}}\right)
\end{align}

 In addition, for the \textbf{centralized method} where $\check{\boldsymbol{w}}_n$ is $0$, the $O(\mu^2)$ term in \eqref{sum_tildew4} is eliminated. Consequently, we have
\begin{align}
\mathds{E}\Vert{\widetilde{\boldsymbol{\scriptstyle\mathcal{W}}}}_{n}\Vert^2  \le O\left(\frac{\mu}{B}\right) + O\left(\epsilon^{\frac{7}{4}}\right)
\end{align}
\end{proof}

\section{Proof for Lemma \ref{mse_4}}\label{p_mse_4}
\begin{proof}
    In this section, we derive bounds for $\mathds{E}\Vert{\widetilde{\boldsymbol{\scriptstyle\mathcal{W}}}}_{n}\Vert^4$.  First, we establish bounds for $\mathds{E}\Vert\check{\boldsymbol{w}}_{n}\Vert^4$. Basically, using \eqref{sgn_0} and \eqref{sum4_2}, we obtain the fourth-order moments of \eqref{checkw_re1} as follows:
    \begin{align}\label{check_w_4}
    \mathds{E}\Vert\check{\boldsymbol{w}}_{n}\Vert^4 \le& \mathds{E}\Vert\mathcal{P}_\alpha\check{\boldsymbol{w}}_{n-1} -  \mu\mathcal{V}_{\alpha}^{\sf T}\mathcal{A}_2\mathop{\rm col}\limits_k\left\{\nabla J_k(\boldsymbol{w}_{k,n-1}) - \nabla J_k(w^{\star})\right\} -  \mu\mathcal{V}_{\alpha}^{\sf T}\mathcal{A}_2d -\mu\mathcal{V}_{\alpha}^{\sf T}\mathcal{A}_2\boldsymbol{e}_{n-1}\Vert^4\notag\\
        &+ 3\mu^4\mathds{E}\Vert\mathcal{V}_{\alpha}^{\sf T}\mathcal{A}_2\boldsymbol{s}_{n}^B\Vert^4 + 8\mu^2\mathds{E}\Vert\mathcal{P}_\alpha\check{\boldsymbol{w}}_{n-1} -  \mu\mathcal{V}_{\alpha}^{\sf T}\mathcal{A}_2\mathop{\rm col}\limits_k\left\{\nabla J_k(\boldsymbol{w}_{k,n-1}) - \nabla J_k(w^{\star})\right\} \notag\\
        &-  \mu\mathcal{V}_{\alpha}^{\sf T}\mathcal{A}_2d -\mu\mathcal{V}_{\alpha}^{\sf T}\mathcal{A}_2\boldsymbol{e}_{n-1}\Vert^2\Vert\mathcal{V}_{\alpha}^{\sf T}\mathcal{A}_2\boldsymbol{s}_{n}^B\Vert^2
    \end{align}
We proceed by examining each of the three terms on the RHS of \eqref{check_w_4}. Basically, using Jensen's inequality on the first term, we have
    \begin{align}\label{check_w_4_11}
&\mathds{E}\Vert\mathcal{P}_\alpha\check{\boldsymbol{w}}_{n-1} -  \mu\mathcal{V}_{\alpha}^{\sf T}\mathcal{A}_2\mathop{\rm col}\limits_k\left\{\nabla J_k(\boldsymbol{w}_{k,n-1}) - \nabla J_k(w^{\star})\right\} -  \mu\mathcal{V}_{\alpha}^{\sf T}\mathcal{A}_2d -\mu\mathcal{V}_{\alpha}^{\sf T}\mathcal{A}_2\boldsymbol{e}_{n-1}\Vert^4\notag\\
&\le \rho(\mathcal{P}_{\alpha})\mathds{E}\Vert\check{\boldsymbol{w}}_{n-1}\Vert^4 + O(\mu^4)\mathds{E}\Vert\mathop{\rm col}\limits_k\left\{\nabla J_k(\boldsymbol{w}_{k,n-1}) - \nabla J_k(w^{\star})\right\}\Vert^4 + O(\mu^4)\mathds{E}\Vert\mathcal{V}_{\alpha}^{\sf T}\mathcal{A}_2d\Vert^4 \notag\\
&\quad\; + O(\mu^4)\mathds{E}\Vert\mathcal{V}_{\alpha}^{\sf T}\mathcal{A}_2\boldsymbol{e}_{n-1}\Vert^4\notag\\
&\overset{(a)}{\le} \rho(\mathcal{P}_{\alpha})\mathds{E}\Vert\check{\boldsymbol{w}}_{n-1}\Vert^4 + O(\mu^4)\mathds{E}\Vert{\widetilde{\boldsymbol{\scriptstyle\mathcal{W}}}}_{n-1}\Vert^4 + O(\mu^4\epsilon^2)\mathds{E}\Vert{\widetilde{\boldsymbol{\scriptstyle\mathcal{W}}}}_{n-1}\Vert^2 + O(\mu^4)\notag\\
&\overset{(b)}{\le} (\rho(\mathcal{P}_{\alpha}) + O(\mu^4))\mathds{E}\Vert\check{\boldsymbol{w}}_{n-1}\Vert^4 + O(\mu^4)\mathds{E}\Vert\bar{\boldsymbol{w}}_{n-1}\Vert^4  + O(\mu^4)
    \end{align}
where $(a)$ follows from \eqref{affine_l_eq_2} and \eqref{check_w_212}, namely,
\begin{align}\label{check_w_4_12}
    \Vert\mathop{\rm col}\limits_k\left\{\nabla J_k(\boldsymbol{w}_{k,n-1}) - \nabla J_k(w^{\star})\right\}\Vert^4 = &\left(\sum_k \Vert\nabla J_k(\boldsymbol{w}_{k,n-1}) - \nabla J_k(w^{\star})\Vert^2\right)^2 \notag\\
    \le & \left(2L^2\Vert{\widetilde{\boldsymbol{\scriptstyle\mathcal{W}}}}_{n-1}\Vert^2 + O(\epsilon^2)\right)^2\notag\\
    =& 4L^4\Vert{\widetilde{\boldsymbol{\scriptstyle\mathcal{W}}}}_{n-1}\Vert^4 +O(\epsilon^2)\Vert{\widetilde{\boldsymbol{\scriptstyle\mathcal{W}}}}_{n-1}\Vert^2 +O(\epsilon^4)
\end{align}
and, $(b)$ follows from \eqref{sum_tildew4} and the following inequality:
\begin{align}\label{check_w_4_13}
\left\Vert{\widetilde{\boldsymbol{\scriptstyle\mathcal{W}}}}_{n}\right\Vert^4 = (\Vert{\widetilde{\boldsymbol{\scriptstyle\mathcal{W}}}}_{n}\Vert^2)^2= (\Vert\bar{\boldsymbol{w}}_{n}\Vert^2+\Vert\check{\boldsymbol{w}}_{n}\Vert^2)^2 \le 2\Vert\bar{\boldsymbol{w}}_{n}\Vert^4 + 2\Vert\check{\boldsymbol{w}}_{n}\Vert^4
\end{align}

Then, for the second term on the RHS of \eqref{check_w_4} associated with the fourth-order moments of the gradient noise, we have
\begin{align}\label{check_w_4_21}
    3\mu^4\mathds{E}\Vert\mathcal{V}_{\alpha}^{\sf T}\mathcal{A}_2\boldsymbol{s}_{n}^B\Vert^4 \le O(\mu^4)\mathds{E}\Vert\boldsymbol{s}_{n}^B\Vert^4
\end{align}
Note that 
\begin{align}\label{check_w_4_22}
\mathds{E}\{\Vert\boldsymbol{s}_{n}^B\Vert^4\vert\mathcal{F}_{n-1}\} = & \mathds{E}\left\{\left(\sum_k\Vert\boldsymbol{s}_{k,n}^B\Vert^2\right)^2\bigg\vert\mathcal{F}_{n-1}\right\}\notag\\
\overset{(a)}{\le} &K\sum_k\mathds{E}\left\{\Vert\boldsymbol{s}_{k,n}^B\Vert^4\vert\mathcal{F}_{n-1}\right\} \notag\\
\overset{(b)}{\le}& O\left(\frac{1}{B^2}\right)\sum_k\Vert\widetilde{\boldsymbol{w}}_{k,n-1}\Vert^4 + O\left(\frac{1}{B^2}\right)\notag\\
\le & O\left(\frac{1}{B^2}\right)\Vert{\widetilde{\boldsymbol{\scriptstyle\mathcal{W}}}}_{n-1}\Vert^4 + O\left(\frac{1}{B^2}\right)
\end{align}
where $(a)$ follows from Jensen's inequality, and $(b)$ follows from \eqref{sgn_4}.
Substituting \eqref{check_w_4_22} and \eqref{check_w_4_13} into \eqref{check_w_4_21}, we obtain
\begin{align}
 3\mu^4\mathds{E}\Vert\mathcal{V}_{\alpha}^{\sf T}\mathcal{A}_2\boldsymbol{s}_{n}^B\Vert^4 \le    O\left(\frac{\mu^4}{B^2}\right)\mathds{E}\Vert\bar{\boldsymbol{w}}_{n-1}\Vert^4 + \left(\frac{\mu^4}{B^2}\right)\mathds{E}\Vert\check{\boldsymbol{w}}_{n-1}\Vert^4 + O\left(\frac{\mu^4}{B^2}\right) 
\end{align}

Next, for the cross term in \eqref{check_w_4}, we have
\begin{align}\label{check_w_4_3}
&8\mu^2\mathds{E}\Vert\mathcal{P}_\alpha\check{\boldsymbol{w}}_{n-1} -  \mu\mathcal{V}_{\alpha}^{\sf T}\mathcal{A}_2\mathop{\rm col}\limits_k\left\{\nabla J_k(\boldsymbol{w}_{k,n-1}) - \nabla J_k(w^{\star})\right\} -  \mu\mathcal{V}_{\alpha}^{\sf T}\mathcal{A}_2d -\mu\mathcal{V}_{\alpha}^{\sf T}\mathcal{A}_2\boldsymbol{e}_{n-1}\Vert^2\Vert\mathcal{V}_{\alpha}^{\sf T}\mathcal{A}_2\boldsymbol{s}_{n}^B\Vert^2 \notag\\
&\overset{(a)}{\le} O\left(\frac{\mu^2}{B}\right)\mathds{E}\bigg\{\Big(\rho(\mathcal{P}_{\alpha})\Vert\check{\boldsymbol{w}}_{n-1}\Vert^2 + O(\mu^2)\Vert\mathop{\rm col}\limits_k\left\{\nabla J_k(\boldsymbol{w}_{k,n-1}) - \nabla J_k(w^{\star})\right\}\Vert^2 \notag\\
&\quad\; + O(\mu^2)\Vert\mathcal{V}_{\alpha}^{\sf T}\mathcal{A}_2d\Vert^2 + O(\mu^2)\Vert\mathcal{V}_{\alpha}^{\sf T}\mathcal{A}_2\boldsymbol{e}_{n-1}\Vert^2\Big)\Big(O(1)\Vert\bar{\boldsymbol{w}}_{n-1}\Vert^2 + O(1)\Vert\check{\boldsymbol{w}}_{n-1}\Vert^2 + O(1) \Big)\bigg\}\notag\\
&\overset{(b)}{\le}  O\left(\frac{\mu^2}{B}\right)\mathds{E}\bigg\{ \Big(\rho(\mathcal{P}_{\alpha})\Vert\check{\boldsymbol{w}}_{n-1}\Vert^2 + O(\mu^2)\Vert\bar{\boldsymbol{w}}_{n-1}\Vert^2 + O(\mu^2)\Vert\check{\boldsymbol{w}}_{n-1}\Vert^2 + O(\mu^2)\Big)\notag\\
&\quad \times\Big(O(1)\Vert\bar{\boldsymbol{w}}_{n-1}\Vert^2 + O(1)\Vert\check{\boldsymbol{w}}_{n-1}\Vert^2 + O(1)\Big)\bigg\}\notag\\
&= O\left(\frac{\mu^2}{B}\right)\mathds{E}\Vert\check{\boldsymbol{w}}_{n-1}\Vert^4 + O\left(\frac{\mu^2}{B}\right)\mathds{E}\Vert\bar{\boldsymbol{w}}_{n-1}\Vert^4 + O\left(\frac{\mu^2}{B}\right)\mathds{E}\Vert\check{\boldsymbol{w}}_{n-1}\Vert^2 + O\left(\frac{\mu^4}{B}\right)\mathds{E}\Vert\bar{\boldsymbol{w}}_{n-1}\Vert^2 +  O\left(\frac{\mu^4}{B}\right)\notag\\
&\overset{(c)}{\le}  O\left(\frac{\mu^2}{B}\right)\mathds{E}\Vert\check{\boldsymbol{w}}_{n-1}\Vert^4 + O\left(\frac{\mu^2}{B}\right)\mathds{E}\Vert\bar{\boldsymbol{w}}_{n-1}\Vert^4 + O\left(\frac{\mu^4}{B}\right) + O\left(\frac{\mu^2}{B}\epsilon^{\frac{7}{4}}\right) + o\left(\frac{\mu^3}{B^2}\right)
\end{align}
where $(a)$ follows from Jensen's inequality, \eqref{check_w_211} and \eqref{tw-bw-cw}, $(b)$ follows from \eqref{bar_w_211}, \eqref{check_w_212} and  \eqref{tw-bw-cw}, and $(c)$ follows from \eqref{sum_tildew3}.

Substituting \eqref{rmuB}, \eqref{check_w_4_11}, \eqref{check_w_4_21}, and \eqref{check_w_4_3} into \eqref{check_w_4}, and grouping terms, we obtain
\begin{align}\label{check_w_4_f}
\mathds{E}\Vert\check{\boldsymbol{w}}_{n}\Vert^4 \le \left(\rho(\mathcal{P}_{\alpha}) + O\left(\mu^3\right)\right)\mathds{E}\Vert\check{\boldsymbol{w}}_{n-1}\Vert^4 + O\left({\mu^3}\right)\mathds{E}\Vert\bar{\boldsymbol{w}}_{n-1}\Vert^4  + O(\mu^4) + O(\mu^3\epsilon^{\frac{7}{4}})
\end{align}

We next analyze $\mathds{E}\Vert\bar{\boldsymbol{w}}_{n}\Vert^4$. Using \eqref{sgn_0} and \eqref{sum4_2}, we obtain the fourth-order moment of \eqref{barw_re1} as follows:
\begin{align}\label{bar_w_4}
\mathds{E}\Vert\bar{\boldsymbol{w}}_{n}\Vert^4 \le &\mathds{E}\left\Vert\bar{\boldsymbol{w}}_{n-1} - \frac{\mu}{\sqrt{K}}\sum_k(\nabla J_k(\boldsymbol{w}_{k,n-1}) - \nabla J_k(w^{\star})) - \frac{\mu}{\sqrt{K}}\sum_k \boldsymbol{e}_{k,n-1} \right\Vert^4 \notag\\
& + {3\mu^4}\mathds{E}\left\Vert\frac{1}{\sqrt{K}}\sum_k\boldsymbol{s}_{k,n}^B(\boldsymbol{w}_{k,n-1})\right\Vert^4  + 8\mu^2 \mathds{E}\Bigg\Vert\bar{\boldsymbol{w}}_{n-1} - \frac{\mu}{\sqrt{K}}\sum_k(\nabla J_k(\boldsymbol{w}_{k,n-1}) - \nabla J_k(w^{\star}))\notag\\
& - \frac{\mu}{\sqrt{K}}\sum_k \boldsymbol{e}_{k,n-1} \Bigg\Vert^2\left\Vert\frac{1}{\sqrt{K}}\sum_k \boldsymbol{s}_{k,n}^B(\boldsymbol{w}_{k,n-1})\right\Vert^2
\end{align}
Again, we examine the three terms on the RHS of \eqref{bar_w_4} separately. Basically, considering $t = 1 - O(\mu)$, and using Jensen's inequality on the first term, we obtain
\begin{align}\label{bar_w_4_11}
    &\mathds{E}\left\Vert\bar{\boldsymbol{w}}_{n-1} - \frac{\mu}{\sqrt{K}}\sum_k(\nabla J_k(\boldsymbol{w}_{k,n-1}) - \nabla J_k(w^{\star})) - \frac{\mu}{\sqrt{K}}\sum_k \boldsymbol{e}_{k,n-1} \right\Vert^4\notag\\
    & \le \frac{1}{t^3}\mathds{E}\Vert\bar{\boldsymbol{w}}_{n-1}\Vert^4 + O(\mu)\mathds{E}\left\Vert\frac{1}{\sqrt{K}}\sum_k(\nabla J_k(\boldsymbol{w}_{k,n-1}) - \nabla J_k(w^{\star}))\right\Vert^4 +  O(\mu)\mathds{E}\left\Vert\frac{1}{\sqrt{K}}\sum_k \boldsymbol{e}_{k,n-1} \right\Vert^4\notag\\
\end{align}
which can be bounded using the following inequalities:
\begin{align}\label{bar_w_4_12}
\left\Vert\frac{1}{\sqrt{K}}\sum_k(\nabla J_k(\boldsymbol{w}_{k,n-1}) - \nabla J_k(w^{\star}))\right\Vert^4 \le& K\sum_k \Vert\nabla J_k(\boldsymbol{w}_{k,n-1}) - \nabla J_k(w^{\star})\Vert^4\notag\\
\le &K\left(\sum_k \Vert\nabla J_k(\boldsymbol{w}_{k,n-1}) - \nabla J_k(w^{\star})\Vert^2\right)^2\notag\\
\overset{(a)}{\le} & 4L^4K\Vert{\widetilde{\boldsymbol{\scriptstyle\mathcal{W}}}}_{n-1}\Vert^4+ O(\epsilon^2)\Vert{\widetilde{\boldsymbol{\scriptstyle\mathcal{W}}}}_{n-1}\Vert^2 + O(\epsilon^4)
\end{align}
where $(a)$ follows from \eqref{check_w_4_12}. In addition,
\begin{align}\label{bar_w_4_13}
    \left\Vert\frac{1}{\sqrt{K}}\sum_k \boldsymbol{e}_{k,n-1} \right\Vert^4 \le K\sum_k\Vert\boldsymbol{e}_{k,n-1}\Vert^4 \overset{(a)}{\le} O(\epsilon^4)
\end{align}
where $(a)$ follows from \eqref{check_w_212}.

Substituting \eqref{bar_w_4_12} and \eqref{bar_w_4_13} into \eqref{bar_w_4_11}, we obtain
\begin{align}\label{bar_w_4_14}
    &\mathds{E}\left\Vert\bar{\boldsymbol{w}}_{n-1} - \frac{\mu}{\sqrt{K}}\sum_k(\nabla J_k(\boldsymbol{w}_{k,n-1}) - \nabla J_k(w^{\star})) - \frac{\mu}{\sqrt{K}}\sum_k \boldsymbol{e}_{k,n-1} \right\Vert^4\notag\\
    & \le \frac{1}{(1-O(\mu))^3}\mathds{E}\left\Vert\bar{\boldsymbol{w}}_{n-1}\right\Vert^4 + O(\mu)\mathds{E}\Vert{\widetilde{\boldsymbol{\scriptstyle\mathcal{W}}}}_{n-1}\Vert^4 + O(\mu\epsilon^2)\mathds{E}\Vert{\widetilde{\boldsymbol{\scriptstyle\mathcal{W}}}}_{n-1}\Vert^2 + O(\mu\epsilon^4)\notag\\
    & \overset{(a)}{\le} \frac{1}{(1-O(\mu))^3}\mathds{E}\left\Vert\bar{\boldsymbol{w}}_{n-1}\right\Vert^4 + O(\mu)\mathds{E}\Vert{\widetilde{\boldsymbol{\scriptstyle\mathcal{W}}}}_{n-1}\Vert^4 + O\left(\frac{\mu^2\epsilon^2}{B}\right) + O\left(\mu^3\epsilon^2\right) + O(\mu\epsilon^{\frac{15}{4}})
\end{align}
where $(a)$ follows from \eqref{sum_tildew4}.

Then, using \eqref{check_w_4_22}, the second term on the RHS of \eqref{bar_w_4}, which involves the fourth-order moment of the gradient noise, can be bounded as follows:
\begin{align}\label{bar_w_4_2}
 {3\mu^4}\mathds{E}\left\{\left\Vert\frac{1}{\sqrt{K}}\sum_k\boldsymbol{s}_{k,n}^B(\boldsymbol{w}_{k,n-1})\right\Vert^4\bigg\vert\mathcal{F}_{n-1}\right\}\le &O(\mu^4)\sum_k\mathds{E}\left\{\Vert\boldsymbol{s}_{k,n}^B(\boldsymbol{w}_{k,n-1})\Vert^4\vert\mathcal{F}_{n-1}\right\}\notag\\
\le& O\left(\frac{\mu^4}{B^2}\right)\Vert{\widetilde{\boldsymbol{\scriptstyle\mathcal{W}}}}_{n}\Vert^4 + O\left(\frac{\mu^4}{B^2}\right)
\end{align}

As for the cross term in \eqref{bar_w_4}, we have
\begin{align}\label{bar_w_4_3}
&8\mu^2 \mathds{E}\Bigg\Vert\bar{\boldsymbol{w}}_{n-1} - \frac{\mu}{\sqrt{K}}\sum_k(\nabla J_k(\boldsymbol{w}_{k,n-1}) - \nabla J_k(w^{\star}))- \frac{\mu}{\sqrt{K}}\sum_k \boldsymbol{e}_{k,n-1} \Bigg\Vert^2\left\Vert\frac{1}{\sqrt{K}}\sum_k \boldsymbol{s}_{k,n}^B(\boldsymbol{w}_{k,n-1})\right\Vert^2\notag\\
&\overset{(a)}{\le} O\left(\frac{\mu^2}{B}\right)\mathds{E}\Bigg\{\Bigg(\frac{1}{1-O(\mu)}\Vert\bar{\boldsymbol{w}}_{n-1}\Vert^2 + O(\mu)\left\Vert\frac{1}{\sqrt{K}}\sum_k(\nabla J_k(\boldsymbol{w}_{k,n-1}) - \nabla J_k(w^{\star}))\right\Vert^2 + O(\mu\epsilon^2)\Bigg)\notag\\
&\quad\;\times\Big(O(1)\Vert\bar{\boldsymbol{w}}_{n-1}\Vert^2+ O(1)\Vert\check{\boldsymbol{w}}_{n-1}\Vert^2 + O(1) \Big)\Bigg\}\notag\\
&\overset{(b)}{\le} O\left(\frac{\mu^2}{B}\right)\mathds{E}\Big\{\big(O(1)\Vert\bar{\boldsymbol{w}}_{n-1}\Vert^2 + O(\mu)\Vert\check{\boldsymbol{w}}_{n-1}\Vert^2 + O(\mu\epsilon^2)\big)\big(O(1)\Vert\bar{\boldsymbol{w}}_{n-1}\Vert^2+ O(1)\Vert\check{\boldsymbol{w}}_{n-1}\Vert^2 + O(1) \big)\Big\}\notag\\
&\le O\left(\frac{\mu^2}{B}\right)\mathds{E}\Big\{O(1)\Vert\bar{\boldsymbol{w}}_{n-1}\Vert^4 + O(1)\Vert\check{\boldsymbol{w}}_{n-1}\Vert^4 + O(1)\Vert\bar{\boldsymbol{w}}_{n-1}\Vert^2+ O(\mu)\Vert\check{\boldsymbol{w}}_{n-1}\Vert^2+ O(\mu\epsilon^2)\Big\}\notag\\
&\overset{(c)}{\le}O\left(\frac{\mu^2}{B}\right)\mathds{E}\Vert\bar{\boldsymbol{w}}_{n-1}\Vert^4 + O\left(\frac{\mu^2}{B}\right)\mathds{E}\Vert\check{\boldsymbol{w}}_{n-1}\Vert^4 + O\left(\frac{\mu^3}{B^2}\right) + O\left(\frac{\mu^4}{B}\right) + O\left(\frac{\mu^2}{B}\epsilon^{\frac{7}{4}}\right)
\end{align}
where $(a)$ follows from Jensen's inequality, \eqref{check_w_212}, \eqref{check_w_211}, and \eqref{tw-bw-cw}, $(b)$ follows from \eqref{bar_w_211}, and $(c)$ follows from \eqref{sum_tildew3}.

Substituting \eqref{rmuB}, \eqref{bar_w_4_14}, \eqref{bar_w_4_2}, and \eqref{bar_w_4_3} into \eqref{bar_w_4}, and grouping terms, we obtain
\begin{align}\label{bar_w_4_f}
\mathds{E}\Vert\bar{\boldsymbol{w}}_{n}\Vert^4 \le &\left(\frac{1}{(1 - O(\mu))^3}+ O(\mu)\right)\mathds{E}\Vert\bar{\boldsymbol{w}}_{n-1}\Vert^4  + O(\mu)\mathds{E}\Vert\check{\boldsymbol{w}}_{n-1}\Vert^4 + O(\mu^5)+ O\left(\mu^3\epsilon^{\frac{7}{4}}\right) + O\left(\mu\epsilon^{\frac{15}{4}}\right)
\end{align}

Combining \eqref{check_w_4_f} and \eqref{bar_w_4_f} into a single compact recursive inequality, we obtain
\begin{align}\label{tilde_w_4_1}
\mathds{E}\left[\begin{array}{c}
         \Vert\bar{\boldsymbol{w}}_{n}\Vert^4\\
          \Vert\check{\boldsymbol{w}}_{n}\Vert^4
    \end{array}\right]\le& \left[\begin{array}{cc}
       \frac{1}{(1-O(\mu))^3} + O(\mu) & O(\mu) \\
         O(\mu^3)&  \rho(P_\alpha)+O(\mu^3)
    \end{array}\right]\mathds{E}\left[\begin{array}{c}\Vert\bar{\boldsymbol{w}}_{n-1}\Vert^4\\
         \Vert\check{\boldsymbol{w}}_{n-1}\Vert^4
    \end{array}\right] \notag\\
    &+ \left[\begin{array}{c}
         O(\mu^5) + O\left(\mu^3\epsilon^{\frac{7}{4}}\right) + O\left(\mu\epsilon^{\frac{15}{4}}\right)\\
         O(\mu^4) + O(\mu^3\epsilon^{\frac{7}{4}})
    \end{array}\right]
\end{align}
To proceed, we define the coefficient matrix $\Gamma_2$:
\begin{align}\label{gamma_tilde_w4}
    \Gamma_2 \overset{\Delta}{=} \left[\begin{array}{cc}
       \frac{1}{(1-O(\mu))^3} + O(\mu)& O(\mu) \\
         O(\mu^3)&  \rho(P_\alpha)+O(\mu^3)
    \end{array}\right] = \left[\begin{array}{cc}
       \frac{1+O(\mu)}{(1-O(\mu))^3} & O(\mu) \\
         O(\mu^3)&  \rho(P_\alpha)+O(\mu^3)
    \end{array}\right]
\end{align}
with which and iterating \eqref{tilde_w_4_1}, we obtain
\begin{align}\label{tilde_w_4_2}
\mathds{E}\left[\begin{array}{c}
         \Vert\bar{\boldsymbol{w}}_{n}\Vert^4\\
          \Vert\check{\boldsymbol{w}}_{n}\Vert^4
    \end{array}\right]\le \Gamma_2^{n+1}\mathds{E}\left[\begin{array}{c}
         \Vert\bar{\boldsymbol{w}}_{-1}\Vert^4\\
          \Vert\check{\boldsymbol{w}}_{-1}\Vert^4
    \end{array}\right] + (I-\Gamma_2)^{-1}(I - \Gamma_2^{n+1})\left[\begin{array}{c}
          O(\mu^5) + O\left(\mu^3\epsilon^{\frac{7}{4}}\right) + O\left(\mu\epsilon^\frac{15}{4}\right)\\
         O(\mu^4) + O(\mu^3\epsilon^{\frac{7}{4}})
    \end{array}\right]
\end{align}
Note that
\begin{align}\label{gamma2_inverse}
 (I-\Gamma_2)^{-1} =&  \left[\begin{array}{cc}
       1 - \frac{1+O(\mu)}{(1-O(\mu))^3} & -O(\mu) \\
         -O(\mu^3)&  1 - \rho(P_\alpha) - O(\mu^3)
    \end{array}\right]^{-1} \notag\\
    =& \left[\begin{array}{cc}
       -O(\mu) & -O(\mu) \\
         -O(\mu^3)&  O(1)
    \end{array}\right]^{-1} \notag\\
    =&  \left[\begin{array}{cc}
       -O(\frac{1}{\mu}) & -O(1) \\
        -O(\mu^2)&  O(1)
    \end{array}\right]
\end{align}
\begin{align}\label{gamma2_n}
      \Gamma_2^{n+1}=& \left[\begin{array}{cc}
       \frac{1+O(\mu)}{(1-O(\mu))^3} & O(\mu) \\
         O(\mu^3)&   \rho(P_\alpha)+O(\mu^3)
    \end{array}\right]^{n+1}\notag\\
    =& \left[\begin{array}{cc}
       (\frac{(1+\mu L)^4}{(1-O(\mu))^3})^{n+1} +o(\mu) & O(\mu) \\
         O(\mu^3)&  \rho(P_\alpha)^{n+1} + o(\mu) 
    \end{array}\right] 
\end{align}
Similar to \eqref{gamma1_p1}, we resort to Lemma 2 in \cite{VlaskiS21}. Accordingly, for $n\le O\left(\frac{1}{\mu}\right)$, it holds that
\begin{align}\label{gamma2_p2}
    \left(\frac{1+O(\mu)}{(1-O(\mu))^3}\right)^{n+1}   = O(1), \quad\quad 1- \left(\frac{(1+O(\mu)}{(1-O(\mu))^3}\right)^{n+1}   = -O(1)
\end{align}
so that
\begin{align}\label{gamma2_n2}
    I - \Gamma_2^{n+1}  =  \left[\begin{array}{cc}
       -O(1) & -O(\mu) \\
         -O(\mu^3)&  O(1) 
    \end{array}\right]
\end{align}
\begin{align}\label{gamma2_n3}
    \Gamma_2^{n+1}= \left[\begin{array}{cc}
       O(1) & O(\mu) \\
         O(\mu^3)&  O(1)
    \end{array}\right] 
\end{align}

Substituting \eqref{assump_ip_eq_2}, \eqref{rmuB}, \eqref{gamma2_inverse}, \eqref{gamma2_n2}, and \eqref{gamma2_n3} into \eqref{tilde_w_4_2}, we get
\begin{align}\label{tilde_w_4_3}
    \mathds{E}\left[\begin{array}{c}
         \Vert\bar{\boldsymbol{w}}_{n}\Vert^4\\
          \Vert\check{\boldsymbol{w}}_{n}\Vert^4
    \end{array}\right]\le& o(\frac{\mu^2}{B^2}) + O(\epsilon^{\frac{7}{2}}) + \left[\begin{array}{cc}
       -O(\frac{1}{\mu}) & -O(1) \\
         -O(\mu^2)&  O(1)
    \end{array}\right]\left[\begin{array}{cc}
       -O(1) & -O(\mu) \\
        - O(\mu^3)&  O(1) 
    \end{array}\right]\notag\\
    &\times\left[\begin{array}{c}
         O(\mu^5) + O(\mu^3\epsilon^{\frac{7}{4}}) + O(\mu\epsilon^{\frac{15}{4}})\\
         O(\mu^4) + O(\mu^3\epsilon^{\frac{7}{4}})
    \end{array}\right] \notag\\
     \le& o(\mu^4) + O(\epsilon^{\frac{7}{2}}) + \left[\begin{array}{c}
         O(\mu^4) + O(\mu^2\epsilon^{\frac{7}{4}}) + O(\epsilon^{\frac{15}{4}})\\
         O(\mu^4) + O(\mu^3\epsilon^{\frac{7}{4}})
    \end{array}\right]
\end{align}
with which and \eqref{check_w_4_13}, we have
\begin{align}\label{tilde_w_4_4}
\mathds{E}\Vert{\widetilde{\boldsymbol{\scriptstyle\mathcal{W}}}}_{n}\Vert^4 \le 2\mathds{E}\left(\Vert\bar{\boldsymbol{w}}_{n}\Vert^4 + \Vert\check{\boldsymbol{w}}_{n}\Vert^4\right) \le O(\mu^4) + O(\mu^2\epsilon^{\frac{7}{4}}) + O(\epsilon^{\frac{7}{2}}) \overset{(a)}{\le} O(\mu^4) + O(\epsilon^{\frac{7}{2}})
\end{align}
where $(a)$ follows from the following inequality:
\begin{align}
    \mu^2\epsilon^{\frac{7}{4}} \le \frac{1}{2}\mu^4+ \frac{1}{2}\epsilon^\frac{7}{2}
\end{align}
Also, comparing \eqref{tilde_w_4_4} and \eqref{sum_tildew4}, we observe that
\begin{align}\label{c_tildew2_4}
\mathds{E}\Vert{\widetilde{\boldsymbol{\scriptstyle\mathcal{W}}}}_{n}\Vert^2 = O((\mathds{E}\Vert{\widetilde{\boldsymbol{\scriptstyle\mathcal{W}}}}_{n}\Vert^4)^{\frac{1}{2}})
\end{align}
Moreover, by using the Jensen's inequality, we have
 \begin{align}
\mathds{E}\Vert{\widetilde{\boldsymbol{\scriptstyle\mathcal{W}}}}_{n}\Vert^3 = \mathds{E}(\Vert{\widetilde{\boldsymbol{\scriptstyle\mathcal{W}}}}_{n}\Vert^4)^{\frac{3}{4}} \le (\mathds{E}\Vert{\widetilde{\boldsymbol{\scriptstyle\mathcal{W}}}}_{n}\Vert^4)^{\frac{3}{4}} = O\left((\mathds{E}\Vert{\widetilde{\boldsymbol{\scriptstyle\mathcal{W}}}}_{n}\Vert^2)^{1.5}\right)
 \end{align}
\end{proof}
\section{Proof for Lemma \ref{mse_ae}}\label{p_mse_ae}
\begin{proof}
We now analyze the approximation error introduced by the short-term model. The proof techniques are inspired by Appendix E of  \cite{cao2024trade}, with our focus shifted to the adversarial environment, where it is essential to analyze the impact of adversarial perturbations. Recalling the true recursion in \eqref{unified_re_2}:
\begin{align}\label{unified_re_2_2}
\widetilde{\boldsymbol{\scriptstyle\mathcal{W}}}_{n} = \mathcal{A}_2\left(\mathcal{A}_1 - \mu\mathcal{H}_{n-1}\right)\widetilde{\boldsymbol{\scriptstyle\mathcal{W}}}_{n-1} -  \mu\mathcal{A}_2d - \mu\mathcal{A}_2 \boldsymbol{e}_{n-1} - \mu\mathcal{A}_2\boldsymbol{s}_{n}^{B}
\end{align}
and the approximate model in \eqref{approx_unified_re_2}:
\begin{align}\label{approx_unified_re_2_2}
{\widetilde{\boldsymbol{\scriptstyle\mathcal{W}}}}_{n}' = \mathcal{A}_2\left(\mathcal{A}_1 - \mu\mathcal{H}\right)\widetilde{\boldsymbol{\scriptstyle\mathcal{W}}}_{n-1}' -  \mu\mathcal{A}_2d - \mu\mathcal{A}_2\boldsymbol{s}_{n}^{B}
\end{align}
where ${\widetilde{\boldsymbol{\scriptstyle\mathcal{W}}}}_{-1}' = {\widetilde{\boldsymbol{\scriptstyle\mathcal{W}}}}_{-1}$. We introduce the error vector which quantifies the difference between ${\widetilde{\boldsymbol{\scriptstyle\mathcal{W}}}}_{n}$ and ${\widetilde{\boldsymbol{\scriptstyle\mathcal{W}}}}_{n}'$:
\begin{align}\label{zn_1}
    \boldsymbol{z}_n \overset{\Delta}{=} & {\widetilde{\boldsymbol{\scriptstyle\mathcal{W}}}}_{n}' - {\widetilde{\boldsymbol{\scriptstyle\mathcal{W}}}}_{n}\notag\\
    =&  \mathcal{A}_2\left(\mathcal{A}_1 - \mu\mathcal{H}\right)(\widetilde{\boldsymbol{\scriptstyle\mathcal{W}}}_{n-1}'- \widetilde{\boldsymbol{\scriptstyle\mathcal{W}}}_{n-1}) + \mu \mathcal{A}_2(\mathcal{H}_{n-1} - \mathcal{H})\widetilde{\boldsymbol{\scriptstyle\mathcal{W}}}_{n-1} + \mu\mathcal{A}_2\boldsymbol{e}_{n-1}\notag\\
    =& \mathcal{V}(\mathcal{P} - \mu \mathcal{V}^{\sf T}\mathcal{A}_2\mathcal{H}\mathcal{V})\mathcal{V}^{\sf T}\boldsymbol{z}_{n-1}+\mu\mathcal{A}_2(\mathcal{H}_{n-1} - \mathcal{H})\widetilde{\boldsymbol{\scriptstyle\mathcal{W}}}_{n-1} + \mu\mathcal{A}_2\boldsymbol{e}_{n-1}
\end{align}

Similar to \eqref{unified_decomp}, multiplying both sides of \eqref{zn_1} by $\mathcal{V}^{\sf T}$, and using \eqref{vpv1}, we obtain
\begin{align}
         \mathcal{V}^{\sf T}{{\boldsymbol{{z}}}}_{n} {=} &\left[\begin{array}{c}
         \frac{1}{\sqrt{K}}\mathds{1}^{\sf  T}\boldsymbol{z}_n \\
         \mathcal{V}_\alpha^{\sf T}\boldsymbol{z}_n 
    \end{array}\right] \overset{\Delta}{=} \left[\begin{array}{c}
         \bar{\boldsymbol{z}}_n  \\
          \check{\boldsymbol{z}}_n
    \end{array}\right] \notag\\
    =&  \left(\left[\begin{array}{cc}
        I_M & 0 \\
        0 & \mathcal{P}_\alpha
    \end{array}\right] - \mu\left[\begin{array}{c}\frac{1}{\sqrt{K}}\mathds{1}^{\sf  T}\\
    \mathcal{V}_\alpha^{\sf T}
    \end{array}\right]\mathcal{A}_2\mathcal{H}\left[\frac{1}{\sqrt{K}}\mathds{1} \quad \mathcal{V}_\alpha\right]\right) \left[\begin{array}{c}
    \bar{\boldsymbol{z}}_{n-1}  \\
          \check{\boldsymbol{z}}_{n-1}
    \end{array}\right] \notag\\
    & + \mu \mathcal{V}^{\sf T}\mathcal{A}_2(\mathcal{H}_{n-1} - \mathcal{H})\widetilde{\boldsymbol{\scriptstyle\mathcal{W}}}_{n-1} + \mu \mathcal{V}^{\sf T}\mathcal{A}_2\boldsymbol{e}_{n-1}
\end{align}
Consequently, using \eqref{glbhsws}, we have
\begin{align}
\label{barz}
    \bar{\boldsymbol{z}}_n &= (I_M - \mu\bar{H})\bar{\boldsymbol{z}}_{n-1} - \frac{\mu}{\sqrt{K}}\mathds{1}^{\sf  T}\mathcal{H}\mathcal{V}_\alpha\check{\boldsymbol{z}}_{n-1}+ \frac{\mu}{\sqrt{K}}\mathds{1}^{\sf  T}(\mathcal{H}_{n-1} - \mathcal{H})\widetilde{\boldsymbol{\scriptstyle\mathcal{W}}}_{n-1} + \frac{\mu}{\sqrt{K}}\mathds{1}^{\sf  T}\boldsymbol{e}_{n-1}\\
    \label{checkz}
    \check{\boldsymbol{z}}_{n} &= (\mathcal{P}_{\alpha} - \mu\mathcal{V}_{\alpha}^{\sf T}\mathcal{A}_2\mathcal{H}\mathcal{V}_\alpha)\check{\boldsymbol{z}}_{n-1} - \frac{\mu}{\sqrt{K}}\mathcal{V}_{\alpha}^{\sf T}\mathcal{A}_2\mathcal{H}\mathds{1}\bar{\boldsymbol{z}}_{n-1} + \mu\mathcal{V}_{\alpha}^{\sf T}\mathcal{A}_2(\mathcal{H}_{n-1} - \mathcal{H})\widetilde{\boldsymbol{\scriptstyle\mathcal{W}}}_{n-1} + \mu\mathcal{V}_{\alpha}^{\sf T}\mathcal{A}_2\boldsymbol{e}_{n-1}
\end{align}

We now establish bounds for $\mathds{E}\Vert\bar{\boldsymbol{z}}_n\Vert^2$. Given that $w^{\star}$ is a local minimizer of $J(w)$, with $\bar{H}$ representing  the associated Hessian matrix, we leverage the local strong convexity of $J(w)$ at local minima. Basically, we have
\begin{align}
0< \Vert I - \mu\bar{H}\Vert = 1 - O(\mu) < 1
\end{align}
Accordingly, letting $t =\Vert I - \mu\bar{H}\Vert$,  and using Jensen's inequality on \eqref{barz}, we have
\begin{align}\label{barz1}
    \mathds{E}\Vert \bar{\boldsymbol{z}}_{n}\Vert^2 = &\mathds{E}\Vert (I_M - \mu \bar{H})\bar{\boldsymbol{z}}_{n-1} - \frac{\mu}{\sqrt{K}}\mathds{1}^{\sf T}\mathcal{H}\mathcal{V}_{\alpha}\check{\boldsymbol{z}}_{n-1} + \frac{\mu}{\sqrt{K}}\mathds{1}^{\sf  T}(\mathcal{H}_{n-1} - \mathcal{H})\widetilde{\boldsymbol{\scriptstyle\mathcal{W}}}_{n-1} + \frac{\mu}{\sqrt{K}}\mathds{1}^{\sf  T}\boldsymbol{e}_{n-1}\Vert^2 \notag\\
{\le} & t\mathds{E}\Vert\bar{\boldsymbol{z}}_{n-1} \Vert^2+ \frac{\mu^2}{K(1-t)}\mathds{E}\Vert\mathds{1}^{\sf T}\mathcal{H}\mathcal{V}_{\alpha}\check{\boldsymbol{z}}_{n-1}\Vert^2+ \frac{\mu^2}{K(1-t)}\mathds{E}\Vert\mathds{1}^{\sf  T}(\mathcal{H}_{n-1} - \mathcal{H})\widetilde{\boldsymbol{\scriptstyle\mathcal{W}}}_{n-1}\Vert^2 \notag\\
&+  \frac{\mu^2}{K(1-t)}\mathds{E}\Vert\mathds{1}^{\sf  T}\boldsymbol{e}_{n-1}\Vert^2
\end{align}
Note that although the Hessian matrix of $J(w)$ is not guaranteed to be uniformly bounded over all $w$, the Hessian matrix at a local minimum $H_{k}^{\star}$ for $k=1,\ldots,K$ are constants. Thus, the terms related to $\mathcal{H}$, the collection of $H_{k}^{\star}$ as defined in \eqref{collec_hk}, can be bounded. Accordingly, we have
\begin{align}\label{barz1_1}
\frac{\mu^2}{K(1-t)}\mathds{E}\Vert\mathds{1}^{\sf T}\mathcal{H}\mathcal{V}_{\alpha}\check{\boldsymbol{z}}_{n-1}\Vert^2 \le O(\mu^2) \mathds{E}\Vert\check{\boldsymbol{z}}_{n-1}\Vert^2 
\end{align}
In addition, using Taylor expansion, when $\boldsymbol{w}_{k,n}$ is sufficiently close to $w^{\star}$, we have
\begin{align}
     H_{k,n-1} - H_k^\star =  \nabla H_k^\star\tilde{\boldsymbol{w}}_{k,n-1} + O(\Vert\tilde{\boldsymbol{w}}_{k,n-1}\Vert^2)
\end{align}
and 
\begin{align}
    \Vert (H_{k,n-1} - H_k^\star)\tilde{\boldsymbol{w}}_{k,n-1}\Vert \le \Vert \nabla H_k^\star\Vert\Vert\tilde{\boldsymbol{w}}_{k,n-1}\Vert^2 +O(\Vert\tilde{\boldsymbol{w}}_{k,n-1}\Vert^3)
\end{align}
As a result,
\begin{align}\label{barz1_2}
   \Vert (\mathcal{H}_{n-1} - \mathcal{H})\widetilde{\boldsymbol{\scriptstyle\mathcal{W}}}_{n-1}\Vert^2 = \sum_k \Vert (H_{k,n-1} - H_k^\star)\tilde{\boldsymbol{w}}_{k,n-1}\Vert^2 \le O(\Vert\widetilde{\boldsymbol{\scriptstyle\mathcal{W}}}_{n-1}\Vert^4)
\end{align}

Substituting \eqref{barz1_1}, \eqref{barz1_2} and \eqref{check_w_212} into \eqref{barz1} gives
\begin{align}\label{barz_2}
 \mathds{E}\Vert \bar{\boldsymbol{z}}_{n}\Vert^2    \le (1- O(\mu) )\mathds{E}\Vert \bar{\boldsymbol{z}}_{n-1}\Vert^2 + O(\mu)\mathds{E}\Vert \check{\boldsymbol{z}}_{n-1}\Vert^2 + O(\mu)\mathds{E}\Vert\widetilde{\boldsymbol{\scriptstyle\mathcal{W}}}_{n-1}\Vert^4 + O(\mu\epsilon^2)
\end{align}

We then analyze $\mathds{E}\Vert \check{\boldsymbol{z}}_{n}\Vert^2$. Recalling \eqref{checkz} gives
\begin{align}\label{checkz1}
  \mathds{E}\Vert \check{\boldsymbol{z}}_{n}\Vert^2 = &\mathds{E}\Vert(\mathcal{P}_{\alpha} - \mu\mathcal{V}_{\alpha}^{\sf T}\mathcal{A}_2\mathcal{H}\mathcal{V}_\alpha)\check{\boldsymbol{z}}_{n-1} - \frac{\mu}{\sqrt{K}}\mathcal{V}_{\alpha}^{\sf T}\mathcal{A}_2\mathcal{H}\mathds{1}\bar{\boldsymbol{z}}_{n-1} + \mu\mathcal{V}_{\alpha}^{\sf T}\mathcal{A}_2(\mathcal{H}_{n-1} - \mathcal{H})\widetilde{\boldsymbol{\scriptstyle\mathcal{W}}}_{n-1} + \mu\mathcal{V}_{\alpha}^{\sf T}\mathcal{A}_2\boldsymbol{e}_{n-1} \Vert^2 \notag\\
  \overset{(a)}{\le} & \left(\rho(\mathcal{P}_\alpha) + O(\mu^2)\right)\mathds{E}\Vert\check{\boldsymbol{z}}_{n-1}\Vert^2 + O(\mu^2)\mathds{E}\Vert \bar{\boldsymbol{z}}_{n-1}\Vert^2 + O(\mu^2)\mathds{E}\Vert\widetilde{\boldsymbol{\scriptstyle\mathcal{W}}}_{n-1}\Vert^4 + O(\mu^2\epsilon^2)
\end{align}
where $(a)$ follows from Jensen's inequality, \eqref{barz1_2}, and \eqref{check_w_212}.

Next, combining \eqref{barz} and \eqref{checkz1}, we obtain
\begin{align}\label{zn_2}
     \mathds{E}\left[\begin{array}{c}
         \Vert\bar{\boldsymbol{z}}_{n}\Vert^2\\
          \Vert\check{\boldsymbol{z}}_{n}\Vert^2
    \end{array}\right]\le& \left[\begin{array}{cc}
       1 - O(\mu)  & O(\mu) \\
         O(\mu^2)&  \rho(P_\alpha)+O(\mu^2)
    \end{array}\right] \mathds{E}\left[\begin{array}{c}
         \Vert\bar{\boldsymbol{z}}_{n-1}\Vert^2\\
          \Vert\check{\boldsymbol{z}}_{n-1}\Vert^2
    \end{array}\right] \notag\\
&+ \left[\begin{array}{c}
O(\mu)\mathds{E}\Vert\widetilde{\boldsymbol{\scriptstyle\mathcal{W}}}_{n-1}\Vert^4 + O(\mu\epsilon^2) \\
O(\mu^2)\mathds{E}\Vert\widetilde{\boldsymbol{\scriptstyle\mathcal{W}}}_{n-1}\Vert^4 + O(\mu^2\epsilon^2)
    \end{array}\right]  
\end{align}
where $\boldsymbol{z}_{-1} = 0$. To proceed, we introduce the coefficient matrix:
\begin{align}\label{de_gamma3}
    \Gamma_3 = \left[\begin{array}{cc}
       1 - O(\mu) & O(\mu) \\
         O(\mu^2)&  \rho(P_\alpha)+O(\mu^2)
    \end{array}\right]
\end{align}
with which and iterating \eqref{zn_2}, we have
\begin{align}\label{zn_21}
    \mathds{E}\left[\begin{array}{c}
         \Vert\bar{\boldsymbol{z}}_{n}\Vert^2\\
          \Vert\check{\boldsymbol{z}}_{n}\Vert^2
    \end{array}\right] = (I - \Gamma_3)^{-1}(I - \Gamma_3^{n+1})\left[\begin{array}{c}
O(\mu)\mathds{E}\Vert\widetilde{\boldsymbol{\scriptstyle\mathcal{W}}}_{n-1}\Vert^4 + O(\mu\epsilon^2) \\
O(\mu^2)\mathds{E}\Vert\widetilde{\boldsymbol{\scriptstyle\mathcal{W}}}_{n-1}\Vert^4 + O(\mu^2\epsilon^2)
    \end{array}\right]  
\end{align}
for which we have
\begin{align}\label{gamma3in}
    (I - \Gamma_3)^{-1} =  \left[\begin{array}{cc}
       O(\mu) & -O(\mu) \\
         -O(\mu^2)&  1 - \rho(P_\alpha)+O(\mu^2)
    \end{array}\right]^{-1} = \left[\begin{array}{cc}
       O(\frac{1}{\mu}) & O(1) \\
         O(\mu)&  O(1)
    \end{array}\right]
\end{align}
and 
\begin{align}\label{gamma3n}
    I - \Gamma_3^{n+1} =&  \left[\begin{array}{cc}
       1 - (1 - O(\mu))^{n+1} & -O(\mu) \\
         -O(\mu^2)&  1 - (\rho(P_\alpha)+O(\mu^2))^{n+1}
    \end{array}\right] \notag\\
    =& \left[\begin{array}{cc}
       1 - (1 - O(\mu))^{n+1} & -O(\mu) \\
         -O(\mu^2)&  O(1)
    \end{array}\right]
\end{align}
where
\begin{align}\label{gamma3n_1}
   1 - (1 - O(\mu))^{n+1} < 1 = O(1) 
\end{align}

Substituting \eqref{gamma3in}, \eqref{gamma3n}, and \eqref{tilde_w_4_4} into \eqref{zn_21}, we obtain
\begin{align}
 \mathds{E}\left[\begin{array}{c}
         \Vert\bar{\boldsymbol{z}}_{n}\Vert^2\\
          \Vert\check{\boldsymbol{z}}_{n}\Vert^2
    \end{array}\right]\le&  \left[\begin{array}{cc}
       O(\frac{1}{\mu}) & O(1) \\
         O(\mu)&  O(1)
    \end{array}\right] \left[\begin{array}{cc}
       O(1) & -O(\mu) \\
         -O(\mu^2)&  O(1)
    \end{array}\right] \left[\begin{array}{c}
O(\mu^5) + O(\mu\epsilon^2) \\
O(\mu^6)  + O(\mu^2\epsilon^2)
    \end{array}\right]    \notag\\
    = & \left[\begin{array}{c}
O(\mu^4) + O(\epsilon^2) \\
O(\mu^6) + O(\mu^2\epsilon^2)
    \end{array}\right]
\end{align}
Therefore,
\begin{align}\label{zn_f}
\mathds{E}\Vert\widetilde{\boldsymbol{\scriptstyle\mathcal{W}}}_n' - \widetilde{\boldsymbol{\scriptstyle\mathcal{W}}}_n \Vert^2 = \mathds{E}  \Vert{{\boldsymbol{{z}}}}_{n}\Vert^2 = \mathds{E}(\Vert\check{\boldsymbol{z}}_{n}\Vert^2+\Vert\bar{\boldsymbol{z}}_{n}\Vert^2) = O(\mu^4) + O(\epsilon^2)
\end{align}

In the subsequent analysis, we examine how well $\mathds{E}\Vert\widetilde{\boldsymbol{\scriptstyle\mathcal{W}}}_n'\Vert^2$ can represent $\mathds{E}\Vert\widetilde{\boldsymbol{\scriptstyle\mathcal{W}}}_n\Vert^2$. Specifically, we analyze the size of $\vert \mathds{E}\Vert\widetilde{\boldsymbol{\scriptstyle\mathcal{W}}}_n' \Vert^2 - \mathds{E}\Vert\widetilde{\boldsymbol{\scriptstyle\mathcal{W}}}_n\Vert^2\vert$. Basically,
\begin{align}
\mathds{E}\Vert\widetilde{\boldsymbol{\scriptstyle\mathcal{W}}}_n' \Vert^2 &= \mathds{E}\Vert\widetilde{\boldsymbol{\scriptstyle\mathcal{W}}}_n' -\widetilde{\boldsymbol{\scriptstyle\mathcal{W}}}_n+\widetilde{\boldsymbol{\scriptstyle\mathcal{W}}}_n\Vert^2\notag\\
&\le \mathds{E}\Vert\widetilde{\boldsymbol{\scriptstyle\mathcal{W}}}_n'- \widetilde{\boldsymbol{\scriptstyle\mathcal{W}}}_n\Vert^2 + \mathds{E}\Vert\widetilde{\boldsymbol{\scriptstyle\mathcal{W}}}_n\Vert^2 + 2\vert\mathds{E}(\widetilde{\boldsymbol{\scriptstyle\mathcal{W}}}_n' -\widetilde{\boldsymbol{\scriptstyle\mathcal{W}}}_n)^{\sf T}\widetilde{\boldsymbol{\scriptstyle\mathcal{W}}}_n\vert\notag\\
    &\le \mathds{E}\Vert\widetilde{\boldsymbol{\scriptstyle\mathcal{W}}}_n'- \widetilde{\boldsymbol{\scriptstyle\mathcal{W}}}_n\Vert^2 + \mathds{E}\Vert\widetilde{\boldsymbol{\scriptstyle\mathcal{W}}}_n\Vert^2 + 2\sqrt{\mathds{E}\Vert\widetilde{\boldsymbol{\scriptstyle\mathcal{W}}}_n'- \widetilde{\boldsymbol{\scriptstyle\mathcal{W}}}_n\Vert^2\mathds{E}\Vert\widetilde{\boldsymbol{\scriptstyle\mathcal{W}}}_n\Vert^2}\notag\\
    &\overset{(a)}{\le} O(\mu^2) + O(\epsilon^{\frac{7}{4}})
\end{align}
where $(a)$ follows from \eqref{zn_f} and \eqref{mse_2_unified}. Moreover, 
\begin{align}
\label{approx_e1}
\mathds{E}\Vert\widetilde{\boldsymbol{\scriptstyle\mathcal{W}}}_n' \Vert^2 - \mathds{E}\Vert\widetilde{\boldsymbol{\scriptstyle\mathcal{W}}}_n\Vert^2 &\le \mathds{E}\Vert\widetilde{\boldsymbol{\scriptstyle\mathcal{W}}}_n'- \widetilde{\boldsymbol{\scriptstyle\mathcal{W}}}_n\Vert^2 + 2\sqrt{\mathds{E}\Vert\widetilde{\boldsymbol{\scriptstyle\mathcal{W}}}_n'- \widetilde{\boldsymbol{\scriptstyle\mathcal{W}}}_n\Vert^2\mathds{E}\Vert\widetilde{\boldsymbol{\scriptstyle\mathcal{W}}}_n\Vert^2}\\
\label{approx_e2}
\mathds{E}\Vert\widetilde{\boldsymbol{\scriptstyle\mathcal{W}}}_n\Vert^2 - \mathds{E}\Vert\widetilde{\boldsymbol{\scriptstyle\mathcal{W}}}_n'\Vert^2 &\le \mathds{E}\Vert\widetilde{\boldsymbol{\scriptstyle\mathcal{W}}}_n'- \widetilde{\boldsymbol{\scriptstyle\mathcal{W}}}_n\Vert^2 + 2\sqrt{\mathds{E}\Vert\widetilde{\boldsymbol{\scriptstyle\mathcal{W}}}_n'- \widetilde{\boldsymbol{\scriptstyle\mathcal{W}}}_n\Vert^2\mathds{E}\Vert\widetilde{\boldsymbol{\scriptstyle\mathcal{W}}}_n'\Vert^2}
\end{align}
Combining \eqref{approx_e1} and \eqref{approx_e2}, we have
\begin{align}\label{approx_e3}
    \left\vert \mathds{E}\Vert\widetilde{\boldsymbol{\scriptstyle\mathcal{W}}}_n' \Vert^2 - \mathds{E}\Vert\widetilde{\boldsymbol{\scriptstyle\mathcal{W}}}_n\Vert^2 \right\vert \le &\mathds{E}\Vert\widetilde{\boldsymbol{\scriptstyle\mathcal{W}}}_n'- \widetilde{\boldsymbol{\scriptstyle\mathcal{W}}}_n\Vert^2 + 2\sqrt{\mathds{E}\Vert\widetilde{\boldsymbol{\scriptstyle\mathcal{W}}}_n'- \widetilde{\boldsymbol{\scriptstyle\mathcal{W}}}_n\Vert^2\mathds{E}\Vert\widetilde{\boldsymbol{\scriptstyle\mathcal{W}}}_n\Vert^2}\notag\\
    \overset{(a)}{\le}& O(\mu^4)+ O(\epsilon^2) + \sqrt{\Big(O(\mu^4)+ O(\epsilon^2)\Big)\times\Big(O(\mu^2) + O(\epsilon^{\frac{7}{4}})\Big)}
\end{align}
where $(a)$ follows from \eqref{zn_f} and \eqref{mse_2_unified}. To guarantee that $\mathds{E}\Vert\widetilde{\boldsymbol{\scriptstyle\mathcal{W}}}_n'\Vert^2$ can  well represent $\mathds{E}\Vert\widetilde{\boldsymbol{\scriptstyle\mathcal{W}}}_n\Vert^2$, i.e., the size of $\left\vert \mathds{E}\Vert\widetilde{\boldsymbol{\scriptstyle\mathcal{W}}}_n' \Vert^2 - \mathds{E}\Vert\widetilde{\boldsymbol{\scriptstyle\mathcal{W}}}_n\Vert^2 \right\vert$ is negligible compared to the size of $\mathds{E}\Vert\widetilde{\boldsymbol{\scriptstyle\mathcal{W}}}_n'\Vert^2$ and $\mathds{E}\Vert\widetilde{\boldsymbol{\scriptstyle\mathcal{W}}}_n\Vert^2$, the value of $\epsilon$ should be sufficiently small such that $O(\epsilon^{\frac{7}{4}})$ is dominated by $O(\mu^2)$. Basically, assume
\begin{align}
    \epsilon \le O(\mu^\alpha)
\end{align}
If $\alpha \ge \frac{8}{7}$, then $\left\vert \mathds{E}\Vert\widetilde{\boldsymbol{\scriptstyle\mathcal{W}}}_n' \Vert^2 - \mathds{E}\Vert\widetilde{\boldsymbol{\scriptstyle\mathcal{W}}}_n\Vert^2 \right\vert$ is negligible. For simplicity, in this paper, we consider $\alpha = 2$
\begin{align}
    \epsilon \le O(\mu^2)
\end{align}
As a result,
\begin{align}\label{approx_e4}
    \left\vert \mathds{E}\Vert\widetilde{\boldsymbol{\scriptstyle\mathcal{W}}}_n' \Vert^2 - \mathds{E}\Vert\widetilde{\boldsymbol{\scriptstyle\mathcal{W}}}_n\Vert^2 \right\vert {\le}& O(\mu^3) = O\left((\mathds{E}\Vert{\widetilde{\boldsymbol{\scriptstyle\mathcal{W}}}}_{n}\Vert^2)^{1.5}\right)
\end{align}
\end{proof}

\section{Additional simulation results}\label{appen_as}
\subsection{Experimental Setup}
We begin by presenting the details of our experimental setup. The simulations are performed on CIFAR-10, CIFAR-100, and Tiny ImageNet-200 datasets. We randomly sample a small subset from the original training set for validation. Specifically, for CIFAR-10 and CIFAR-100, the size of the validation dataset is 1008 or 1024, while for Tiny ImageNet-200, the validation set contains 5024 samples. For centralized training, we follow the conventional single-agent setting with a global batch size of batch size $KB$, where $K\in \{8,16,32\}$ denotes the number of agents and $B\in \{128, 256\}$ is the local batch size. In the decentralized training setup, the training set is randomly partitioned into $K$ disjoint subsets. Each agents maintains a local model and has access only to its assigned subset. Communication among agents during training is governed by the graph topology. For $K=16$, we consider three graph topologies, including a randomly generated graph based on the Metropolis rule~\cite{sayed2014adaptation}, the \emph{ring} topology, and the \emph{grid} topology. For the ring topology, we additionally consider network sizes with $K\in\{8,16,32\}$. The random and grid topologies are evaluated only for $K=16$.

We investigate three training methods: (i) PGD training using the standard SGD optimizer, (ii) PGD training using SGD with momentum, and (iii) TRADES, which incorporates an additional regularization term related to the clean accuracy relative to standard PGD training. We adopt the hyper-parameter settings recommended in \cite{abs-2010-03593}. Specifically, the momentum parameter for the SGD momentum optimizer is set to $0.9$. In TRADES, the coefficient controlling the trade-off between clean accuracy and robustness is set to $6.0$. Note that (i) is the main focus of this paper because it is consistent with our theoretical analysis, while settings (ii) and (iii) are included as complementary empirical evaluations. The hyper-parameters used during training are summarized in Table \ref{tabl_hs}.

\begin{table*}[htbp]
\footnotesize
\centering
\begin{spacing}{1.6}
\caption{Training Hyperparameter Settings}
\label{tabl_hs}
\begin{tabular}{ccccccccc}
\toprule
Data & Neural Network & Initial Learning Rate(LR) & Weight Decay
& Optimizer & K & B & Epoch & LR Milestones \\
\cline{1-9}
\multirow{6}{*}{CIFAR-10}
& \multirow{6}{*}{WideResNet-28-10}
& \multirow{6}{*}{0.1}
& \multirow{6}{*}{5e-4}
& \multirow{4}{*}{SGD}
& 8 & 128 & 350 & [140, 210, 280] \\
\cline{6-9}
&&&&&\multirow{2}{*}{16}&128&350 & [140, 210, 280]\\
\cline{7-9}
&&&&&&256&400&[200, 280, 360]\\
\cline{6-9}
&&&&&32&128&400&[200, 280, 360]\\
\cline{5-9}
&&&&\multirow{2}{*}{SGD Momentum}&\multirow{2}{*}{16}&128&350& [140, 210, 280]\\
\cline{7-9}
&&&&&&256&400&[200, 280, 360]\\
\cline{1-9}

\multirow{6}{*}{CIFAR-100}
& \multirow{6}{*}{WideResNet-34-10}
& \multirow{6}{*}{0.1}
& \multirow{6}{*}{1e-4}
& \multirow{4}{*}{SGD}
& 8 & 128 & 400&[200, 280, 360] \\
\cline{6-9}
&&&&&\multirow{2}{*}{16}&128& 350& [140, 210, 280] \\
\cline{7-9}
&&&&&&256&500&[250, 350, 450]\\
\cline{6-9}
&&&&&32&128&500&[250, 350, 450]\\
\cline{5-9}
&&&&\multirow{2}{*}{SGD Momentum}&\multirow{2}{*}{16}&128&350&[175, 265, 315]\\
\cline{7-9}
&&&&&&256&400&[200, 300, 360]\\
\cline{1-9}
\multirow{2}{*}{Tiny ImageNet-200}& \multirow{2}{*}{WideResNet-28-10}
& \multirow{2}{*}{0.1}
& \multirow{2}{*}{1e-4}
& SGD
& 16 & 128 & 350 & [175, 265, 315] \\
\cline{5-9}
&&&&SGD Momentum&16&128&300&[150, 215,  270]\\
\bottomrule
\end{tabular}
\end{spacing}
\end{table*}

We consider adversarial perturbations constrained by $\ell_2$ and $\ell_\infty$ norms. Specifically, for $\ell_\infty$ attacks, we simulate $\epsilon= 8/255$ and $\epsilon=3/255$ for CIFAR-10 and CIFAR-100, and simulate $\epsilon=3/255$ for Tiny ImageNet-200. For $\ell_2$ attacks, we use the bound $\epsilon=128/255$ for all datasets. The inner maximization is approximated using 10 gradient-ascent steps. The step sizes are $15/255$ for $128/255$, $1/255$ for perturbation radii $3/255$, and $2/255$ for $8/255$, respectively. We use early-stopping, a widely used technique in the literature \cite{abs-2010-03593, WuX020}, to select the \emph{best} model based on its performance on the validation dataset. Note that the \emph{best} model is chosen according to its robust validation accuracy to PGD attacks. It is possible that the \emph{final} model exhibit slightly better robustness to AutoAttack than the early-stopped model when the overfitting is not severe.

We next report the hardware configurations and per-epoch training times for reference. This paper is a fundamental study of the optimization bias of distributed adversarial training algorithms. Our computational resources are limited, and implementing and evaluating the algorithms on a fully distributed platform would require substantial engineering effort. We therefore conduct simulations primarily on a single GPU, without assigning a separate GPU to each agent. For decentralized training, all agents are simulated on 1 GPU. For centralized training, we use 2 GPUs only when the global batch size exceeds the memory capacity of a single GPU.

The experiments are conducted using Nvidia A100, H100 or H200 GPUs. Table \ref{table_et} reports the per-epoch training time and hardware for each experiment. Note that for settings in which all three methods use the same GPU type, the centralized training has a shorter epoch time than decentralized implementations. Here we clarify that this difference reflects our simulation procedure and does not establish that centralized training is more efficient. Specifically, all algorithms are simulated mainly using a single GPU. The centralized algorithm follows the conventional single-agent training with a global batch size of $KB$, as they are fundamentally equivalent. By contrast, the decentralized implementation contains an additional inner loop over all agents to simulate their local updates and parameter exchanges with neighboring agents. This additional loop introduces additional computational overhead. Therefore, the reported times describe the computational costs of our simulations and are provided only for reference. They should not be interpreted as a system-level comparison between centralized and decentralized training.

The potential system-level benefits of decentralized learning strategies in the clean case have been analyzed in~\cite{abs-2111-04287}, which introduces BlueFog, a framework supporting neighbor-based communication, and reported reduced communication overhead and improved throughput for decentralized training. Although this framework provides a useful reference, deploying and systematically benchmarking the distributed adversarial training algorithms considered here would require additional computational resources and engineering effort. Evaluating their communication costs and runtime on a fully distributed platform is beyond the scope of this paper and is a direction for future work.

\begin{table*}[htbp]
\footnotesize
\centering
\begin{spacing}{1.6}
\caption{The training time of each epoch}
\label{table_et}
\begin{tabular}{ccccccccc}
\toprule
Data & K & B & Method & Hardware & Epoch Time \\
\cline{1-6}
\multirow{12}{*}{CIFAR-10}
& \multirow{3}{*}{8} & \multirow{3}{*}{128} & Centralized & 1*A100 & 220s \\
&&&Consensus&1*A100&270s\\
&&&Diffusion&1*H100&139s\\
\cline{2-6}
&\multirow{6}{*}{16}&\multirow{3}{*}{128} & Centralized&1*H100&106s\\
&&&Consensus&1*H100&164s\\
&&&Diffusion&1*H100&162s\\
\cline{3-6}
&&\multirow{3}{*}{256} & Centralized&1*H100&105s\\
&&&Consensus&1*H100&151s\\
&&&Diffusion&1*H100&151s\\
\cline{2-6}
&\multirow{3}{*}{32}&\multirow{3}{*}{128} & Centralized&1*H200&97s\\
&&&Consensus&1*A100&385s\\
&&&Diffusion&1*A100&387s\\
\hline
\multirow{12}{*}{CIFAR-100}
& \multirow{3}{*}{8} & \multirow{3}{*}{128} & Centralized & 1*H100 & 132s \\
&&&Consensus&1*H100&169s\\
&&&Diffusion&1*H100&170s\\
\cline{2-6}
&\multirow{6}{*}{16}&\multirow{3}{*}{128} & Centralized&1*H100&130s\\
&&&Consensus&1*H100&201s\\
&&&Diffusion&1*H100&202s\\
\cline{3-6}
&&\multirow{3}{*}{256} & Centralized&1*H200&120s\\
&&&Consensus&1*H100&188s\\
&&&Diffusion&1*H100&188s\\
\cline{2-6}
&\multirow{3}{*}{32}&\multirow{3}{*}{128} & Centralized&1*H200&119s\\
&&&Consensus&1*H100&250s\\
&&&Diffusion&1*H100&250s\\
\hline
\multirow{3}{*}{Tiny ImageNet-200}&\multirow{3}{*}{16}&\multirow{3}{*}{128}&Centralized&2*H200&410s\\
&&&Consensus&1*H100&1602s\\
&&&Diffusion&1*H100&1606s\\
\bottomrule
\end{tabular}
\end{spacing}
\end{table*}

\subsection{Additional Results of PGD training with SGD}\label{appendix_gb}

\subsubsection{Flatness of models with other graph structures and network sizes}\label{appendix_gb1}
In this section, we visualize the flatness of models obtained by centralized, consensus, and diffusion adversarial training algorithms on CIFAR-10, CIFAR-100, and Tiny ImageNet-200 datasets. The results for the \emph{final} models under the \emph{random} graph have been reported in Section \ref{flat_visual_main} of the main text. We now provide additional results under different graph topologies and network sizes, together with the results for the \emph{best} models under all experimental settings as further supplementary empirical results.  In addition to the \emph{random} graph, we simulate \emph{ring} and \emph{grid} topologies to examine whether the observations remain consistent across different network structures. For the ring topology, we further consider different networks sizes with $K=\{8,16,32\}$ to examine the effect of number of nodes.

The results of the \emph{final} models under these additional experimental settings are shown in Figures \ref{loss_visual_ring_cifar10}--\ref{loss_visual_grid_cifar100}. The results of the \emph{best} models under all experimental settings are reported in Figures  \ref{loss_visual_cifar10_sgd_best}--\ref{loss_visual_grid_cifar100_best}. Consistent with the observations in the main text, under relatively mild attacks, i.e., $\epsilon=3/255$ for $\ell_\infty$ perturbations and $\epsilon=128/255$ for $\ell_2$ perturbations, decentralized methods generally converge to flatter models than the centralized approach. However, this trend may not hold when the attack strength increases to $\epsilon=8/255$ under the $\ell_\infty$ norm. Moreover, consensus generally produces slightly flatter  models than diffusion.
Overall, these observations are consistent across different datasets, network sizes, and graph topologies.

\begin{figure*}[htbp]
\begin{center}
\subfigure[CIFAR-10: B=128, K=8]{\includegraphics[width=.3\linewidth]
{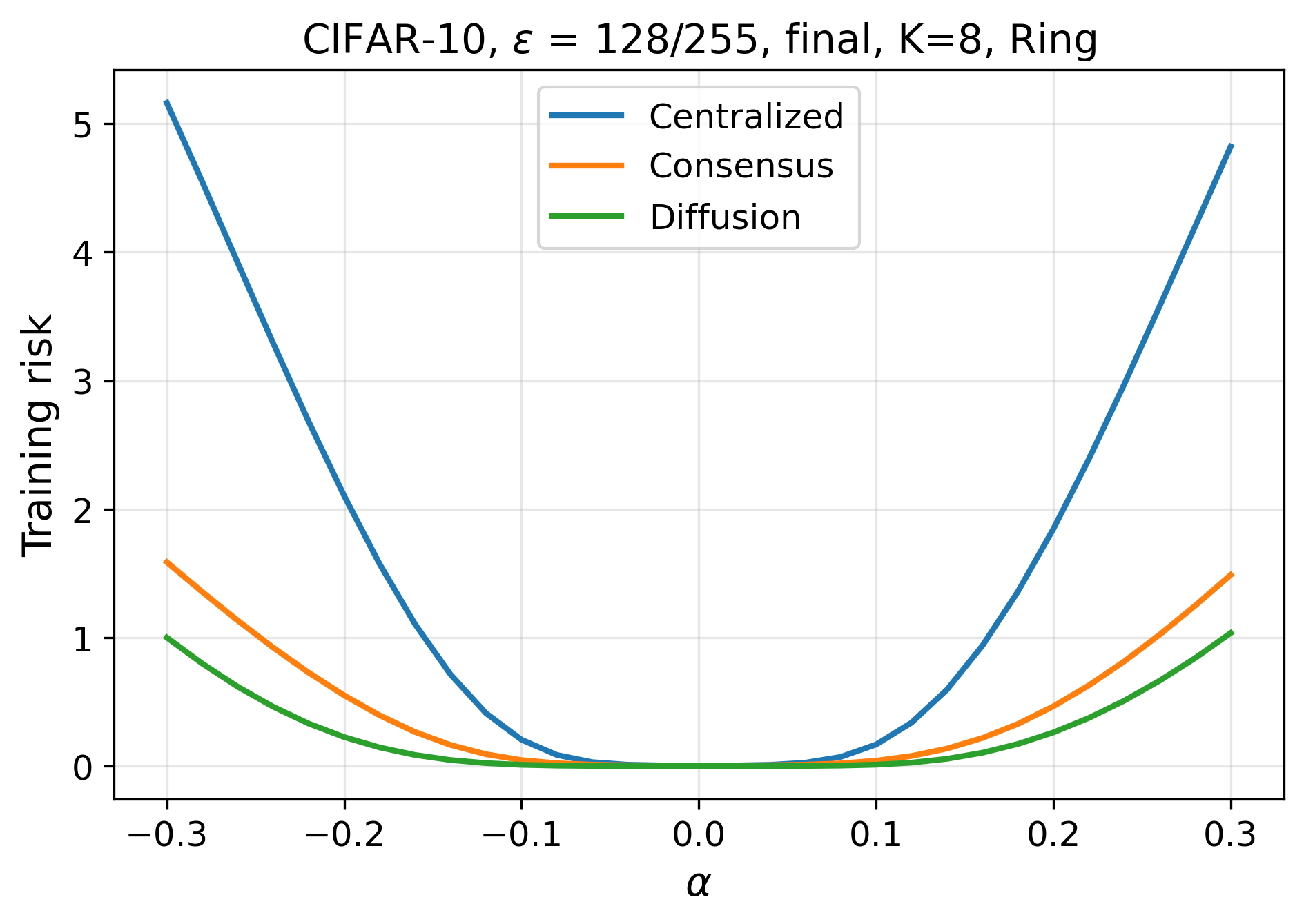}}
\subfigure[CIFAR-10: B=128, K=8]
{\includegraphics[width=.3\linewidth]{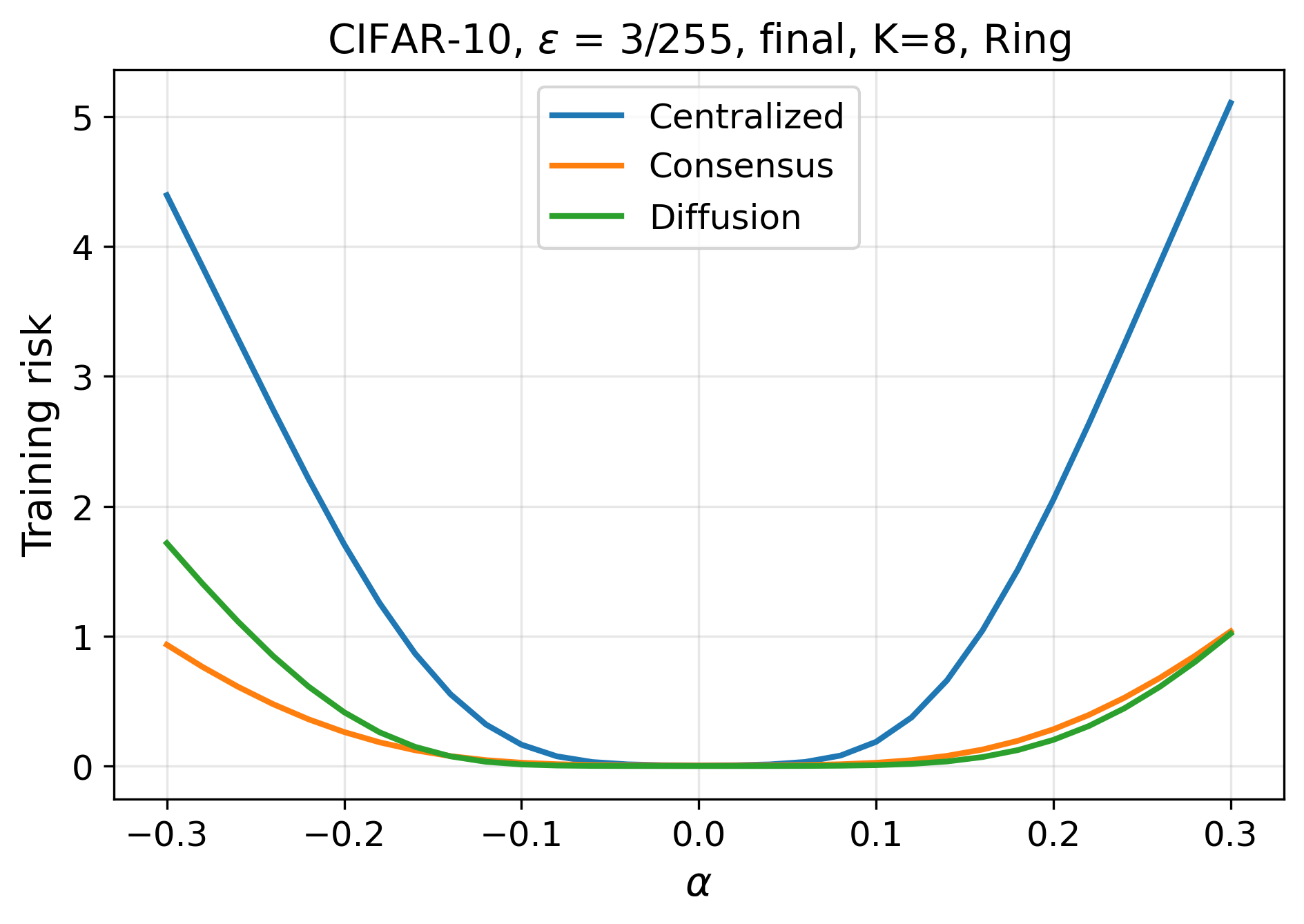}}
\subfigure[CIFAR-10: B=128, K=8]
{\includegraphics[width=.3\linewidth]
{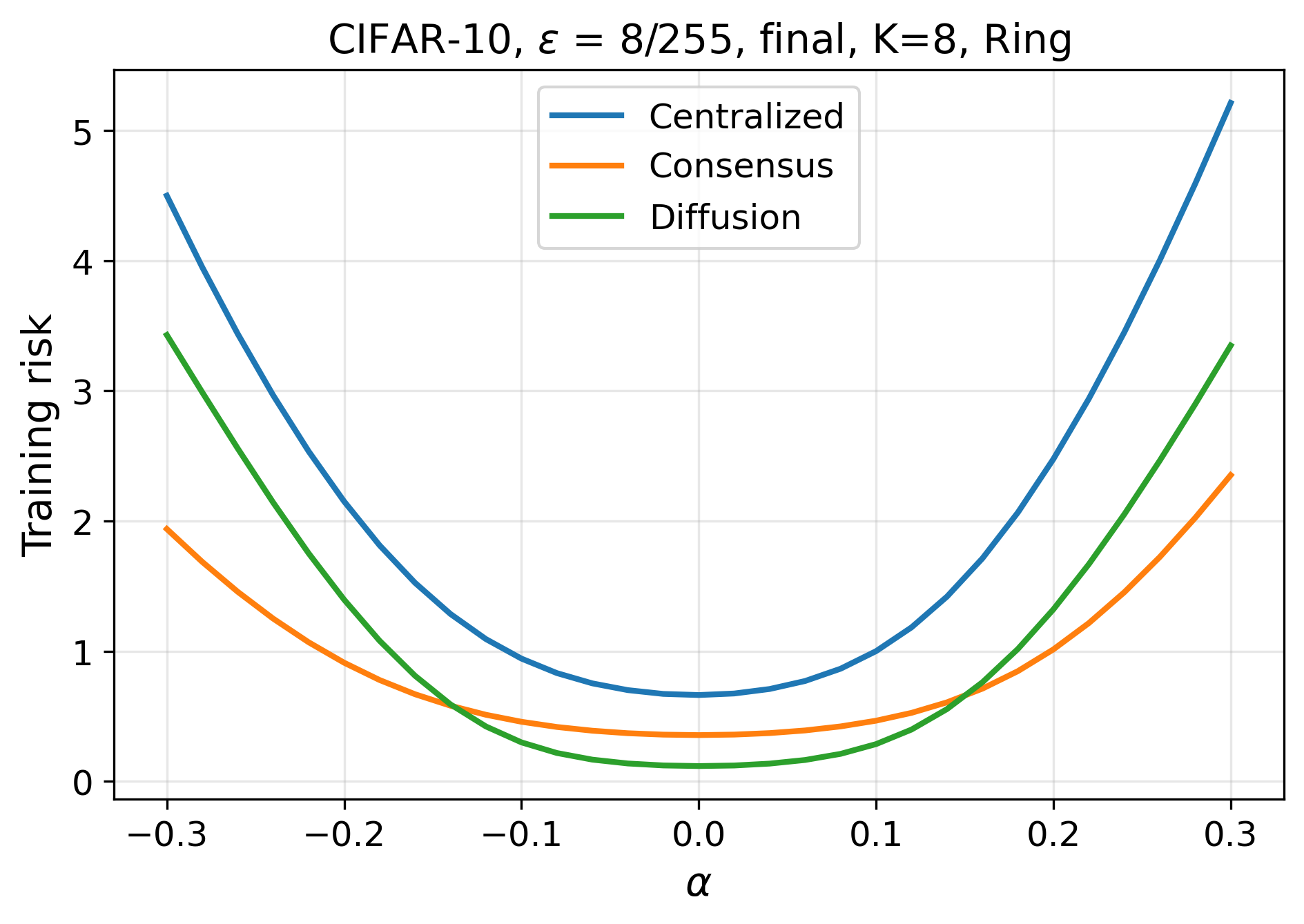}}

\subfigure[CIFAR-10: B=128, K=16]{\includegraphics[width=.3\linewidth]
{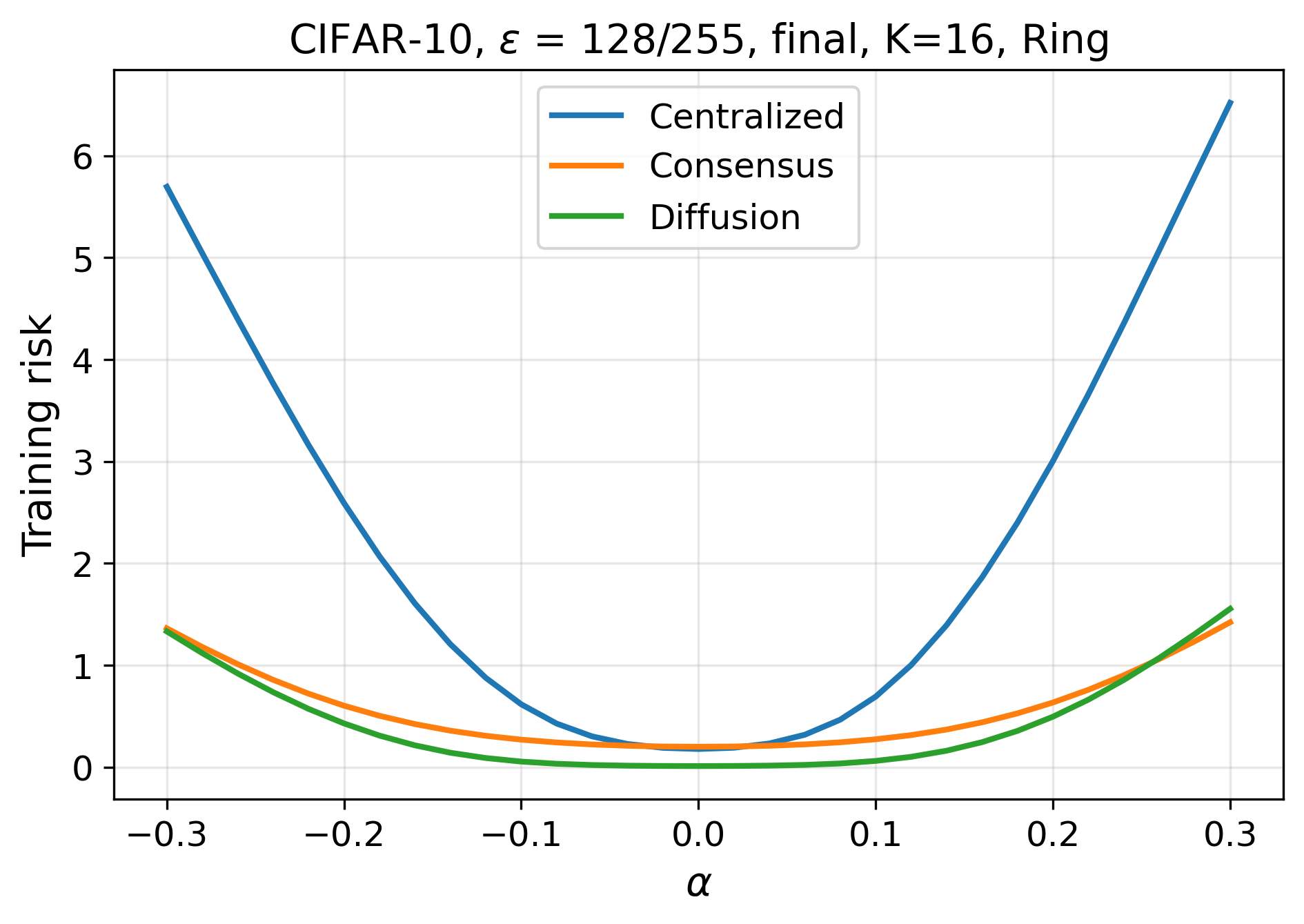}}
\subfigure[CIFAR-10: B=128, K=16]
{\includegraphics[width=.3\linewidth]{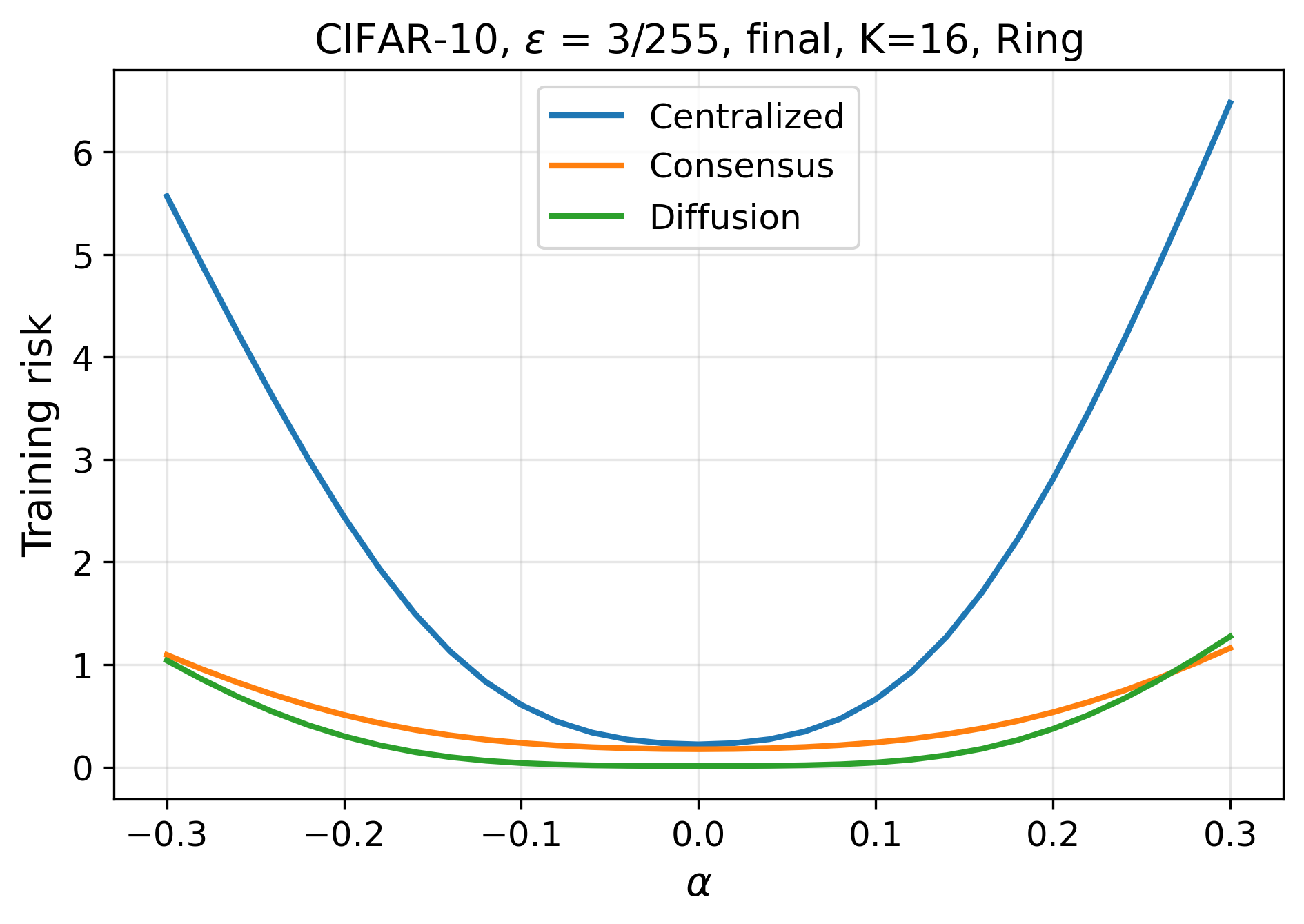}}
\subfigure[CIFAR-10: B=128, K=16]
{\includegraphics[width=.3\linewidth]
{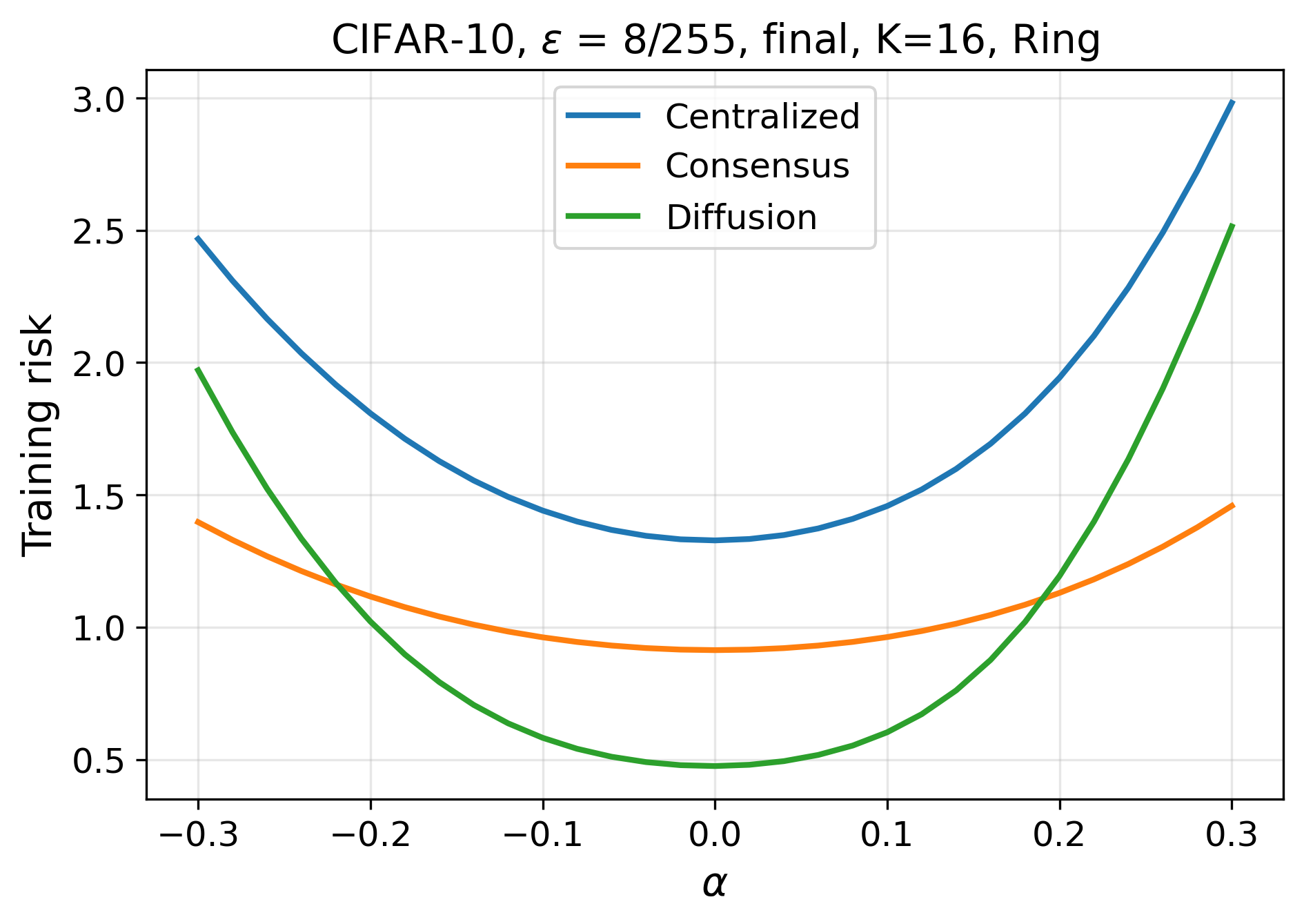}}

\subfigure[CIFAR-10: B=128, K=32]{\includegraphics[width=.3\linewidth]
{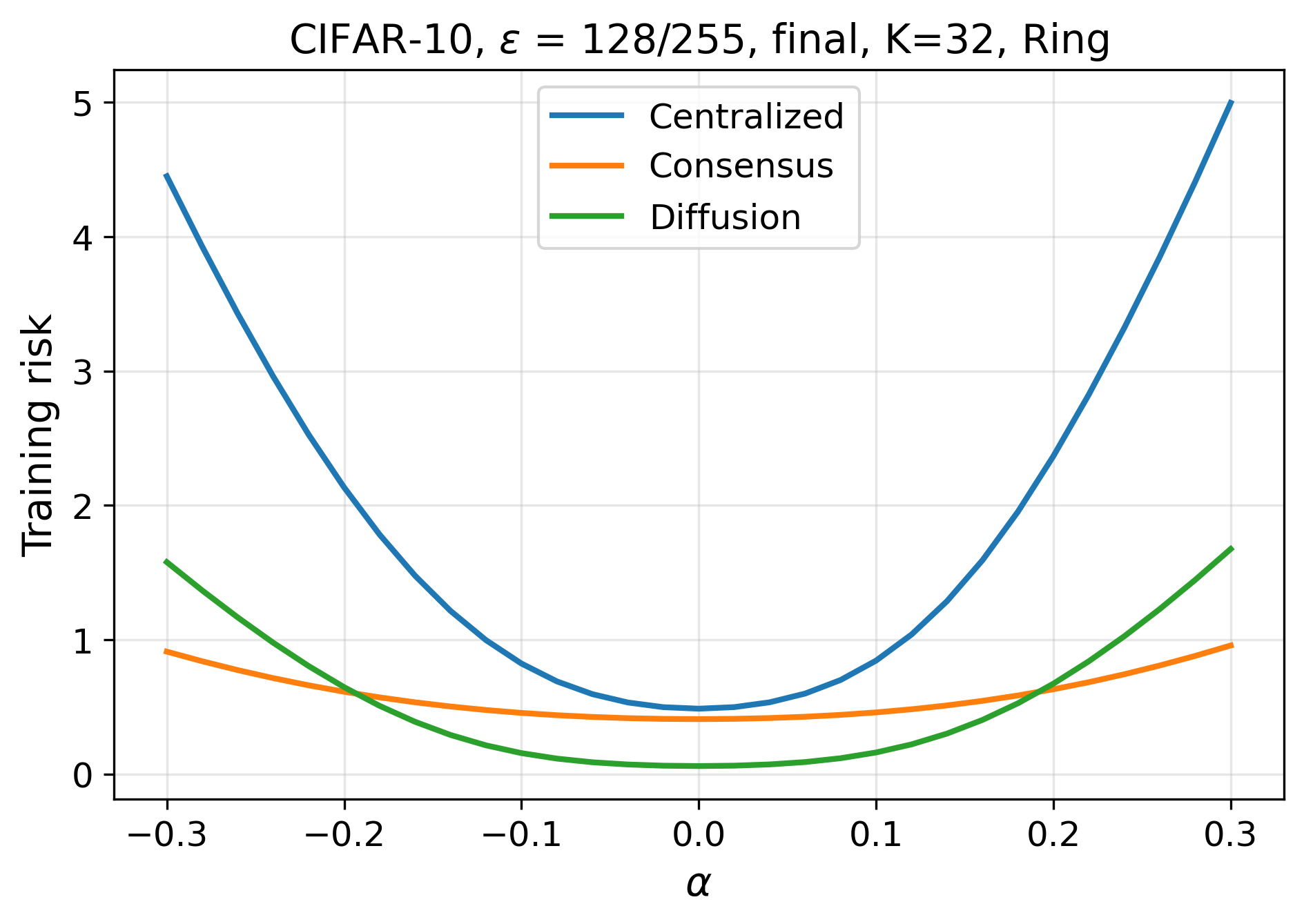}}
\subfigure[CIFAR-10: B=128, K=32]
{\includegraphics[width=.3\linewidth]{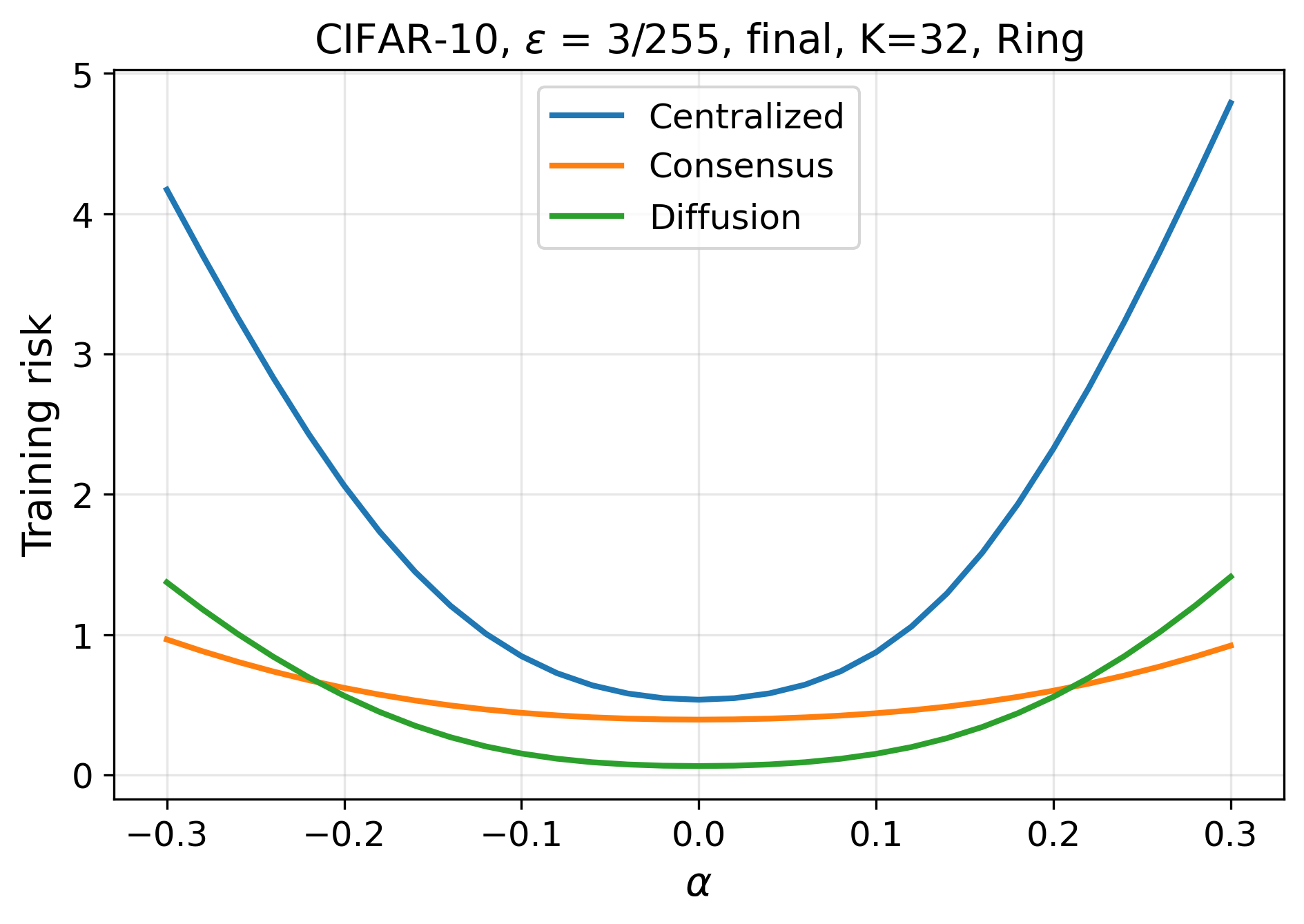}}
\subfigure[CIFAR-10: B=128, K=32]
{\includegraphics[width=.3\linewidth]
{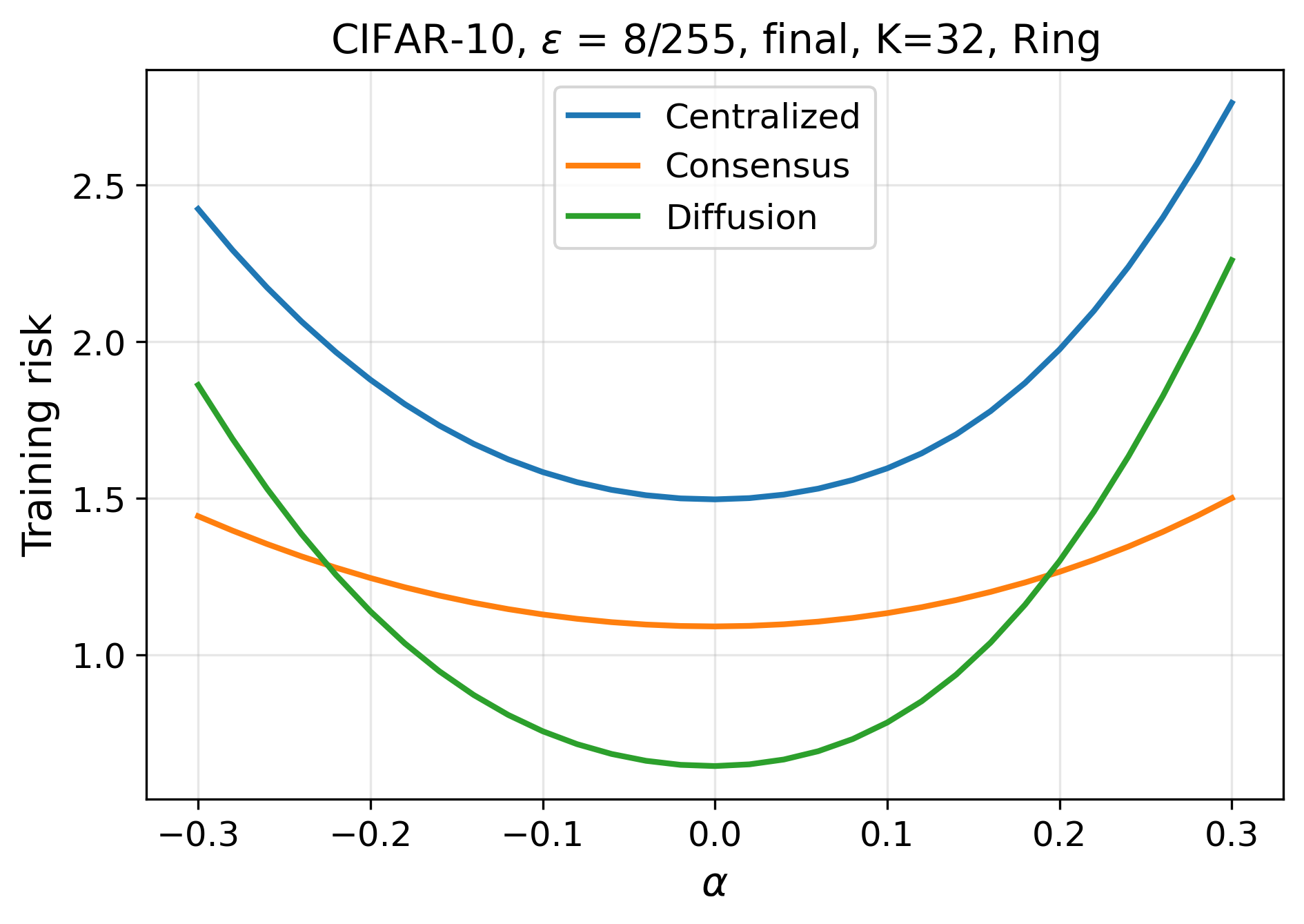}}
\caption{Flatness visualization of the \emph{final} models trained on the CIFAR-10 dataset using the \emph{ring} topology with $K=\{8,16,32\}$. Under relatively mild attacks, i.e.,  $\ell_\infty$ perturbations bounded by $\epsilon=3/255$ and $\ell_2$ perturbations bounded by $128/255$, decentralized adversarial training methods, including consensus and diffusion, converge to flatter models than the centralized strategy. However, when the attack strength increases to $\epsilon=8/255$ with the $\ell_\infty$ norm, diffusion can yield sharper models than centralized training. These observations remain consistent across different network sizes. }
\label{loss_visual_ring_cifar10}
\end{center}
\end{figure*}

\begin{figure*}[htbp]
\begin{center}
\subfigure[CIFAR-10: B=128, K=16]{\includegraphics[width=.3\linewidth]
{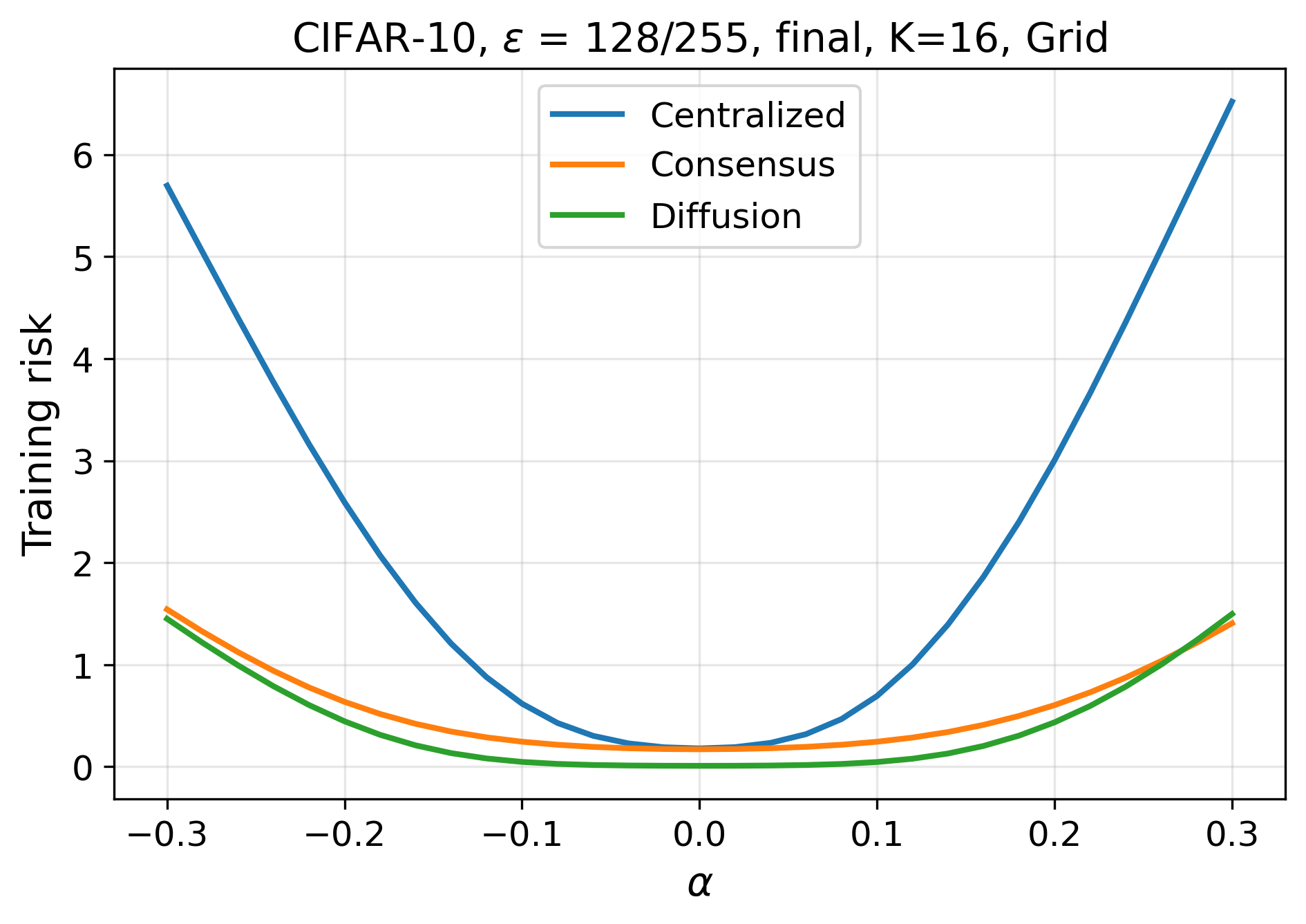}}
\subfigure[CIFAR-10: B=128, K=16]
{\includegraphics[width=.3\linewidth]{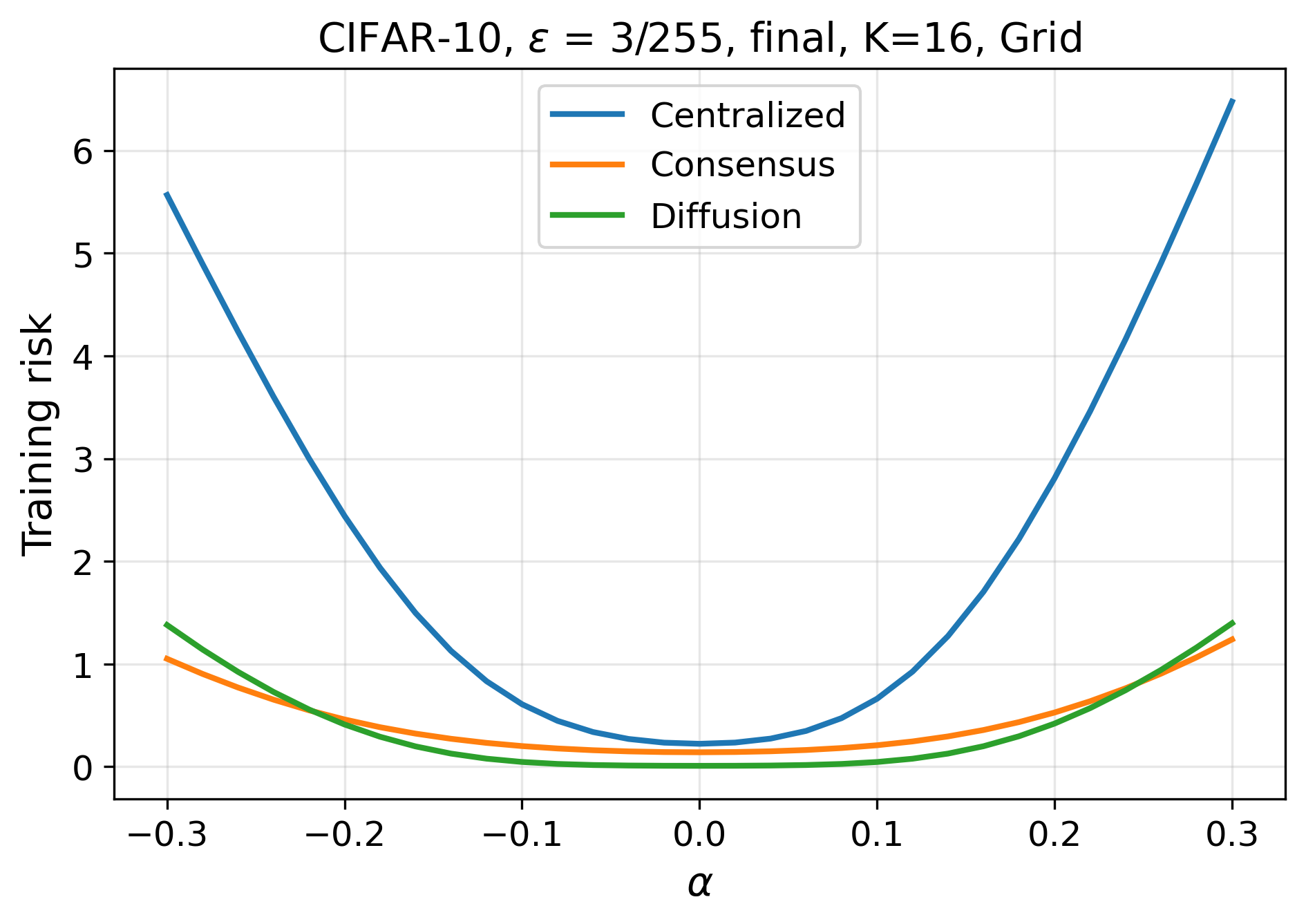}}
\subfigure[CIFAR-10: B=128, K=16]
{\includegraphics[width=.3\linewidth]
{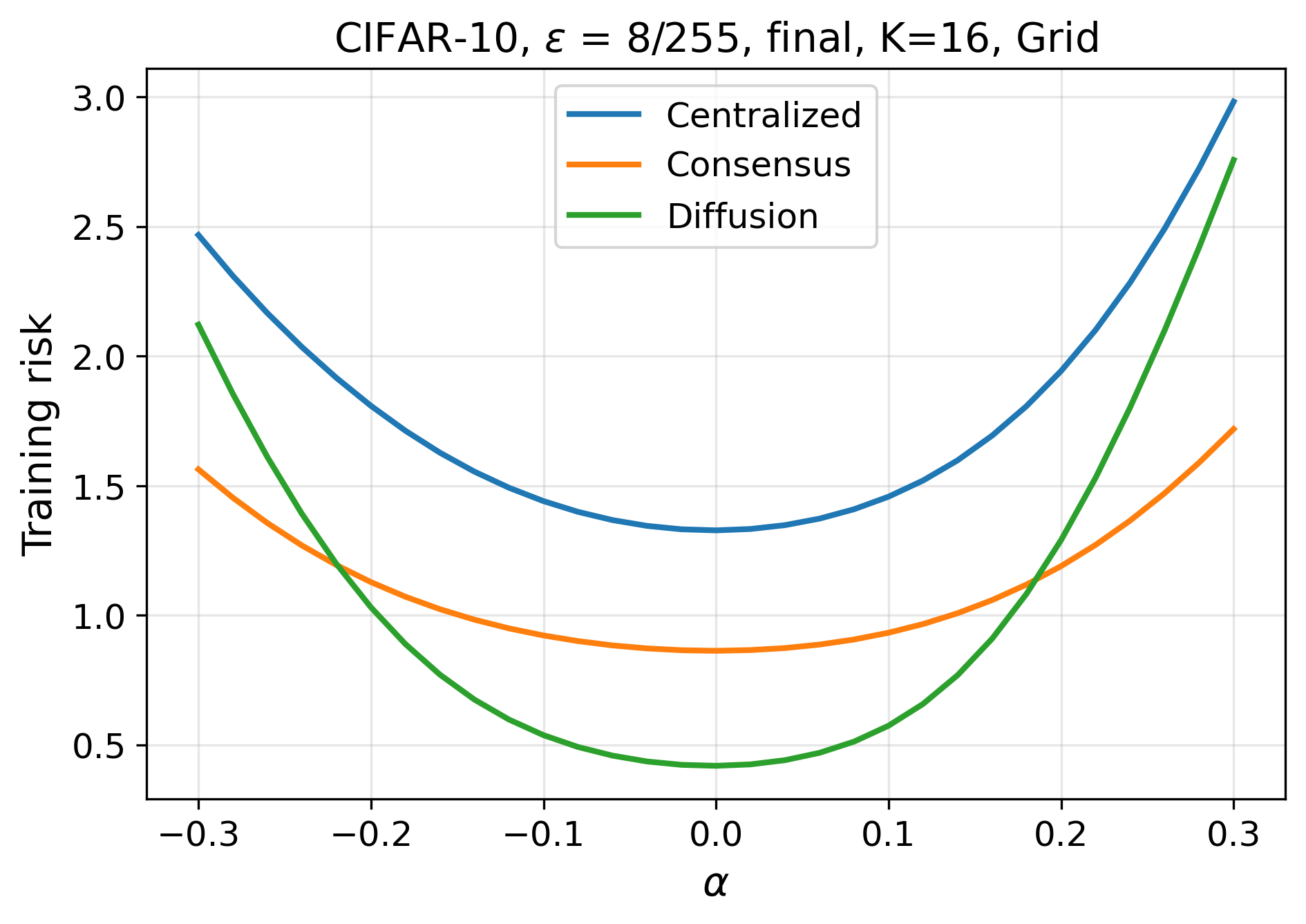}}
\caption{Flatness visualization of the \emph{final} models trained on the CIFAR-10 dataset using the \emph{grid} topology with $K=16$. In (c), diffusion yields a sharper model than centralized. Combining these results with those in Figures \ref{loss_visual_cifar10_sgd} and \ref{loss_visual_ring_cifar10}, we conclude that a similar phenomenon holds across different graph topologies.}
\label{loss_visual_grid_cifar10}
\end{center}
\end{figure*}

\begin{figure*}[htbp]
\begin{center}
\subfigure[CIFAR-100: B=128, K=8]{\includegraphics[width=.3\linewidth]
{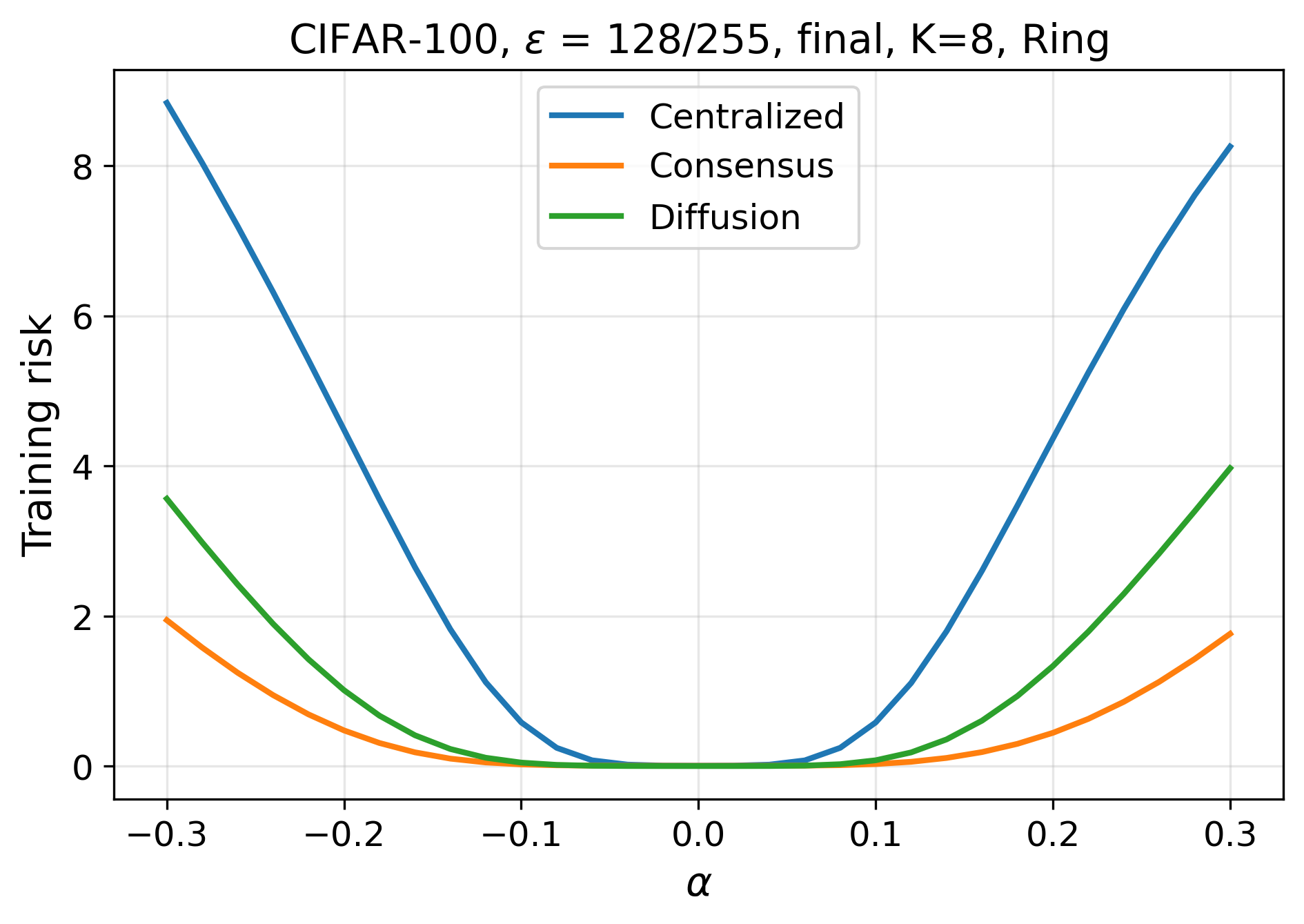}}
\subfigure[CIFAR-100: B=128, K=8]
{\includegraphics[width=.3\linewidth]{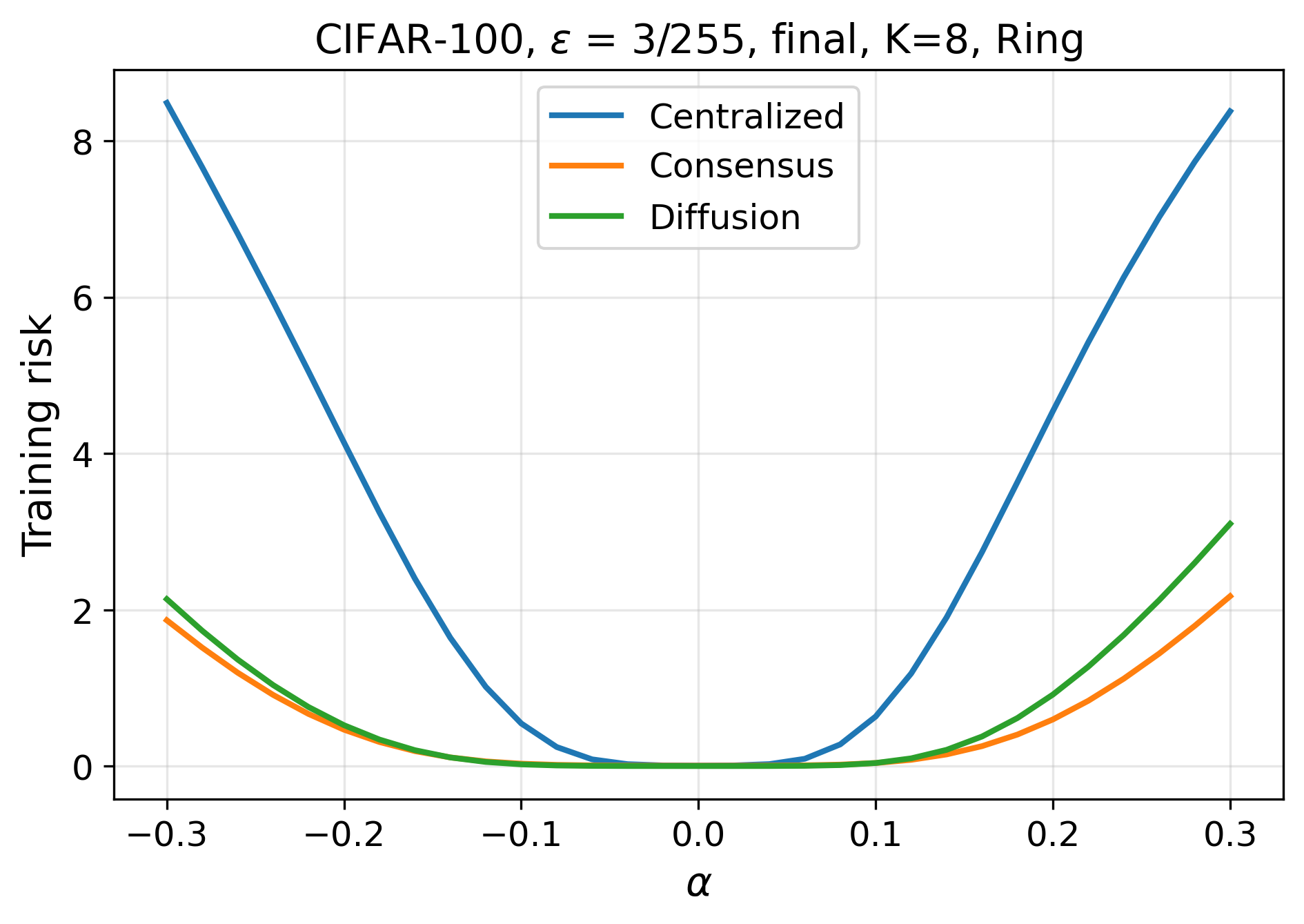}}
\subfigure[CIFAR-100: B=128, K=8]
{\includegraphics[width=.3\linewidth]
{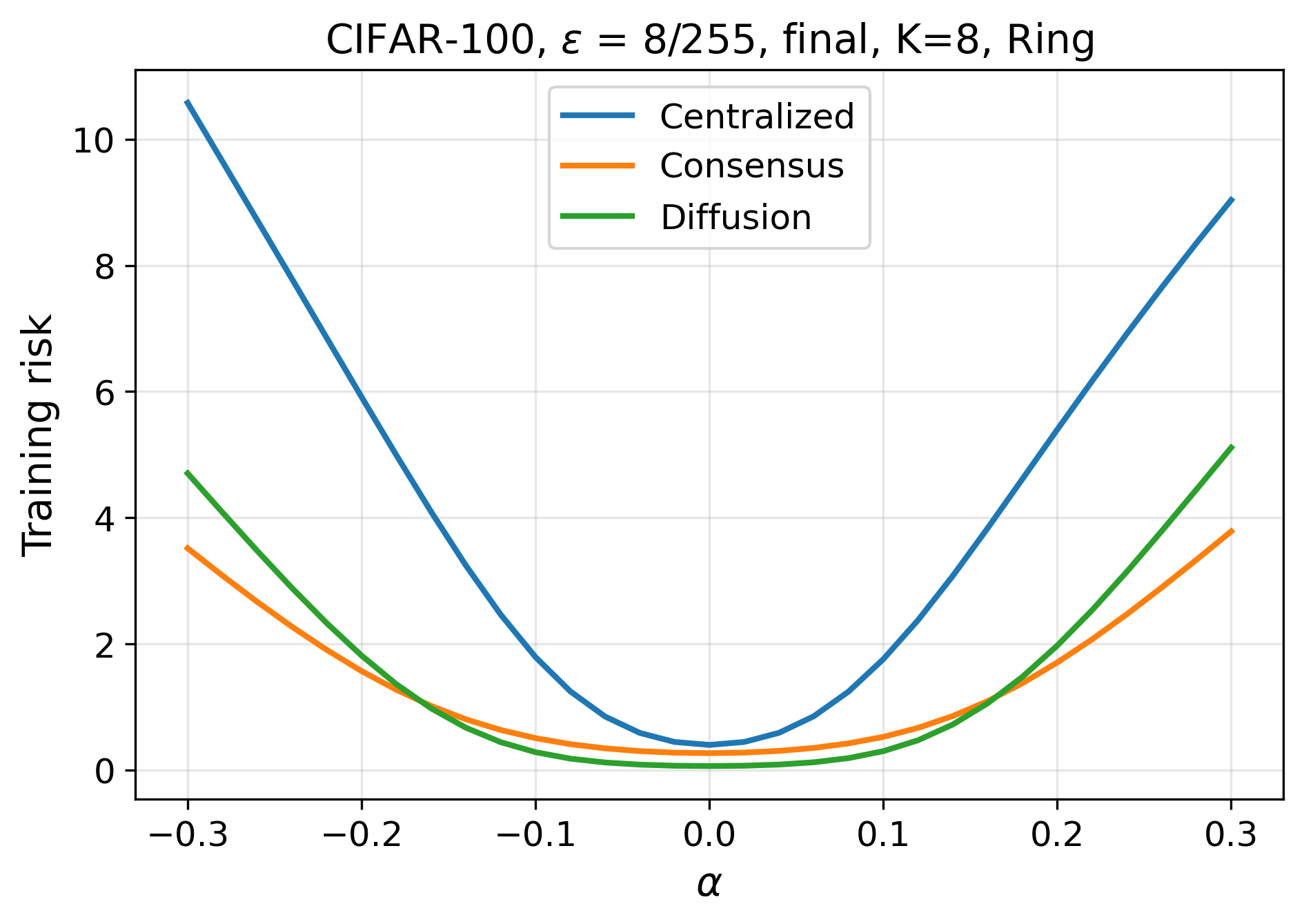}}

\subfigure[CIFAR-100: B=128, K=16]{\includegraphics[width=.3\linewidth]
{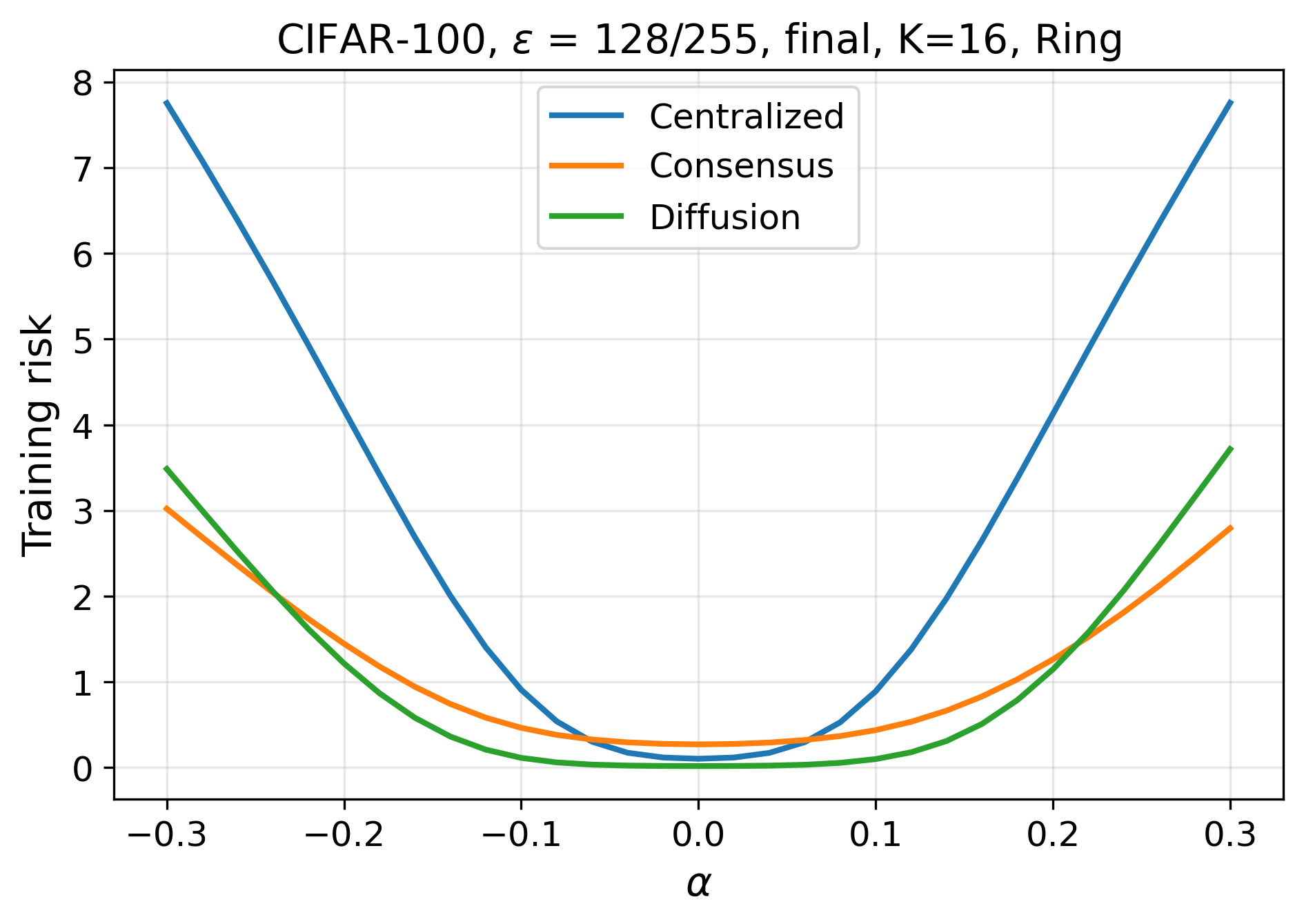}}
\subfigure[CIFAR-100: B=128, K=16]
{\includegraphics[width=.3\linewidth]{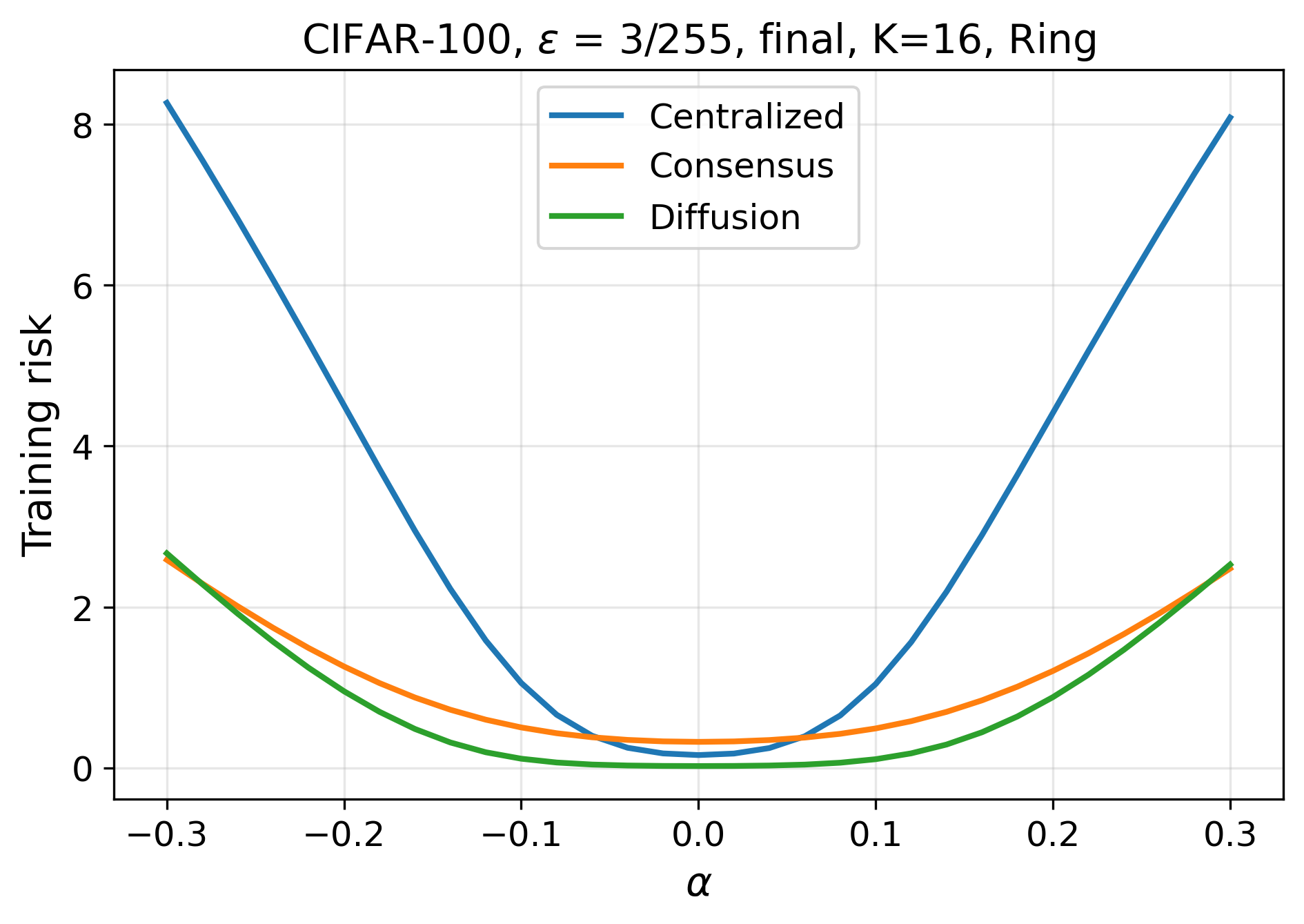}}
\subfigure[CIFAR-100: B=128, K=16]
{\includegraphics[width=.3\linewidth]
{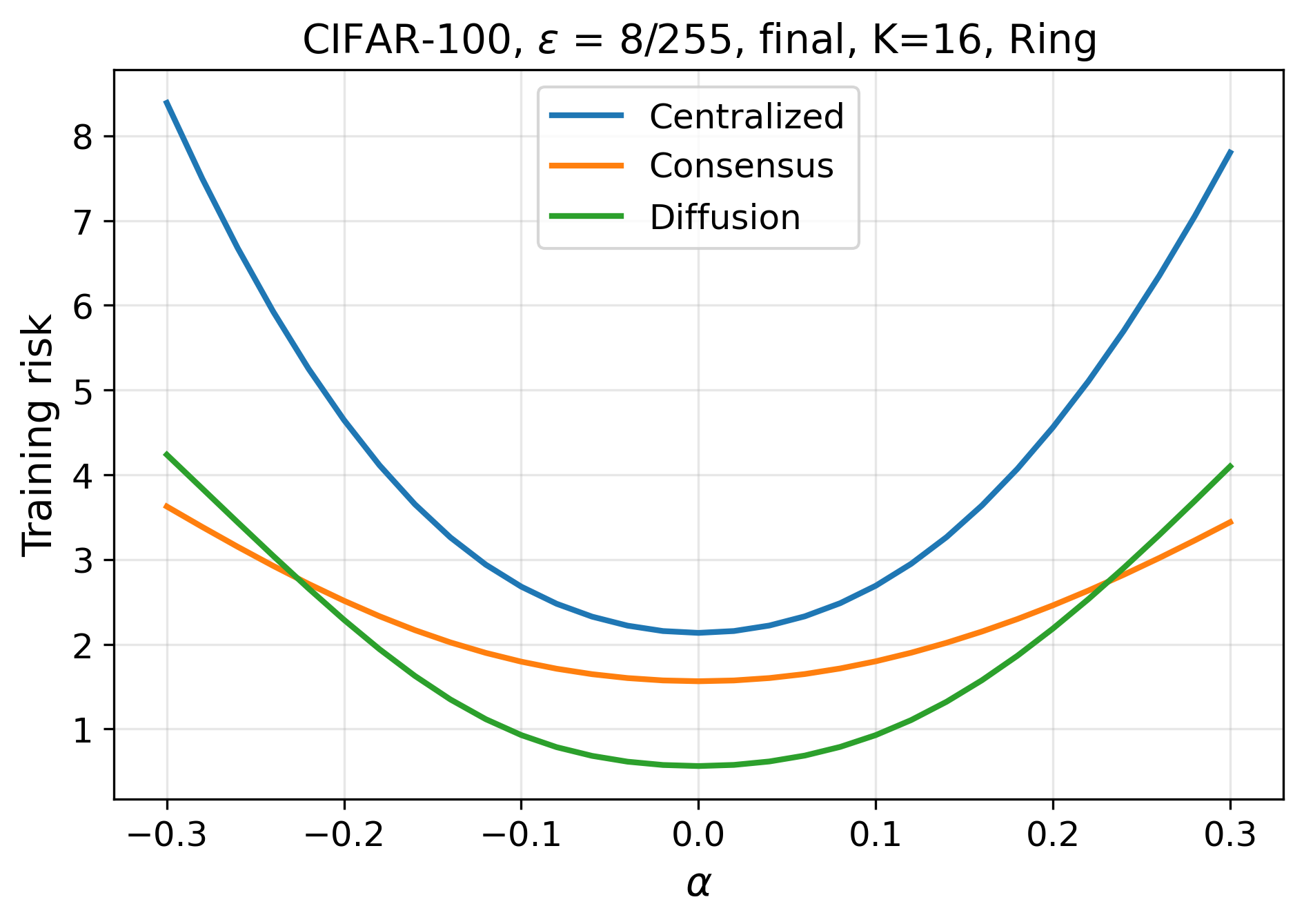}}

\subfigure[CIFAR-100: B=128, K=32]{\includegraphics[width=.3\linewidth]
{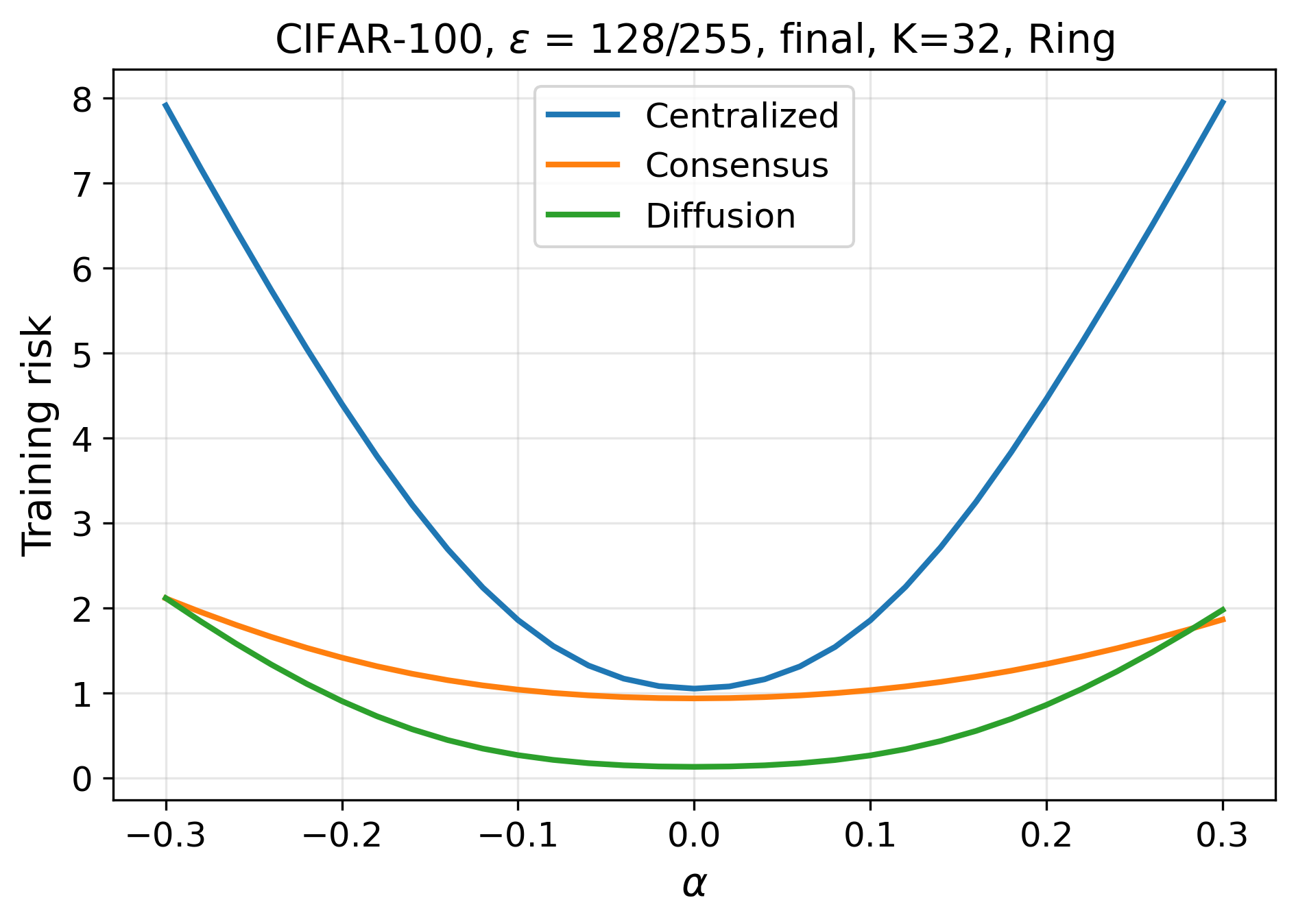}}
\subfigure[CIFAR-100: B=128, K=32]
{\includegraphics[width=.3\linewidth]{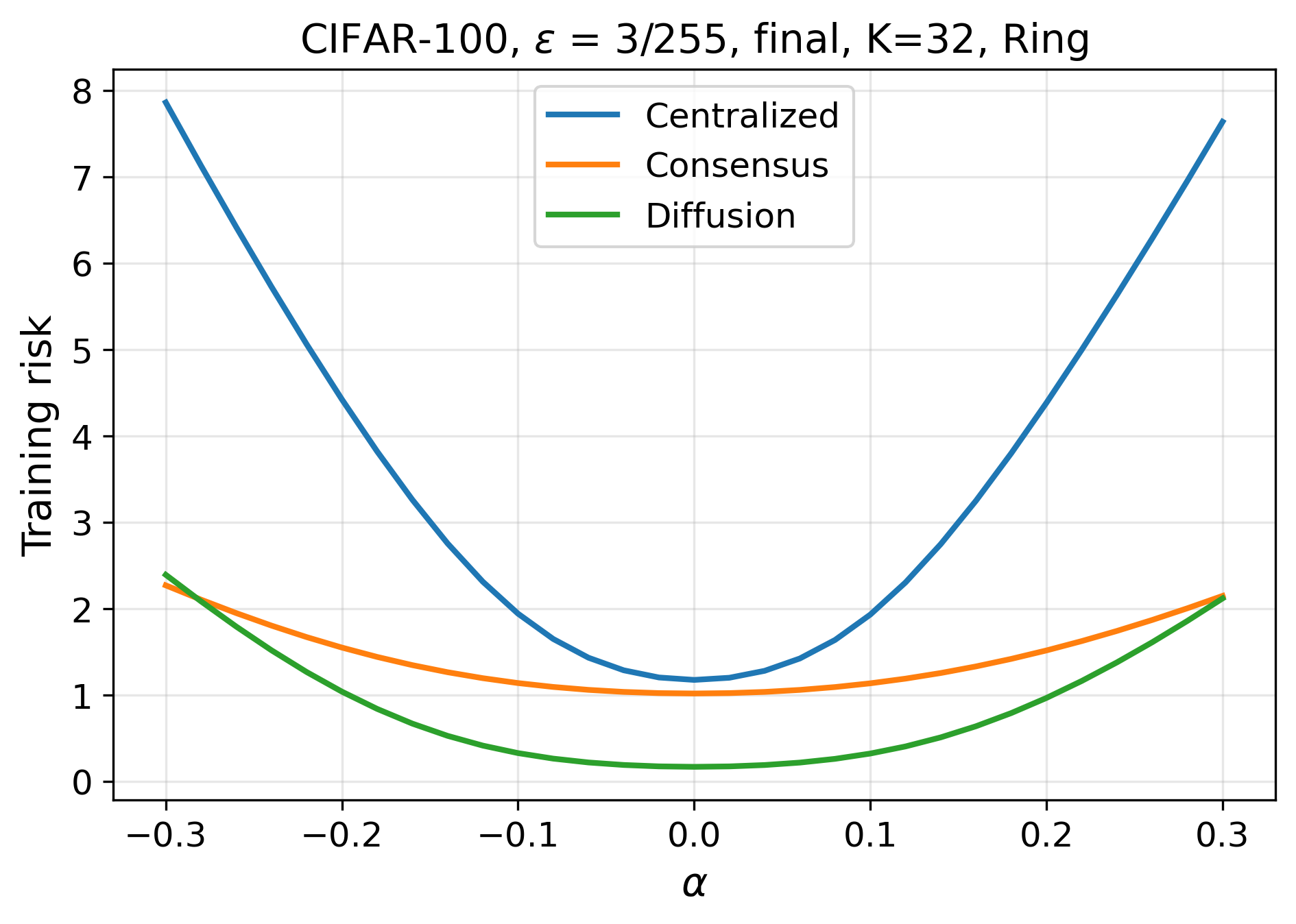}}
\subfigure[CIFAR-100: B=128, K=32]
{\includegraphics[width=.3\linewidth]
{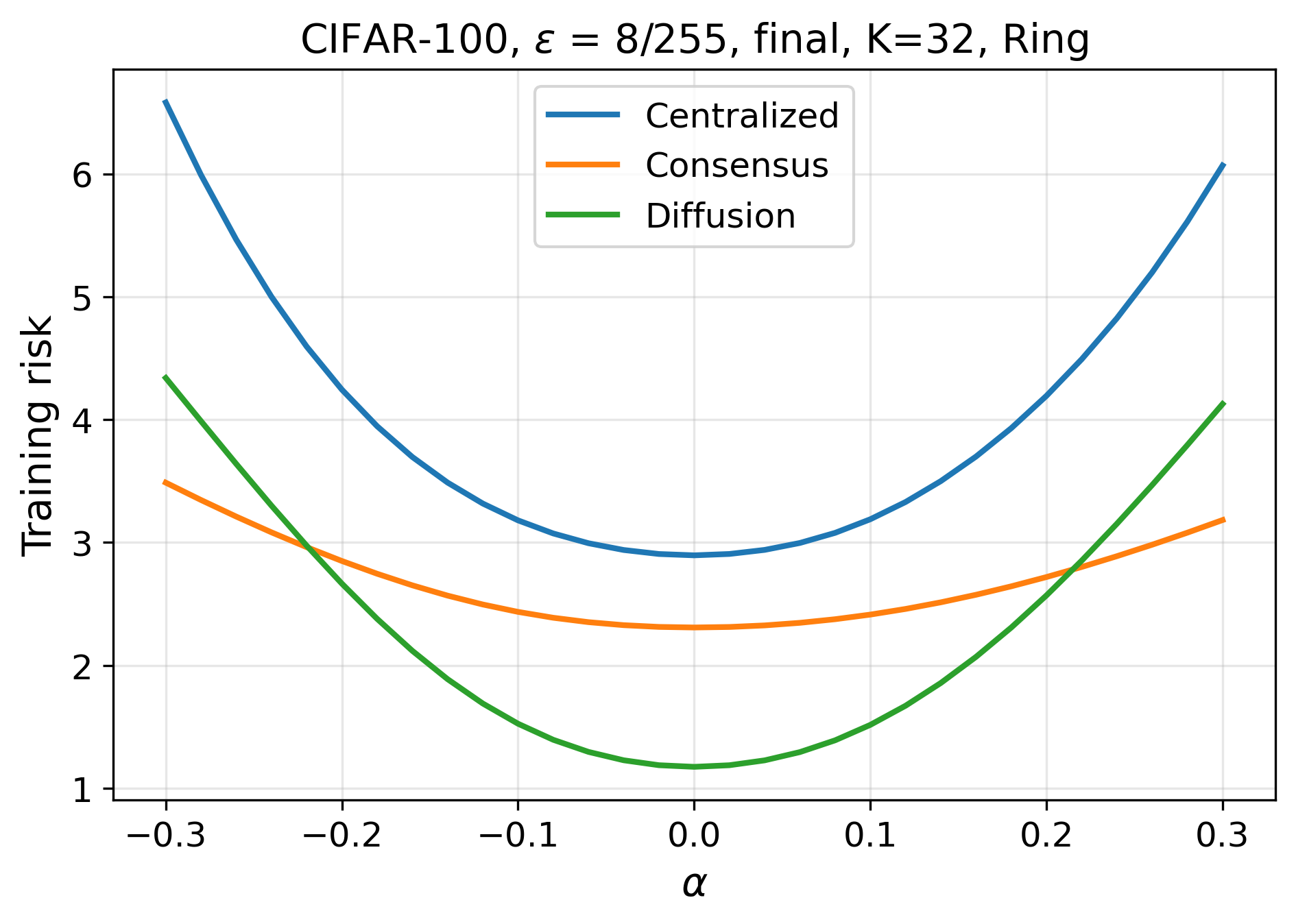}}
\caption{Flatness visualization of the \emph{final} models trained on the CIFAR-100 dataset using the \emph{ring} topology with $K=\{8,16,32\}$. The results shows a similar phenomenon to that observed in Figure \ref{loss_visual_ring_cifar10}. }
\label{loss_visual_ring_cifar100}
\end{center}
\end{figure*}

\begin{figure*}[htbp]
\begin{center}
\subfigure[CIFAR-100: B=128, K=16]{\includegraphics[width=.3\linewidth]
{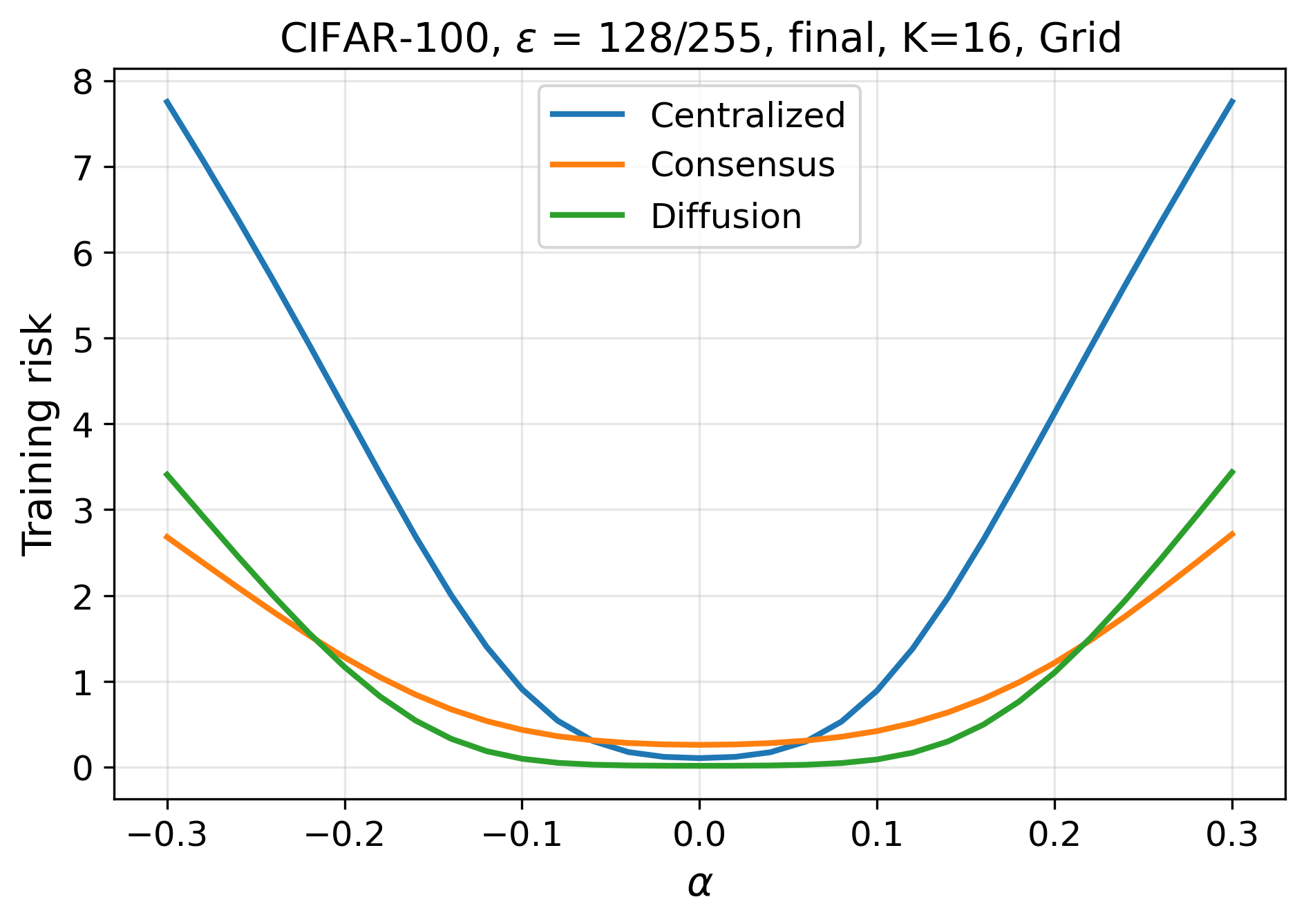}}
\subfigure[CIFAR-100: B=128, K=16]
{\includegraphics[width=.3\linewidth]{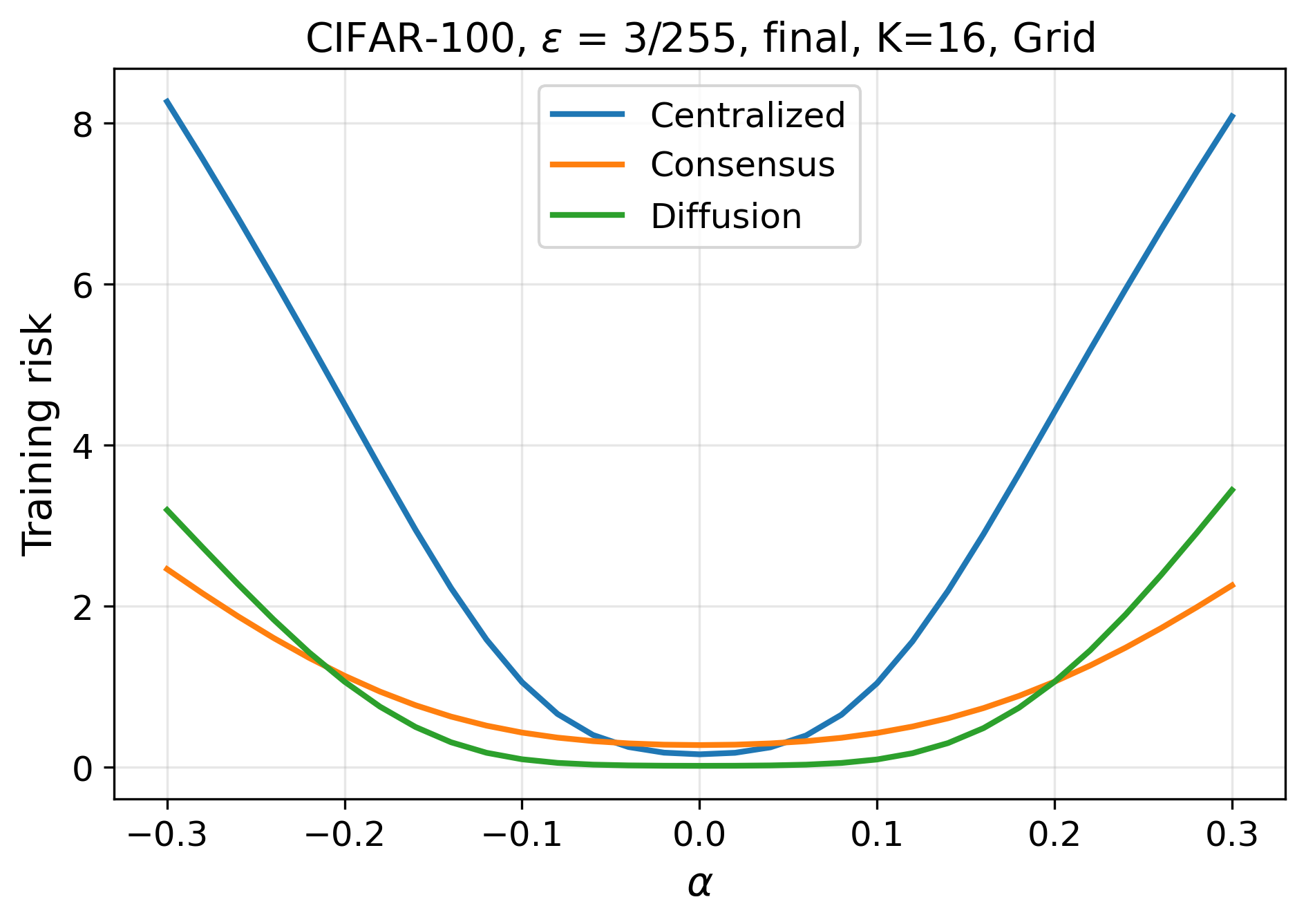}}
\subfigure[CIFAR-100: B=128, K=16]
{\includegraphics[width=.3\linewidth]
{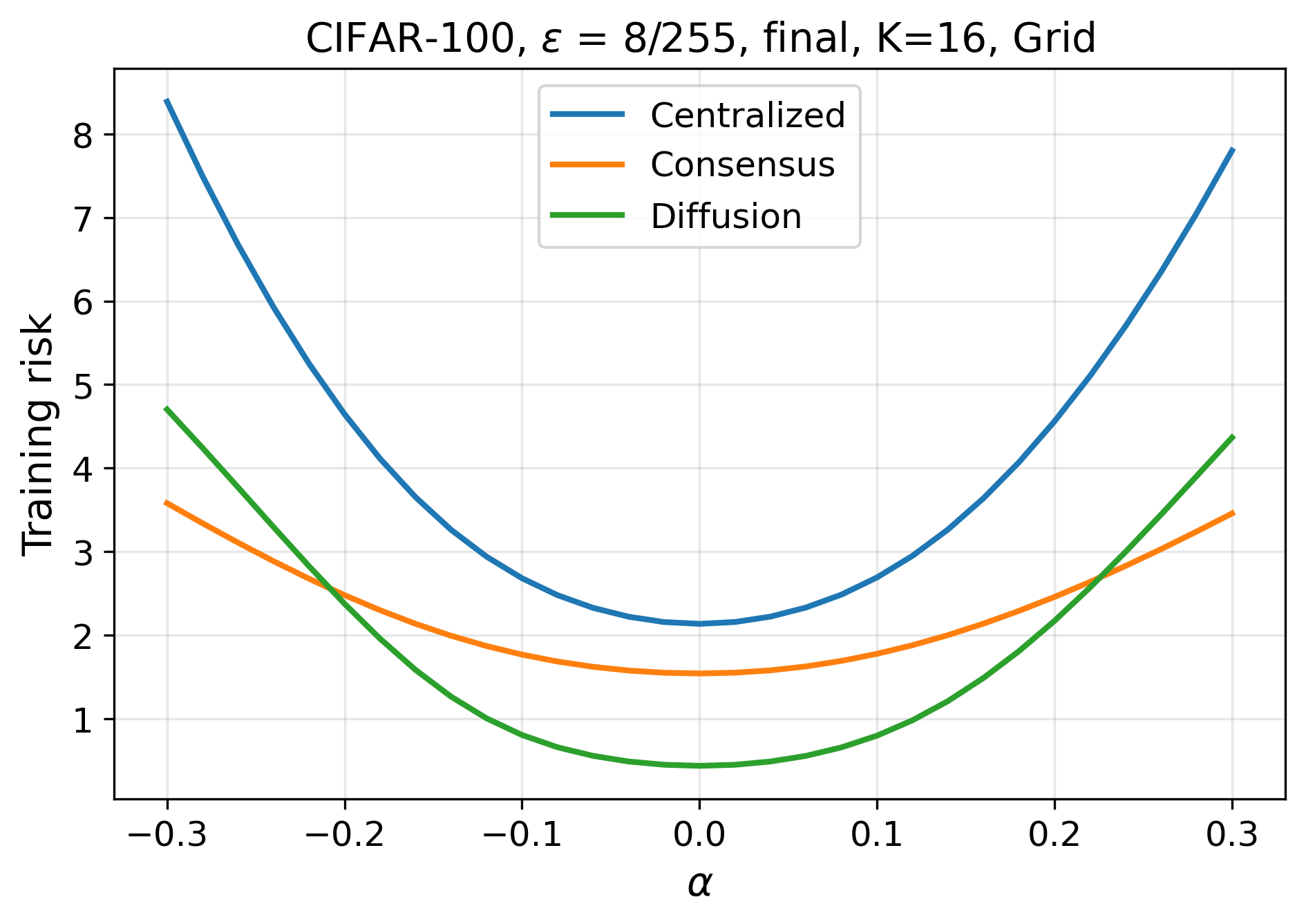}}
\caption{Flatness visualization of the \emph{final} models trained on the CIFAR-100 dataset using the \emph{grid} topology with $K=16$. The results are similar to Figure \ref{loss_visual_grid_cifar10}.}
\label{loss_visual_grid_cifar100}
\end{center}
\end{figure*}

\begin{figure*}[htbp]
\begin{center}
\subfigure[CIFAR-10: B=128]{\includegraphics[width=.3\linewidth]
{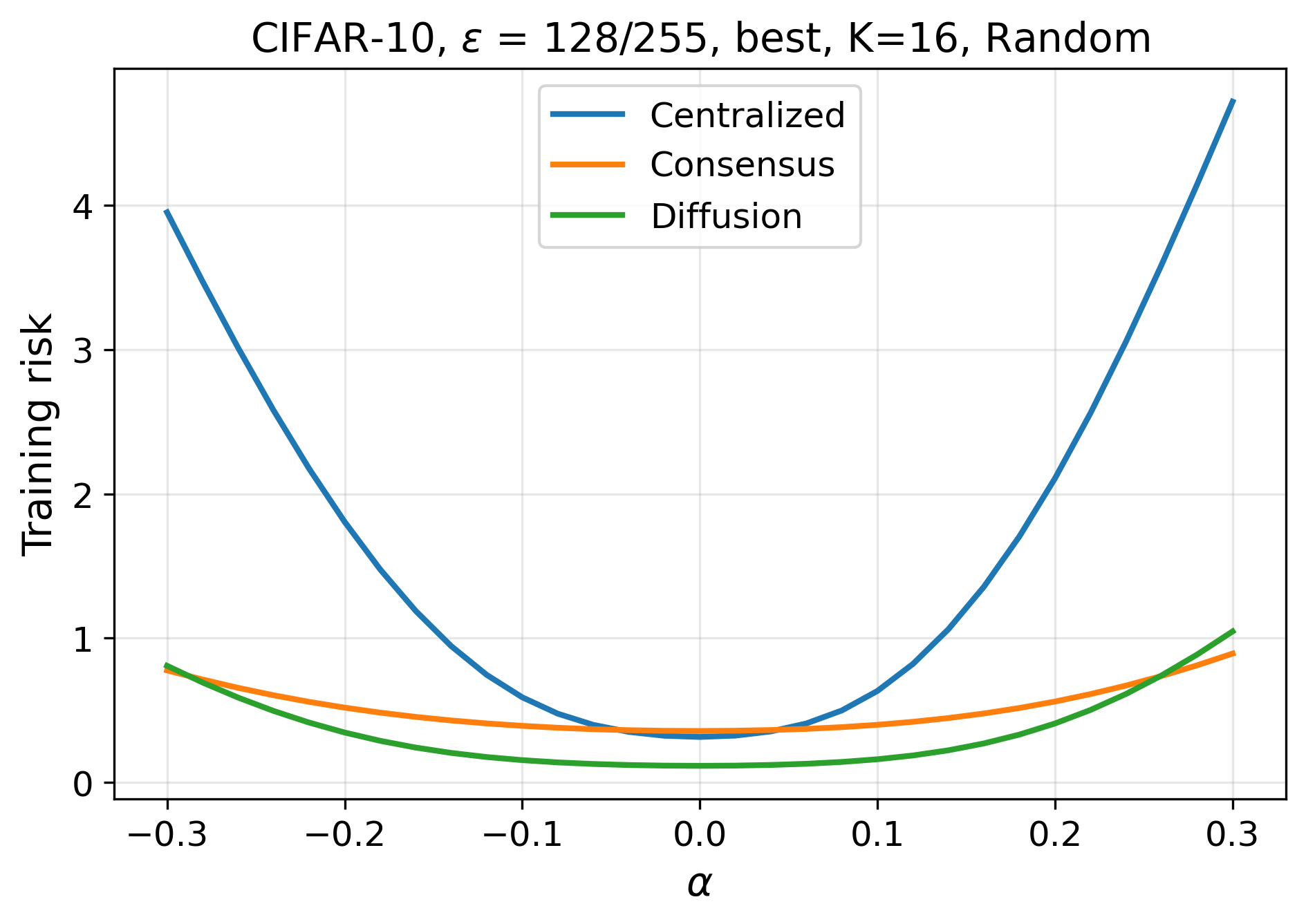}}
\subfigure[CIFAR-10: B=128]
{\includegraphics[width=.3\linewidth]{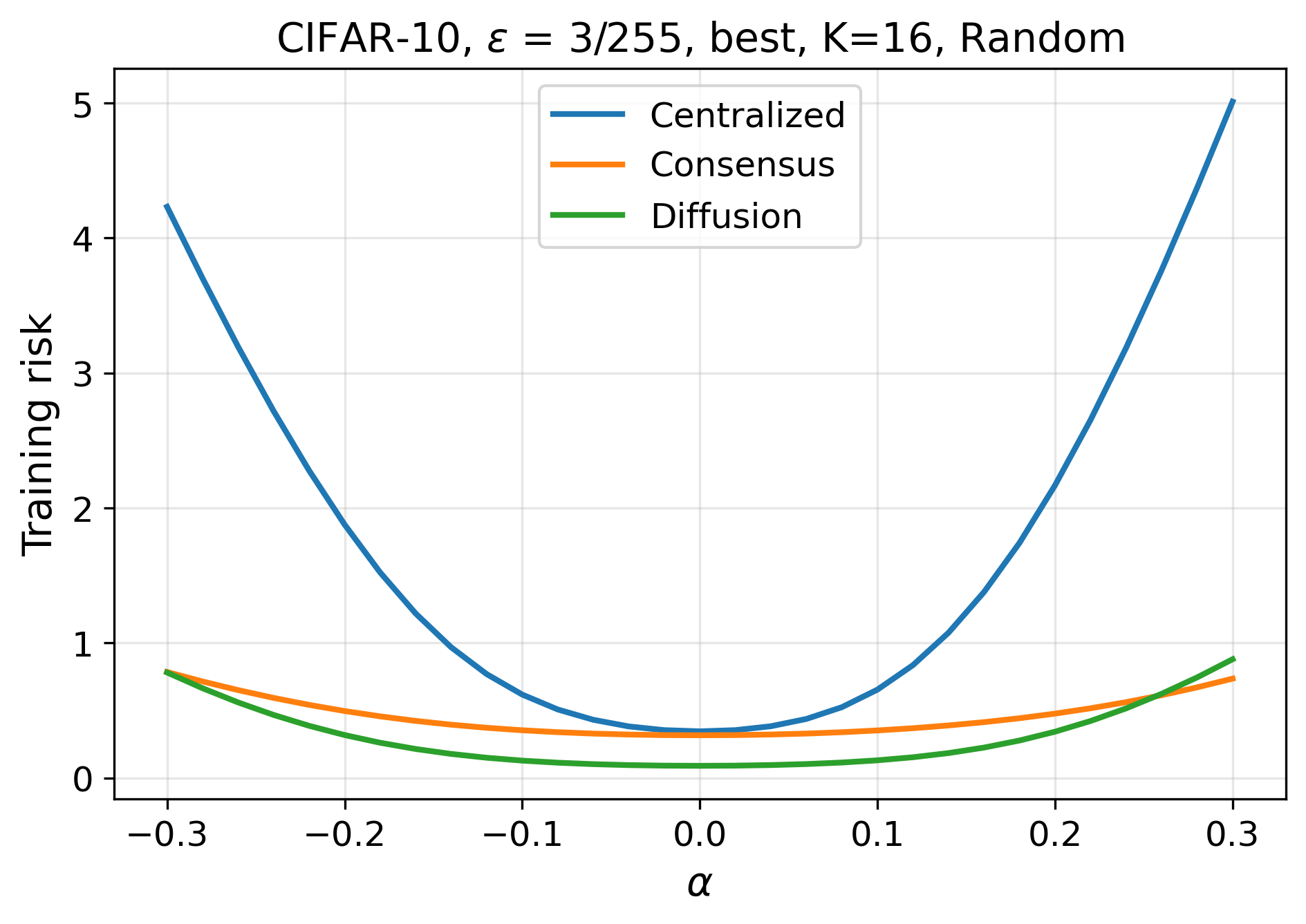}}
\subfigure[CIFAR-10: B=128]
{\includegraphics[width=.3\linewidth]
{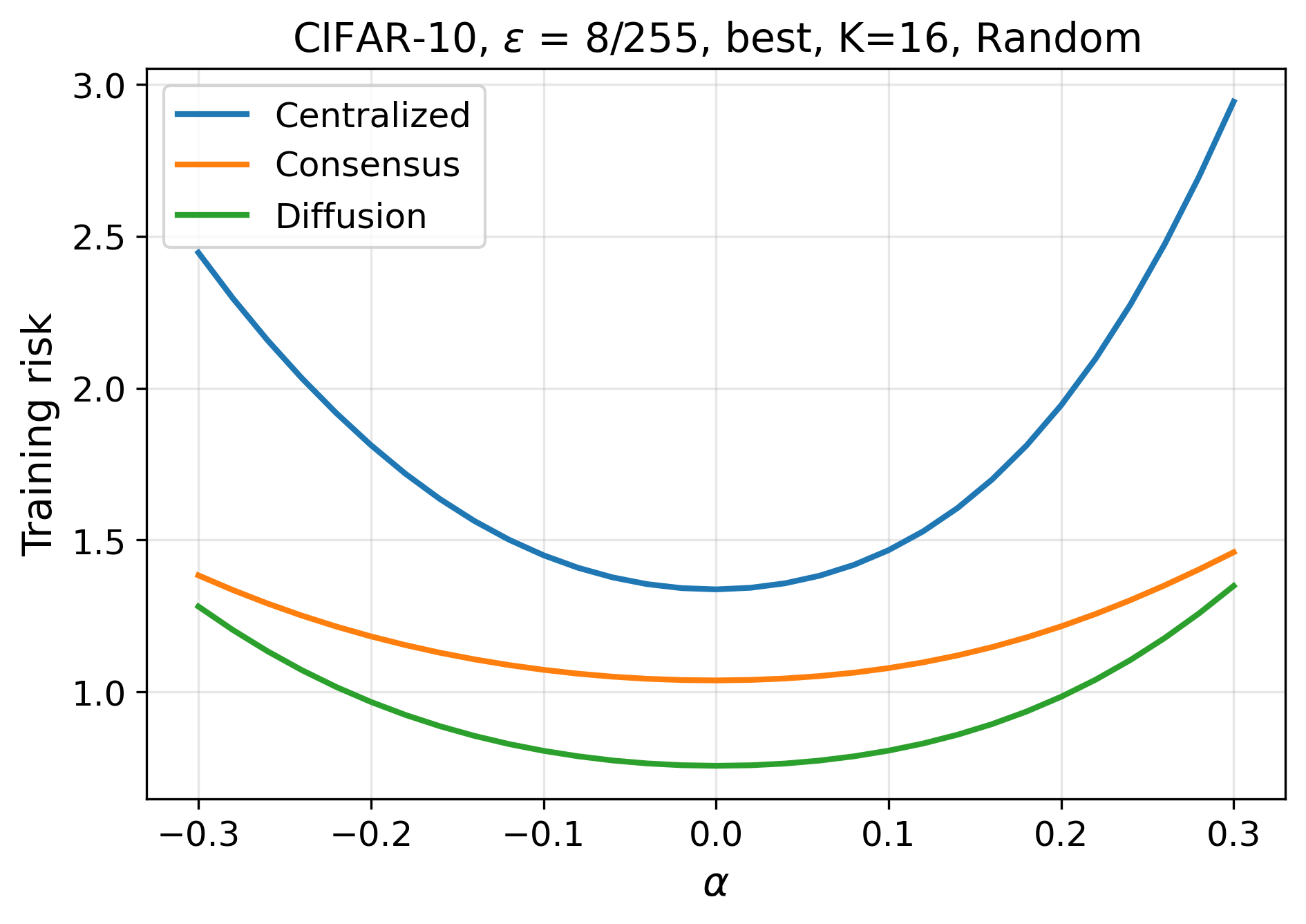}}

\subfigure[CIFAR-10: B=256]{\includegraphics[width=.3\linewidth]
{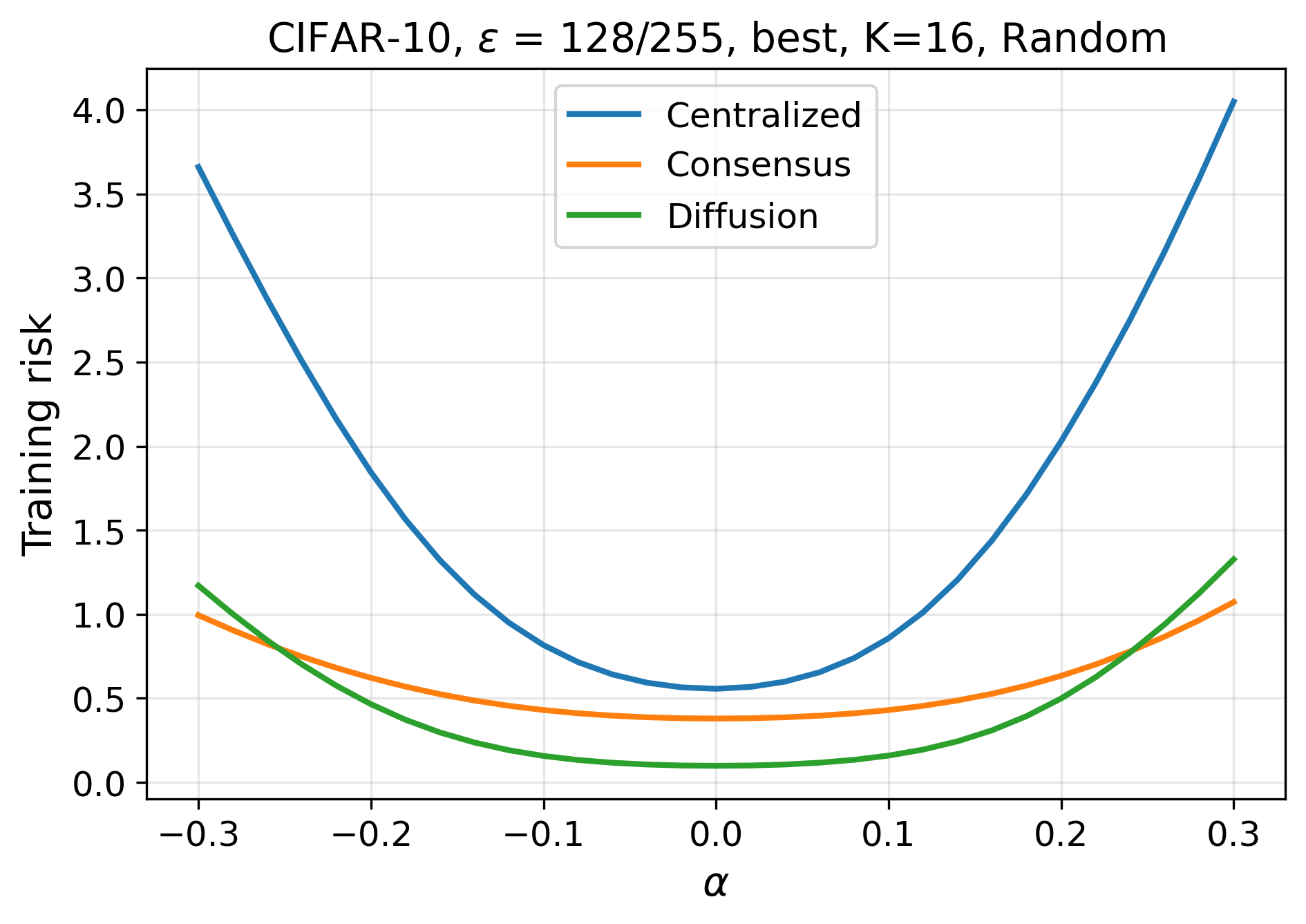}}
\subfigure[CIFAR-10: B=256 ]
{\includegraphics[width=.3\linewidth]{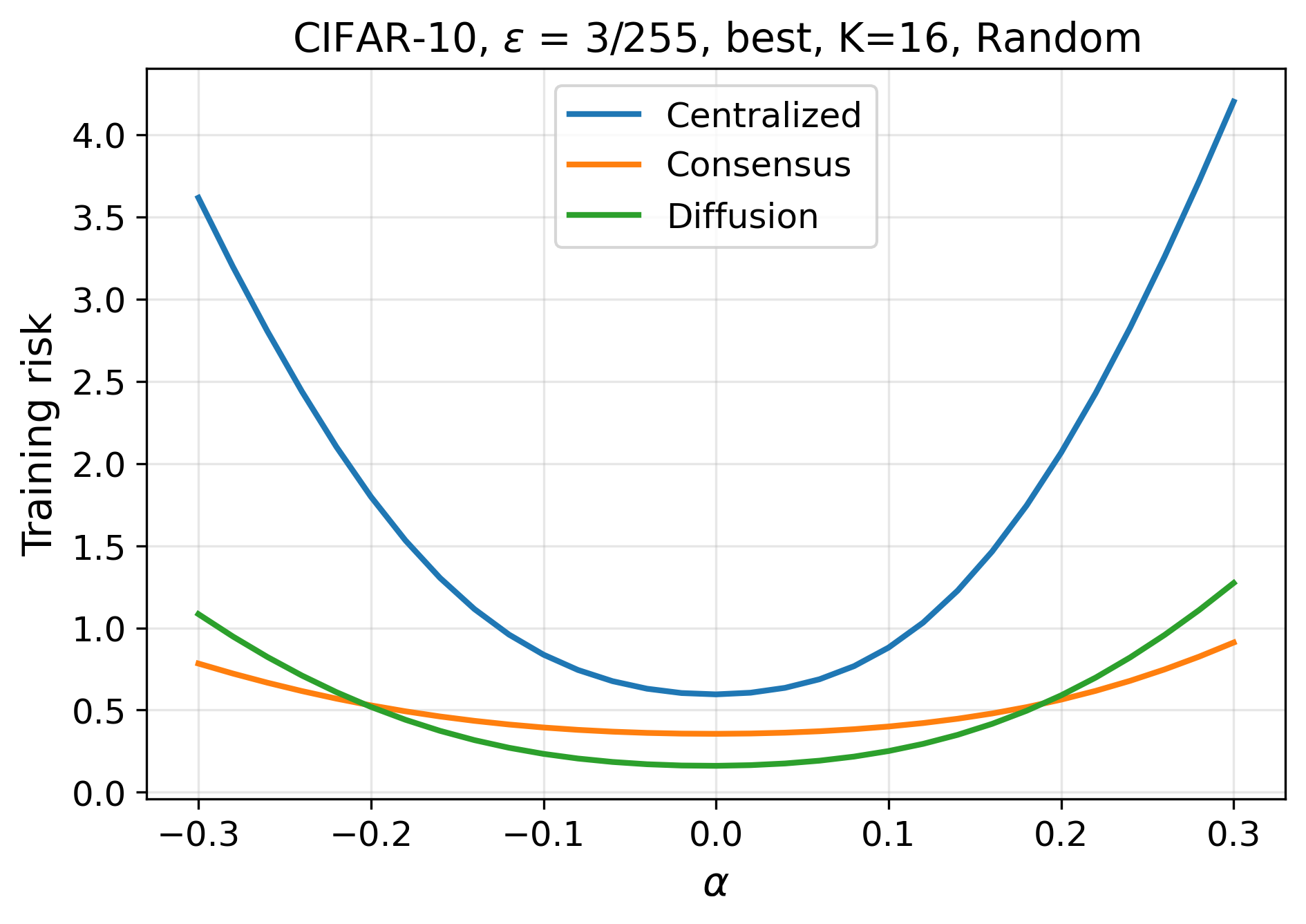}}
\subfigure[CIFAR-10: B=256 ]
{\includegraphics[width=.3\linewidth]
{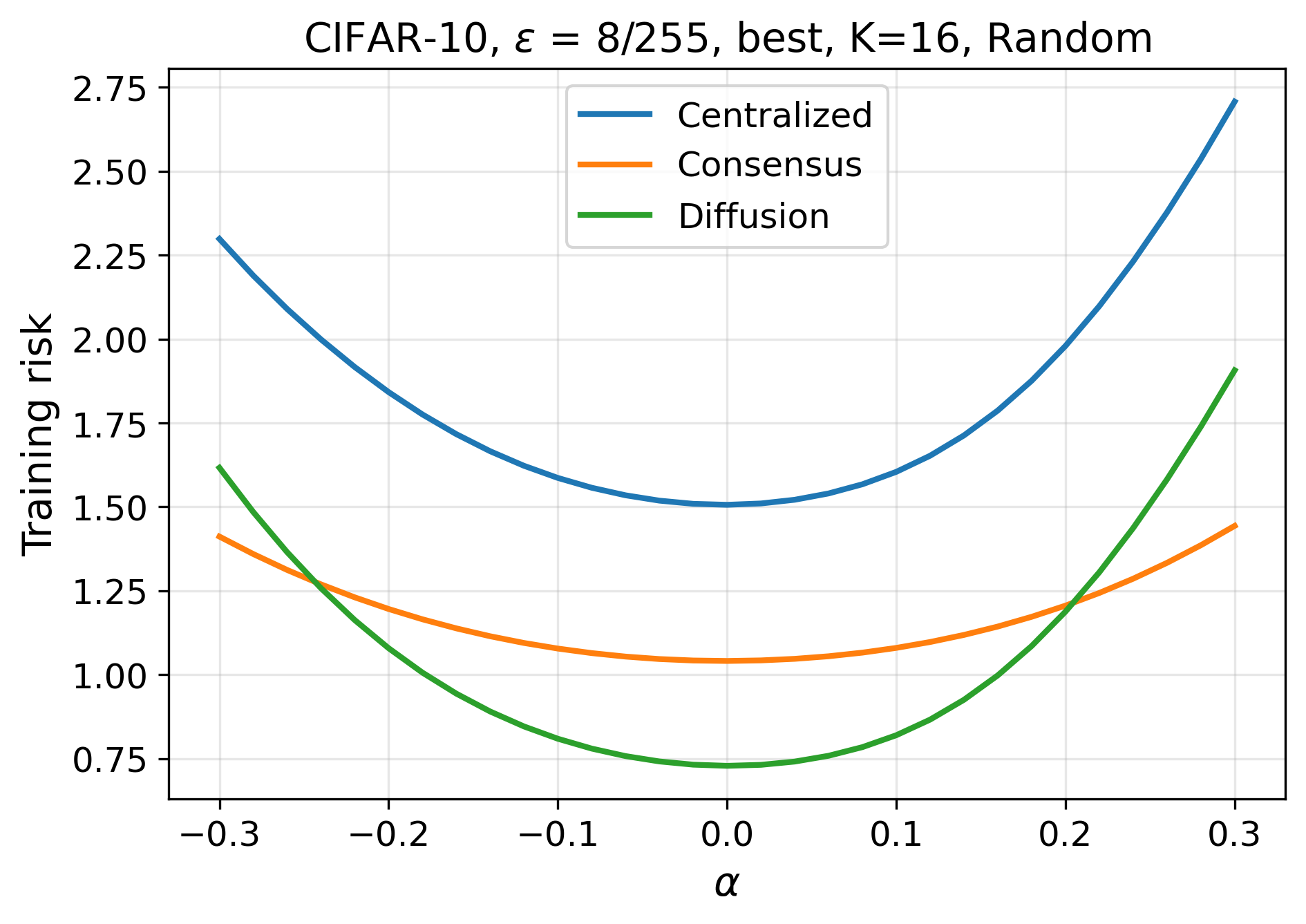}}
\caption{Flatness visualization of the \emph{best} models trained on the CIFAR-10 dataset using the \emph{random} topology with $K=16$. The results are similar to Figure \ref{loss_visual_cifar10_sgd}.}
\label{loss_visual_cifar10_sgd_best}
\end{center}
\end{figure*}

\begin{figure*}[htbp]
\begin{center}
\subfigure[CIFAR-100: B=128]
{\includegraphics[width=.3\linewidth]
{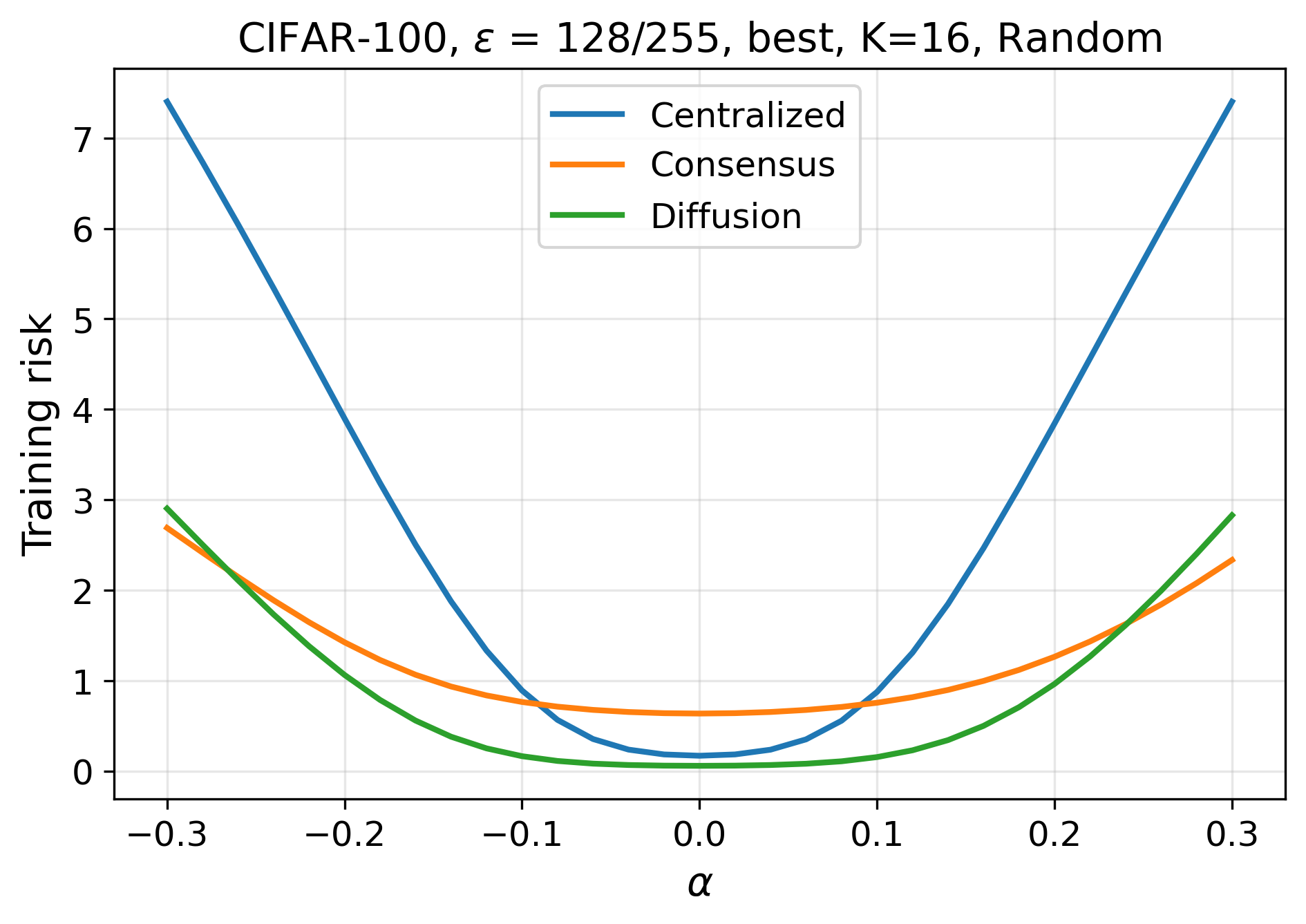}}
\subfigure[CIFAR-100: B=128]
{\includegraphics[width=.3\linewidth]
{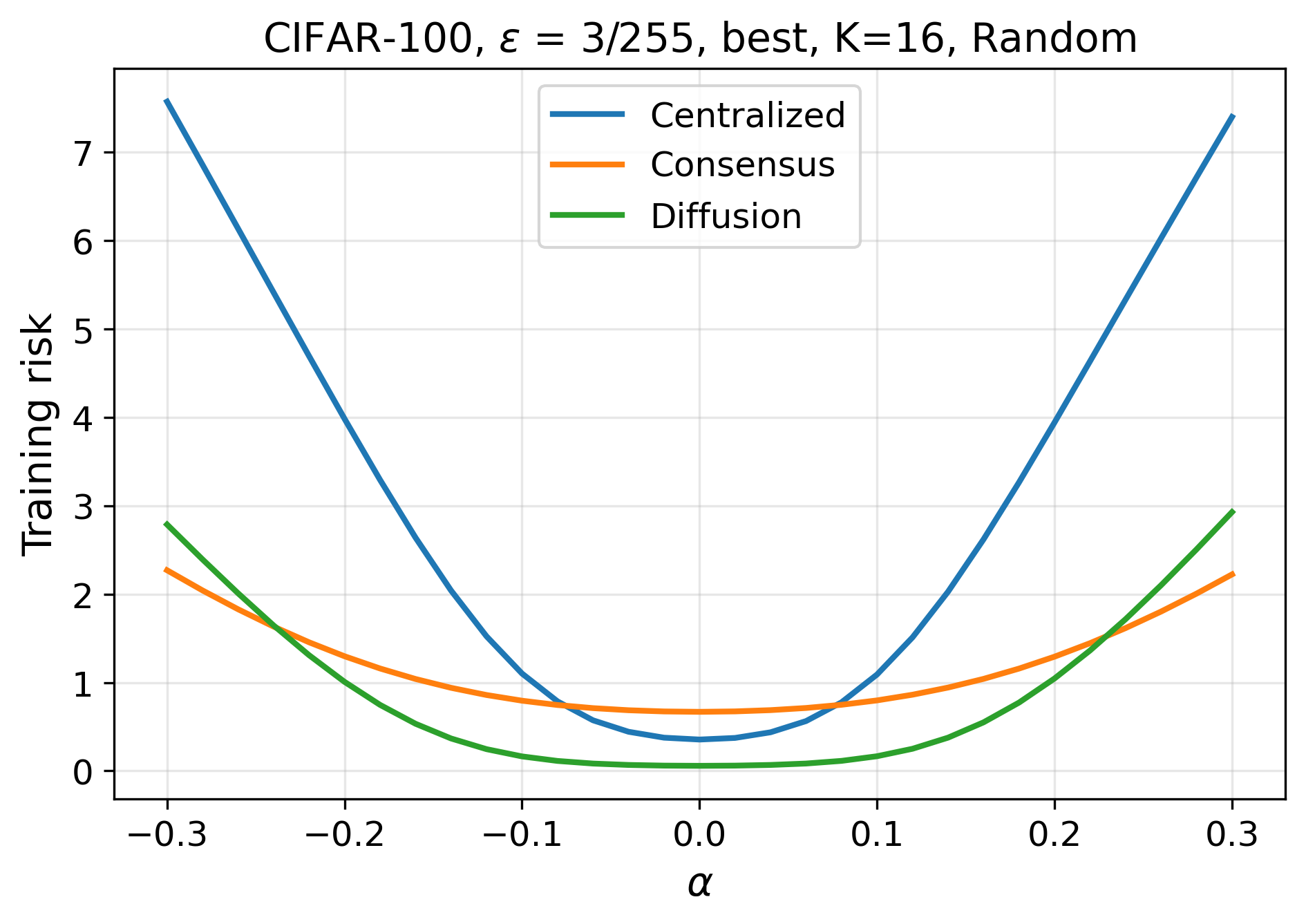}}
\subfigure[CIFAR-100: B=128]
{\includegraphics[width=.3\linewidth]
{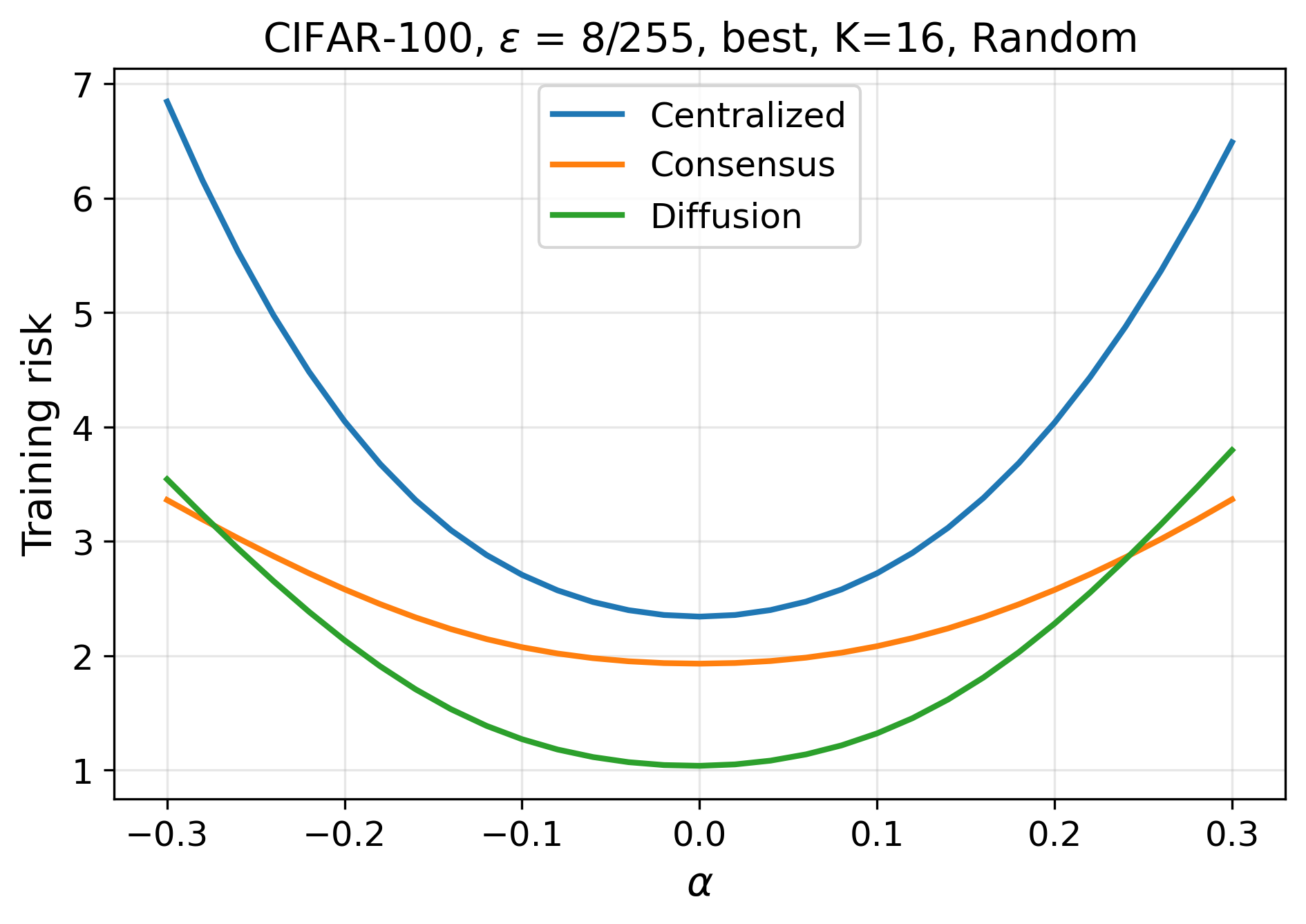}}
\subfigure[CIFAR-100: B=256 ]
{\includegraphics[width=.3\linewidth]
{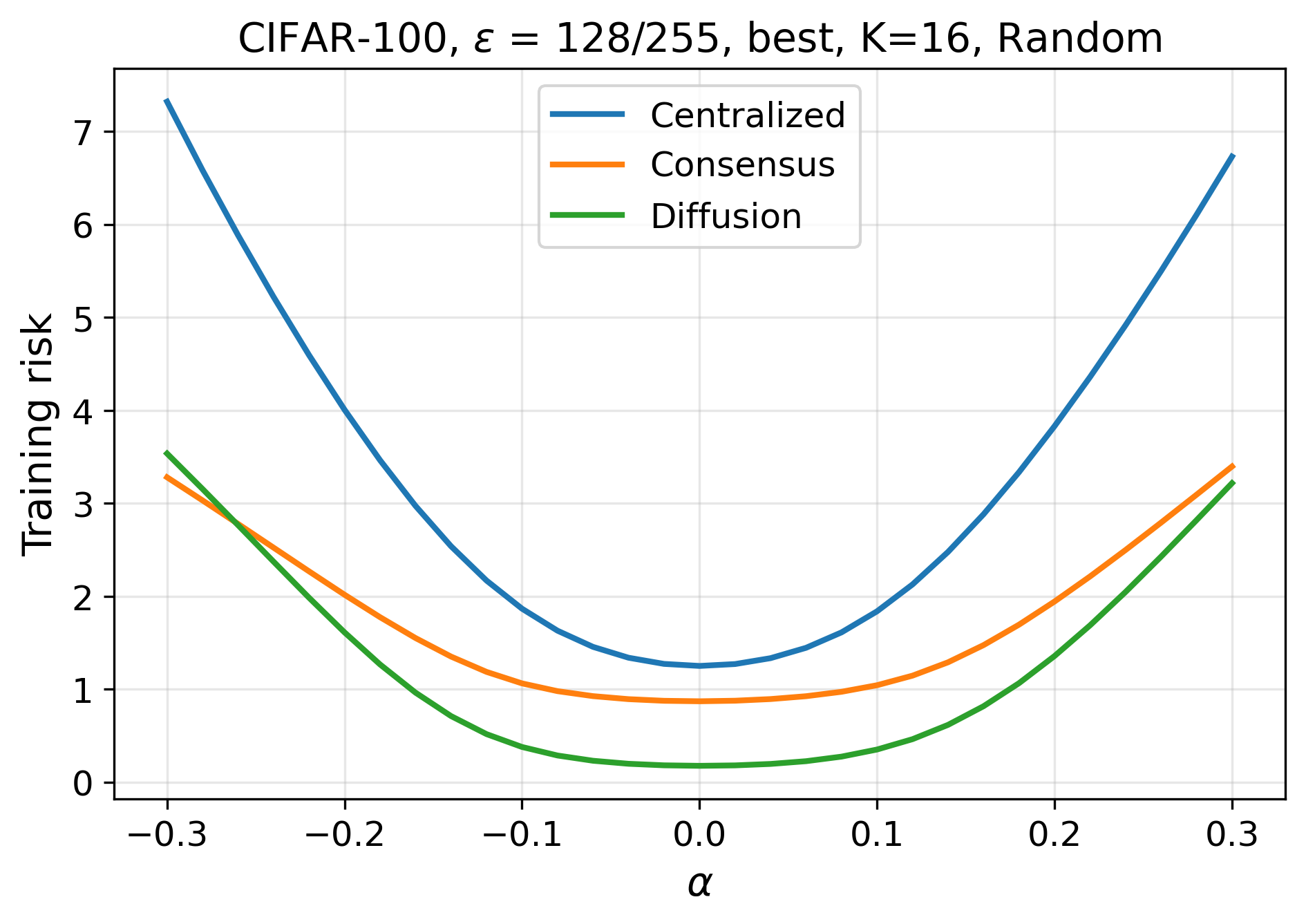}}
\subfigure[CIFAR-100: B=256]
{\includegraphics[width=.3\linewidth]
{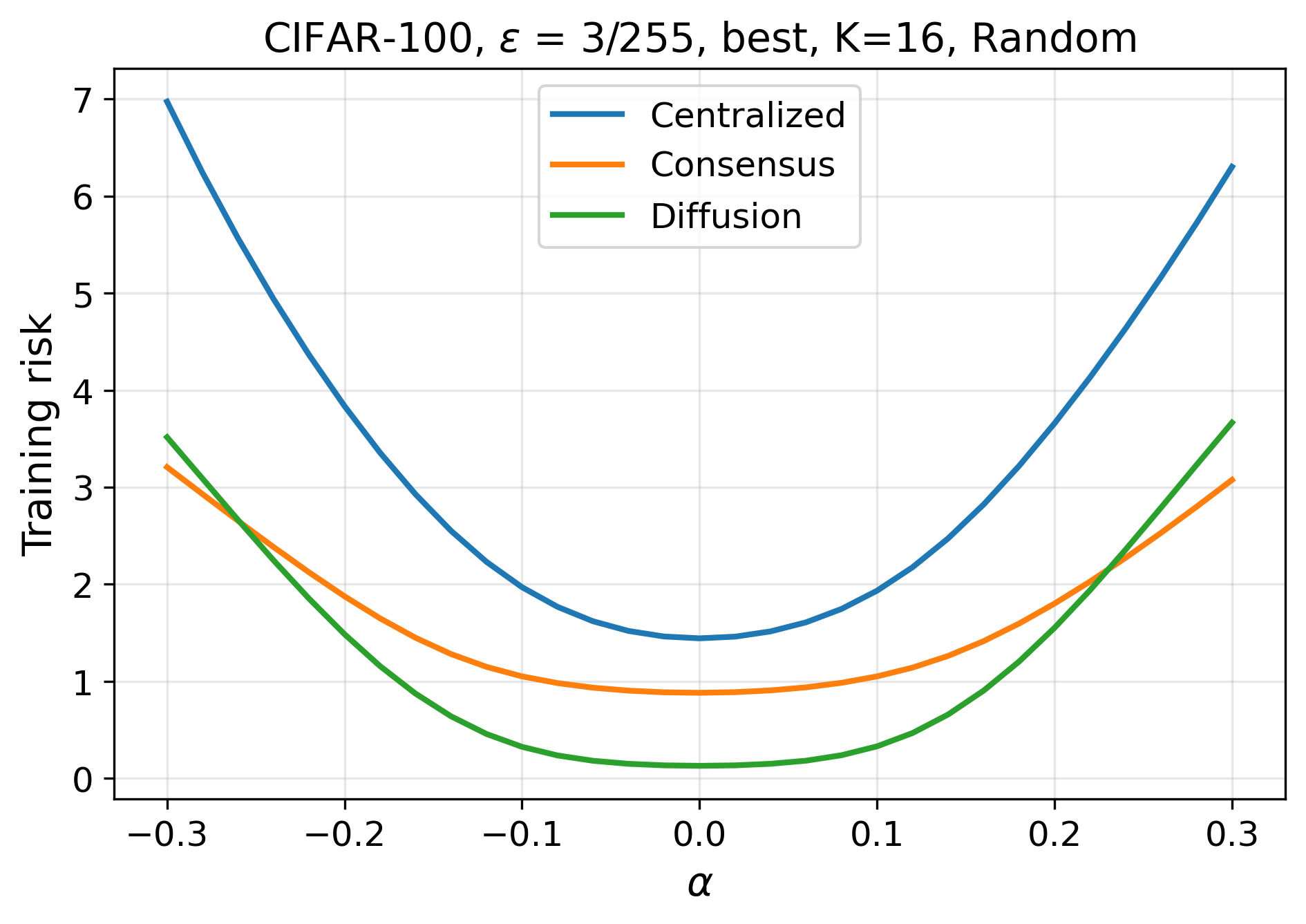}}
\subfigure[CIFAR-100: B=256]
{\includegraphics[width=.3\linewidth]
{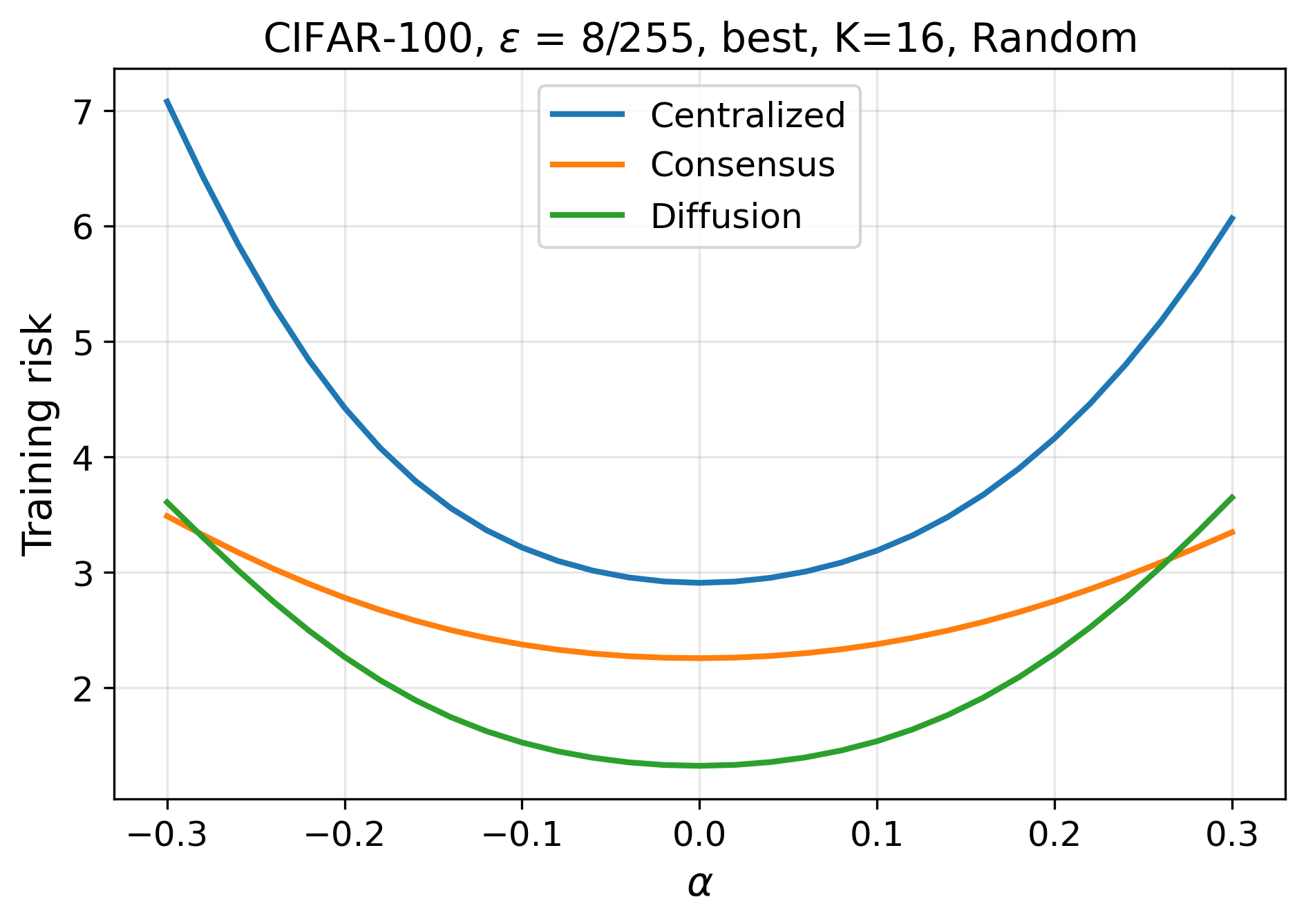}}
\caption{Flatness visualization of the \emph{best} models trained on the CIFAR-100 dataset using the \emph{random} topology with $K=16$ nodes. The results are similar to Figure \ref{loss_visual_cifar100_sgd}.}
\label{loss_visual_cifar100_sgd_best}
\end{center}
\end{figure*}

\begin{figure*}[htbp]
\begin{center}
\subfigure[Tiny ImageNet-200: B=128]{\includegraphics[width=.3\linewidth]
{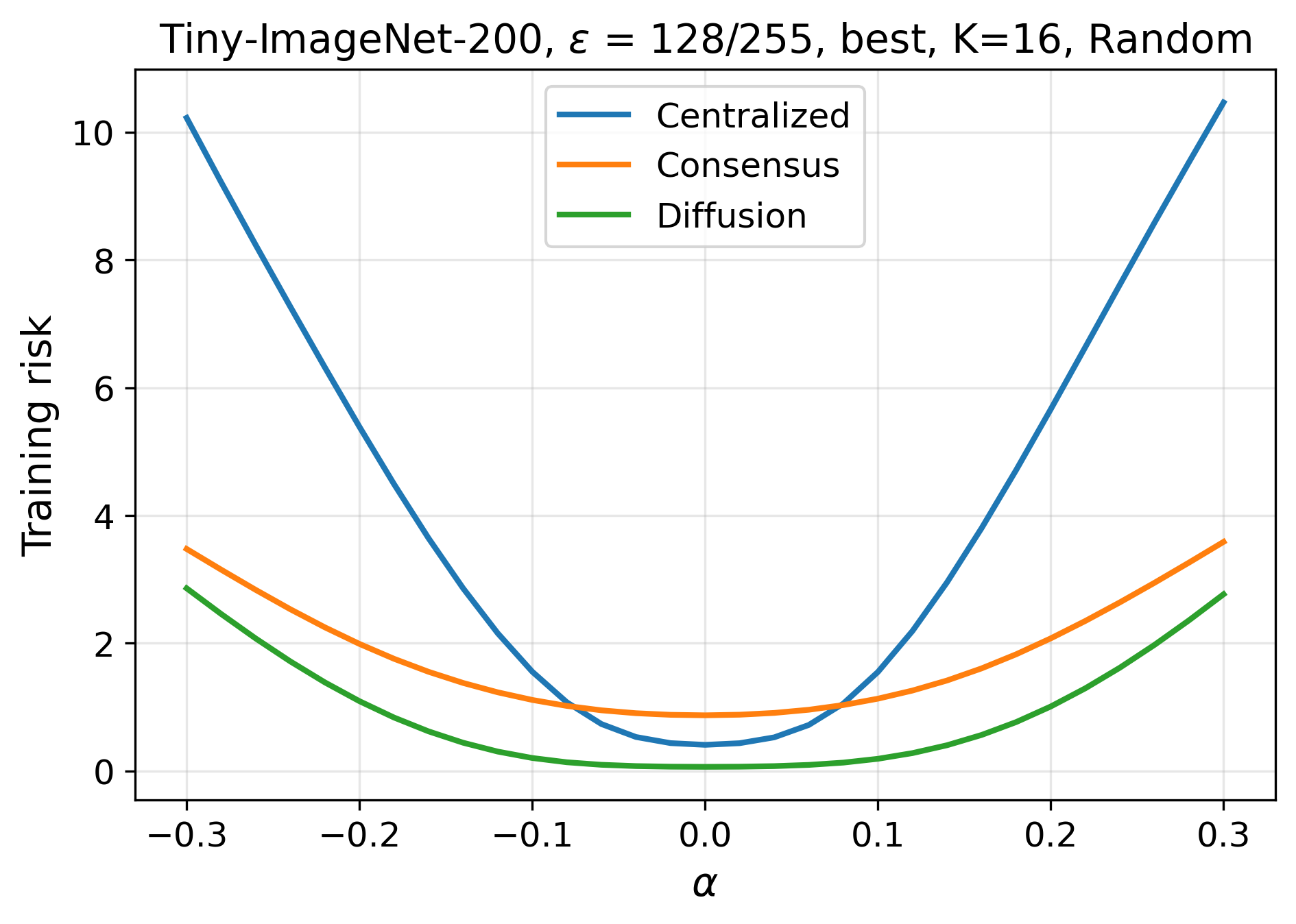}}
\subfigure[Tiny ImageNet-200: B=128]
{\includegraphics[width=.3\linewidth]
{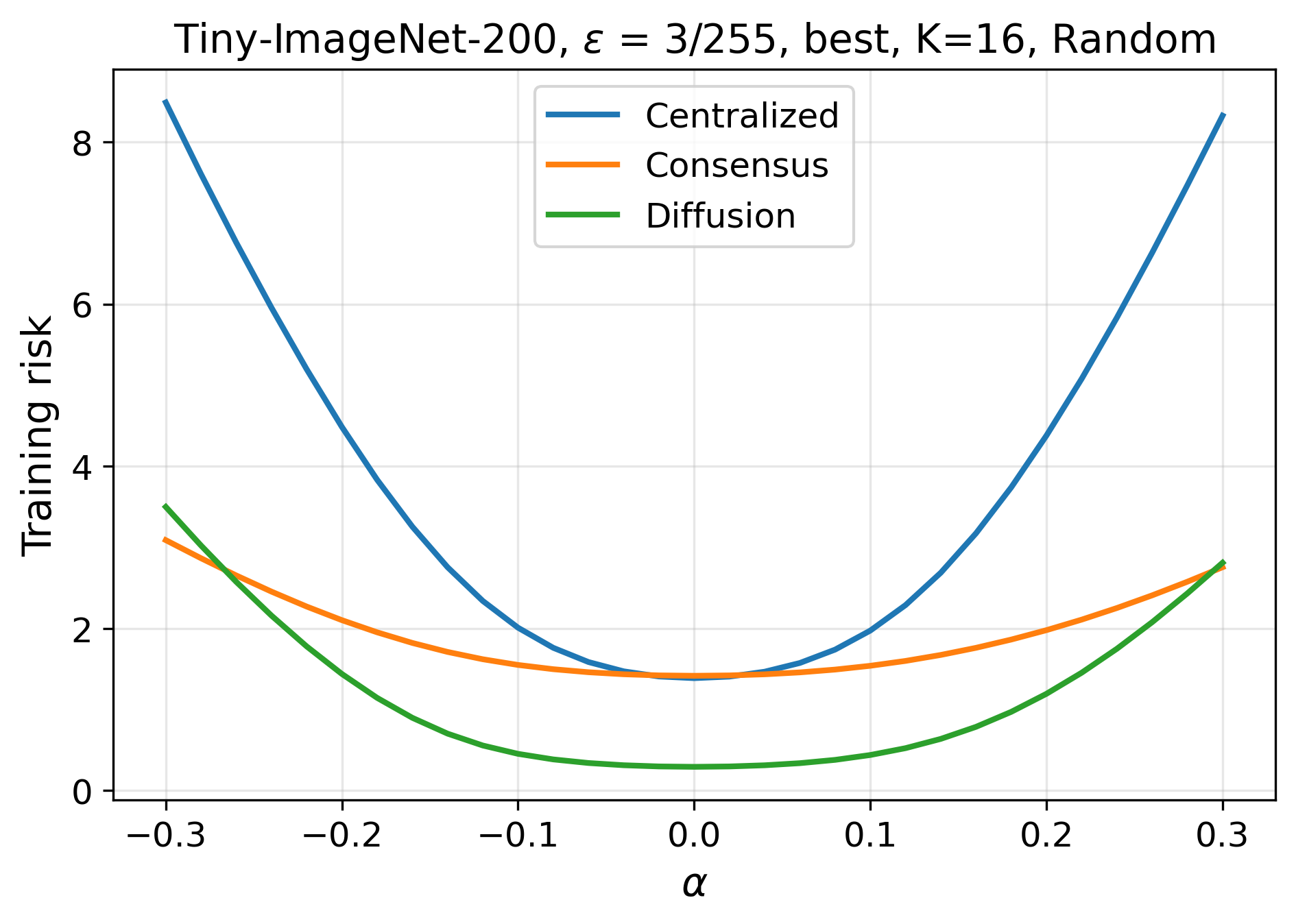}}
\caption{Flatness visualization of the \emph{best} models trained on the Tiny ImageNet-200 dataset using the \emph{random} topology with $K=16$ nodes. The results are similar to Figure \ref{loss_visual_imagenet_sgd}.}
\label{loss_visual_imagenet_sgd_best}
\end{center}
\end{figure*}

\begin{figure*}[htbp]
\begin{center}
\subfigure[CIFAR-10: B=128, K=8]{\includegraphics[width=.3\linewidth]
{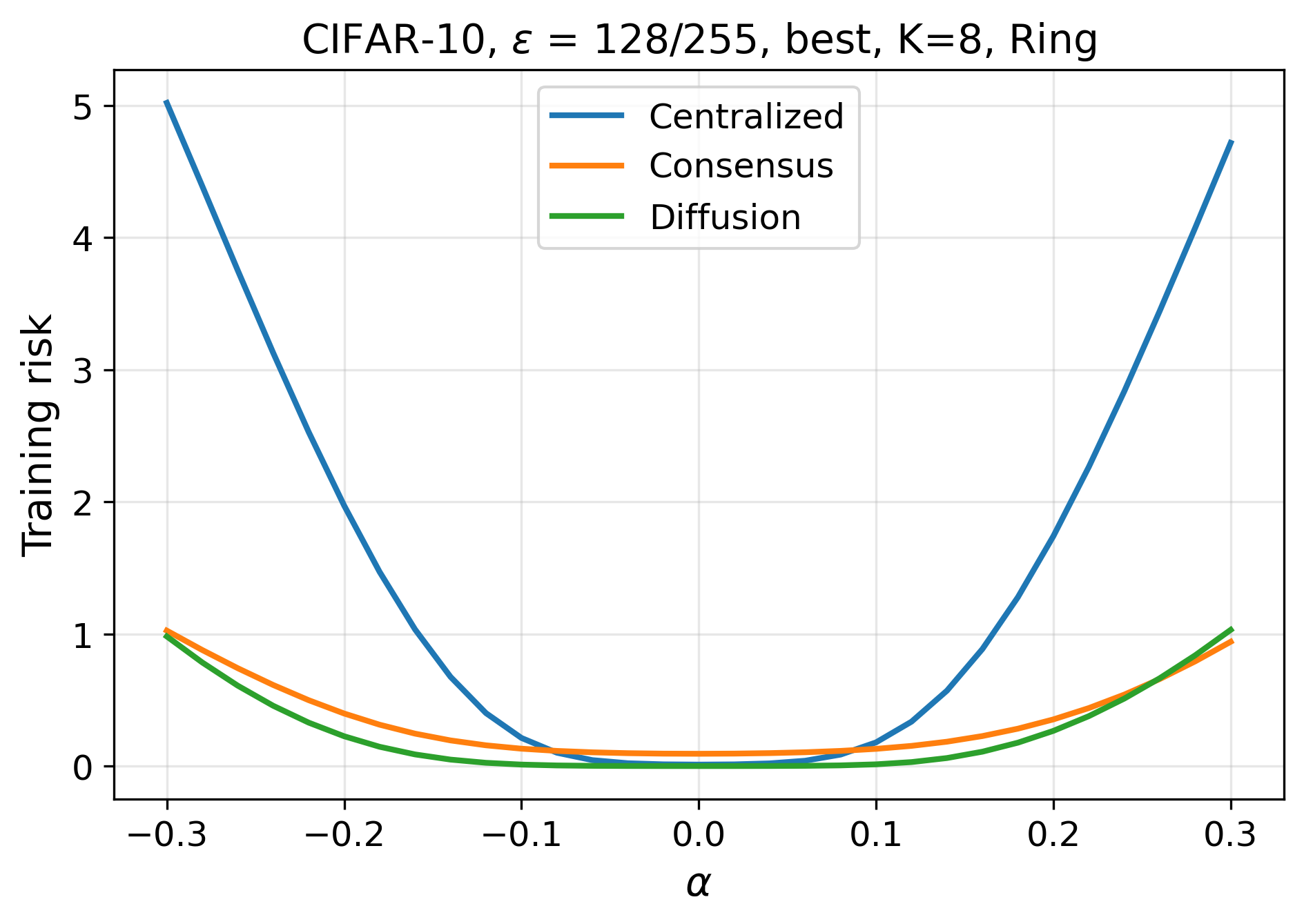}}
\subfigure[CIFAR-10: B=128, K=8]
{\includegraphics[width=.3\linewidth]{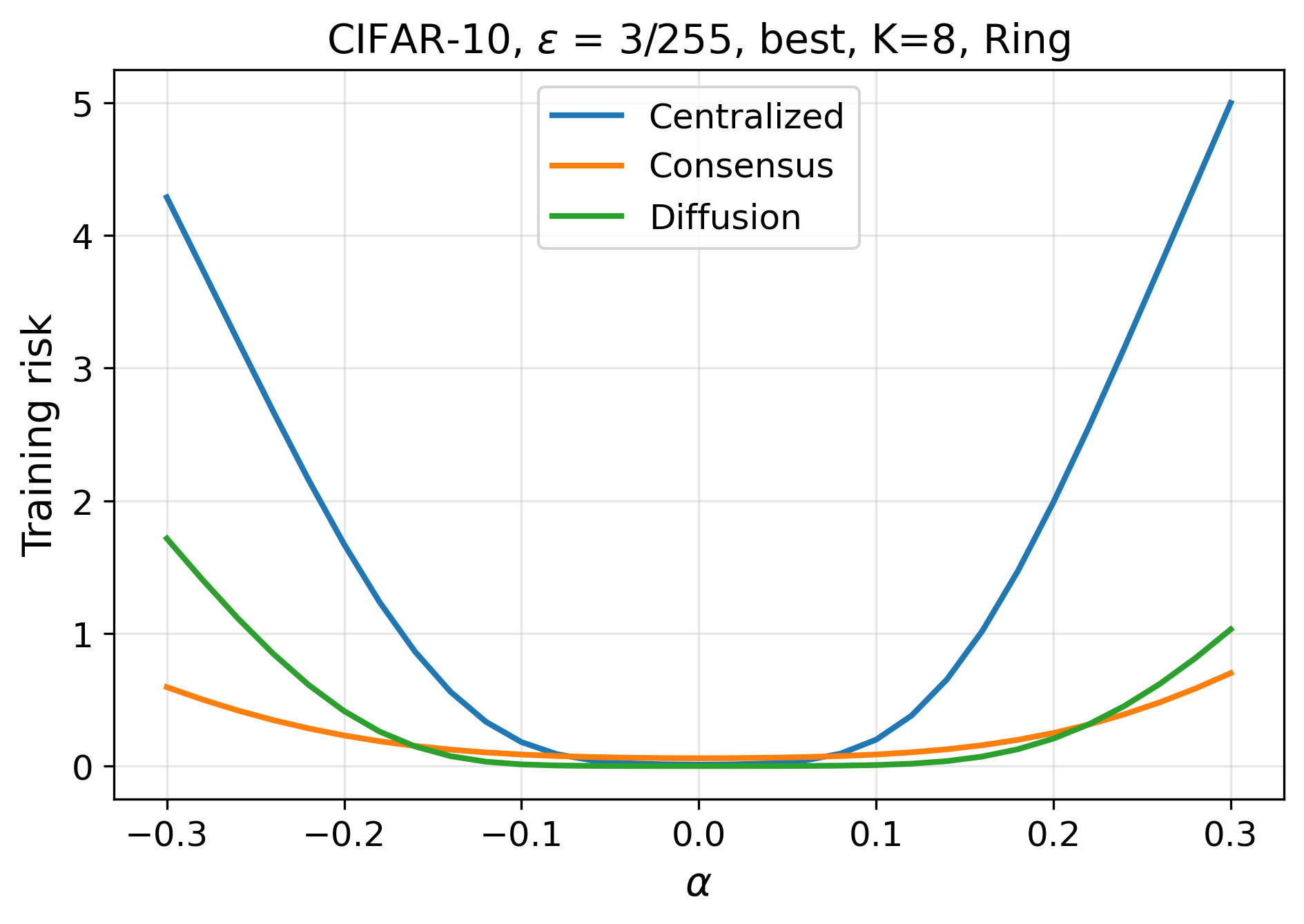}}
\subfigure[CIFAR-10: B=128, K=8]
{\includegraphics[width=.3\linewidth]
{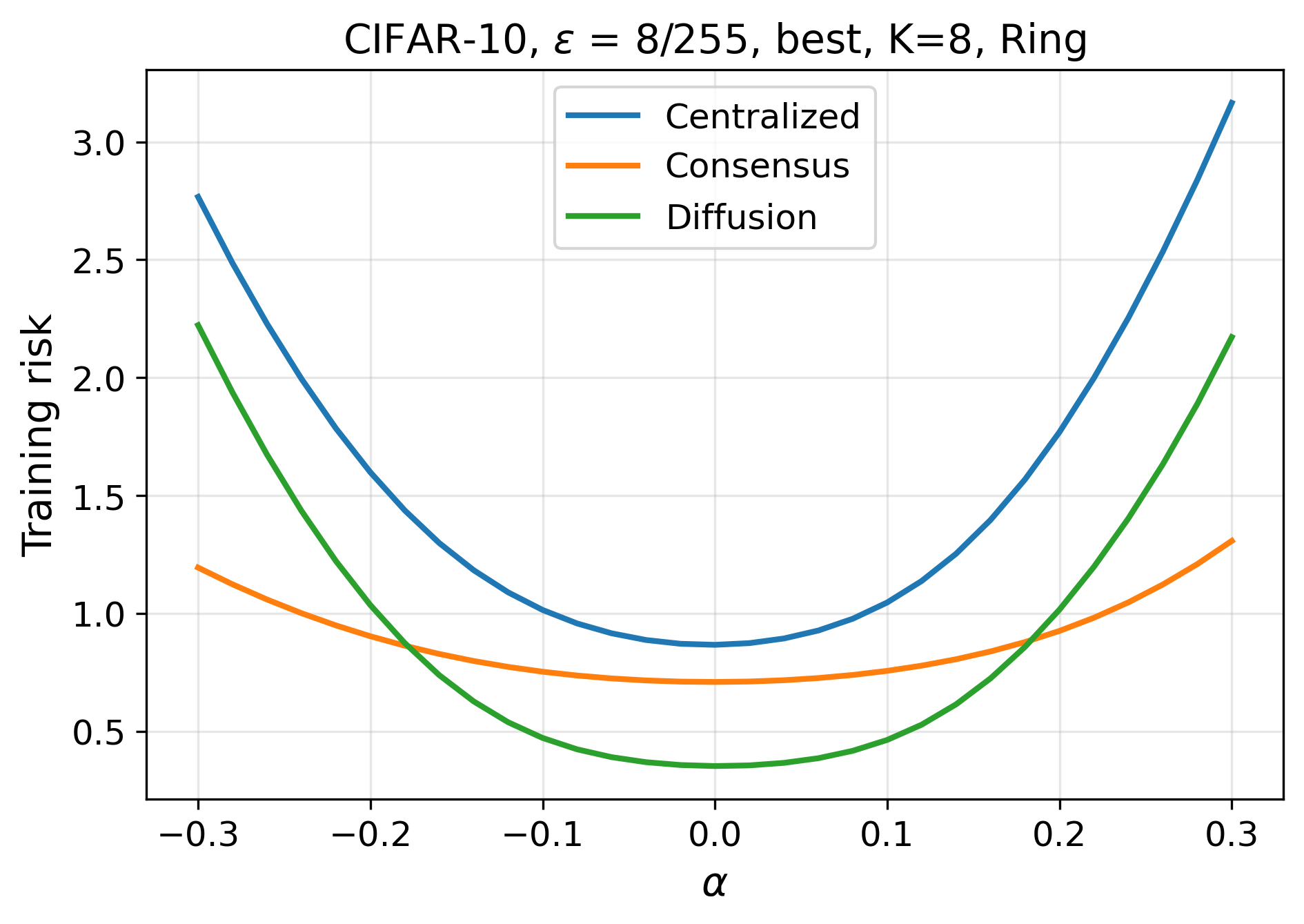}}

\subfigure[CIFAR-10: B=128, K=16]{\includegraphics[width=.3\linewidth]
{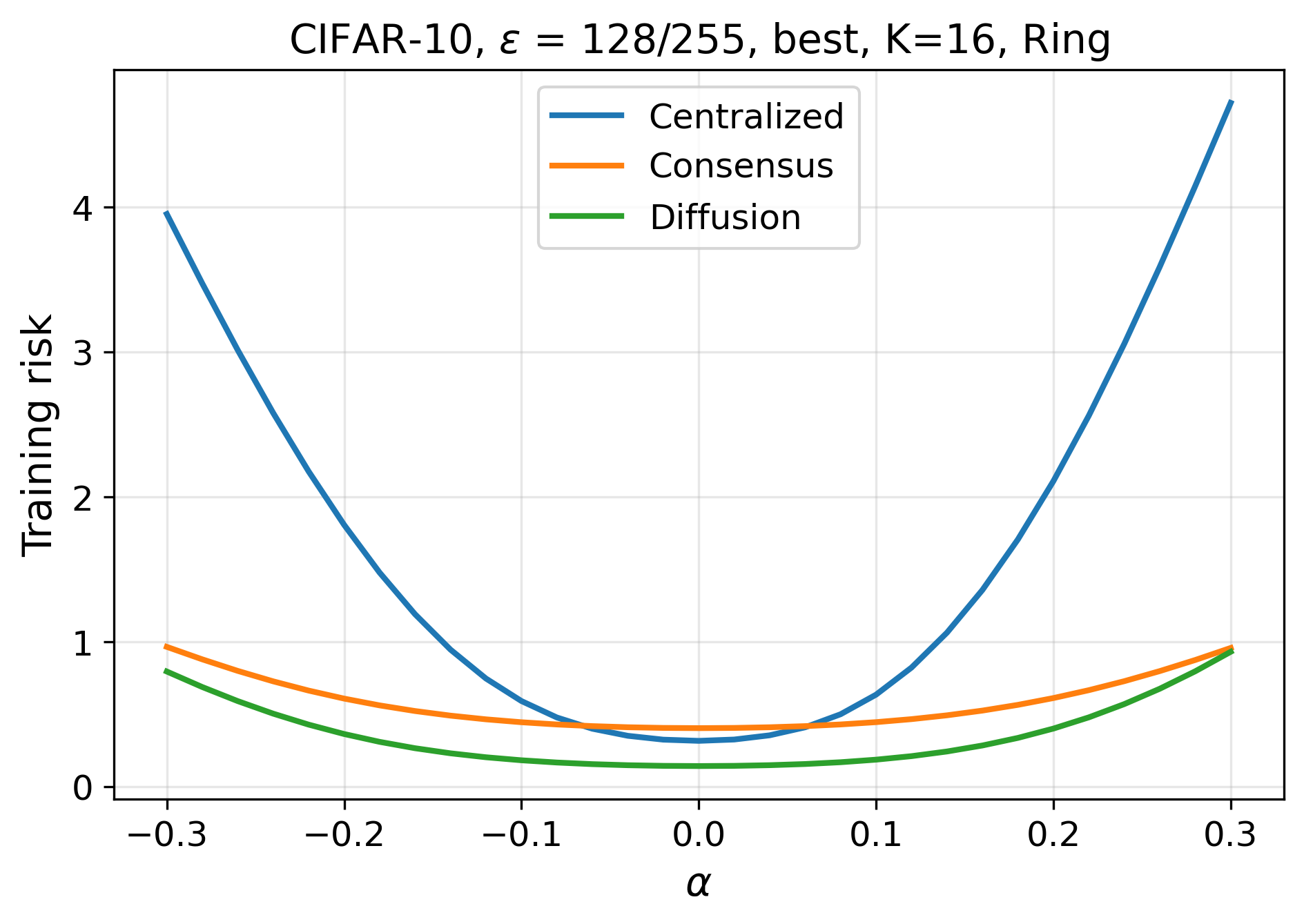}}
\subfigure[CIFAR-10: B=128, K=16]
{\includegraphics[width=.3\linewidth]{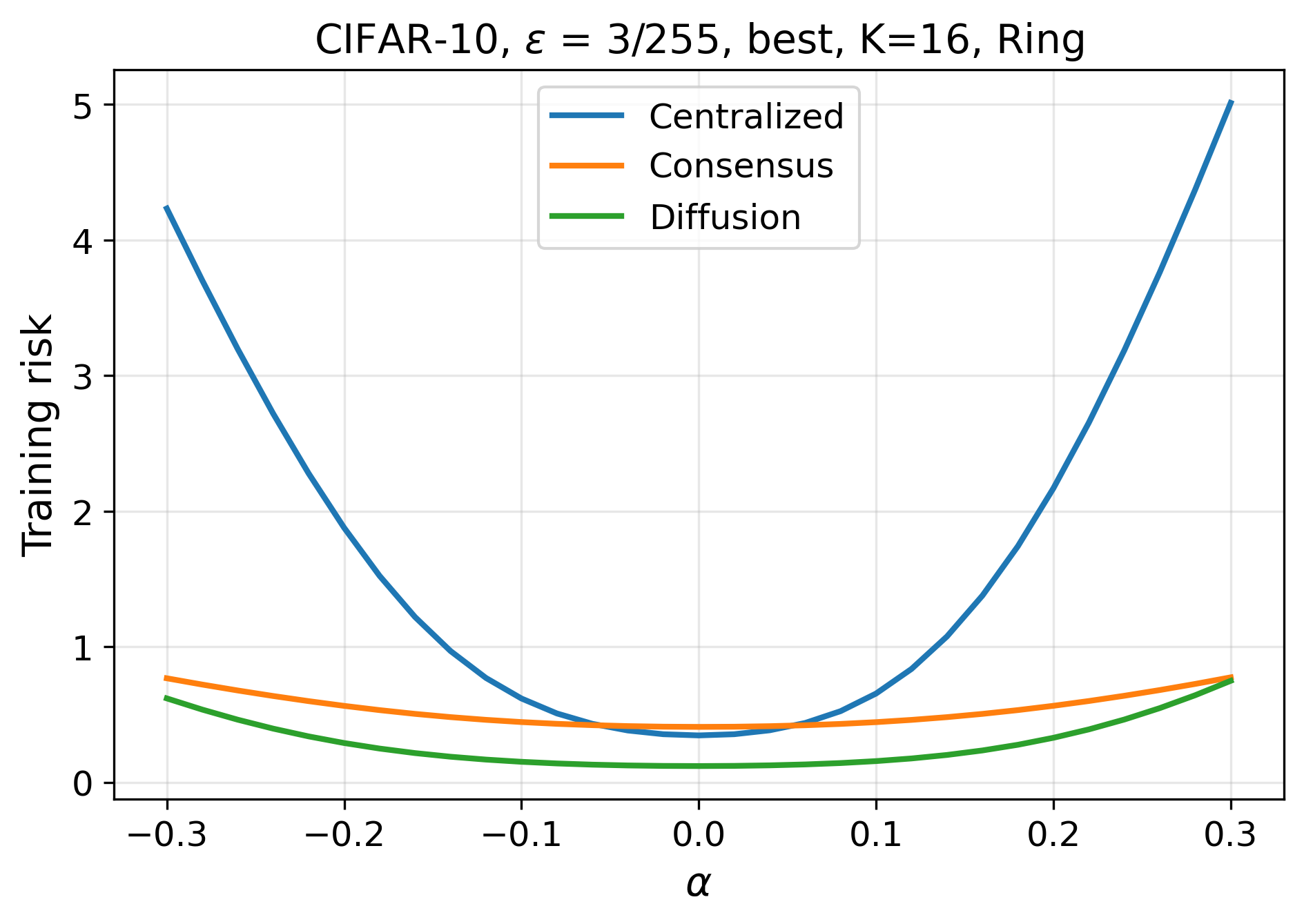}}
\subfigure[CIFAR-10: B=128, K=16]
{\includegraphics[width=.3\linewidth]
{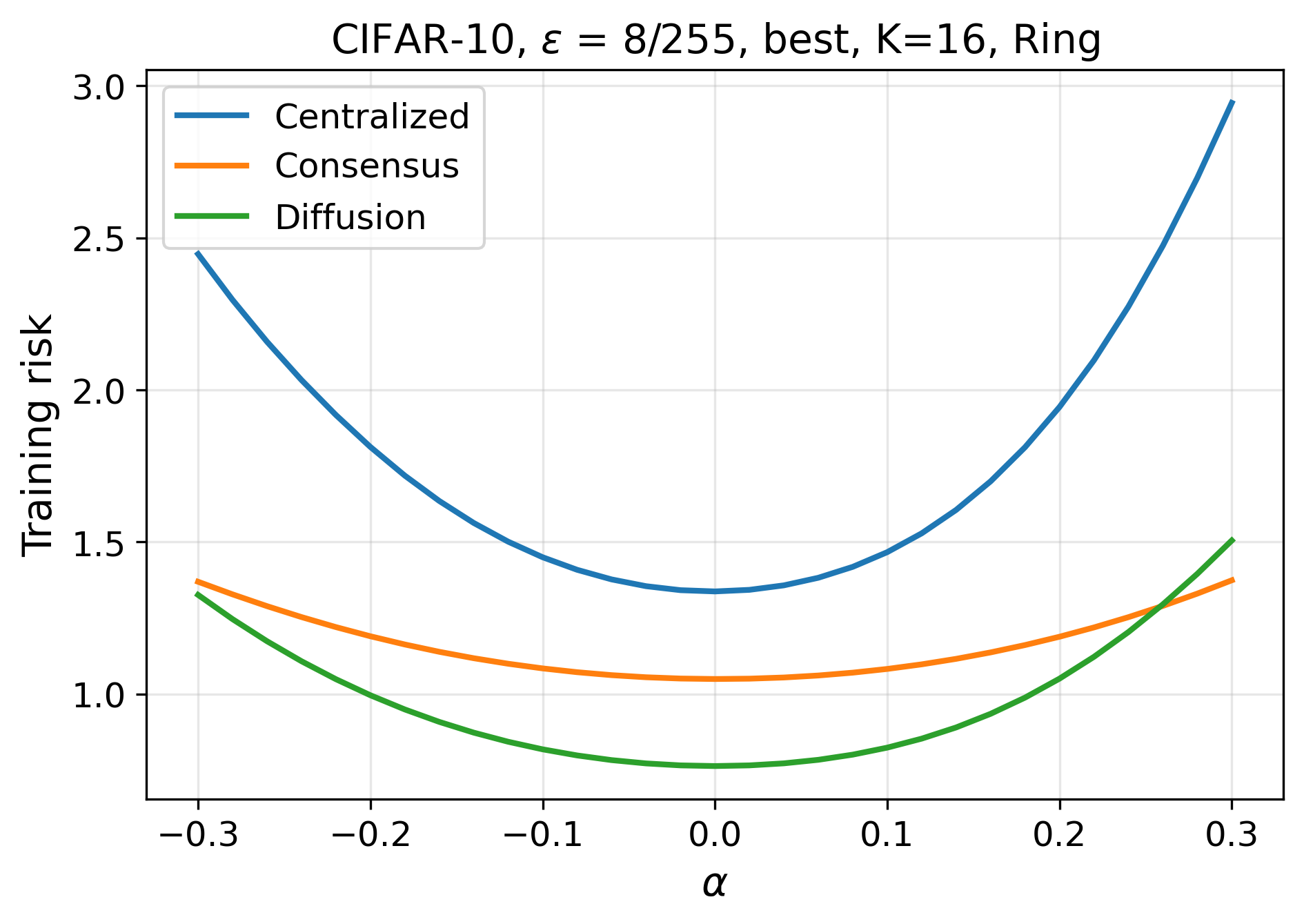}}

\subfigure[CIFAR-10: B=128, K=32]{\includegraphics[width=.3\linewidth]
{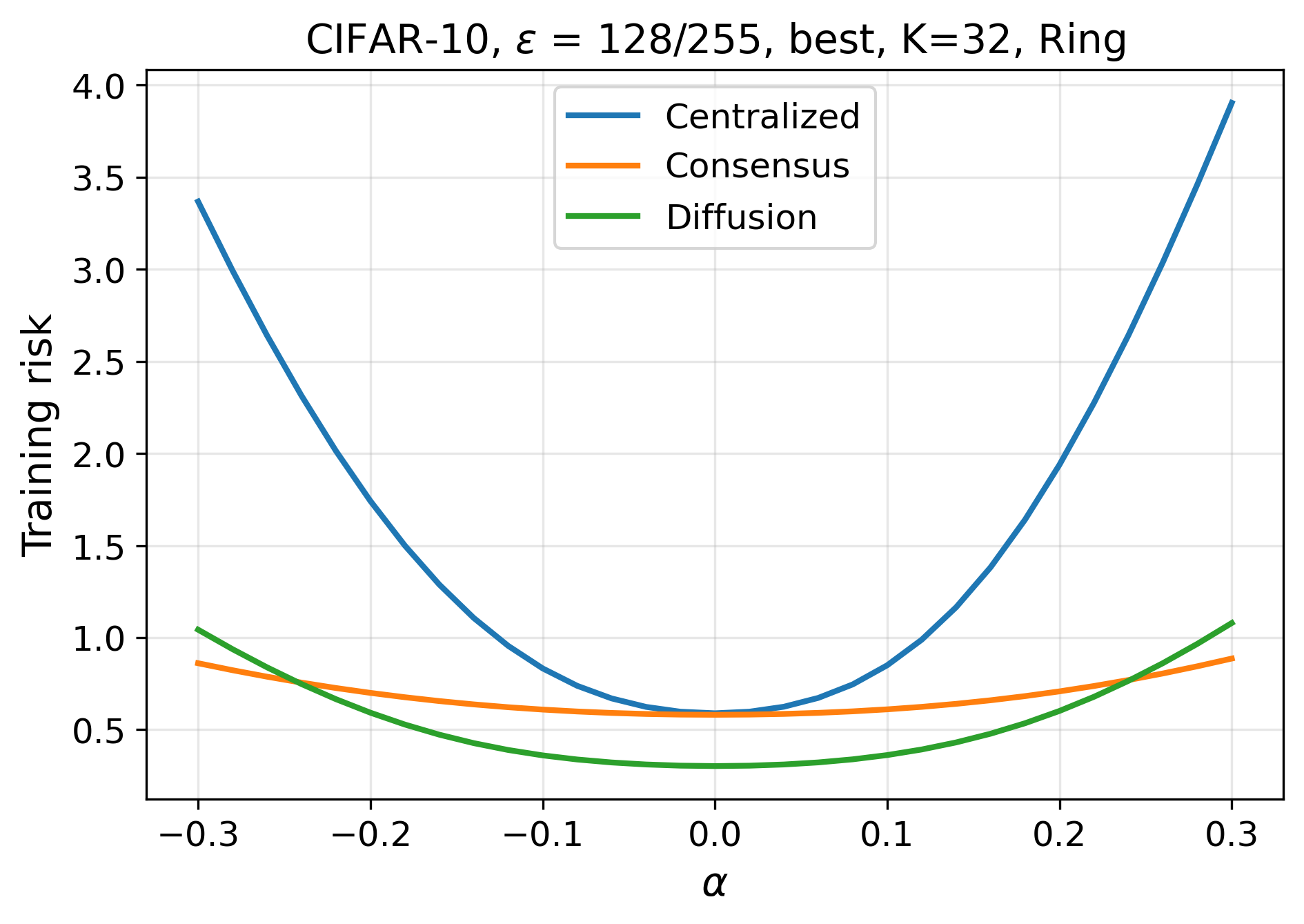}}
\subfigure[CIFAR-10: B=128, K=32]
{\includegraphics[width=.3\linewidth]{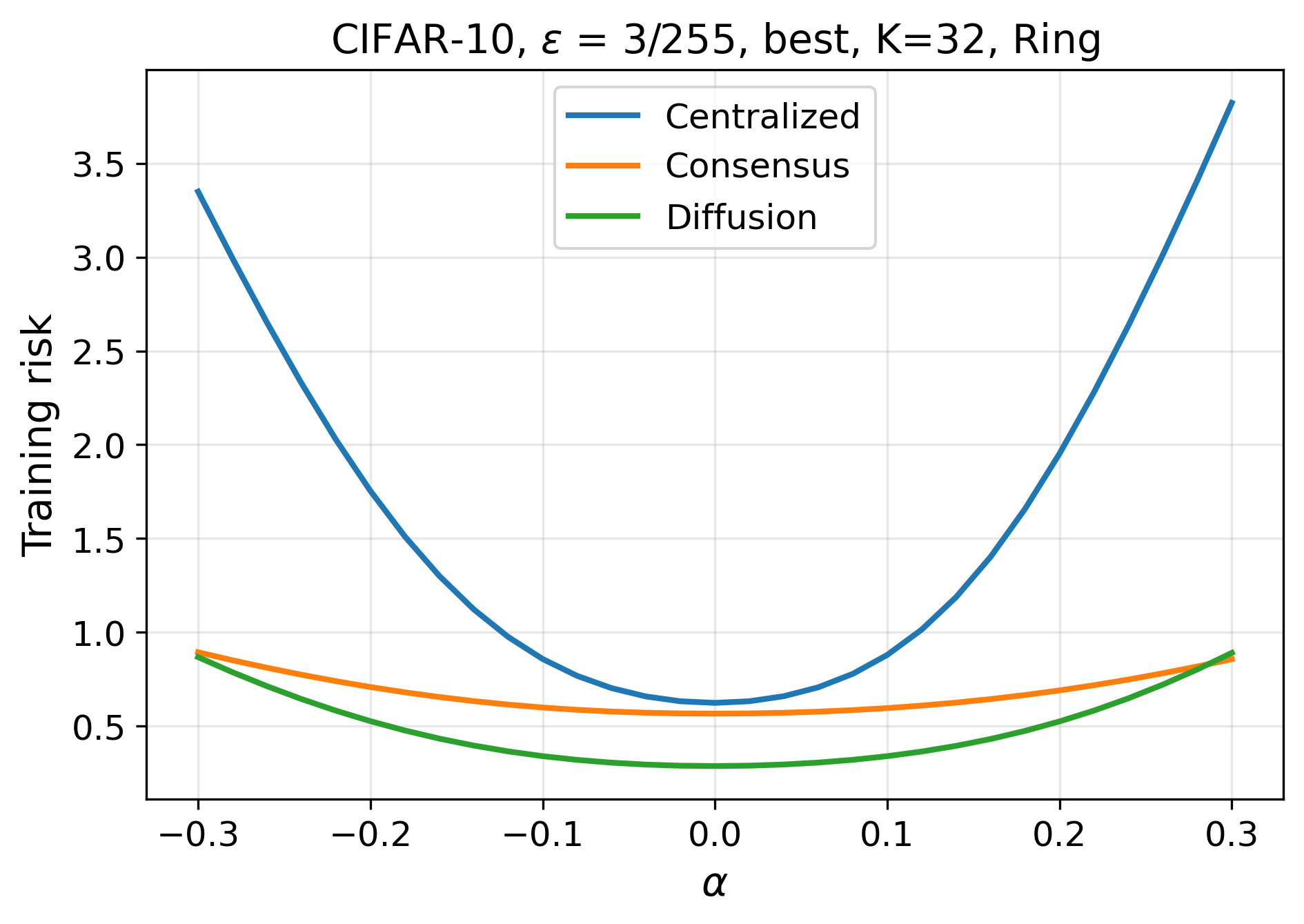}}
\subfigure[CIFAR-10: B=128, K=32]
{\includegraphics[width=.3\linewidth]
{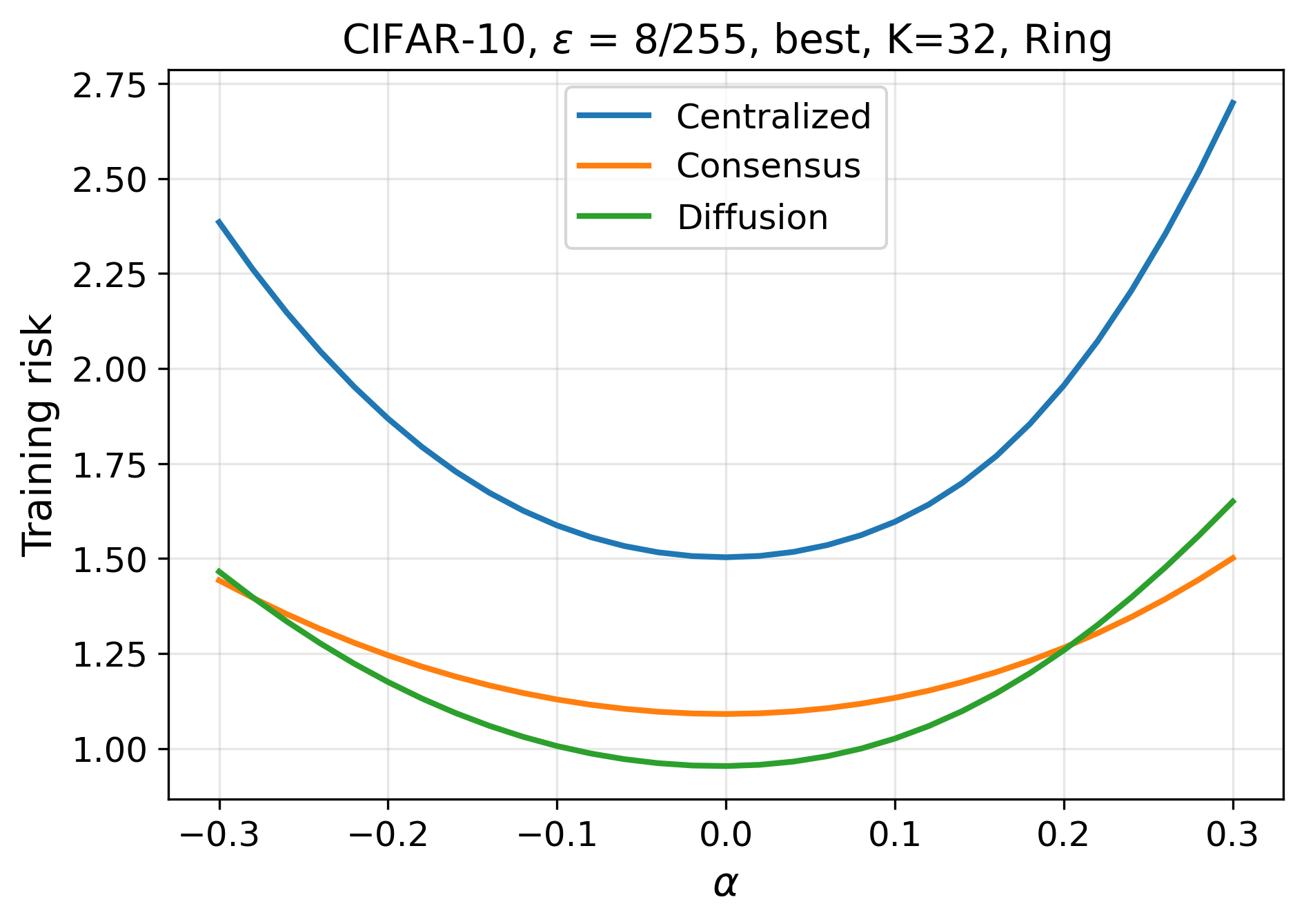}}
\caption{Flatness visualization of the \emph{best} models trained on the CIFAR-10 dataset using the \emph{ring} topology with $K=\{8,16,32\}$. The results are similar to Figure \ref{loss_visual_ring_cifar10}.}
\label{loss_visual_ring_cifar10_best}
\end{center}
\end{figure*}

\begin{figure*}[htbp]
\begin{center}
\subfigure[CIFAR-10: B=128, K=16]{\includegraphics[width=.3\linewidth]
{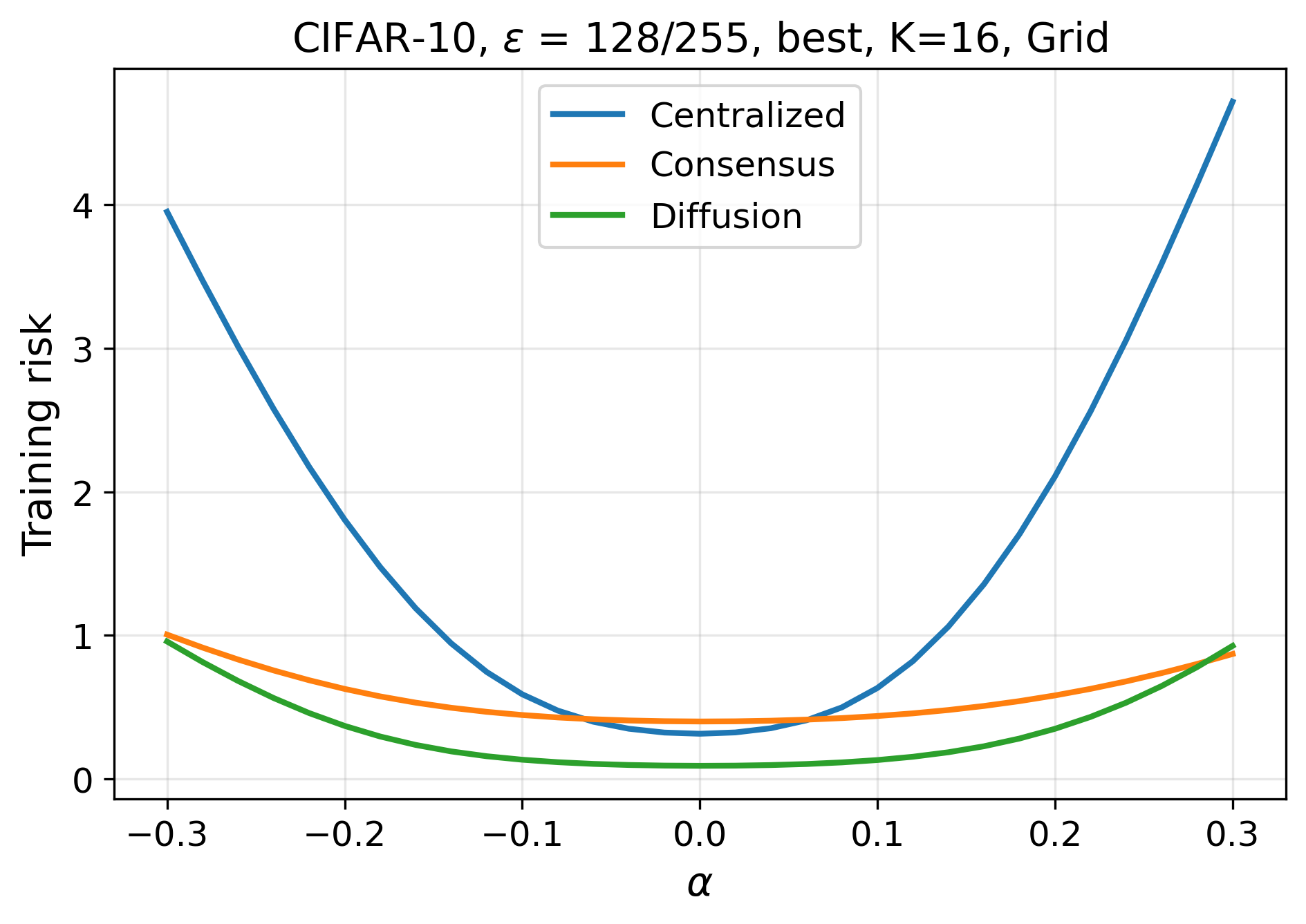}}
\subfigure[CIFAR-10: B=128, K=16]
{\includegraphics[width=.3\linewidth]{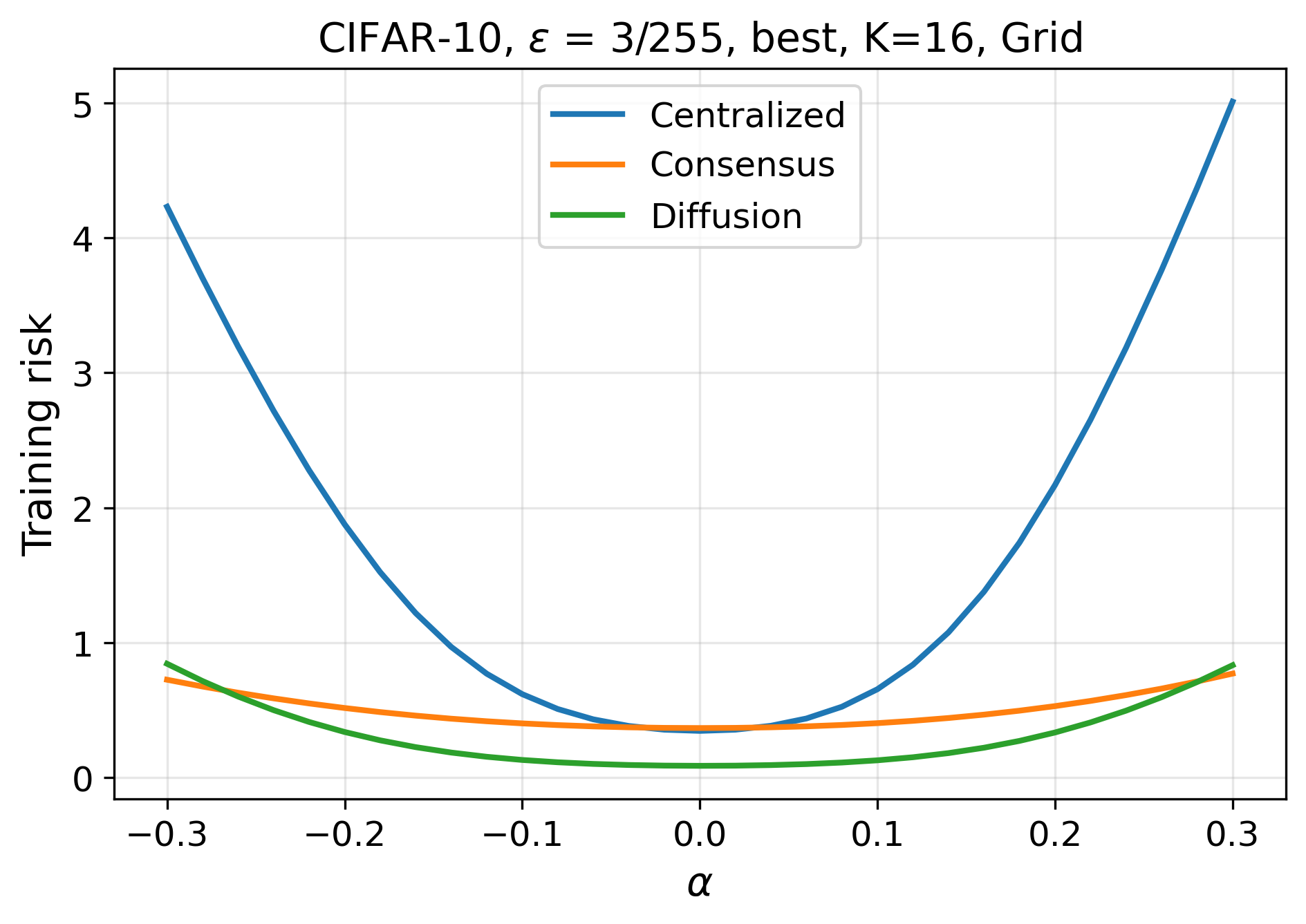}}
\subfigure[CIFAR-10: B=128, K=16]
{\includegraphics[width=.3\linewidth]
{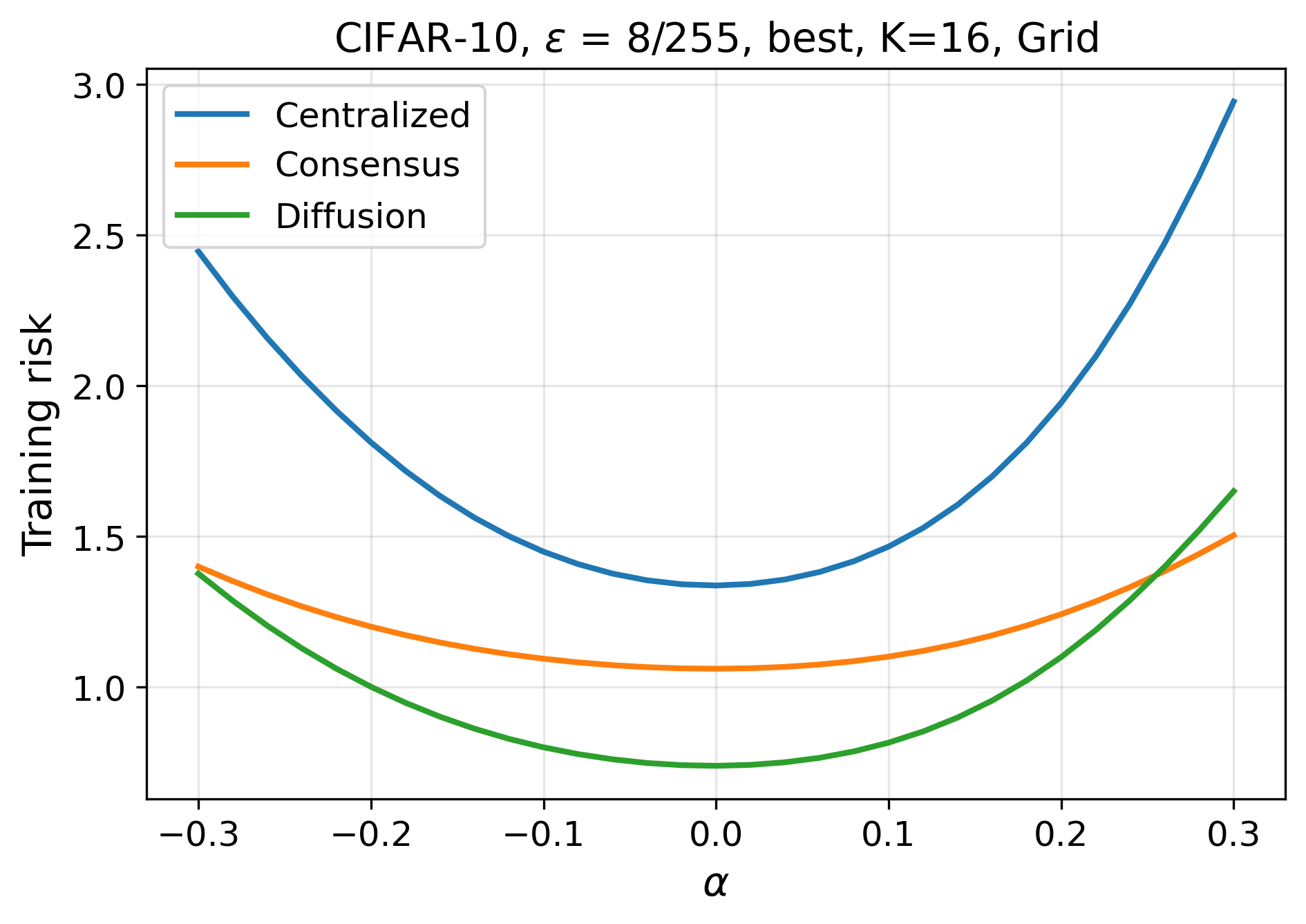}}
\caption{Flatness visualization of the \emph{best} models trained on the CIFAR-10 dataset using the \emph{grid} topology with $K=16$. The results are similar to Figure \ref{loss_visual_grid_cifar10}.}
\label{loss_visual_grid_cifar10_best}
\end{center}
\end{figure*}

\begin{figure*}[htbp]
\begin{center}
\subfigure[CIFAR-100: B=128, K=8]{\includegraphics[width=.3\linewidth]
{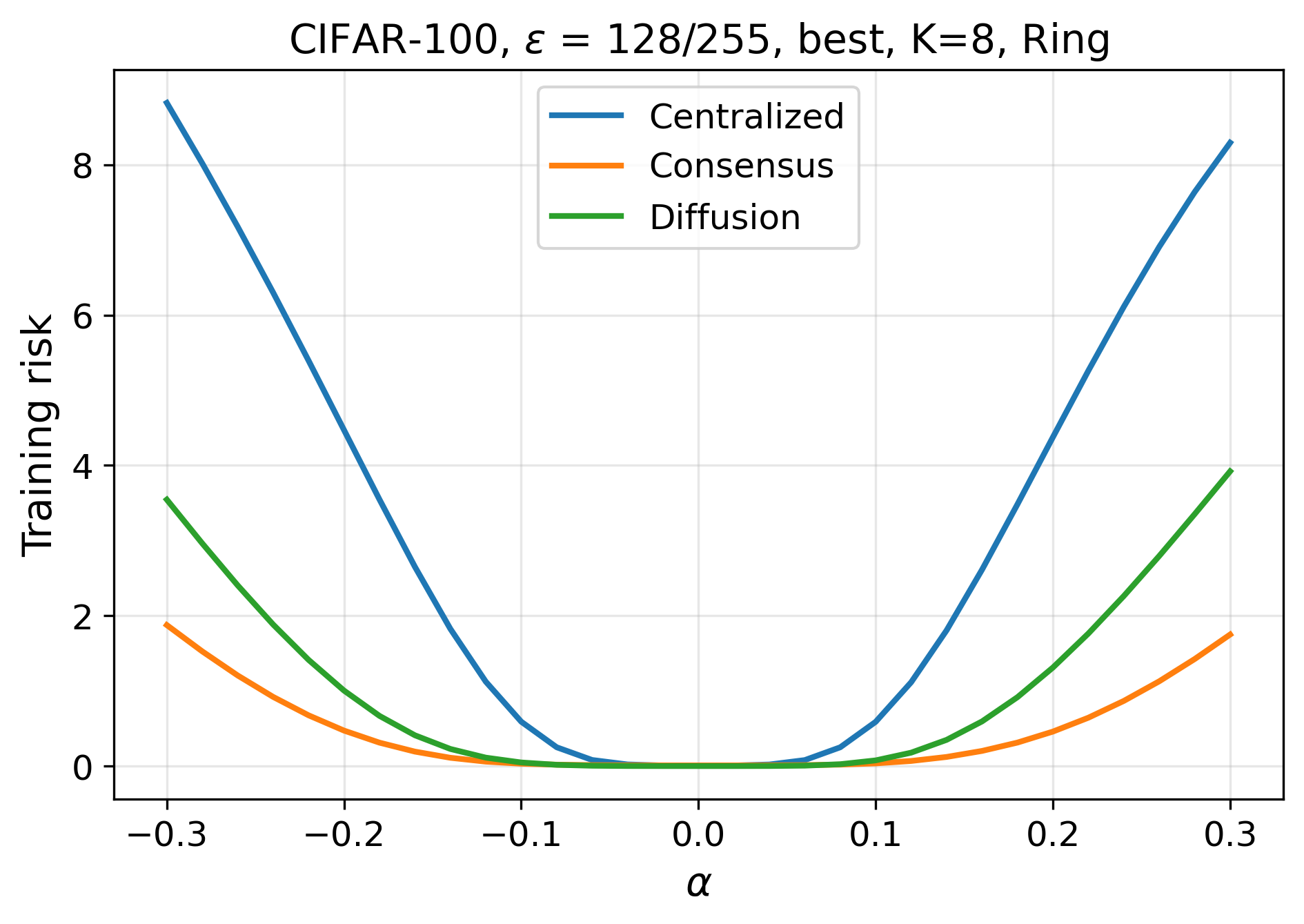}}
\subfigure[CIFAR-100: B=128, K=8]
{\includegraphics[width=.3\linewidth]{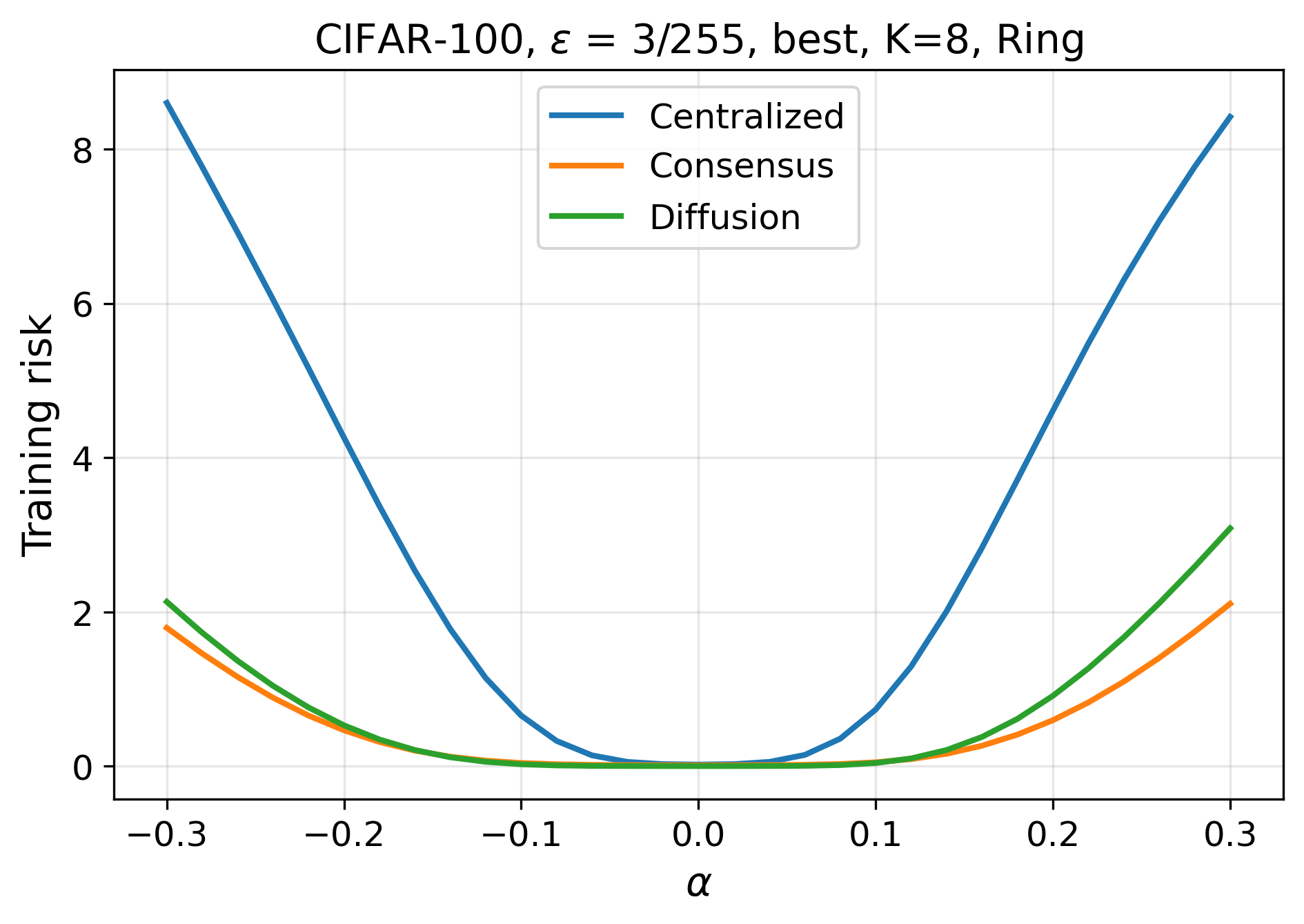}}
\subfigure[CIFAR-100: B=128, K=8]
{\includegraphics[width=.3\linewidth]
{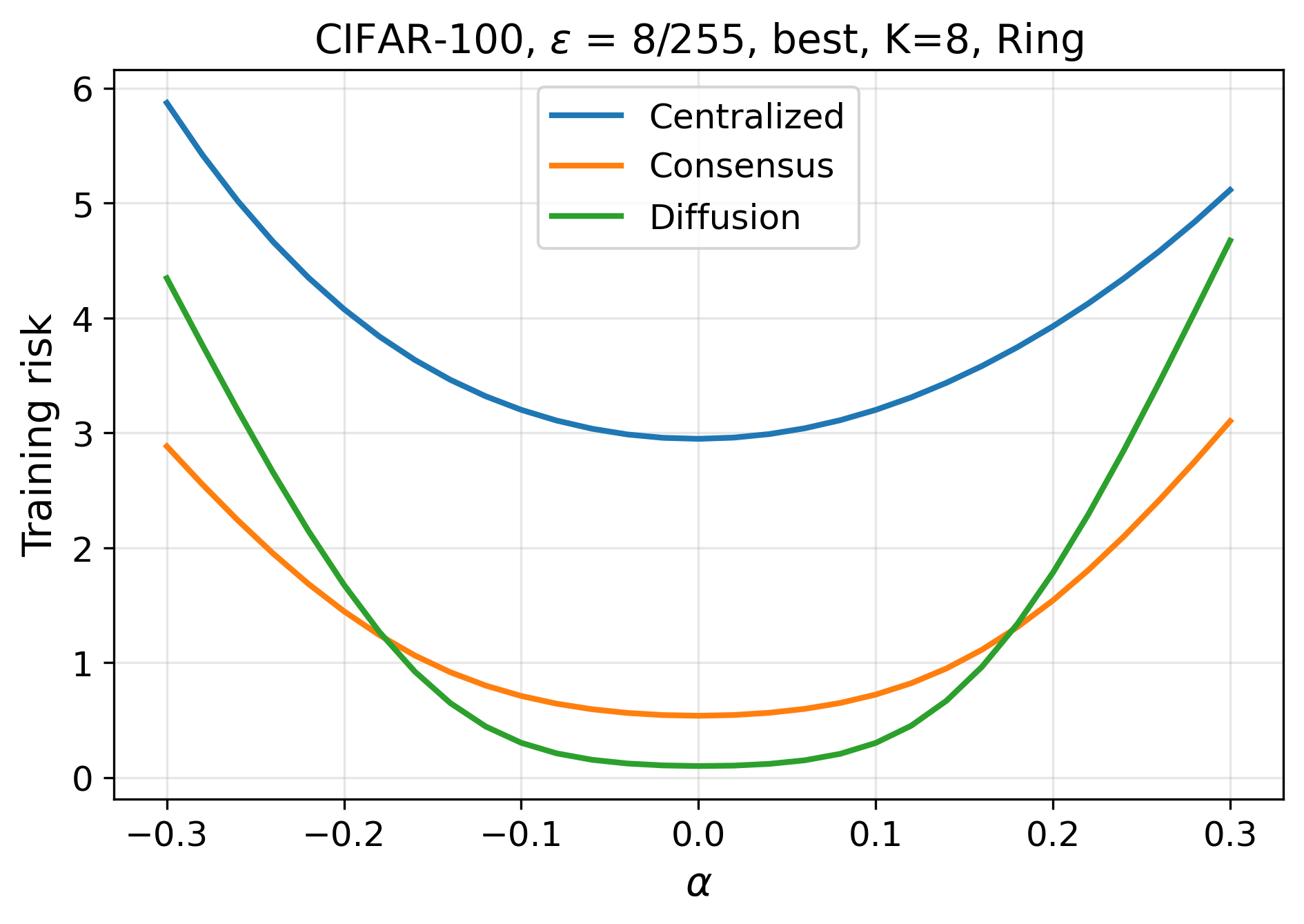}}

\subfigure[CIFAR-100: B=128, K=16]{\includegraphics[width=.3\linewidth]
{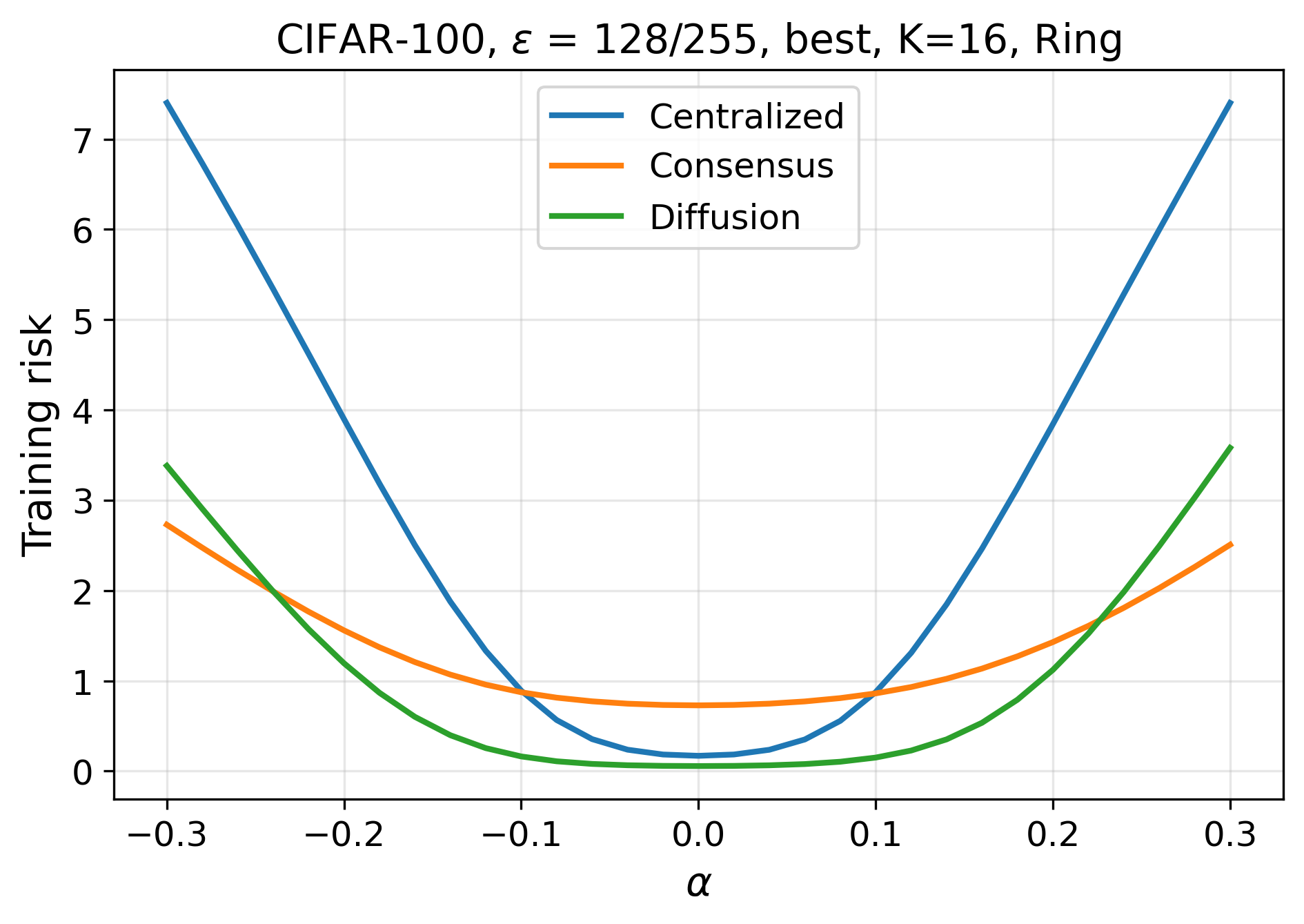}}
\subfigure[CIFAR-100: B=128, K=16]
{\includegraphics[width=.3\linewidth]{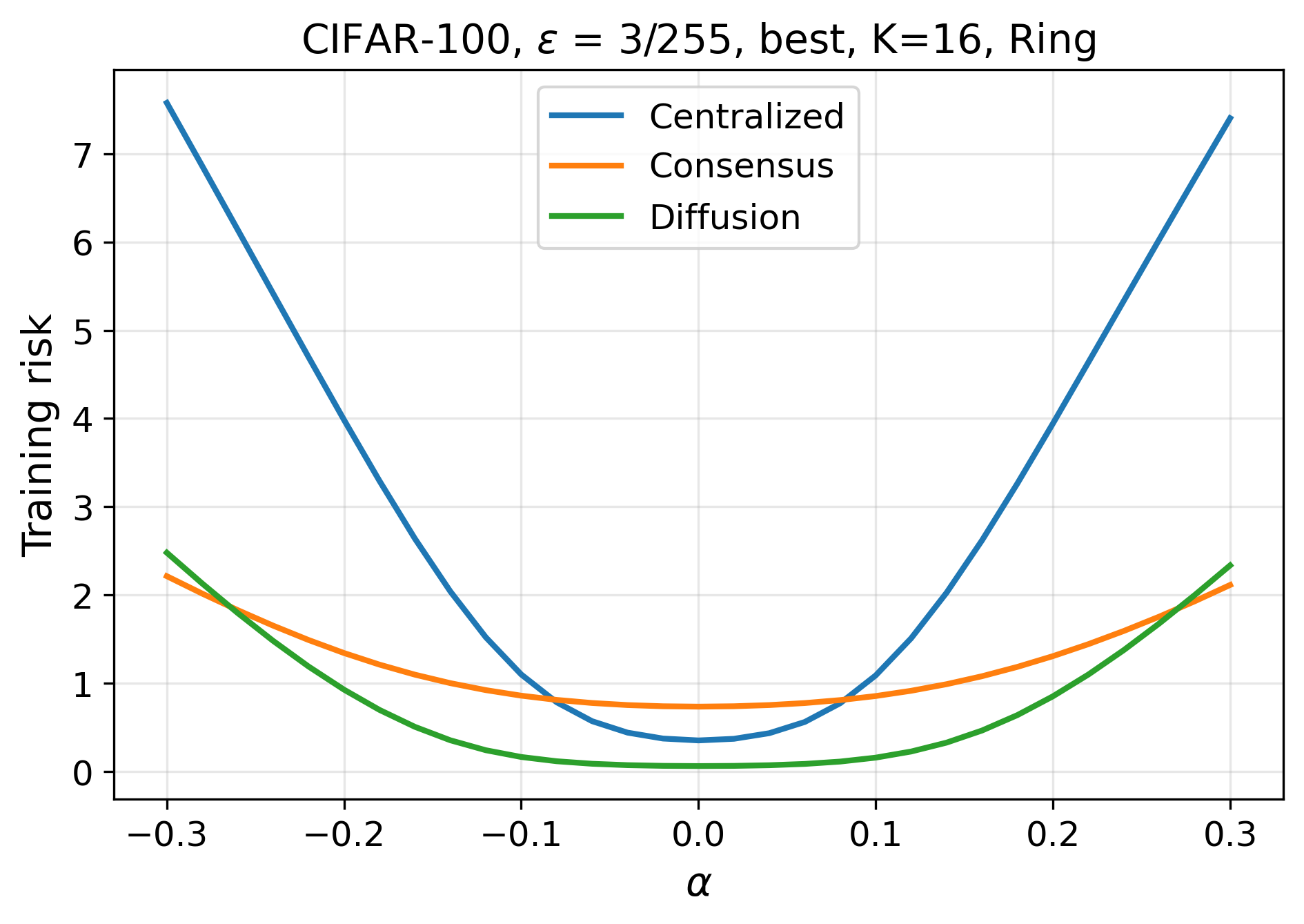}}
\subfigure[CIFAR-100: B=128, K=16]
{\includegraphics[width=.3\linewidth]
{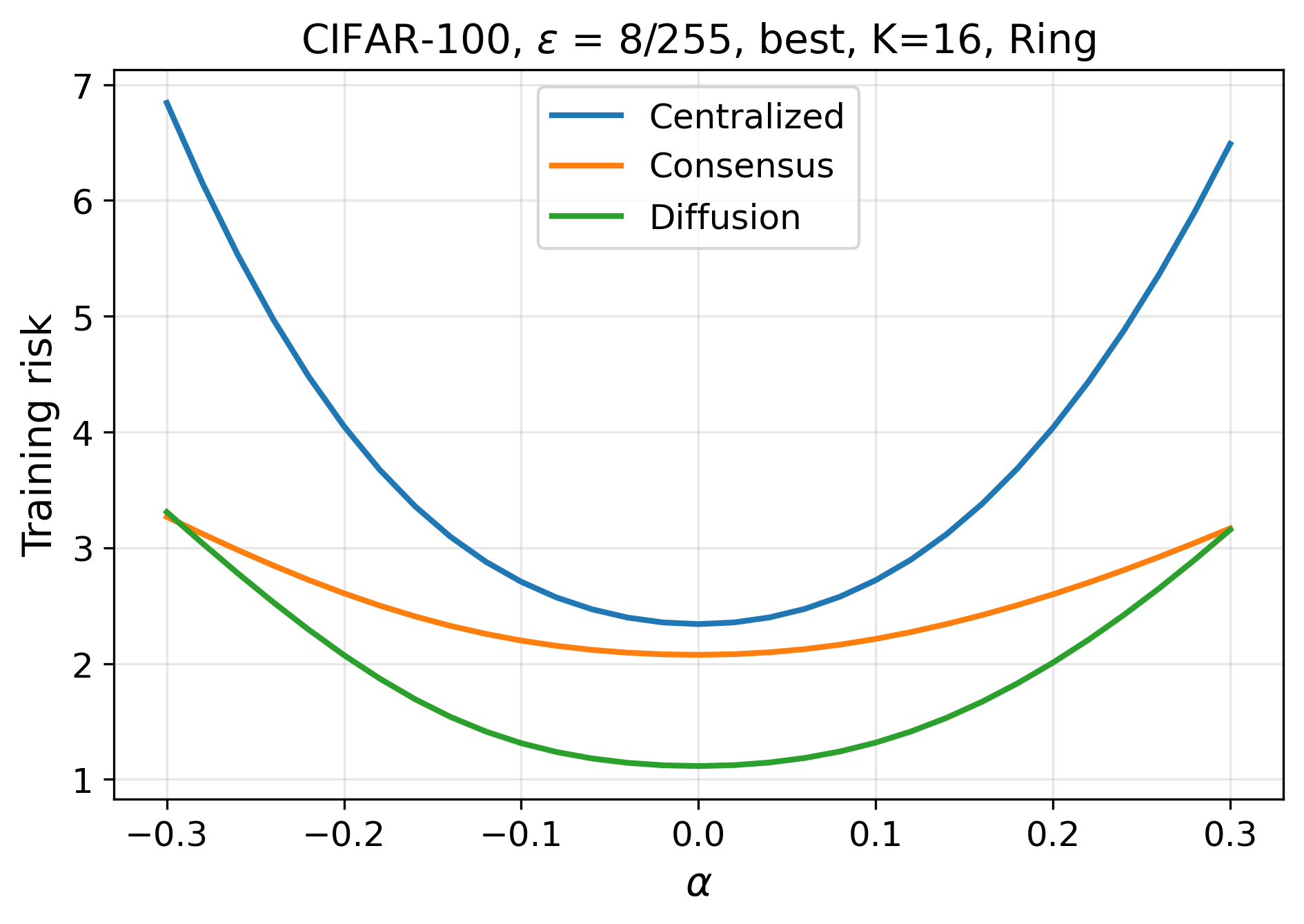}}

\subfigure[CIFAR-100: B=128, K=32]{\includegraphics[width=.3\linewidth]
{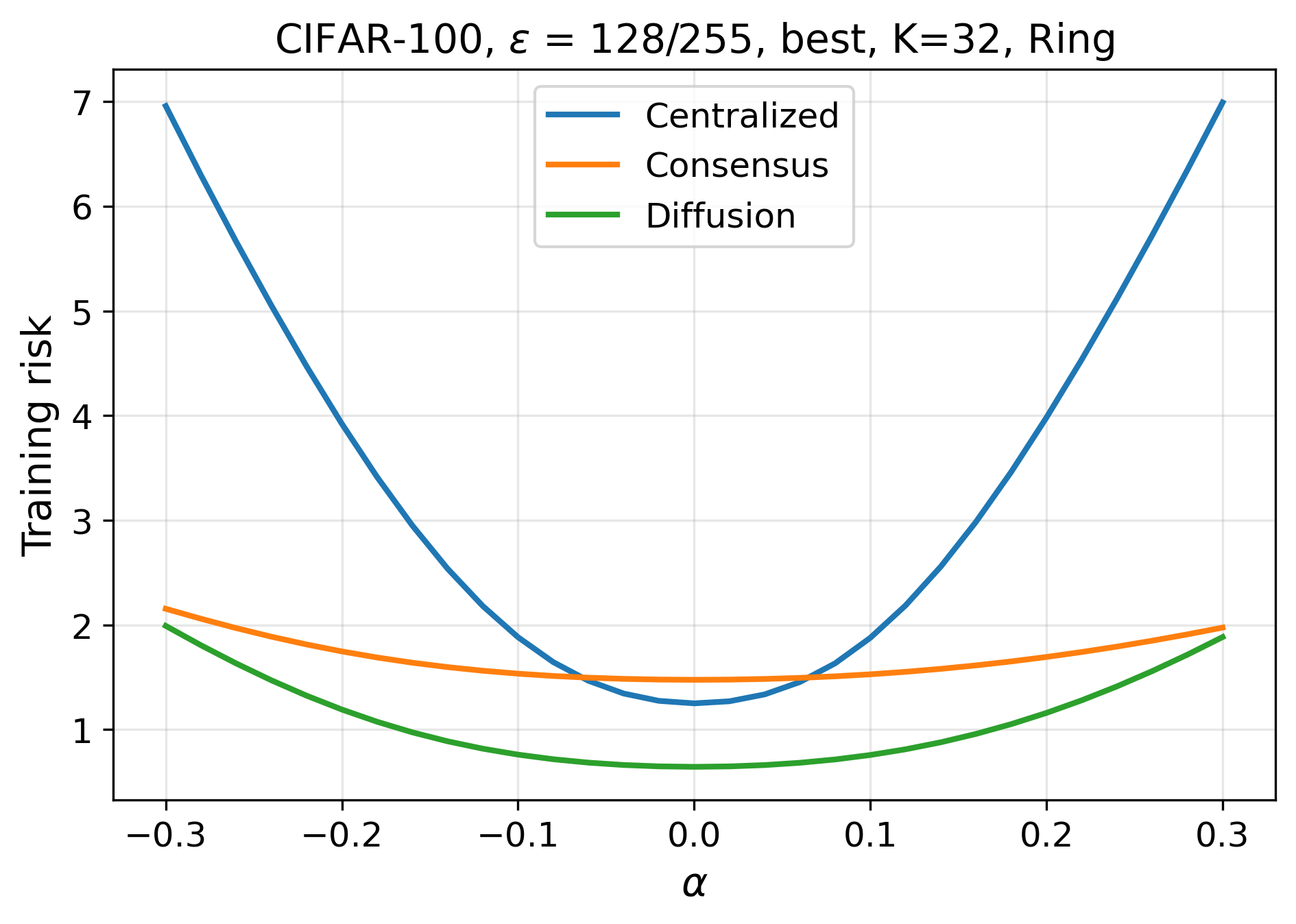}}
\subfigure[CIFAR-100: B=128, K=32]
{\includegraphics[width=.3\linewidth]{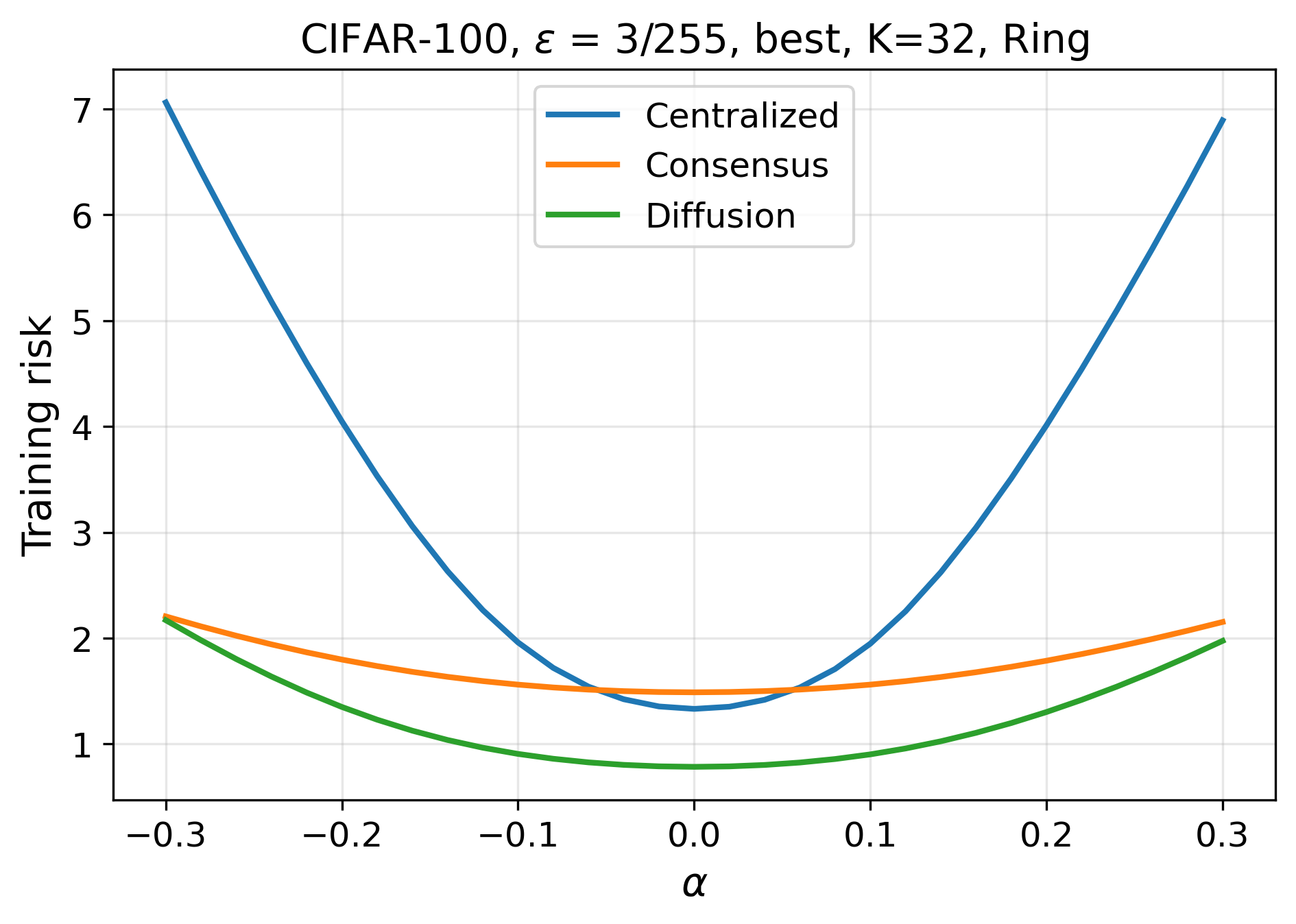}}
\subfigure[CIFAR-100: B=128, K=32]
{\includegraphics[width=.3\linewidth]
{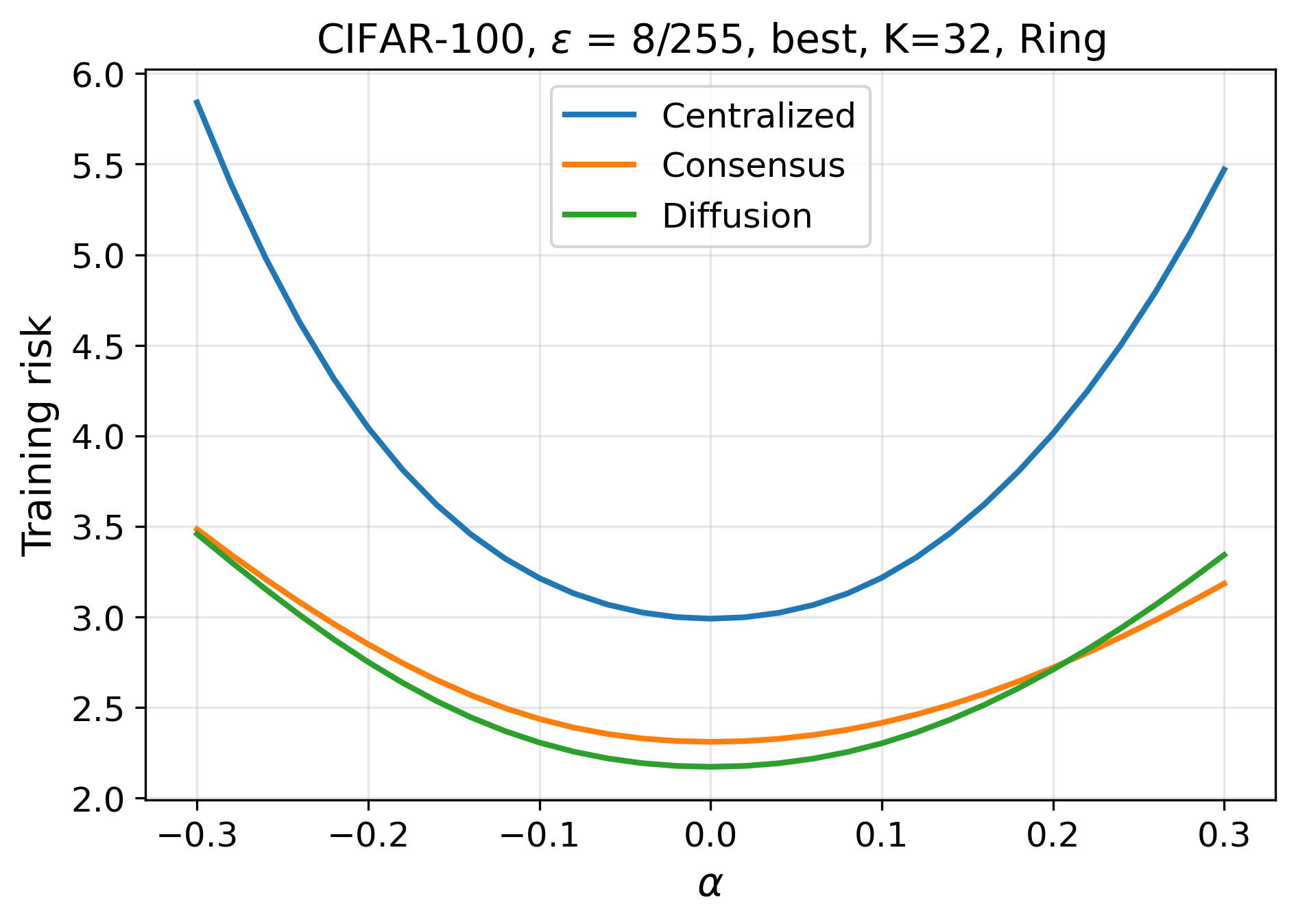}}
\caption{Flatness visualization of the \emph{best} models trained on the CIFAR-100 dataset using the \emph{ring} topology with $K=\{8,16,32\}$. The results are similar to Figure \ref{loss_visual_ring_cifar100}.}
\label{loss_visual_ring_cifar100_best}
\end{center}
\end{figure*}

\begin{figure*}[htbp]
\begin{center}
\subfigure[CIFAR-100: B=128, K=16]{\includegraphics[width=.3\linewidth]
{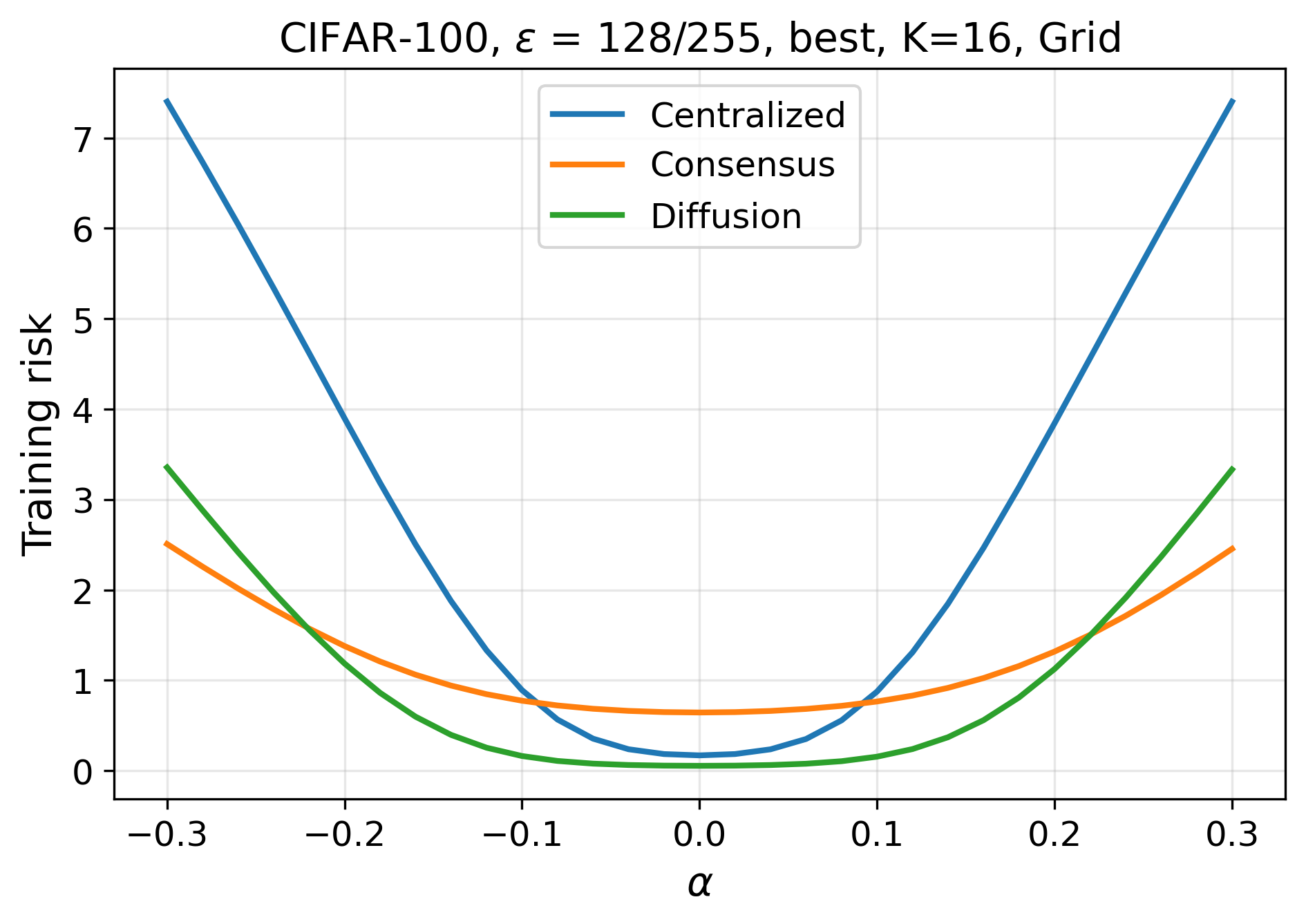}}
\subfigure[CIFAR-100: B=128, K=16]
{\includegraphics[width=.3\linewidth]{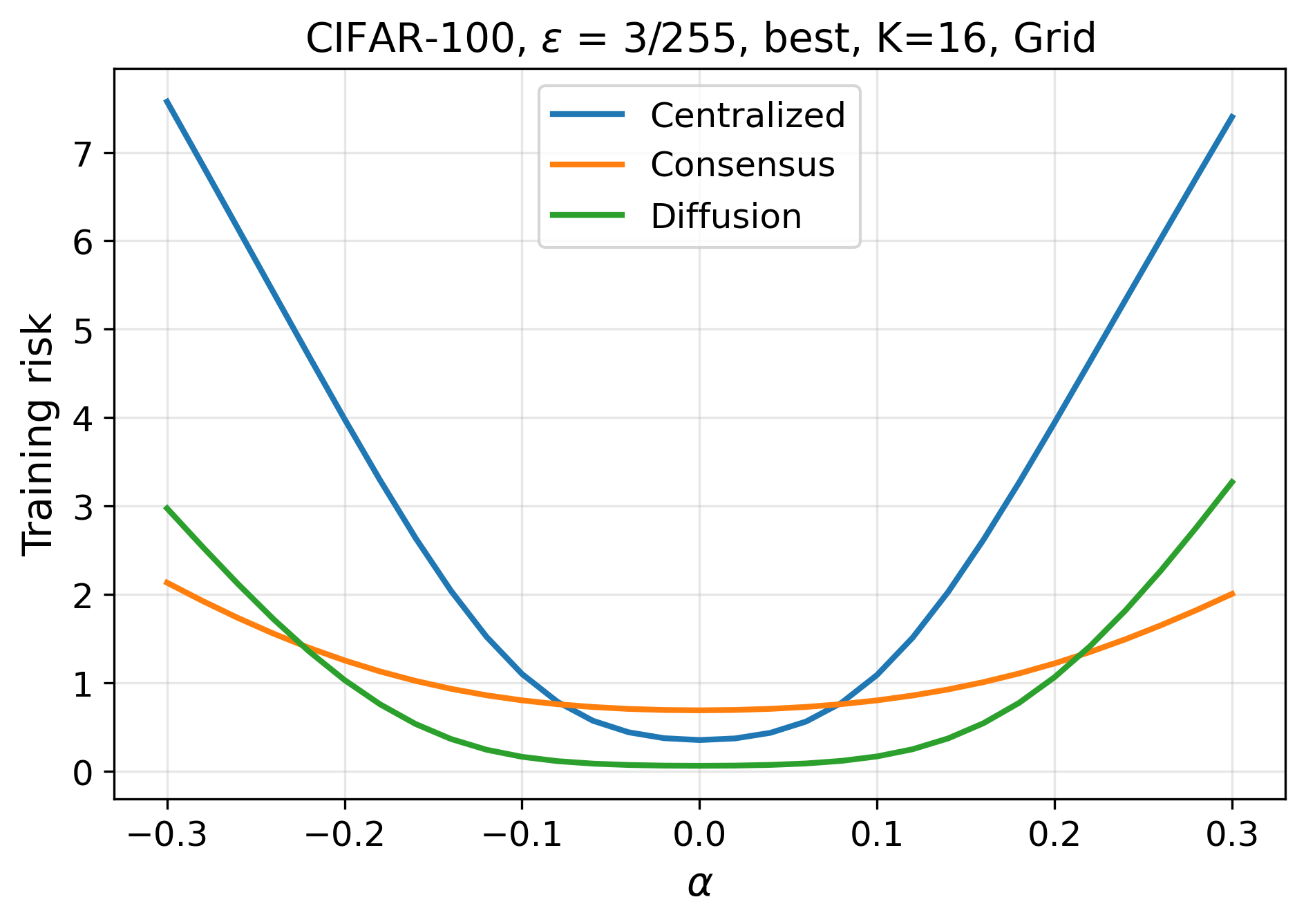}}
\subfigure[CIFAR-100: B=128, K=16]
{\includegraphics[width=.3\linewidth]
{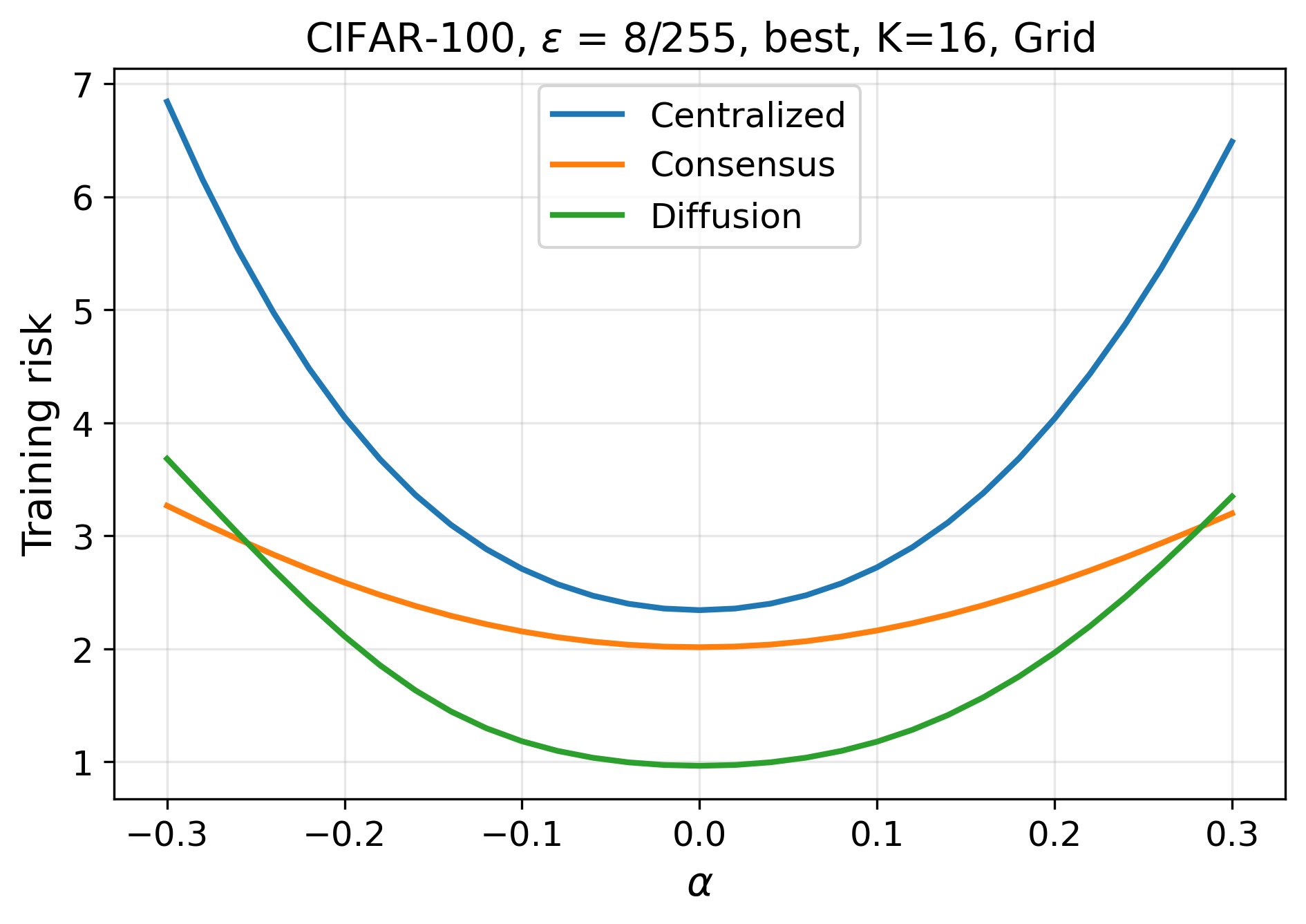}}
\caption{Flatness visualization of the \emph{best} models trained on the CIFAR-100 dataset using the \emph{grid} topology with $K=16$. The results are similar to Figure \ref{loss_visual_grid_cifar100}.}
\label{loss_visual_grid_cifar100_best}
\end{center}
\end{figure*}
\subsubsection{Clean and robust accuracy of additional results}
We compare the clean and robust accuracy of models trained by the three algorithms under the additional settings described in Section \ref{appendix_gb1}.  The results are reported in Tables \ref{robustacc_sgd_cifar10-ring}--\ref{robustacc_sgd_cifar100-grid}, which show that decentralized methods consistently achieve higher accuracy on both clean and perturbed data than the centralized baseline. We also show the corresponding optimization plots in Figures \ref{optim_ring_cifar10}--\ref{optim_grid_16_cifar100}. Consistent with the main text, the centralized adversarial training algorithm exhibits particularly high training error when $\epsilon=8/255$. This explains why diffusion achieve better accuracy than centralized despite sometimes exhibiting sharper loss landscapes.

\begin{table*}[!t]
\footnotesize
\centering
\begin{spacing}{1.6}
\caption{Clean and robust accuracy on CIFAR-10 for models trained using a \emph{ring} topology}
\label{robustacc_sgd_cifar10-ring}
\begin{tabular}{cccccccccccccccccc}
\toprule
\multirow{2}{*}{dataset} & \multirow{2}{*}{graph}& \multirow{2}{*}{K}&\multirow{2}{*}{norm} & \multirow{2}{*}{$\epsilon$} & \multirow{2}{*}{local batch}  & \multirow{2}{*}{method} &  \multicolumn{2}{c}{best} & \multicolumn{2}{c}{final} \\
\cline{8-9}\cline{10-11}
&&&&&&&clean($\%$) & AA($\%$) & clean($\%$) & AA($\%$) \\
\hline
\multirow{27}{*}{CIFAR-10}&\multirow{27}{*}{ring} &\multirow{9}{*}{8}& \multirow{6}{*}{$\ell_\infty$} & \multirow{3}{*}{$8/255$} & \multirow{3}{*}{128} &   centralized & 78.36&37.32&78.68&33.84  \\
    &&&&&&consensus& 84.63& \textbf{44.99}& \textbf{84.97}&\textbf{38.64} \\
    &&&&&&diffusion & \textbf{85.44}& 39.96&84.85&37.38\\
\cline{5-11}
&&&& \multirow{3}{*}{$3/255$} &\multirow{3}{*}{128} & centralized & 86.49&62.04&86.45&61.97 \\
    &&&&&&consensus & \textbf{91.53}& \textbf{71.40}& 91.17&69.50\\
    &&&&&&diffusion& 91.24& 70.33& \textbf{91.21}& \textbf{70.21} \\
\cline{4-11}
&&& \multirow{3}{*}{$\ell_2$} & \multirow{3}{*}{$128/255$} & \multirow{3}{*}{128} & centralized & 85.61&55.40&85.66& 55.23\\
    &&&&&&consensus&\textbf{90.46}&\textbf{63.85}&\textbf{90.15}&60.04 \\
    &&&&&&diffusion &90.05&62.18&90.02&\textbf{61.91}\\
\cline{3-11}

&&\multirow{9}{*}{16}&
\multirow{6}{*}{$\ell_\infty$} & \multirow{3}{*}{$8/255$} & \multirow{3}{*}{128} &  centralized&71.11&35.89&71.44&35.86  \\
    &&&&&&consensus& 81.40&44.36&82.81&\textbf{43.49}&\\
    &&&&&&diffusion &\textbf{84.21}&\textbf{44.84}&\textbf{84.55}&39.85\\
\cline{5-11}
&&&& \multirow{3}{*}{$3/255$} &\multirow{3}{*}{128} & centralized&84.91&60.97&84.50&59.55  \\
    &&&&&&consensus&89.64&71.60&90.72&\textbf{69.56}\\
    &&&&&&diffusion &\textbf{91.26}&\textbf{71.97}&\textbf{90.74}&69.38\\
\cline{4-11}
&&& \multirow{3}{*}{$\ell_2$} & \multirow{3}{*}{$128/255$} & \multirow{3}{*}{128} & centralized&83.96&54.03&83.60&52.26 \\
    &&&&&&consensus &89.03&64.01&\textbf{89.85}&\textbf{60.67}\\
    &&&&&&diffusion &\textbf{90.21}&\textbf{64.02}&89.75&60.28\\
\cline{3-11}

&&\multirow{9}{*}{32}&
\multirow{6}{*}{$\ell_\infty$} & \multirow{3}{*}{$8/255$} & \multirow{3}{*}{128} &  centralized & 65.98& 34.09&66.40&34.13  \\
    &&&&&&consensus &79.78&43.02&79.77&\textbf{43.05} \\
    &&&&&&diffusion &\textbf{83.14}&\textbf{43.37}&\textbf{84.09}&39.62\\
\cline{5-11}
&&&& \multirow{3}{*}{$3/255$} &\multirow{3}{*}{128} & centralized&83.26&59.71&83.50&59.66 \\
    &&&&&&consensus &88.09&68.88&89.44&\textbf{69.20} \\
    &&&&&&diffusion &\textbf{90.38}&\textbf{70.65}&\textbf{90.43}&67.90\\
\cline{4-11}
&&& \multirow{3}{*}{$\ell_2$} & \multirow{3}{*}{$128/255$} & \multirow{3}{*}{128} & centralized& 82.22&54.75&82.46&53.15 \\
    &&&&&&consensus&86.84&62.07&88.39&\textbf{61.56} \\
    &&&&&&diffusion &\textbf{88.88}&\textbf{62.30}&\textbf{88.93}&58.56\\
\bottomrule      
\end{tabular}
\end{spacing}
\end{table*}

\begin{table*}[!t]
\footnotesize
\centering
\begin{spacing}{1.6}
\caption{Clean and robust accuracy on CIFAR-10 for models trained using a \emph{grid} topology}
\label{robustacc_sgd_cifar10-grid}
\begin{tabular}{cccccccccccccccccc}
\toprule
\multirow{2}{*}{dataset} & \multirow{2}{*}{graph}& \multirow{2}{*}{K}&\multirow{2}{*}{norm} & \multirow{2}{*}{$\epsilon$} & \multirow{2}{*}{local batch}  & \multirow{2}{*}{method} &  \multicolumn{2}{c}{best} & \multicolumn{2}{c}{final} \\
\cline{8-9}\cline{10-11}
&&&&&&&clean($\%$) & AA($\%$) & clean($\%$) & AA($\%$) \\
\hline
\multirow{9}{*}{CIFAR-10}&\multirow{9}{*}{grid} &\multirow{9}{*}{16}& \multirow{6}{*}{$\ell_\infty$} & \multirow{3}{*}{$8/255$} & \multirow{3}{*}{128} &  centralized &71.11&35.89&71.44&35.86 \\
    &&&&&&consensus & 80.82&\textbf{44.85}&82.99&\textbf{42.96} \\
    &&&&&&diffusion & \textbf{84.55}&44.67&\textbf{84.85}&39.65\\
\cline{5-11}
&&&& \multirow{3}{*}{$3/255$} &\multirow{3}{*}{128} & centralized& 84.91&60.97&84.50&59.55  \\
    &&&&&&consensus & 90.18&\textbf{71.77}&\textbf{91.00}&\textbf{69.22}\\
    &&&&&&diffusion & \textbf{91.32}&71.32&90.95&68.79\\
\cline{4-11}
&&& \multirow{3}{*}{$\ell_2$} & \multirow{3}{*}{$128/255$} & \multirow{3}{*}{128} & centralized &83.96&54.03&83.60&52.26\\
    &&&&&&consensus &88.92&\textbf{64.28}&89.68&\textbf{60.64}\\
    &&&&&&diffusion &\textbf{90.36}&62.26&\textbf{90.27}&59.37\\
\bottomrule      
\end{tabular}
\end{spacing}
\end{table*}

\begin{table*}[!t]
\footnotesize
\centering
\begin{spacing}{1.6}
\caption{Clean and robust accuracy on CIFAR-100 for models trained using a \emph{ring} topology}
\label{robustacc_sgd_cifar100-ring}
\begin{tabular}{cccccccccccccccccc}
\toprule
\multirow{2}{*}{dataset} & \multirow{2}{*}{graph}& \multirow{2}{*}{K}&\multirow{2}{*}{norm} & \multirow{2}{*}{$\epsilon$} & \multirow{2}{*}{local batch}  & \multirow{2}{*}{method} &  \multicolumn{2}{c}{best} & \multicolumn{2}{c}{final} \\
\cline{8-9}\cline{10-11}
&&&&&&&clean($\%$) & AA($\%$) & clean($\%$) & AA($\%$) \\
\hline
\multirow{27}{*}{CIFAR-100}&\multirow{27}{*}{ring} &\multirow{9}{*}{8}& \multirow{6}{*}{$\ell_\infty$} & \multirow{3}{*}{$8/255$} & \multirow{3}{*}{128} &  centralized & 40.51&15.23&48.97&14.30 \\
    &&&&&&consensus& \textbf{59.08}&\textbf{20.55}&\textbf{57.68}&\textbf{19.19 }\\
    &&&&&&diffusion & 57.02& 19.29&56.92&18.91\\
\cline{5-11}
&&&& \multirow{3}{*}{$3/255$} &\multirow{3}{*}{128} & centralized& 56.73&31.10&57.28&31.28  \\
&&&&&&consensus&\textbf{67.60}&\textbf{40.17}&67.38&\textbf{39.56} \\
    &&&&&&diffusion&67.44&39.44&\textbf{67.53}&39.37 \\
\cline{4-11}
&&& \multirow{3}{*}{$\ell_2$} & \multirow{3}{*}{$128/255$} & \multirow{3}{*}{128} & centralized&56.18&27.41&56.16&27.69 \\
    &&&&&&consensus &\textbf{66.13}&\textbf{35.13}&\textbf{66.14}&\textbf{34.58}\\
    &&&&&&diffusion &65.93&34.60&65.91&34.58\\
\cline{3-11}

&&\multirow{9}{*}{16}&
\multirow{6}{*}{$\ell_\infty$} & \multirow{3}{*}{$8/255$} & \multirow{3}{*}{128} &  centralized& 46.95&16.28&47.09&15.47  \\
    &&&&&&consensus&57.30&\textbf{23.35}&\textbf{58.78}&\textbf{21.43} \\
    &&&&&&diffusion&\textbf{59.12}&22.05&58.10&19.19 \\
\cline{5-11}
&&&& \multirow{3}{*}{$3/255$} &\multirow{3}{*}{128} & centralized& 55.13&29.07&54.57&28.27  \\
    &&&&&&consensus & \textbf{68.53}&\textbf{42.91}&\textbf{67.76}&\textbf{39.50}\\
    &&&&&&diffusion & 68.00&40.30&67.55&39.29\\
\cline{4-11}
&&& \multirow{3}{*}{$\ell_2$} & \multirow{3}{*}{$128/255$} & \multirow{3}{*}{128} & centralized&54.44&25.49&54.07&25.20 \\
    &&&&&&consensus&67.01&\textbf{37.79}&66.26&33.89 \\
    &&&&&&diffusion &\textbf{67.08}&35.27&\textbf{66.67}&\textbf{33.91}\\
\cline{3-11}

&&\multirow{9}{*}{32}&
\multirow{6}{*}{$\ell_\infty$} & \multirow{3}{*}{$8/255$} & \multirow{3}{*}{128} &  centralized & 39.06&15.56&40.10&15.54 \\
    &&&&&&consensus& 54.81&\textbf{21.25}&54.83&\textbf{21.26} \\
    &&&&&&diffusion&\textbf{55.43}&20.65&\textbf{56.38}&17.98 \\
\cline{5-11}
&&&& \multirow{3}{*}{$3/255$} &\multirow{3}{*}{128} & centralized &53.20&30.15&53.64&29.76 \\
    &&&&&&consensus&63.03&\textbf{40.08}&64.90&\textbf{38.59} \\
    &&&&&&diffusion&\textbf{65.79}&39.79&\textbf{65.43}&37.03 \\
\cline{4-11}
&&& \multirow{3}{*}{$\ell_2$} & \multirow{3}{*}{$128/255$} & \multirow{3}{*}{128} & centralized&52.99&26.37&52.87&25.32 \\
    &&&&&&consensus &61.59&\textbf{35.64}&63.17&\textbf{33.67}\\
    &&&&&&diffusion &\textbf{65.01}&34.68&\textbf{64.75}&31.79\\
\bottomrule      
\end{tabular}
\end{spacing}
\end{table*}

\begin{table*}[!t]
\footnotesize
\centering
\begin{spacing}{1.6}
\caption{Clean and robust accuracy on CIFAR-100 for models trained using a \emph{grid} topology}
\label{robustacc_sgd_cifar100-grid}
\begin{tabular}{cccccccccccccccccc}
\toprule
\multirow{2}{*}{dataset} & \multirow{2}{*}{graph}& \multirow{2}{*}{K}&\multirow{2}{*}{norm} & \multirow{2}{*}{$\epsilon$} & \multirow{2}{*}{local batch}  & \multirow{2}{*}{method} &  \multicolumn{2}{c}{best} & \multicolumn{2}{c}{final} \\
\cline{8-9}\cline{10-11}
&&&&&&&clean($\%$) & AA($\%$) & clean($\%$) & AA($\%$) \\
\hline
\multirow{9}{*}{CIFAR-100}&\multirow{9}{*}{grid} &\multirow{9}{*}{16}& \multirow{6}{*}{$\ell_\infty$} & \multirow{3}{*}{$8/255$} & \multirow{3}{*}{128} & centralized& 46.95&16.28&47.09&15.47  \\
    &&&&&&consensus& 58.63&\textbf{23.04}&\textbf{59.40}&\textbf{21.50} \\
    &&&&&&diffusion&\textbf{60.25}&21.63&58.03&19.35 \\
\cline{5-11}
&&&& \multirow{3}{*}{$3/255$} &\multirow{3}{*}{128} & centralized& 55.13&29.07&54.57&28.27  \\
    &&&&&&consensus &\textbf{68.55}&\textbf{42.99}&\textbf{68.02}&39.18\\
    &&&&&&diffusion &68.18&40.76&67.84&\textbf{39.57}\\
\cline{4-11}
&&& \multirow{3}{*}{$\ell_2$} & \multirow{3}{*}{$128/255$} & \multirow{3}{*}{128} &  centralized&54.44&25.49&54.07&25.20\\
    &&&&&&consensus &\textbf{67.61}&\textbf{37.63}&66.44&\textbf{34.33}\\
    &&&&&&diffusion &67.18&35.28&\textbf{66.55}&34.30\\
\bottomrule      
\end{tabular}
\end{spacing}
\end{table*}

\begin{figure*}[htbp]
\begin{center}
\subfigure[CIFAR-10: B=128, K=8]{\includegraphics[width=.3\linewidth]{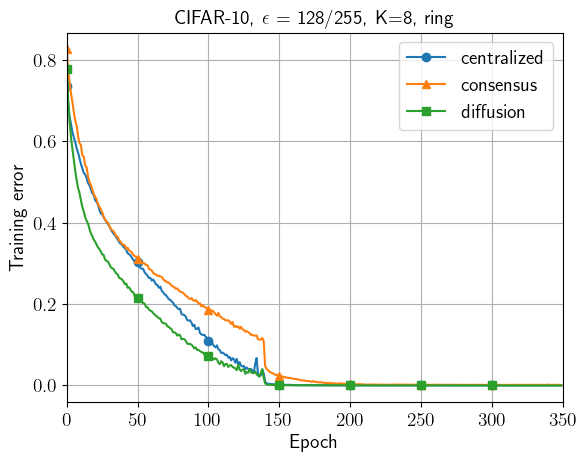}}
\subfigure[CIFAR-10: B=128, K=8]{\includegraphics[width=.3\linewidth]{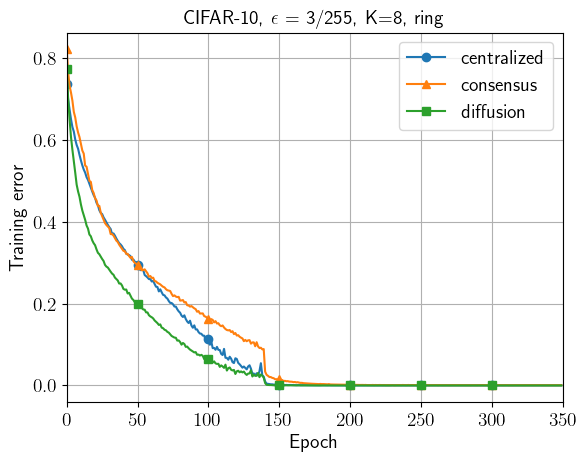}}
\subfigure[CIFAR-10: B=128, K=8]{\includegraphics[width=.3\linewidth]{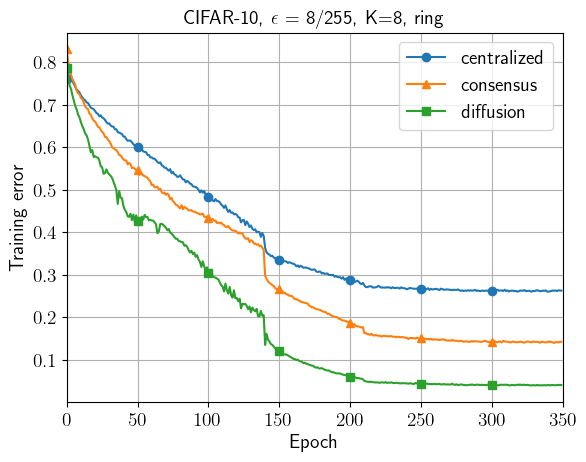}}
\subfigure[CIFAR-10: B=128, K=16]{\includegraphics[width=.3\linewidth]{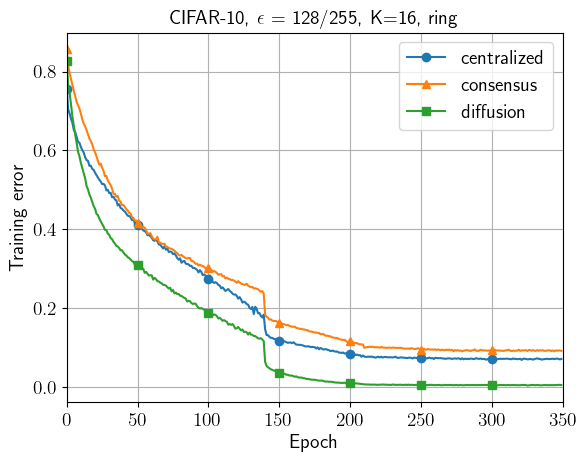}}
\subfigure[CIFAR-10: B=128, K=16]{\includegraphics[width=.3\linewidth]{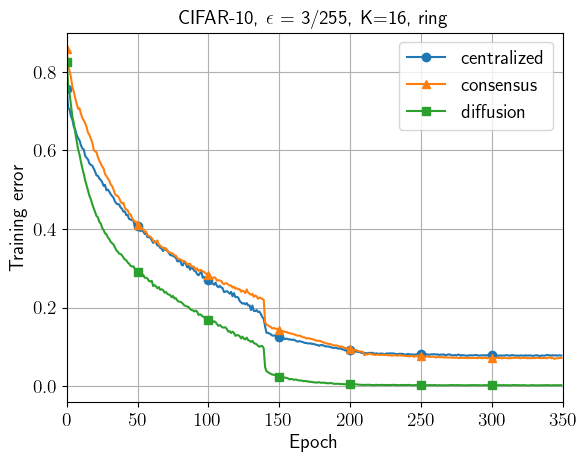}}
\subfigure[CIFAR-10: B=128, K=16]{\includegraphics[width=.3\linewidth]{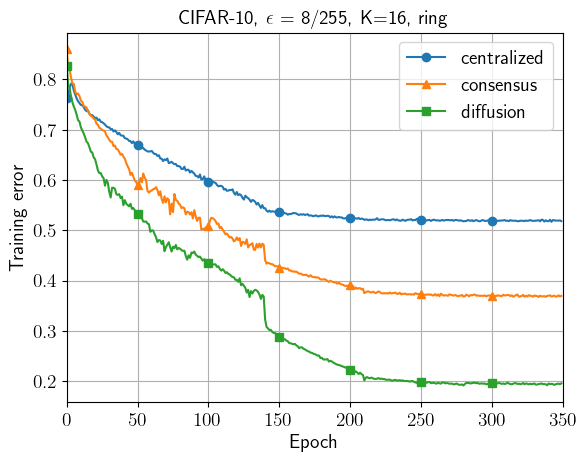}}
\subfigure[CIFAR-10: B=128, K=32]{\includegraphics[width=.3\linewidth]{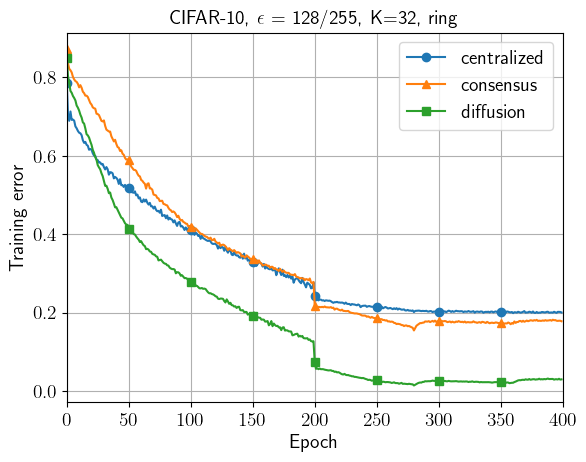}}
\subfigure[CIFAR-10: B=128, K=32]{\includegraphics[width=.3\linewidth]{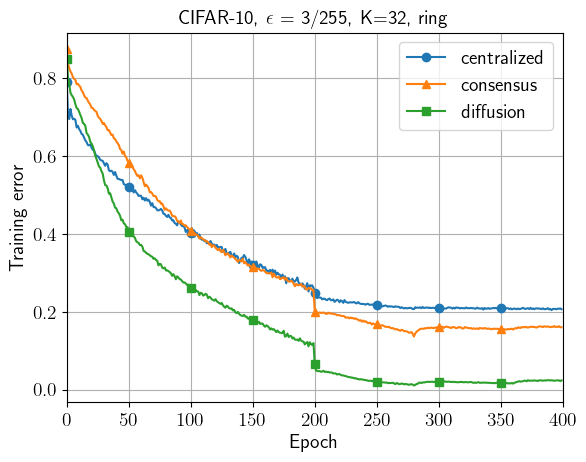}}
\subfigure[CIFAR-10: B=128, K=32]{\includegraphics[width=.3\linewidth]{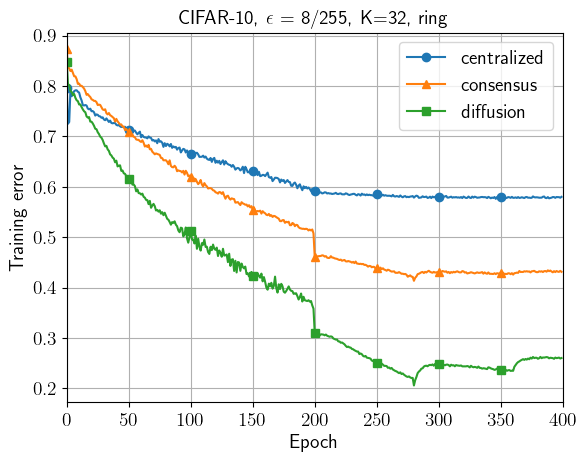}}
\caption{The evolution of the training error on CIFAR-10 using the \emph{ring} graph with $K=\{8,16,32\}$ nodes. The centralized method exhibits particularly high training error when $\epsilon=8/255$.}
\label{optim_ring_cifar10}
\end{center}
\end{figure*}

\begin{figure*}[htbp]
\begin{center}
\subfigure[CIFAR-10: B=128]{\includegraphics[width=.3\linewidth]{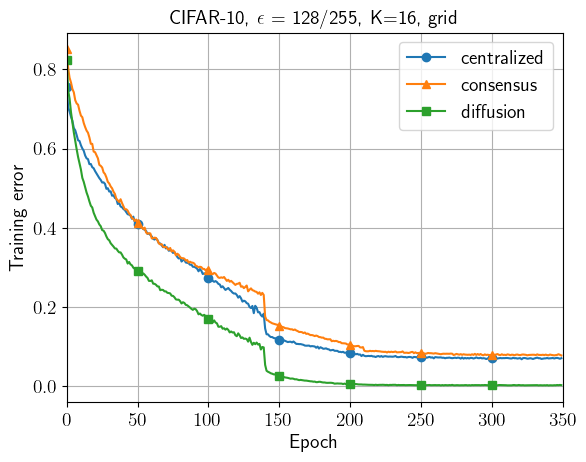}}
\subfigure[CIFAR-10: B=128]{\includegraphics[width=.3\linewidth]{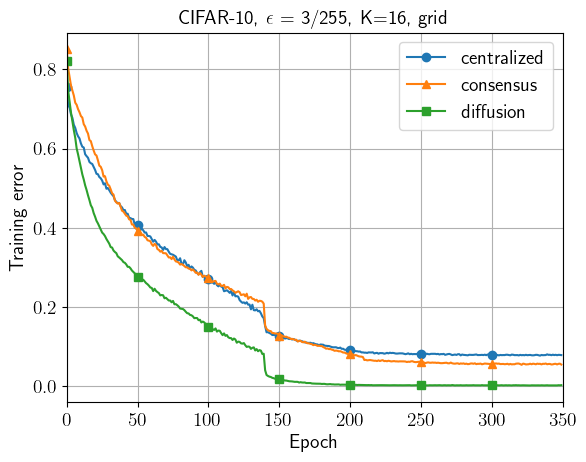}}
\subfigure[CIFAR-10: B=128]{\includegraphics[width=.3\linewidth]{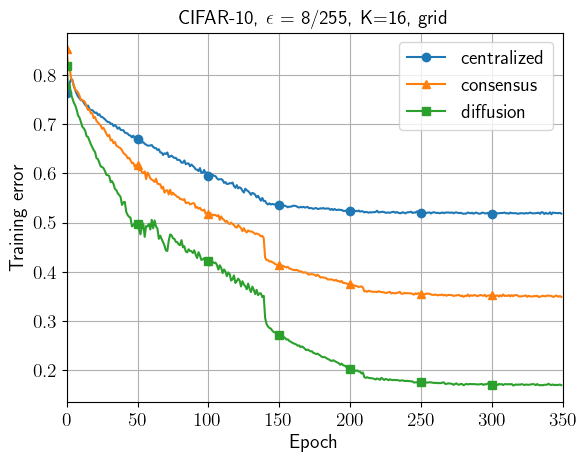}}
\caption{The evolution of the training error on CIFAR-10 using a \emph{grid} graph with $K=16$ nodes. The observed phenomenon is similar to Figure \ref{optim_ring_cifar10}.}
\label{optim_grid_16_cifar10}
\end{center}
\end{figure*}

\begin{figure*}[htbp]
\begin{center}
\subfigure[CIFAR-100: B=128, K=8]{\includegraphics[width=.3\linewidth]{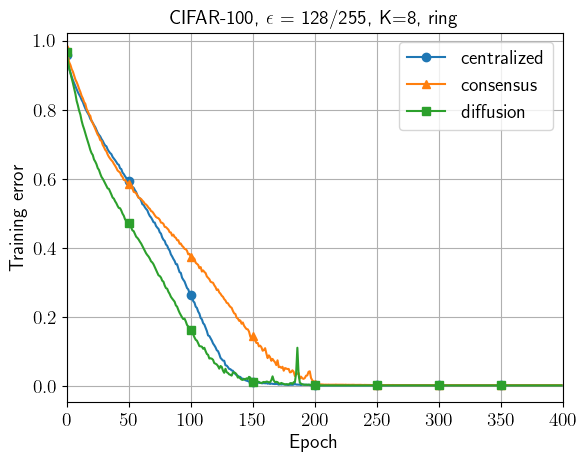}}
\subfigure[CIFAR-100: B=128, K=8]{\includegraphics[width=.3\linewidth]{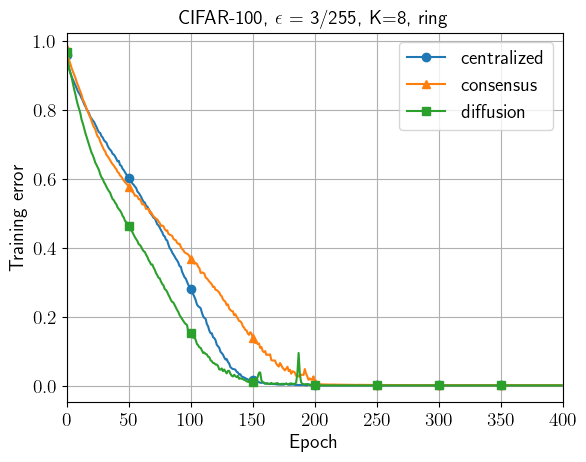}}
\subfigure[CIFAR-100: B=128, K=8]{\includegraphics[width=.3\linewidth]{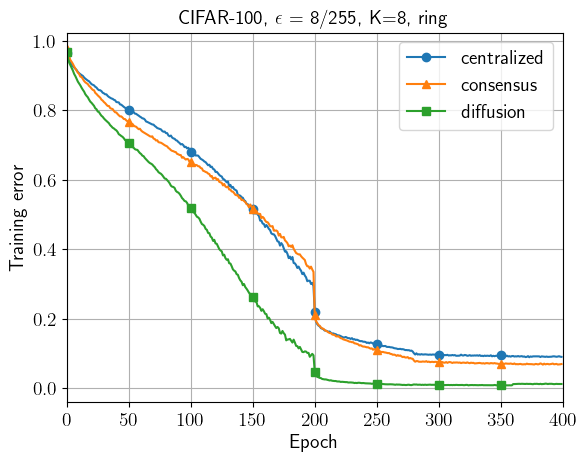}}

\subfigure[CIFAR-100: B=128, K=16]{\includegraphics[width=.3\linewidth]{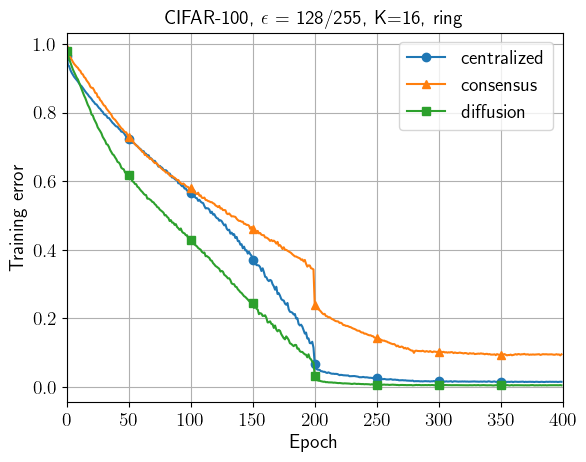}}
\subfigure[CIFAR-100: B=128, K=16]{\includegraphics[width=.3\linewidth]{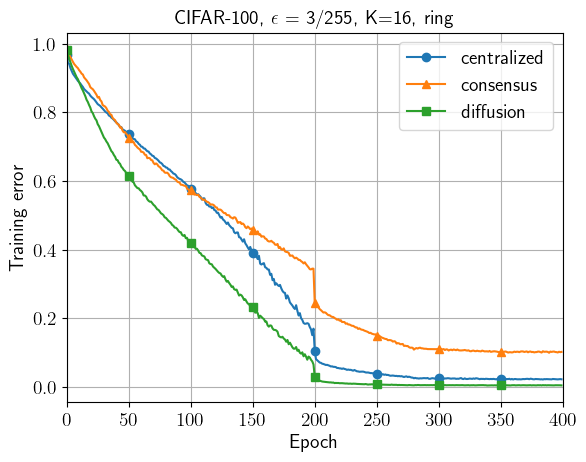}}
\subfigure[CIFAR-100: B=128, K=16]{\includegraphics[width=.3\linewidth]{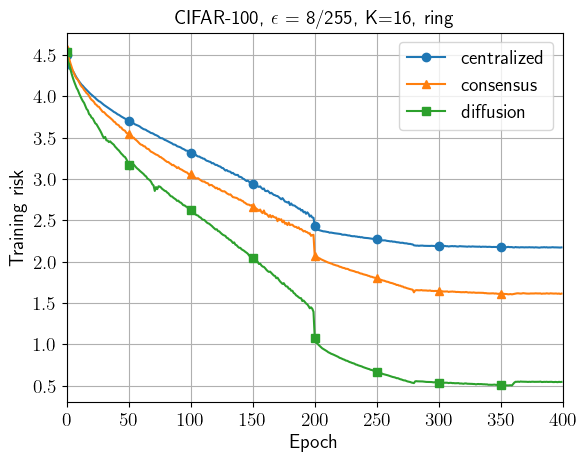}}

\subfigure[CIFAR-100: B=128, K=32]{\includegraphics[width=.3\linewidth]{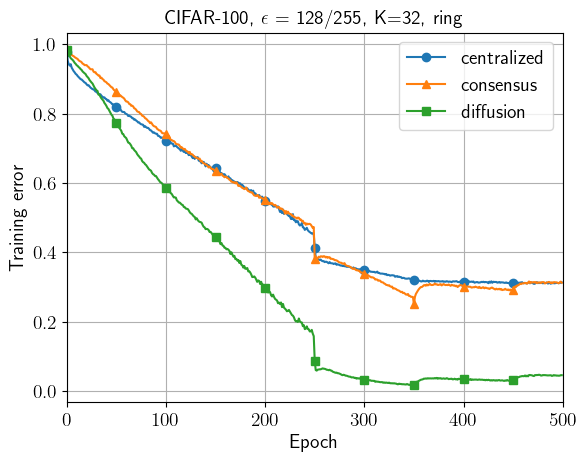}}
\subfigure[CIFAR-100: B=128, K=32]{\includegraphics[width=.3\linewidth]{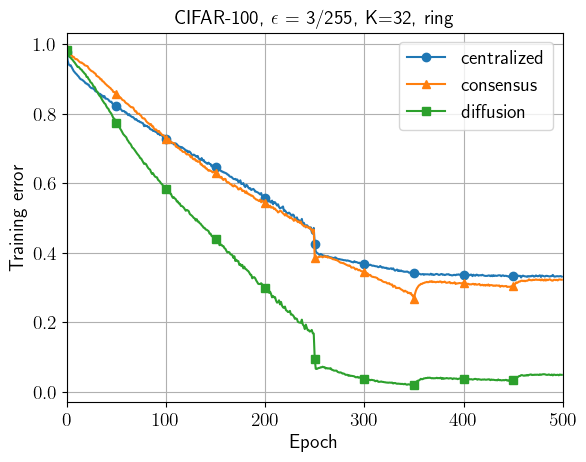}}
\subfigure[CIFAR-100: B=128, K=32]{\includegraphics[width=.3\linewidth]{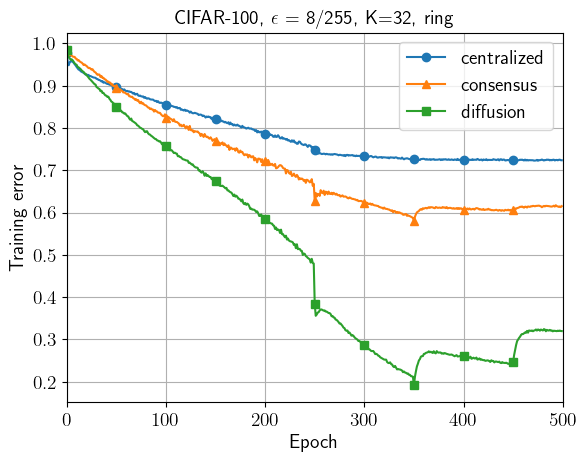}}
\caption{The evolution of the training error on CIFAR-100 using the \emph{ring} graph with $K=\{8,16,32\}$ nodes. The observed phenomenon is similar to Figure \ref{optim_ring_cifar10}.}
\label{optim_ring_cifar100}
\end{center}
\end{figure*}

\begin{figure*}[htbp]
\begin{center}
\subfigure[CIFAR-100: B=128]{\includegraphics[width=.3\linewidth]{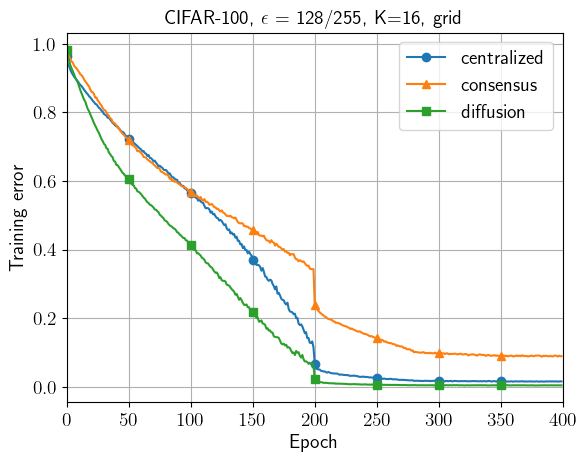}}
\subfigure[CIFAR-100: B=128]{\includegraphics[width=.3\linewidth]{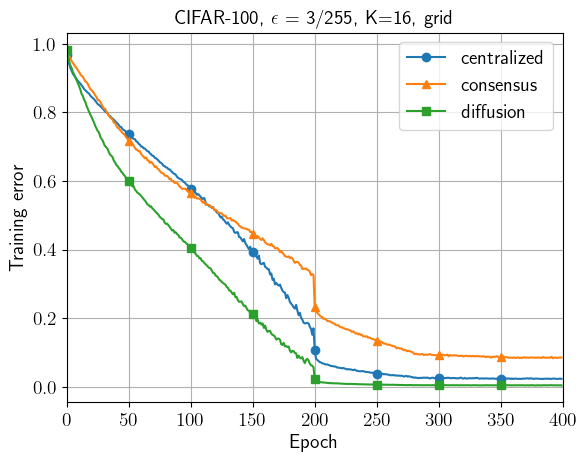}}
\subfigure[CIFAR-100: B=128]{\includegraphics[width=.3\linewidth]{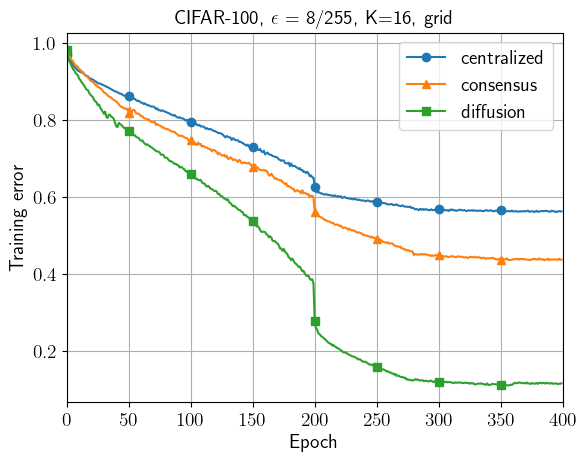}}
\caption{The evolution of the training error on CIFAR-100 using a \emph{grid} graph with $K=16$ nodes. The observed phenomenon is similar to Figure \ref{optim_ring_cifar10}.}
\label{optim_grid_16_cifar100}
\end{center}
\end{figure*}

\subsection{Results of Pretrained Initialization} 
In this section, we simulate the scenario in which  decentralized adversarial training training is initialized with models pretrained by the centralized method on CIFAR-10 and CIFAR-100 with $\epsilon=3/255$ and $\epsilon=128/255$. This setting allows us to examine whether decentralized methods converge to flatter models when initialized near a local minimizer obtained by centralized training, which more directly align with 
 Assumption \ref{assump_ip}. As shown in Figure \ref{flat_pretrained},  decentralized adversarial training methods converge to flatter models than centralized training. These results complement those obtained with random initialization. We focus mainly on the random initialization scenario in this paper it represents a more general setting.

\begin{figure*}[htbp]
\begin{center}
\subfigure[CIFAR-10: B=128]{\includegraphics[width=.4\linewidth]{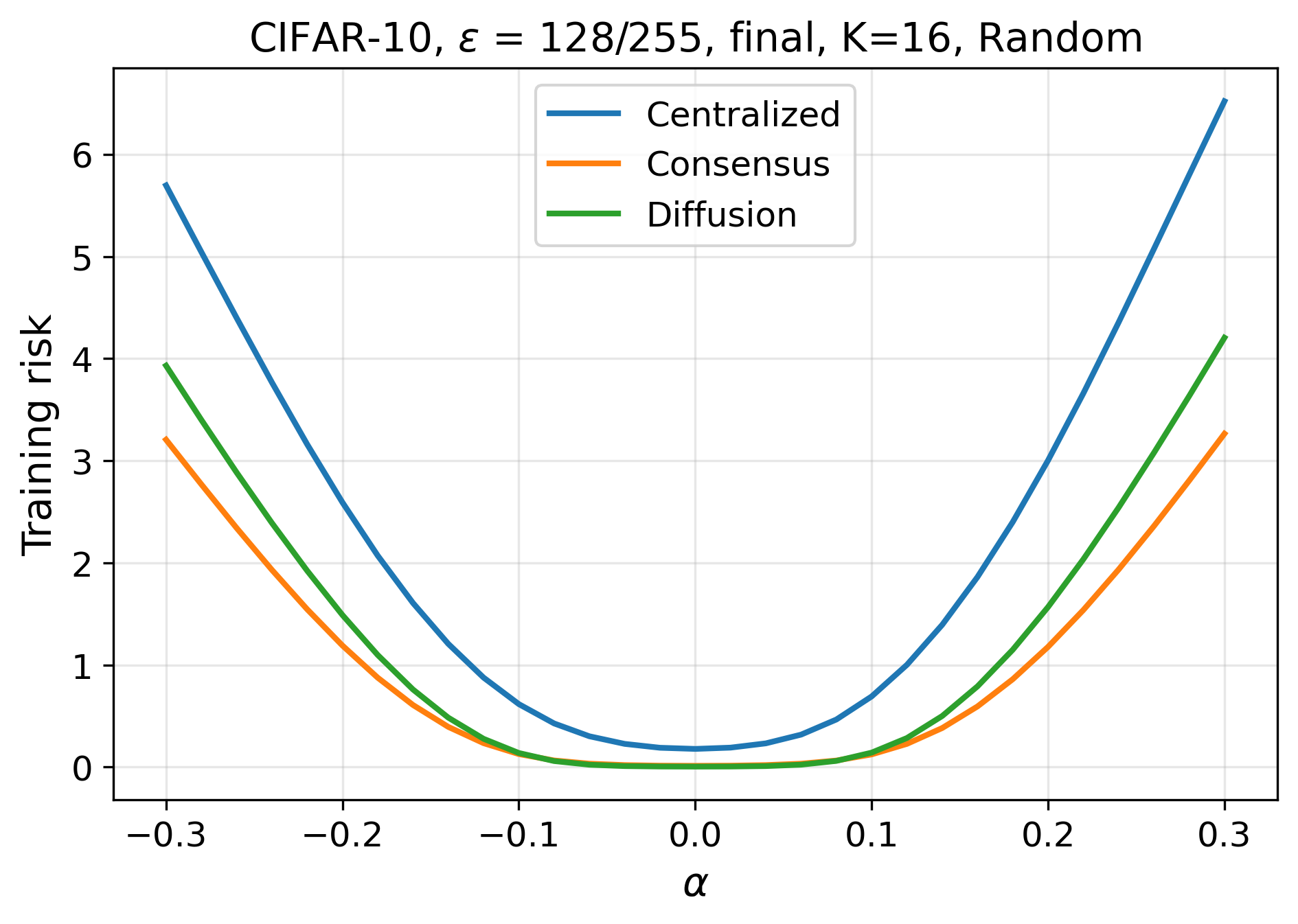}}
\subfigure[CIFAR-10: B=128]{\includegraphics[width=.4\linewidth]{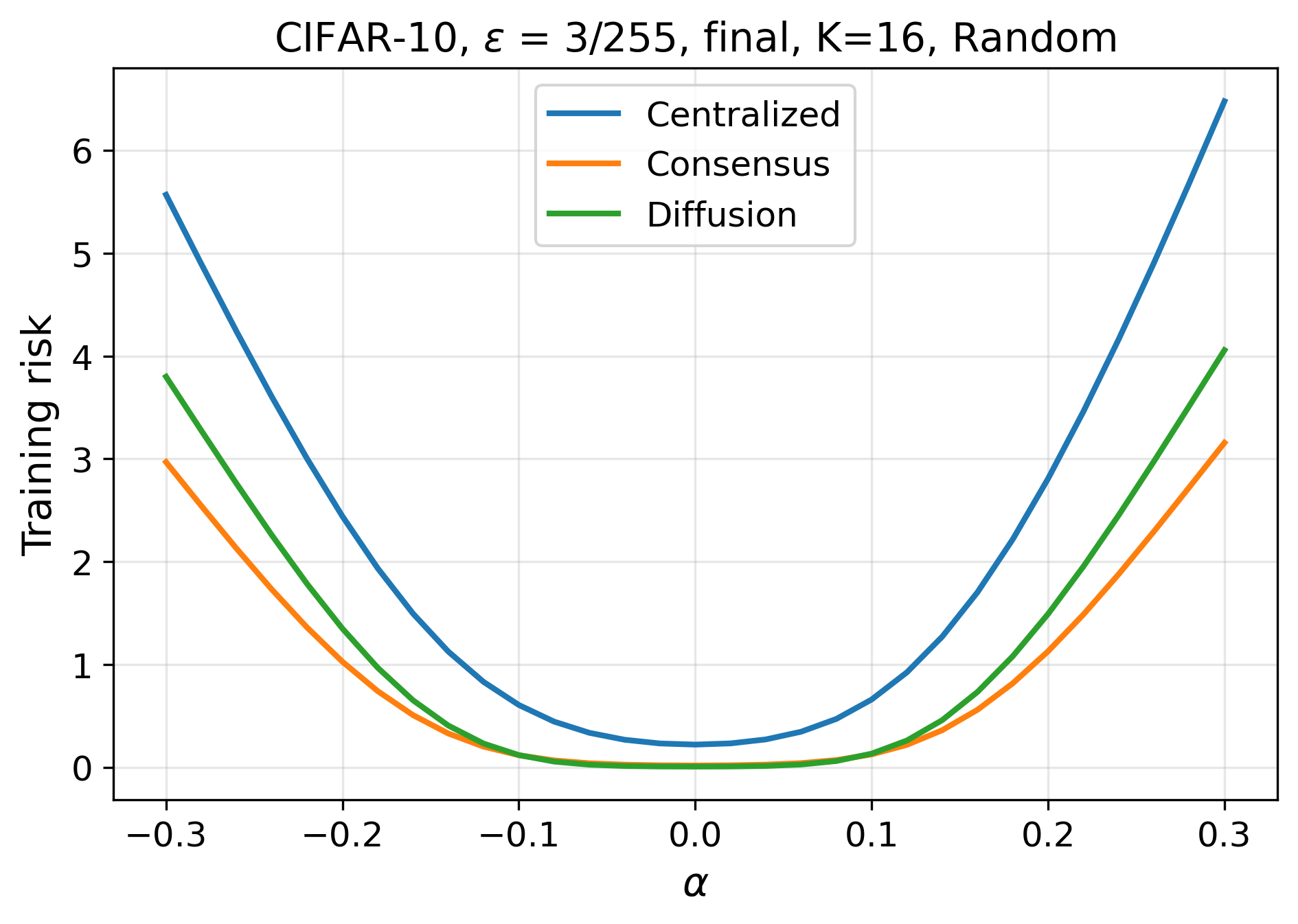}}
\subfigure[CIFAR-100: B=128]{\includegraphics[width=.4\linewidth]{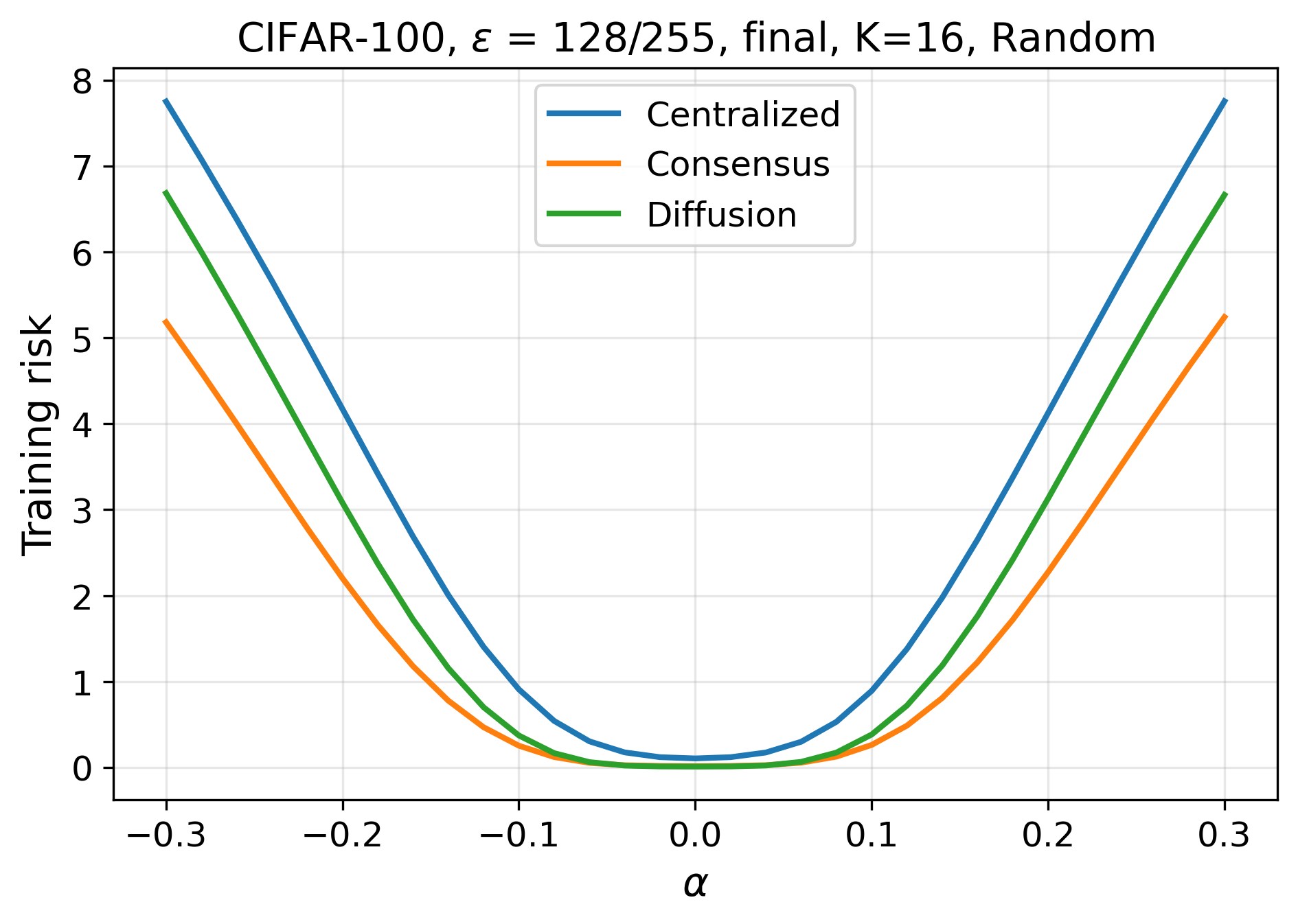}}
\subfigure[CIFAR-100: B=128]{\includegraphics[width=.4\linewidth]{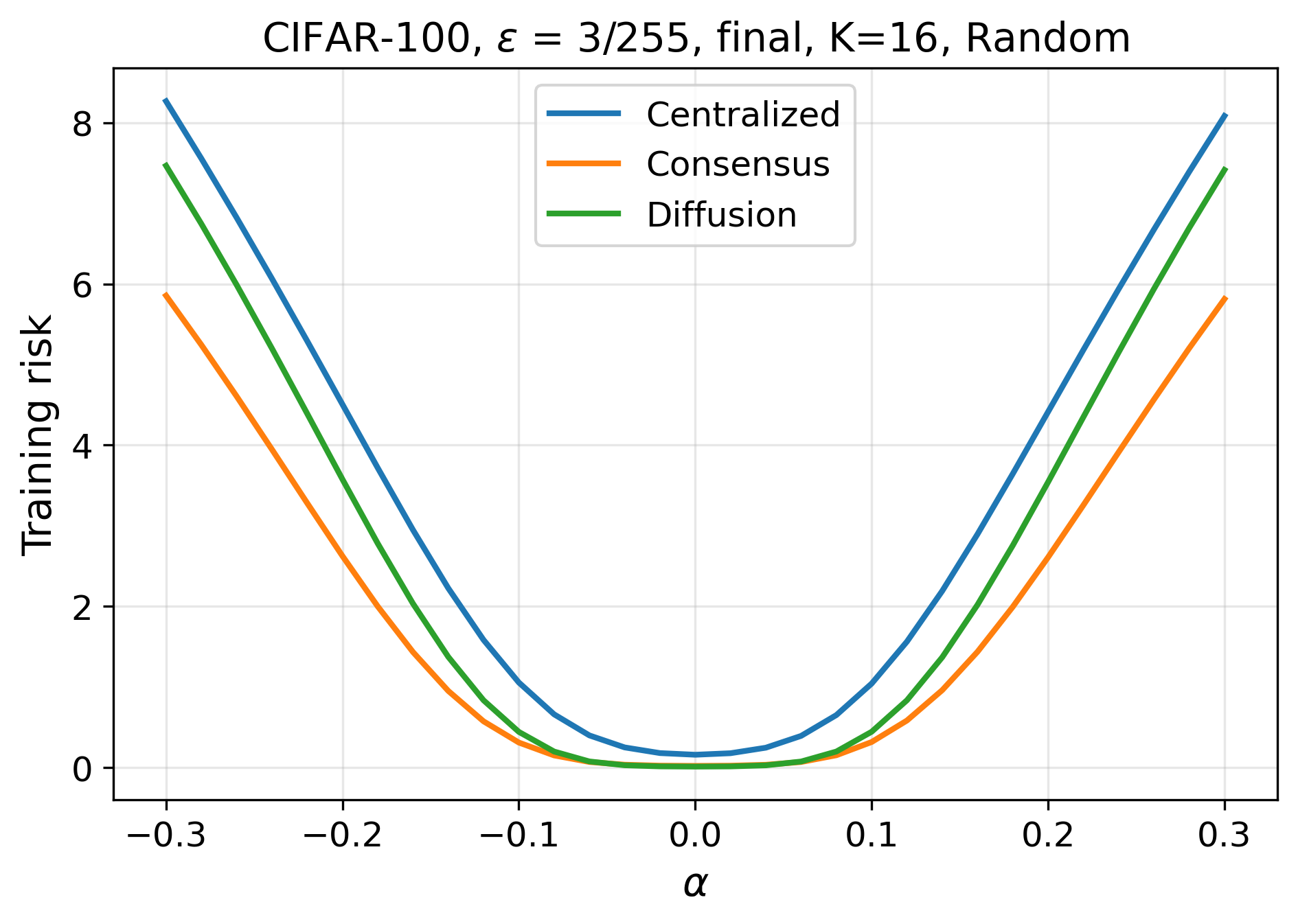}}
\caption{Flatness visualization of models in the pretrained initialization scenario. Basically, when initialized from models pretrained by the centralized approach, decentralized converges to flatter regions than centralized ones.}
\label{flat_pretrained}
\end{center}
\end{figure*}

\subsection{Results of SGD Momentum and Trades}
We have reported the results of PGD training using the standard SGD optimizer. To provide a more comprehensive comparison of model robustness under different distributed adversarial training strategies, we now present the results for PGD training with SGD momentum in Table \ref{robustacc_sgdm}, and for Trades in Table \ref{robustacc_trades}, from which we observe that decentralized adversarial training generally perform better on both clean and adversarial data than the centralized method. Furthermore, consensus and diffusion are comparable to each other.

\begin{table*}[!t]
\centering
\footnotesize
\begin{spacing}{1.6}
\caption{Clean and robust accuracy of models trained using PGD and the SGD momentum optimizer.}
\label{robustacc_sgdm}
\begin{tabular}{cccccccccccccccccc}
\toprule
\multirow{2}{*}{dataset} & \multirow{2}{*}{norm} & \multirow{2}{*}{$\epsilon$} & \multirow{2}{*}{local batch} & \multirow{2}{*}{method} &  \multicolumn{2}{c}{best} & \multicolumn{2}{c}{final} \\
\cline{6-7}\cline{8-9}
&&&&&clean($\%$) & AA($\%$) & clean($\%$) & AA($\%$) \\
\hline
\multirow{18}{*}{CIFAR-10} & \multirow{12}{*}{$\ell_\infty$} & \multirow{6}{*}{$8/255$} & \multirow{3}{*}{128} & centralized & 75.67&39.34&83.88&38.86  \\
    &&&&consensus&\textbf{85.23}&\textbf{42.87}&\textbf{84.89}&40.18 \\
    &&&&diffusion&76.53&41.34&84.75&\textbf{40.36} \\
\cline{4-9}
&&&\multirow{3}{*}{256} & centralized& 76.87&38.70&83.27&37.25  \\
    &&&&consensus &\textbf{83.60}&\textbf{39.98}&\textbf{83.35}&\textbf{39.04}\\
    &&&&diffusion & 76.53&38.20&83.14&38.06 \\
\cline{3-9}
&& \multirow{6}{*}{$3/255$} &\multirow{3}{*}{128} & centralized &90.94&70.44&91.04&70.54 \\
    &&&&consensus &\textbf{91.47}&71.83&\textbf{91.57}&72.00\\
    &&&&diffusion &91.46&\textbf{72.13}&91.57&\textbf{72.17}\\
\cline{4-9}
&&&\multirow{3}{*}{256} & centralized& 89.51&68.15&89.70&68.59 \\
    &&&&consensus&\textbf{91.20}&71.27&\textbf{91.23}&71.28\\
    &&&&diffusion &91.12&\textbf{71.28}&91.16&\textbf{71.31}\\
\cline{2-9}
& \multirow{6}{*}{$\ell_2$} & \multirow{6}{*}{$128/255$} & \multirow{3}{*}{128} & centralized & 89.67&61.78&89.73&61.96  \\
    &&&&consensus& \textbf{90.41}&\textbf{63.42}&\textbf{90.70}&\textbf{63.21}\\
    &&&&diffusion & 90.23&63.01&90.22&63.09\\
\cline{4-9}
&&&\multirow{3}{*}{256} & centralized &88.22&60.33&88.25&60.32 \\
    &&&&consensus&\textbf{90.05}&62.51&\textbf{90.12}&62.51 \\
    &&&&diffusion &89.72&\textbf{62.83}&89.77&\textbf{62.74}\\
\midrule
\multirow{18}{*}{CIFAR-100} & \multirow{12}{*}{$\ell_\infty$} & \multirow{6}{*}{$8/255$} & \multirow{3}{*}{128} & centralized & 46.25&16.32&53.79&17.72 \\
    &&&&consensus&\textbf{56.79}&19.48&\textbf{56.44}&19.39 \\
    &&&&diffusion &56.22&\textbf{19.63}&56.14&\textbf{19.43}\\
\cline{4-9}
&&&\multirow{3}{*}{256} & centralized &40.72&15.39&51.53&16.28 \\
    &&&&consensus&54.20&\textbf{18.68}&54.28&17.17 \\
    &&&&diffusion &\textbf{54.24}&18.62&\textbf{54.38}&\textbf{17.99}\\
\cline{3-9}
&& \multirow{6}{*}{$3/255$} &\multirow{3}{*}{128} & centralized& 62.98&37.68&64.29&37.95 \\
    &&&&consensus &\textbf{67.22}&41.09&\textbf{67.40}&40.89\\
    &&&&diffusion &67.22&\textbf{41.35}&67.22&\textbf{41.39}\\
\cline{4-9}
&&&\multirow{3}{*}{256} & centralized &60.44&34.06&61.28&34.72& \\
    &&&&consensus & \textbf{65.77}&\textbf{39.66}&\textbf{65.74}&\textbf{39.63}\\
    &&&&diffusion & 65.32&38.77&65.38&38.84\\
\cline{2-9}
& \multirow{6}{*}{$\ell_2$} & \multirow{6}{*}{$128/255$} & \multirow{3}{*}{128} & centralized& 62.23& 33.14& 63.03& 33.67 \\
    &&&&consensus &\textbf{66.03}&\textbf{35.58}&\textbf{66.04}&35.67 \\
    &&&&diffusion &65.49& 35.20& 65.97& \textbf{35.71}\\
\cline{4-9}
&&&\multirow{3}{*}{256}& centralized&59.85&31.23&60.09&31.37 \\
    &&&&consensus&\textbf{64.51}&34.38&\textbf{64.52}&\textbf{34.34} \\
    &&&&diffusion&64.16&\textbf{34.58}&64.32&34.39 \\
\midrule
\multirow{6}{*}{Tiny ImageNet-200} & \multirow{3}{*}{$\ell_\infty$} & \multirow{3}{*}{$3/255$} & \multirow{3}{*}{128} & centralized & 53.25&25.57&53.13&25.64 \\
    &&&&consensus& \textbf{55.58}&\textbf{27.93}&\textbf{55.71}&\textbf{27.82} \\
    &&&&diffusion &55.47&27.82&55.38&27.68\\
\cline{2-9}
& \multirow{3}{*}{$\ell_2$} & \multirow{3}{*}{$128/255$} & \multirow{3}{*}{128} & centralized &57.28&31.50&57.35&31.52\\
    &&&&consensus &\textbf{59.45}&33.07&\textbf{59.53}&33.06 \\
    &&&&diffusion & 58.82&\textbf{33.35}&58.78&\textbf{33.29}\\
\bottomrule         
\end{tabular}
\end{spacing}
\end{table*}

\begin{table*}[!t]
\footnotesize
\centering
\begin{spacing}{1.6}
\caption{Clean and robust accuracy of the models trained with Trades.}
\label{robustacc_trades}
\begin{tabular}{cccccccccccccccccc}
\toprule
\multirow{2}{*}{dataset} & \multirow{2}{*}{norm} & \multirow{2}{*}{$\epsilon$} & \multirow{2}{*}{local batch} & \multirow{2}{*}{method} &  \multicolumn{2}{c}{best} & \multicolumn{2}{c}{final} \\
\cline{6-7}\cline{8-9}
&&&&&clean($\%$) & AA($\%$) & clean($\%$) & AA($\%$) \\
\hline
\multirow{18}{*}{CIFAR-10} & \multirow{12}{*}{$\ell_\infty$} & \multirow{6}{*}{$8/255$} & \multirow{3}{*}{128} & centralized & 81.59&43.47&81.68&40.10  \\
    &&&&consensus&82.46&\textbf{45.54}&83.05&41.20 \\
    &&&&diffusion&\textbf{82.54}&45.39&\textbf{83.07}&\textbf{41.46} \\
\cline{4-9}
&&&\multirow{3}{*}{256} &  centralized&80.07&42.20&80.16&37.51 \\
    &&&&consensus &\textbf{81.21}&\textbf{43.83}&\textbf{82.03}&\textbf{40.52}\\
    &&&&diffusion &80.24&41.98&81.31&39.24 \\
\cline{3-9}
&& \multirow{6}{*}{$3/255$} &\multirow{3}{*}{128} & centralized& 89.13&71.51&89.12&71.23  \\
    &&&&consensus&\textbf{89.48}&72.92&\textbf{89.42}&\textbf{72.38} \\
    &&&&diffusion&89.21&\textbf{72.98}&89.34&72.27\\
\cline{4-9}
&&&\multirow{3}{*}{256} & centralized&87.75&68.92&87.83&68.94 \\
    &&&&consensus&88.31&71.30&88.79&71.26\\
    &&&&diffusion &89.11&71.33&89.07&71.27\\
\cline{2-9}
& \multirow{6}{*}{$\ell_2$} & \multirow{6}{*}{$128/255$} & \multirow{3}{*}{128} & centralized&87.57&62.91&87.57&62.87  \\
&&&&consensus&87.74&\textbf{65.39}&88.04&\textbf{64.42}\\
    &&&&diffusion &\textbf{88.71}&64.05&\textbf{88.79}&64.08\\
\cline{4-9}
&&&\multirow{3}{*}{256} & centralized&86.64&61.40&86.35&61.50  \\   &&&&consensus&\textbf{87.90}&62.99&\textbf{88.05}&\textbf{63.18} \\
    &&&&diffusion &87.35&\textbf{63.39}&87.73&62.46\\
\midrule
\multirow{18}{*}{CIFAR-100} & \multirow{12}{*}{$\ell_\infty$} & \multirow{6}{*}{$8/255$} & \multirow{3}{*}{128} & centralized& 52.38&21.30&52.56&20.12& \\
&&&&consensus&\textbf{54.23}&\textbf{23.38}&54.69&\textbf{20.70} \\
    &&&&diffusion&54.17&22.69&\textbf{54.81}&19.99 \\
\cline{4-9}
&&&\multirow{3}{*}{256} & centralized&51.09&20.08&50.96&\textbf{18.87} \\
    &&&&consensus&\textbf{51.16}&20.15&51.43&18.10 \\
    &&&&diffusion&51.11&\textbf{20.21}&\textbf{52.11}&18.33 \\
\cline{3-9}
&& \multirow{6}{*}{$3/255$} &\multirow{3}{*}{128} & centralized& 61.42&40.08&61.53&40.04 \\
    &&&&consensus &63.79&42.52&64.05&42.05\\
    &&&&diffusion &\textbf{64.08}&\textbf{42.54}&\textbf{64.20}&\textbf{42.73}\\
\cline{4-9}
&&&\multirow{3}{*}{256} & centralized & 59.93&37.00&59.91&37.04 \\
    &&&&consensus&\textbf{62.40}&\textbf{39.92}&62.37&\textbf{39.92} \\
    &&&&diffusion&62.26&39.74&\textbf{62.73}&39.41\\
\cline{2-9}
& \multirow{6}{*}{$\ell_2$} & \multirow{6}{*}{$128/255$} & \multirow{3}{*}{128} & centralized&60.76&34.68&60.69&34.90  \\
    &&&&consensus&\textbf{62.14}&\textbf{37.08}&\textbf{62.59}&\textbf{36.45}  \\
    &&&&diffusion &62.04&36.60&62.59&36.09\\
\cline{4-9}
&&&\multirow{3}{*}{256}& centralized &58.64&32.89&58.52&32.95\\
    &&&&consensus&\textbf{61.24}&\textbf{34.77}&61.14&\textbf{34.55} \\
    &&&&diffusion&60.65&34.63&\textbf{61.20}&34.26 \\
\midrule
\multirow{6}{*}{Tiny ImageNet-200} & \multirow{3}{*}{$\ell_\infty$} & \multirow{3}{*}{$3/255$} & \multirow{3}{*}{128} & centralized & 46.07&25.81&49.08&25.96  \\
    &&&&consensus&51.04&\textbf{29.05}&\textbf{51.73}&\textbf{28.40} \\
    &&&&diffusion&\textbf{51.15}&28.33&51.73&27.99 \\
\cline{2-9}
& \multirow{3}{*}{$\ell_2$} & \multirow{3}{*}{$128/255$} & \multirow{3}{*}{128} & centralized & 54.34&31.89&54.42&31.96\\
    &&&&consensus &\textbf{57.12}&34.97&57.16&34.98 \\
    &&&&diffusion&56.98&\textbf{35.22}&\textbf{57.27}&\textbf{35.19}\\
\bottomrule      
\end{tabular}
\end{spacing}
\end{table*}

\end{document}